%% file: main.tex
\documentclass{mimosis}

\input{sources/definitions}

\input{sources/acronyms}

\input{sources/glossary}

\input{sources/abstract-cover.tex}

\makeindex
\makeglossaries

\title{Building Compact and Robust Deep Neural Networks with Toeplitz Matrices}
\author{Alexandre Araujo}

\pslassetspath{figures/cover}
\doctoralschool{Sciences de la Décision, des Organisations, de la Société et de l'Echange}{543}
\institute{Université Paris-Dauphine -- PSL Research University}
\specialty{Informatique}
\date{1$^{\text{er}}$ Juin 2021}

\jurymember{1}{Elisa FROMONT}{Professeur, IRISA Rennes}{Présidente du Jury}
\jurymember{2}{Teddy FURON}{Chargé de recherche, INRIA Rennes}{Rapporteur}
\jurymember{3}{Alain RAKOTOMAMONJY}{Professeur, CRITEO}{Rapporteur}
\jurymember{4}{Krzysztof CHOROMANSKI}{Research Scientist, Google}{Examinateur}
\jurymember{5}{Remi GRIBONVAL}{Directeur de recherche, INRIA Lyon}{Examinateur}
\jurymember{6}{Jamal ATIF}{Professeur, Paris-Dauphine}{Directeur de thèse}
\jurymember{7}{Yann CHEVALEYRE}{Professeur, Paris-Dauphine}{Co-Encadrant}
\jurymember{8}{Benjamin NEGREVERGNE}{Maître de conférences, Paris-Dauphine}{Co-Encadrant}

\begin{document}

  \frontmatter
  \maketitle{}
  \input{sources/title}
  \input{sources/dedication.tex}

  \addcontentsline{toc}{chapter}{Remerciements}
  \input{sources/acknowledgments}

  \addcontentsline{toc}{chapter}{Abstract}
  \input{sources/abstract}
  \addcontentsline{toc}{chapter}{Résumé}
  \input{sources/resume}

  \bookmarksetup{startatroot}
  \backmatter
    \begingroup
      \let\clearpage\relax
      \glsaddall
      \etocsettocdepth{2}
      \addcontentsline{toc}{chapter}{Table of Contents}
      \begingroup
        \etocsettocstyle{\addchap*{Table of Contents}}{}
        \tableofcontents
      \endgroup
      \newpage
      \addcontentsline{toc}{chapter}{\listfigurename}
      \listoffigures
      \newpage
      \addcontentsline{toc}{chapter}{\listtablename}
      \listoftables
      \newpage
      \printglossary[type=\acronymtype]
      \newpage
      \printglossary[title={List of Symbols}]
    \endgroup
    \printindex

  \mainmatter
  \input{sources/main/ch1-introduction}

  \input{sources/main/ch2-background}

  \input{sources/main/ch3-related_work}
\input{sources/main/ch4-diagonal_circulant}

  \input{sources/main/ch5-lipschitz_regularization}

  \input{sources/main/ch6-conclusion}

  Appendix
  \renewcommand\appendixpagename{\usekomafont{disposition}Appendices}
  \begin{appendices}
    \renewcommand\chaptername{Appendix}
    \input{sources/appendix/ap1-technical_proofs.tex}
    \input{sources/appendix/ap2-training_video_classification}

    \input{sources/appendix/ap3-randomized_inference}

    \input{sources/appendix/ap4-advocating_for_multiple_defense_strategies}

    \input{sources/appendix/ap5-resume_these_fr.tex}

  \end{appendices} 

  \newrefcontext[sorting=nyt]
  \printbibliography[title=References]

\end{document}

%% file: sources/definitions.tex
\usepackage[utf8]{inputenc}                                 
\usepackage[T1]{fontenc}                                    
\usepackage{url}                                            
\usepackage{microtype}                                      
\usepackage{calc}                                           
\usepackage{xspace}                                         
\usepackage{enumitem}                                       
\usepackage{numprint}                                       
\usepackage{epigraph}                                       
\usepackage[super]{nth}                                     
\usepackage{bbding}                                         
\usepackage[title, toc, titletoc, page, header]{appendix}   
\usepackage{lmodern}
\usepackage{psl-cover}                                      
\usepackage{etoc}                                           
\usepackage{etoolbox}                                       
\usepackage[binary-units=true]{siunitx}                     

\usepackage{mathtools}      
\usepackage{amssymb}
\usepackage{amsthm}         
\usepackage{mathdots}       
\usepackage{amsfonts}       
\usepackage{nicefrac}       
\usepackage{physics}        
\usepackage{bm}             
\usepackage{centernot}      
\usepackage{thmtools}       
\usepackage{afterpage}

\usepackage{pgfplots}
\usepackage{graphicx}       
\usepackage{tikz}           
\usepackage{tikz-cd}        
\usetikzlibrary{positioning}
\usetikzlibrary{arrows.meta}
\usepgfplotslibrary{groupplots}

\usepackage{array}          
\usepackage{multirow}       
\usepackage{tabularx}       
\usepackage{booktabs}       
\usepackage{makecell}

\DeclareSIUnit\px{px}
\newcolumntype{L}[1]{>{\raggedright\arraybackslash}p{#1}}
\newcolumntype{C}[1]{>{\centering\arraybackslash}p{#1}}

\addto\captionsfrench{}

\definecolor{mydarkblue}{rgb}{0.02, 0.40, 0.62}
\usepackage[
  colorlinks=true,
  allcolors=mydarkblue,
  ]{hyperref}
\usepackage{cleveref}       
\usepackage{bookmark}       

\usepackage[chapter]{algorithm} 
\usepackage{algpseudocode}      

\makeatletter
\newcounter{algorithmicH}
\let\oldalgorithmic\algorithmic
\renewcommand{\algorithmic}{%
  \stepcounter{algorithmicH}
  \oldalgorithmic}
\renewcommand{\theHALG@line}{ALG@line.\thealgorithmicH.\arabic{ALG@line}}
\makeatother

\DeclareFontFamily{T1}{calligra}{}
\DeclareFontShape{T1}{calligra}{m}{n}{<->s*[1.44]callig15}{}
\DeclareMathAlphabet\mathcalligra   {T1} {calligra} {m}  {n}
\DeclareMathAlphabet\mathzapf       {T1} {pzc}      {mb} {it}
\DeclareMathAlphabet\mathchorus     {T1} {qzc}      {m}  {n}
\DeclareMathAlphabet\mathrsfso      {U}  {rsfso}    {m}  {n}

\DeclareUnicodeCharacter{2212}{-}

\newlength\tocrulewidth
\etocsettocstyle{
  \addsec*{Contents \\ \vspace{-0.5cm}
    \rule{\textwidth}{\tocrulewidth}
    \vspace{-1cm plus0mm minus0mm}
  }
}{
  \noindent\rule{\linewidth}{\tocrulewidth}
}

\newcommand{\localtoc}{
  \begingroup
    \localtableofcontents
  \endgroup
}

\usepackage[%
  style        = authoryear-comp, 
  backend      = biber,
  autocite     = plain,
  uniquename   = false,
  uniquelist   = false,
  sorting      = ynt,
  sortcites    = true,
  doi          = false,
  url          = false,
  giveninits   = true,
  hyperref     = true,
  maxbibnames  = 99,
  maxcitenames = 2,
  natbib       = true,
  backref      = true,
  abbreviate   = true,
  dashed       = false,
  ]{biblatex}

\makeatletter
\let\abx@macro@citeOrig\abx@macro@cite
\renewbibmacro{cite}{%
  \bibhyperref{%
  \let\bibhyperref\relax\relax%
  \abx@macro@citeOrig%
    }%
}

\DeclareCiteCommand{\textcite}
 {\boolfalse{cbx:parens}}
  {\usebibmacro{citeindex}%
    \printtext[bibhyperref]{\printnames{labelname}%
      \printtext{ (\printfield{year}\printtext{)}}}}
 {\ifbool{cbx:parens}
  {\bibcloseparen\global\boolfalse{cbx:parens}}
  {}%
 \textcitedelim}
{\usebibmacro{textcite:postnote}}
\makeatother

\DefineBibliographyStrings{english}{and={and}}
\DefineBibliographyStrings{english}{in={in}}
\DefineBibliographyStrings{english}{page={page}}
\DefineBibliographyStrings{english}{pages={pages}}
\DefineBibliographyStrings{english}{andothers={et al.}}
\DefineBibliographyStrings{english}{byeditor={editor}}
\DefineBibliographyStrings{english}{volume={volume}}
\DefineBibliographyStrings{english}{backrefpage={see page}}
\DefineBibliographyStrings{english}{backrefpages={see pages}}
\DefineBibliographyStrings{english}{phdthesis={PhD Thesis}}
\renewcommand{\cite}[1]{\citep{#1}}
\renewcommand{\finalnamedelim}{ \& }
\renewcommand{\textcitedelim}{
  \iflastcitekey{
    \unspace\textcolor{black}{\text{\ and}}\unspace
  }{
  \addcomma}\space}

\input{bibliography-mimosis}
\DefineBibliographyExtras{french}{\restorecommand\mkbibnamefamily}

\declaretheoremstyle[
  spaceabove=6pt, spacebelow=6pt,
  bodyfont=\itshape,
  postheadspace=1em,
  qed=,
]{style}

\declaretheoremstyle[
  spaceabove=6pt, spacebelow=6pt,
  bodyfont=\itshape,
  postheadspace=1em,
  qed=,
  shaded={
    rulecolor=mydarkblue,
    rulewidth=2pt,
    bgcolor=white,
    padding=0.02\textwidth,
    textwidth=0.96\textwidth},
]{maintheoremstyle}

\declaretheorem[
  numberwithin=chapter,
  name=Theorem,
  style=style,
  refname={theorem,theorems},
  Refname={Theorem,Theorems}]{theorem}
\declaretheorem[
  numberwithin=chapter,
  name=Corollary,
  style=style,
  refname={corollary,corollaries},
  Refname={Corollary,Corollaries}]{corollary}

\declaretheorem[
  numberwithin=chapter,
  name=Definition,
  style=style,
  refname={definition,definitions},
  Refname={Definition,Definitions}]{definition}
\declaretheorem[
  numberwithin=chapter,
  name=Lemma,
  style=style,
  refname={lemma,lemmas},
  Refname={Lemma,Lemmas}]{lemma}
\declaretheorem[
  numberwithin=chapter,
  name=Proposition,
  style=style,
  refname={proposition,propositions},
  Refname={Proposition,Propositions}]{proposition}

\declaretheorem[
  numberwithin=chapter,
  numbered=no,
  name=Remark,
  style=style,
  refname={remark,remarks},
  Refname={Remark,Remarks}]{remark}

\declaretheorem[
  numberwithin=chapter,
  name=Theorem,
  style=maintheoremstyle,
  sibling=theorem,
  refname={theorem,theorems},
  Refname={Theorem,Theorems}]{maintheorem}
\declaretheorem[
  numberwithin=chapter,
  name=Corollary,
  style=maintheoremstyle,
  sibling=corollary,
  refname={corollary,corollaries},
  Refname={Corollary,Corollaries}]{maincorollary}

\declaretheorem[
  numberwithin=chapter,
  numbered=no,
  name=Properties,
  style=style,
  refname={property,properties},
  Refname={Property,Properties}]{properties}

\declaretheoremstyle[
  notebraces={}{},
  headfont=\bfseries\itshape,
  notefont=\bfseries\itshape,
  postheadspace=1em,
  postfoothook=\vspace{2em},
  qed=$\blacksquare$,
  mdframed={
    topline=false,
    bottomline=false,
    rightline=false,
    linecolor=lightgray,
    linewidth=3pt,
    backgroundcolor=white,
    innerleftmargin=0.02\textwidth,
    innerrightmargin=0pt,
  }
]{proofstyle}

\declaretheorem[
  numbered=no,
  name=Proof of,
  style=proofstyle,
  refname={proof,proofs},
  Refname={Proof,Proofs}
]{proof}

\newcommand{\yt}    {\textit{\mbox{YouTube-8M}}\xspace}
\newcommand{\eg}    {\textit{e.g.}\xspace}
\newcommand{\ie}    {\textit{i.e.}\xspace}
\newcommand{\cf}    {\textit{cf.}\xspace}
\newcommand{\aka}   {\textit{a.k.a.}\xspace}
\newcommand{\pdf}   {p.d.f.\xspace}
\newcommand{\st}    { s.t.\xspace}
\newcommand{\vs}    {\textit{vs.}\xspace}
\newcommand{\wrt}   {w.r.t.\xspace}

\DeclareMathOperator*{\argmax}       {arg\,max}
\DeclareMathOperator*{\argmin}       {arg\,min}
\DeclareMathOperator*{\Ebb}          {\mathbb{E}}
\DeclareMathOperator*{\Pbb}          {\mathbb{P}}
\DeclareMathOperator*{\Risk}         {Risk}
\DeclareMathOperator*{\advRisk}      {Risk_{\alpha}}
\DeclareMathOperator*{\PCadvRisk}    {PC-Risk_{\alpha}}

\DeclareMathOperator*{\B}            {B(\alpha)}
\DeclareMathOperator*{\probmap}      {M}
\DeclareMathOperator*{\EoT}          {EoT}
\DeclareMathOperator*{\Vol}          {Vol}

\newcommand{\relu}     {\ensuremath{\mathrm{ReLU}}\xspace}

\newcommand{\lipp}[2]  {\ensuremath{\mathrm{Lip}_{#1} \left(#2\right)}}
\newcommand{\lip}[1]   {\ensuremath{\mathrm{Lip}\left(#1\right)}}
\newcommand{\lipbound} {\ensuremath{\mathrm{LipBound}}}
\newcommand{\diag}     {\ensuremath{\mathrm{diag}}}
\newcommand{\bdiag}    {\ensuremath{\mathrm{bdiag}}}

\renewcommand{\mod}[1] {\ \mathrm{mod}\ #1}

\newcommand{\circulant}   {\ensuremath \mathrm{circ}}
\newcommand{\diagonal}    {\ensuremath \mathrm{diag}}
\newcommand{\krylov}      {\ensuremath \mathrm{krylov}}
\newcommand{\Cov}         {\ensuremath \mathrm{Cov}}
\newcommand{\Var}         {\ensuremath \mathrm{Var}}
\newcommand{\cin}         {\ensuremath c_{in}}
\newcommand{\cout}        {\ensuremath c_{out}}
\newcommand{\margin}      {\ensuremath M}

\newcommand{\scalefigure}{0.80}
\newcommand{\removespace}{\vspace{-\baselineskip}}

\newcommand{\ci}{\ensuremath \mathbf{i}}

\newcommand{\lzero} {\ensuremath \ell_0 \xspace}
\newcommand{\lone}  {\ensuremath \ell_1 \xspace}
\newcommand{\ltwo}  {\ensuremath \ell_2 \xspace}
\newcommand{\linf}  {\ensuremath \ell_\infty \xspace}
\newcommand{\lp}    {\ensuremath \ell_p \xspace}
\newcommand{\fro}   {\ensuremath \mathrm{F}}

\newcommand{\onevec}[1]{\ensuremath \mathbf{1}_{#1}}
\newcommand{\zerovec}[1]{\ensuremath \mathbf{0}_{#1}}

\newcommand{\nn}{N}
\newcommand{\act}{\rho}
\newcommand{\layer}{\phi}
\newcommand{\weights}{\Omega}
\newcommand{\depth}{p}
\newcommand{\dimw}{w}
\newcommand{\dimb}{b}
\newcommand{\adv}{\boldsymbol{\tau}}
\newcommand*\diff{\mathop{}\!\mathrm{d}}

\newcommand{\triangleopup}{\boldsymbol{\bigtriangleup}}
\newcommand{\triangleopdown}{\rotatebox[x=0pt,y=4pt]{180}{$\triangleopup$}}

\newcommand{\leftmatrix}       {\begin{pmatrix}}
\newcommand{\rightmatrix}      {\end{pmatrix}}
\newcommand{\leftmatrixsmall}  {\begin{psmallmatrix}}
\newcommand{\rightmatrixsmall} {\end{psmallmatrix}}
\newcommand{\leftmat}          {\left(}
\newcommand{\rightmat}         {\right)}

\def\ddefloop#1{\ifx\ddefloop#1\else\ddef{#1}\expandafter\ddefloop\fi}

\def\ddef#1{\expandafter\def\csname #1bb\endcsname{\ensuremath{\mathbb{#1}}}}
\ddefloop ABCDEFGHIJKLMNOPQRSTUVWXYZ\ddefloop

\def\ddef#1{\expandafter\def\csname #1set\endcsname{\ensuremath{\mathcal{#1}}}}
\ddefloop ABCDEFGHIJKLMNOPQRSTUVWXYZ\ddefloop

\def\ddef#1{\expandafter\def\csname #1mat\endcsname{\ensuremath{\mathbf{#1}}}}
\ddefloop ABCDEFGHIJKLMNOPQRSTUVWXYZ\ddefloop

\def\ddef#1{\expandafter\def\csname #1matsf\endcsname{\ensuremath{\mathsf{#1}}}}
\ddefloop ABCDEFGHIJKLMNOPQRSTUVWXYZ\ddefloop

\def\ddef#1{\expandafter\def\csname #1vec\endcsname{\ensuremath{\mathbf{#1}}}}
\ddefloop abcdefghijklmnopqrstuvwxyz\ddefloop

\newcommand{\bigO}{\ensuremath \mathcal{O}}

\usepackage{lineno}
\usepackage{todonotes}
\usepackage{printlen}

\let\todoorg\todo
\renewcommand{\todo}[1]{\todoorg[inline]{#1}}



%% file: bibliography-mimosis.tex
\AtBeginBibliography{%
  \renewcommand*{\finalnamedelim}{%
    \ifthenelse{\value{listcount} > 2}{%
      \addcomma
      \addspace
      \bibstring{and}%
    }{%
      \addspace
      \bibstring{and}%
    }
  }
}

\renewbibmacro{in:}{%
  \ifentrytype{article}
  {%
  }%
  {%
    \printtext{\bibstring{in}\intitlepunct}%
  }%
}

\renewbibmacro*{issue+date}{%
  \setunit{\addcomma\space}
    \iffieldundef{issue}
      {\usebibmacro{date}}
      {\printfield{issue}%
       \setunit*{\addspace}%
       \usebibmacro{date}}%
  \newunit}

\renewbibmacro*{volume+number+eid}{%
  \printfield{volume}%
  \setunit*{\addcolon}%
  \printfield{number}%
  \setunit{\addcomma\space}%
  \printfield{eid}%
}

\renewbibmacro*{publisher+location+date}{%
  \printlist{publisher}%
  \setunit*{\addcomma\space}%
  \printlist{location}%
  \setunit*{\addcomma\space}%
  \usebibmacro{date}%
  \newunit%
}

\renewbibmacro*{organization+location+date}{%
  \printlist{location}%
  \setunit*{\addcomma\space}%
  \printlist{organization}%
  \setunit*{\addcomma\space}%
  \usebibmacro{date}%
  \newunit%
}

\DefineBibliographyStrings{english}{%
  techreport = {technical report},
}

\DeclareFieldFormat{labelnumberwidth}{#1\adddot}

\DeclareFieldFormat{doi}{%
  \mkbibacro{DOI}\addcolon\addnbspace
    \ifhyperref
      {\href{http://dx.doi.org/#1}{\nolinkurl{#1}}}
      {\nolinkurl{#1}}
}

\DeclareLanguageMapping{british}{bibliography-correct-ordinals}
\DeclareLanguageMapping{english}{bibliography-correct-ordinals}

%% file: sources/acronyms.tex
\newacronym{aka}{\aka}{\textbf{A}lso \textbf{K}own \textbf{a}s}
\newacronym{eg}{\eg}{\textbf{E}xempli \textbf{G}ratia}
\newacronym{ie}{\ie}{\textbf{I}d \textbf{E}st}
\newacronym{cf}{\cf}{\textbf{C}on\textbf{f}er}
\newacronym{st}{\st}{\textbf{S}uch \textbf{T}hat}
\newacronym{iid}{i.i.d.}{\textbf{I}dentically and \textbf{I}ndependently \textbf{D}istributed}
\newacronym{pdf}{p.d.f.}{\textbf{P}robability \textbf{D}ensity \textbf{F}unction}

\newacronym{ERM}{ERM}{\textbf{E}mpirical \textbf{R}isk \textbf{M}inimization}
\newacronym{CNN}{CNN}{\textbf{C}onvolutional \textbf{N}eural \textbf{N}etwork}
\newacronym{DCNN}{DCNN}{\textbf{D}iagonal \textbf{C}irculant \textbf{N}eural \textbf{N}etwork}
\newacronym{PGD}{PGD}{\textbf{P}rojected \textbf{G}radient \textbf{D}escent}
\newacronym{CW}{C\&W}{\textbf{C}arlini and \textbf{W}agner}

\newacronym{DFT}{DFT}{\textbf{D}iscrete \textbf{F}ourier \textbf{T}ransform}
\newacronym{FFT}{FFT}{\textbf{F}ast \textbf{F}ourier \textbf{T}ransform}
\newacronym{IFFT}{IFFT}{\textbf{I}nverse \textbf{F}ast \textbf{F}ourier \textbf{T}ransform}
\newacronym{SVD}{SVD}{\textbf{S}ingular \textbf{V}alue \textbf{D}ecomposition}

%% file: sources/glossary.tex
\newglossaryentry{Natural numbers}{%
  name        = {$\Nbb$},
  description = {The set of natural numbers},
}

\newglossaryentry{Real numbers}{%
  name        = {$\Rbb, \Rbb^n, \Rbb^{n \times m}$},
  description = {The set of real numbers, vectors and matrices},
}

\newglossaryentry{Complex numbers}{%
  name        = {$\Cbb, \Cbb^n, \Cbb^{n \times m}$},
  description = {The set of complex numbers, vectors and matrices},
}

\newglossaryentry{Positive real numbers}{%
  name        = {$\Rbb_+$, $\Rbb_-$},
  description = {The set of positive and negative real numbers respectivly},
}

\newglossaryentry{integer set}{%
  name        = {$[n]$},
  description = {The set $\left\{ x \in \Nbb \mid 1 \leq x \leq n \right\}$},
}

\newglossaryentry{real numbers interval}{%
  name        = {$[a, b]$},
  description = {The set $\left\{ x \in \Rbb \mid a \leq x \leq b \right\}$},
}

\newglossaryentry{set I}{%
  name        = {$\Iset_n$, $\Iset^+_n$},
  description = {The sets $\left\{ -n+1, \dots, n-1 \right\}$ and $\left\{ 0, \dots, n-1 \right\}$ respectively},
}

\newglossaryentry{imaginary number}{%
  name        = {$\ci = \sqrt{-1}$},
  description = {Imaginary number},
}

\newglossaryentry{real part, imaginary part}{%
  name        = {$\mathfrak{R}(z), \mathfrak{I}(z)$},
  description = {Real part and imaginary part of complex number z},
}

\newglossaryentry{complex modulus}{%
  name        = {$|a + \ci b|$},
  description = {Modulus of complex number $a + \ci b$, \ie, $|a + \ci b| = \sqrt{a^2 + b^2}$},
}

\newglossaryentry{indicators function}{%
  name        = {$\mathds{1}_{[\text{Boolean expres.}]}$},
  description = {Indicator function (equals 1 if expression is true 0 otherwise)}
}

\newglossaryentry{Matrix}{%
  name        = {$\Mmat = \left(a_{ij}\right)$},
  description = {Matrix $\Mmat$ with $(i,j)$-entry $a_{ij}$},
}

\newglossaryentry{Vector}{%
  name        = {$\xvec = \left[ \xvec_{1} \dots \xvec_{n} \right]$},
  description = {Vector $\xvec$ of size $n$ with $\xvec_{i}$ entries},
}

\newglossaryentry{Transpose, Conjugate}{%
  name        = {$\Mmat^\top, \Mmat^*$},
  description = {Transpose and conjugate transpose of matrix $\Mmat$},
}

\newglossaryentry{vector of ones, zero}{%
  name        = {$\onevec{n}, \zerovec{n}$},
  description = {\emph{n-}vector of ones and \emph{n-}vector of zeros},
}

\newglossaryentry{unit vector}{%
  name        = {$\evec^{(i)}$},
  description = {$i$-th unit vector in $\Rbb^n$}
}

\newglossaryentry{identity, reflection}{%
  name        = {$\Imat_n, \Jmat_n$},
  description = {Identity and reflection matrix of size $n \times n$, \ie, $\Jmat^2 = \Imat$},
}

\newglossaryentry{discrete fourier transform}{%
  name        = {$\Umat_n$},
  description = {DFT matrix of size $n \times n$, \ie, 
  $\Umat_n = \leftmatrix e^{-(2 \pi \ci jk)/n} \rightmatrix_{j,k = 0}^{n-1}$},
}

\newglossaryentry{Vector p-norm}{%
  name        = {$\norm{\xvec}_p$},
  description = {Norm $p$ of vector $\xvec$, \ie,
  $\norm{\xvec}_p = \left( \sum_{i=0}^n \xvec_i^p \right)^\frac{1}{p} $},
}

\newglossaryentry{Vector inf-norm}{%
  name        = {$\norm{\xvec}_\infty$},
  description = {Infinity Norm of vector $\xvec$, \ie,
  $\norm{\xvec}_\infty = \max_i |\xvec_i|$},
}

\newglossaryentry{Matrix p-norm}{%
  name        = {$\norm{\Mmat}_p$},
  description = {Norm $p$ of matrix $\Mmat$, \ie, 
 $\norm{\Mmat}_p = \sup_{\norm{\xvec}_p \neq 0} \frac{\norm{\Mmat \xvec}_p}{\norm{\xvec}_p}$},
}

\newglossaryentry{Frobenius Norm}{%
  name        = {$\norm{\Mmat}_\mathrm{F}$},
  description = {Frobenius Norm of matrix $\Mmat$},
}

\newglossaryentry{largest singular value of matrix}{%
  name        = {$\sigma_1(\Mmat)$},
  description = {Largest singular value of matrix $\Mmat$, \ie $\sigma_1(\Mmat) = \norm{\Mmat}_2$},
}

\newglossaryentry{largest eigenvalue of Hermitian matrix}{%
  name        = {$\lambda_1(\Mmat)$},
  description = {Largest eigenvalue of matrix $\Mmat$},
}

\newglossaryentry{semi positive definite matrix}{%
  name        = {$\Mmat \geq 0$},
  description = {$\Mmat$ is positive semidefinite, \ie $\xvec^*\Mmat\xvec \geq 0 \text{ for all } \xvec \in \Cbb^n$},
}

\newglossaryentry{positive definite}{%
  name        = {$\Mmat > 0$},
  description = {$\Mmat$ is positive definite, \ie $\xvec^*\Mmat\xvec > 0 \text{ for all } \xvec \in \Cbb^n$},
}

\newglossaryentry{probability, expectation}{%
  name        = {$\Pbb, \Ebb$},
  description = {Probability and expectation of a random variable},
}

\newglossaryentry{normal distribution}{%
  name        = {$\mathcal{N}$},
  description = {Gaussian distribution},
}

%% file: sources/abstract-cover.tex
\enabstract{
Deep neural networks are state-of-the-art in a wide variety of tasks, however, they exhibit important limitations which hinder their use and deployment in real-world applications.
When developing and training neural networks, the accuracy should not be the only concern, neural networks must also be cost-effective and reliable.
Although accurate, large neural networks often lack these properties.
This thesis focuses on the problem of training neural networks which are not only accurate but also compact, easy to train, reliable and robust to adversarial examples.
To tackle these problems, we leverage the properties of structured matrices from the Toeplitz family to build compact and secure neural networks. 
}
\enkeywords{Deep Learning, Neural networks, Structured Matrices}

\frabstract{
Les réseaux de neurones profonds sont considérés comme étant état de l’art dans une grande variété de tâches, mais ils présentent des limites importantes qui entravent leur utilisation et leur déploiement.
Lors du développement et l’entraînement de réseaux de neurones, la précision ne devrait pas être la seule préoccupation, ils se doivent aussi d’être efficaces et sécurisés.
Bien que précis, les réseaux de neurones dotés de nombreux paramètres n’ont souvent pas ces propriétés.
Cette thèse se concentre sur le problème de l'entraînement de réseaux de neurones qui ne sont pas seulement précis, mais aussi compacts, faciles à entraîner, fiables et robustes aux exemples contradictoires.
Pour résoudre ces problèmes, nous exploitons les propriétés des matrices structurées de la famille de Toeplitz pour construire des réseaux de neurones compacts et sécurisés.
}
\frkeywords{Apprentissage Profond, Réseaux de Neurones, Matrices Structurées}

%% file: sources/dedication.tex
\clearpage
\begin{center}
  \thispagestyle{empty}
  \vspace*{\fill}
  \emph{Dédié à la mémoire de mon grand-père maternel} \\
  \textbf{Jean Marchelie} \\
  $1936$ -- $2021$
  \vspace*{\fill}
\end{center}
\clearpage

\newpage
\null
\thispagestyle{empty}
\newpage

%% file: sources/acknowledgments.tex
\newpage
\begin{center}
  {\Huge \textsc{Remerciements}}
\end{center}
\noindent

Pour commencer, je souhaiterais remercier Teddy Furon et Alain Rakotomamonjy d'avoir accepté d'être examinateurs de cette thèse ainsi qu'Élisa Fromont, Rémi Gribonval et Krzysztof Choromanski de s'être intéressés à mes travaux de recherche et d'avoir accepté d'être membres du jury.
Nos échanges pendant la relecture ainsi que la soutenance ont été très enrichissants.

J'éprouve une profonde gratitude envers mes trois encadrants de thèse, Jamal, Yann et Benjamin durant ces quatre dernières années où j'ai pu découvrir le monde de la recherche et me former à devenir un bon chercheur.
En effet, le métier de chercheur est bien le plus beau métier du monde ! 
Merci d'avoir cru en moi et de m'avoir offert cette chance, merci pour tous ces échanges, désaccords, réunions, sessions de travail et enfin, merci de m'avoir toléré quand j'étais pénible.
Vous m'avez vraiment offert un encadrement exceptionnel: peu de doctorants peuvent se targuer de pouvoir échanger avec leurs encadrants tous les jours !

Cette expérience a été d'autant plus riche grâce aux autres doctorants du laboratoire.
Ainsi, je souhaite remercier Rafael et Laurent pour les collaborations réalisées ensemble, et également Geovani, Raphaël, Alexandre V., Éric, Virginie et Céline pour avoir participé à l'émulation et la bonne ambiance du labo.
Merci également à Florian, Clément et Alexandre A. pour les échanges que l'on a pu avoir.

Cette thèse n'aurait pas été possible sans les financements de Wavestone.
Ainsi, je souhaite remercier Cyril pour m'avoir donné cette opportunité ainsi que Nicolas pour nos longs échanges sur mes travaux de recherche.
Également, je souhaite remercier David, Hugo et Cédric pour m'avoir aidé au cours de cette expérience.

Je souhaite également remercier ma famille, belle-famille et mes amis pour m'avoir toléré, accompagné, conseillé, encouragé et pour avoir essayé de comprendre ce que je faisais ces quatre dernières années.
Désolé d'avoir autant travaillé pendant les week-ends, j'espère ne pas avoir été trop désagréable ou trop ``dans mon monde'' pendant cette période.
Un grand merci à Othmane qui, à travers nos échanges, m'a donné l'idée de faire une thèse et de me lancer dans la recherche.

Pour finir, je souhaite exprimer ma profonde gratitude envers Chuthima, ma partenaire du quotidien.
Merci pour tous les compromis faits pour me laisser réaliser cette thèse, pour le soutien quotidien, les encouragements et les relectures.
La qualité de cette thèse est énormément due à ta présence à mes côtés.

\newpage
\null
\thispagestyle{empty}
\newpage

%% file: sources/abstract.tex
\newpage
\begin{center}
  {\Huge \textsc{Abstract}}
\end{center}
\noindent
Deep neural networks are state-of-the-art in a wide variety of tasks, however, they exhibit important limitations which hinder their use and deployment in real-world applications.
When developing and training neural networks, the accuracy should not be the only concern, neural networks must also be cost-effective and reliable.
Although accurate, large neural networks often lack these properties.
In this thesis, we leverage the properties of structured matrices from the Toeplitz family to build compact and secure neural networks.
Our contributions are twofold.


First, we propose a new neural network architecture that is not only accurate but also compact and easy to train.
The purpose of this contribution is to study deep diagonal-circulant neural networks, which are deep neural networks in which weight matrices are the product of diagonal and circulant ones.
We perform a theoretical analysis of their expressivity and propose an initialization procedure and an intelligent use of nonlinearity functions to facilitate training. 
Furthermore, we show that these networks outperform recently introduced deep networks with other types of structured layers.
We conduct a thorough experimental study to compare the performance of deep diagonal-circulant networks with state-of-the-art models based on structured matrices and with dense models.
We show that our models achieve better accuracy than other structured approaches while requiring 2x fewer weights than the next best approach.
Finally, we train compact and accurate deep diagonal-circulant networks on a real-world video classification dataset with over 3.8 million training examples.

Secondly, we propose an approach to build robust neural networks to adversarial examples.
In this contribution, we introduce a new Lipschitz regularization for Convolutional Neural Networks that improves the robustness of neural networks.
Lipschitz regularity is now established as a key property of modern deep learning with implications in training stability, generalization, robustness against adversarial examples, etc.
However, computing the exact value of the Lipschitz constant of a neural network is known to be NP-hard.
Recent attempts from the literature introduce upper bounds to approximate this constant that are either efficient but loose or accurate but computationally expensive.
In this work, by leveraging the properties of doubly-block Toeplitz matrices, we introduce a new upper bound of the singular values of convolution layers that is both tight and easy to compute.
Based on this result we devise an algorithm to train Lipschitz-regularized Convolutional Neural Networks.

\newpage
\null
\thispagestyle{empty}
\newpage

%% file: sources/resume.tex
\begin{center}
  {\Huge \textsc{Résumé}}
\end{center}
\noindent
Les réseaux de neurones profonds sont considérés comme étant état de l'art dans une grande variété de tâches, mais ils présentent des limites importantes qui entravent leur utilisation et leur déploiement.
Lors du développement et l'entraînement de réseaux de neurones, la précision ne devrait pas être la seule préoccupation, ils se doivent aussi d'être efficaces et sécurisés.
Bien que précis, les réseaux de neurones dotés de nombreux paramètres n'ont souvent pas ces propriétés.
Dans cette thèse, nous exploitons les propriétés des matrices structurées de la famille de Toeplitz pour construire des réseaux de neurones compacts et sécurisés.
Nous réalisons deux contributions sur ces thématiques.


Premièrement, nous proposons une nouvelle architecture de réseau de neurones précise, mais également compacte et facile à entraîner. 
L'objectif de cette contribution est d'étudier les réseaux de neurones diagonaux-circulants, qui sont des réseaux de neurones profonds pour lesquels les matrices de poids sont le produit des matrices diagonales et circulantes.
Nous effectuons une analyse théorique de leur expressivité et proposons une procédure d'initialisation et une utilisation intelligente des fonctions de non-linéarité qui facilitent leur entraînement.
Nous montrons que nos modèles atteignent une meilleure précision que les autres approches structurées tout en nécessitant deux fois moins de paramètres.
Enfin, nous entraînons des réseaux de neurones diagonaux-circulants sur un ensemble de données de classification vidéo qui contient plus de 3,8 millions d'exemples.

%
%
%
%
%

Deuxièmement, en plus d'être compacts et précis, les réseaux de neurones se doivent d'être sécurisés.
Pour améliorer leur robustesse, nous proposons une nouvelle régularisation pour les réseaux convolutifs basée sur la constante de Lipschitz.
La régularisation Lipschitz est maintenant établie comme une propriété clé de l'apprentissage profond avec des implications en stabilité, généralisation et robustesse contre les attaques adversariales, etc.
Cependant, le calcul de la constante de Lipschitz d'un réseau de neurones est connu pour être un problème NP-complet.
De récentes tentatives introduisent des bornes supérieures pour approximer cette constante qui sont soit efficaces, mais peu précises, soit précises mais coûteuses.
Dans cette thèse, en exploitant les propriétés des matrices de Toeplitz à bloc de Toeplitz, nous introduisons une nouvelle borne supérieure de cette constante pour les couches convolutionnelles qui est à la fois précise et facile à calculer.
Sur la base de ce résultat, nous concevons un algorithme pour entraîner des réseaux de neurones convolutifs avec une régularisation Lipschitz.

\newpage
\null
\thispagestyle{empty}
\newpage

%% file: sources/main/ch1-introduction.tex
\chapter{Introduction}
\label{chapter:ch1-introduction}
\localtoc

\section{Context and Motivation}
\label{section:ch1-context_and_motivation}

\begin{figure}[t]
  \centering
  \includegraphics[scale=0.2]{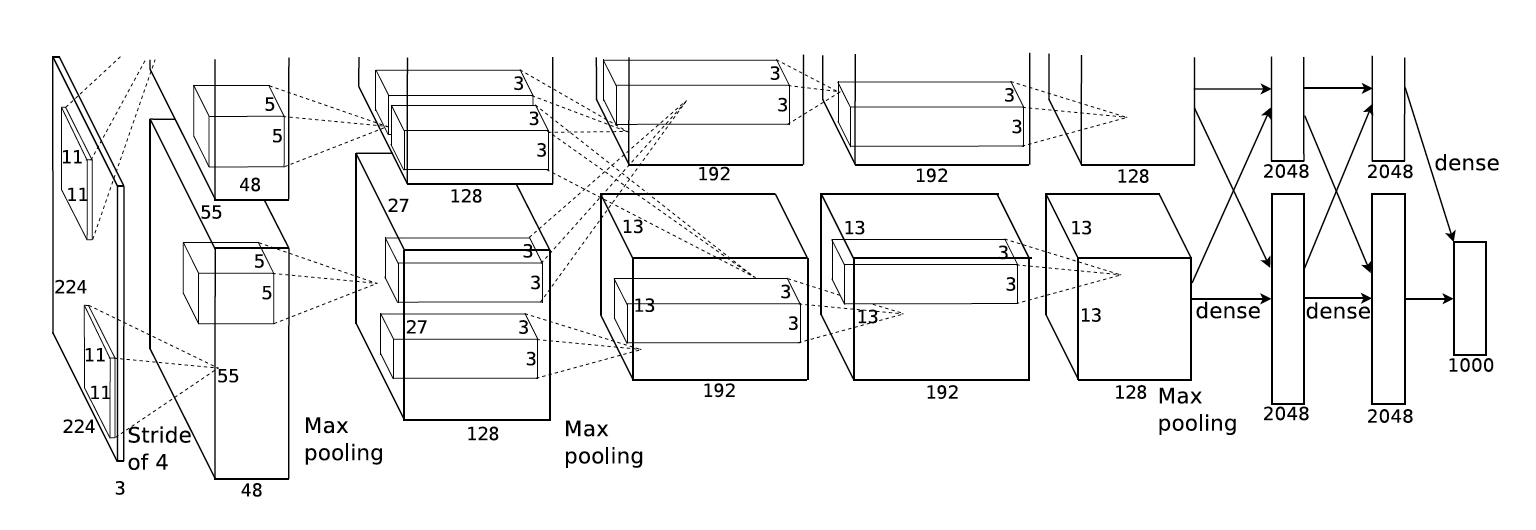}
  \caption{The neural network architecture (AlexNet) proposed by~\citet{krizhevsky2012imagenet} which won the ImageNet Large-Scale Visual Recognition Challenge in 2012.}
  \label{figure:ch1-alexnet_network}
\end{figure}

Since the dawn of computer science, researchers have been trying to emulate intelligence through computers.
Alan Turing was the first, in a paper called \emph{Computing Machinery and Intelligence} \cite{turing1950computing}, to lay the foundation for what we now call \emph{Artificial Intelligence}.
In the last 20 years, with the surge in data collection and computing resources, the interest and use cases for Machine Learning have grown exponentially.
More specifically, Deep Learning, a subfield of Machine Learning, consisting of training Deep Neural Networks on high-level data (images, sounds, texts) have shown great achievements, even outperforming humans on certain tasks.

One of the most remarkable breakthroughs of Deep Learning happened in 2012 during the ImageNet Large-Scale Visual Recognition Challenge~\cite{russakovsky2015imagenet}.
The challenge aims at evaluating different algorithms for object detection and image classification.
In 2012, \citeauthor{krizhevsky2012imagenet} obtained \nth{1} place and beat every other participant by a 10.8\% margin with a neural network architecture called \textbf{AlexNet}.
The main reasons for this success are twofold.
First, they used a convolutional neural network (CNN) with more than 60 million parameters which was one of the largest models of the time.
Secondly, they designed a specific architecture to exploit dual programmable graphics processing units (GPUs) to speed up the arithmetic operations, which enabled them to significantly reduce training time.
\Cref{figure:ch1-alexnet_network} shows the AlexNet architecture which consists of five convolution layers with two fully connected layers at the end.

\begin{table}[t]
  \centering
  \sisetup{
    table-number-alignment = center,
    table-space-text-pre = \ \ \ ,
  }
  \begin{subfigure}[b]{\textwidth}
    \centering
    \begin{tabular}{
      L{5cm}
      L{3.5cm}
      S[table-format=3.0, table-text-alignment=left]@{\,}
      s[table-unit-alignment=left]
      c
    }
      \toprule
      \textbf{Authors} & \textbf{Models} & \multicolumn{2}{c}{\textbf{\#Params}} & \textbf{TOP-5 Acc.} \\
      \midrule
      \citet{krizhevsky2012imagenet} & AlexNet             &  61 & \si{M} & 84.7\% \\
      \citet{simonyan2014very}       & VGG                 & 144 & \si{M} & 92.0\% \\
      \citet{he2016deep}             & ResNet-152          &  60 & \si{M} & 93.8\% \\
      \citet{szegedy2017inception}   & Inception-ResNet-v2 &  56 & \si{M} & 95.1\% \\
      \citet{xie2017aggregated}      & ResNeXt-101         &  84 & \si{M} & 95.6\% \\
      \citet{hu2018squeeze}          & SENet               & 146 & \si{M} & 96.2\% \\
      \citet{real2019regularized}    & AmoebaNet-A         & 469 & \si{M} & 96.7\% \\
      \citet{huang2019gpipe}         & AmoebaNet-B         & 556 & \si{M} & 97.0\% \\
      \bottomrule
    \end{tabular}
    \caption{Computer Vision Models}
    \label{table:ch1-networks_parameters_cv}
  \end{subfigure}
  \par\bigskip
  \begin{subfigure}[b]{\textwidth}
    \centering
    \begin{tabular}{
      L{4.5cm}
      L{4.5cm}
      S[table-format=3.0, table-text-alignment=left]@{\,}
      s[table-unit-alignment=left]
    }
      \toprule
      \textbf{Authors} & \textbf{Models} & \multicolumn{2}{c}{\textbf{\#Params}} \\
      \midrule
      \citet{peters2018deep}         & ELMo            &  94 & \si{M} \\
      \citet{radford2018improving}   & GPT             & 110 & \si{M} \\
      \citet{devlin2019bert}         & BERT            & 340 & \si{M} \\
      \citet{yang2019xlnet}          & XLNet (Large)   & 340 & \si{M} \\
      \citet{liu2019roberta}         & RoBERTa (Large) & 355 & \si{M} \\
      \citet{radford2019language}    & GPT-2           &   1 & \si{B} \\
      \citet{shoeybi2019megatron}    & MegatronLM      &   8 & \si{B} \\
      \citet{raffel2020exploring}    & T5-11B          &  11 & \si{B} \\
      \citet{rosset2020turingnlg}    & T-NLG           &  17 & \si{B} \\
      \citet{brown2020language}      & GPT-3           & 175 & \si{B} \\
      \citet{fedus2021switch}        & Switch Transformers & 1 & \si{T} \\
      \bottomrule
    \end{tabular}
    \caption{Natural Language Processing Models}
    \label{table:ch1-networks_parameters_nlp}
  \end{subfigure}
  \par\bigskip
  \caption{Evolution of the number of parameters for Computer Vision and Natural Language Processing models developed in the years after AlexNet.}
  \label{table:ch1-networks_parameters}
\end{table}

Following this result, many architectures with an increasing number of parameters have been developed.
This growth in the number of parameters has led to an substantial gains in accuracy, exceeding even human performance, on the ImageNet dataset~\cite{he2015delving}.
\Cref{table:ch1-networks_parameters} shows a list of the different state-of-the-art architectures along with their size and accuracy.
As we can see, the accuracy of the models generally improves at the cost of the model size.
For computer vision models, \citet{tan2019efficientnet} have empirically shown that the relationship between model size and accuracy seems to obey a power law.
This relationship has also been observed for neural networks designed for Natural Language Processing (NLP) \cite{rosenfeld2020a,kaplan2020scaling} aided by the availability of large-scale datasets such as the Common Crawl dataset~\cite{raffel2020exploring} which constitutes nearly a trillion words.

As a result of their size and improved accuracy, deep neural networks now achieve state-of-the-art performances in a variety of domains such as image recognition~\cite{lecun1998gradient,krizhevsky2012imagenet,he2016deep,tan2019efficientnet}, object detection~\cite{redmon2016you,liu2016ssd,redmon2017yolo9000}, natural language processing~\cite{merity2016pointer,vaswani2017attention,radford2019language,brown2020language}, speech recognition~\cite{hinton2012deep,abdel2014convolutional,yu2016automatic}, games \cite{silver2017mastering}, etc.
Specifically, computer vision and natural language processing models have achieved sufficient performance for being used in real-world applications such as autonomous vehicles~\cite{sadat2019jointly}, translation~\cite{bahdanau2015neural}, vocal assistants~\cite{li2017acoustic}, etc.

However, accuracy is not the only concern, when implemented in a critical decision process, neural networks need to be compact, cost-effective and secure.
Although accurate, large neural networks often lack these properties.
Indeed, training state-of-the-art models on computer vision or natural language processing tasks requires gigabytes of memory and can take several months on a single GPU~\cite{krizhevsky2012imagenet,brown2020language}.
For example, the GPT-3 model proposed by~\citet{brown2020language}, culminates at 175 billion parameters and requires 355 years of training on a single GPU and \$\numprint{4600000} to train on a cloud-computing platform \cite{li2020overview}.
It has also been estimated by \citet{strubell2019energy} that the training and development costs of the large Transformer model proposed by~\citet{vaswani2017attention} with neural architecture search emits an estimated \numprint{284019} kg of $\mathrm{CO}_2$ whereas a human life will consume an average of \numprint{5000} kg of $\mathrm{CO}_2$ for one year. 
Furthermore, with the rise of smartphones and ``Internet of things'' devices (IoT) with limited computational and memory resources, neural networks also need to be efficient during the inference phase.
In addition, with the growing concern over data privacy, methods such as \emph{federated learning} are gaining ground.
Federated learning involves training a model across multiple decentralized devices (\eg, smartphones) with local data samples.
This avoids the step of centralizing all users' data into one server, thus addressing, even modestly, the issue of data privacy.
Thus, building compact and cost-effective neural networks have been an important goal in order to reduce training time, reduce cost and allow for faster research and development.

In addition to being compact and cost-effective, neural networks also need to be secure.
Due to their high complexity and expressivity, large neural networks exhibit instability to small perturbations of their inputs.
Unstable neural networks tend to be vulnerable to \emph{adversarial examples}, \ie, imperceptible variations of natural examples, crafted to deliberately mislead the models~\cite{globerson2006nightmare,biggio2013evasion,szegedy2013intriguing}.
\Cref{figure:ch1-adversarial_image_example} gives an example of an adversarial attack on an image.
The small perturbation (center) is added to the original image (left) leading to an adversarial image (right).
This behavior can cause serious security problems when neural networks are used for critical decision-making (\eg, self-driving cars, predictive justice, etc.).

This thesis focuses on the problem of training neural networks which are not only accurate but also compact, easy to train, reliable and robust to adversarial examples.

%

\section{Problem Statement and Contributions}
\label{section:ch1-problem_statement_and_contributions}

\input{figures/main/ch1-introduction/example_structure_matrices}

Neural networks, which find their roots in the work of \citet{mcculloch1943logical,rosenblatt1958perceptron}, can be analytically described as a composition of multi-dimensional linear functions interlaced with nonlinear functions (also called activation functions).
More formally, a neural network is a function $N_{\Omega} : \Rbb^n \rightarrow \Rbb^m$ parameterized by a set of weights $\Omega$ of the form:
\begin{equation} \label{equation:ch1-neural_network}
  N_{\Omega}(\xvec) = \psi^{(\depth)} \circ \rho \circ \psi^{(\depth-1)} \cdots \circ \psi^{(2)} \circ \rho \circ \psi^{(1)} (\xvec) \enspace.
\end{equation}
Here, $\depth$ corresponds to the \emph{depth} of the network (\ie, the number of layers) and $\rho$ is a nonlinear function.
Each $\psi^{(i)}$ is a multi-dimensional linear function $\psi^{(i)}: \xvec \mapsto \Wmat^{(i)} \xvec + \bvec^{(i)}$ parameterized by a weight matrix $\Wmat^{(i)}$ and a bias vector $\bvec^{(i)}$ and $\Omega$ is the union of the parameters of all the layers.

Classical neural networks typically have a large number of parameters to train.
If they have no restrictions on the weight matrices $\Wmat^{(i)}$, the layers are said to be \emph{fully connected}.
Typically, fully connected neural networks have a large number of parameters.
For example, a fully connected neural network with $\depth$ layers and $n$ neurons on each layer ($\Wmat^{(i)} \in \Rbb^{n \times n}$) will have $pn (n + 1)$ parameters.
Since the input and output dimensions are generally large (\eg, ImageNet has an input dimension of $224^2 \times 3$ and an output of 1000), simple fully connected neural networks with few layers accumulate over hundreds of millions of parameters.
Generally, this type of neural network has been shown to perform poorly due to a large search space or due to the important expressivity of the model which leads to overfitting\footnote{For Machine Learning models, overfitting is a well understood phenomenon. However, it has been discovered that large deep neural networks exhibit a ``double descent'' phenomenon (see \citet{spigler2019jamming}), where the performance first gets worse (overfits) then gets  better with longer training.}.
Moreover, they are computationally expensive, which makes them impractical for a number of use cases (smartphones, IoT devices, etc.).
To reduce the number of parameters on each layer, researchers have devised specific linear operations that reduce the number of parameters and have better properties for the problem at hand.

\begin{figure}[t]
  \centering
  \includegraphics[width=\textwidth]{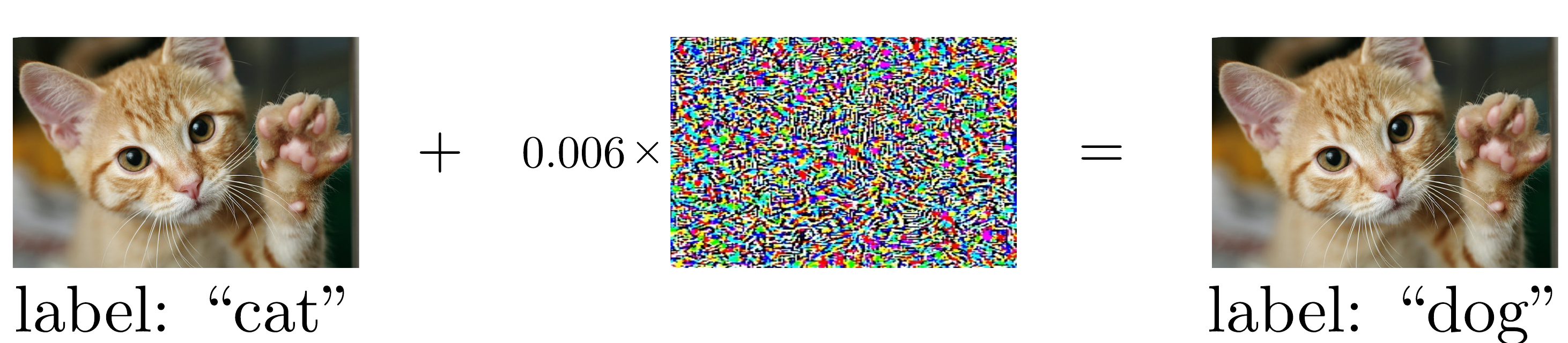}
  \caption{Example of Adversarial Attack on an image.}
  \label{figure:ch1-adversarial_image_example}
\end{figure}

An example of widely used neural networks with specialized and more compact linear operations are \emph{Convolutional Neural Networks} (CNN)~\cite{lecun1998gradient,krizhevsky2012imagenet,he2016deep,tan2019efficientnet} which achieve state-of-the-art results for computer vision tasks.
Convolutional neural networks, which find their roots in the work of~\citet{fukushima1982neocognitron}, use specific weight matrices which encode the translation invariant property often desirable to process images.
Whereas a classical linear layer with a dense matrix will have $n \times n$ parameters, a convolution layer only has $k \times k$ parameters where $k \ll n$ is the kernel size and is usually small (\eg, 3 or 5 for classical convolution layers).
A convolutional neural network is the most common type of \emph{structured} neural networks.
Indeed, the convolution operation can be represented by a structured matrix \ie, a matrix that can be represented with less than $n^2$ parameters.

In addition to offering a more compact representation, the structure of certain matrices can be exploited to obtain better algorithms for the matrix-vector product, thus optimizing memory and computing operations.
Based on the success of convolutional neural networks, researchers have studied and proposed other types of neural networks based on weight matrices with different structures (\eg, ~\citet{moczulski2016acdc,sindhwani2015structured}).
\Cref{figure:ch1-example_structure_matrices} shows different types of structured matrices that have been used for deep learning.
Although convolutional neural networks have been state-of-the-art for computer vision tasks, it remains unclear whether other types of structured networks can be beneficial to other types of applications and which type of structure can provide both accuracy and efficient computation.

The contributions of this thesis lie at the intersection of linear algebra, Fourier analysis and deep learning.
As a result, we build compact and secure neural networks by leveraging the properties of structured matrices from the Toeplitz family.
Hereafter, we summarize our contributions.

\subsection{Training Compact Neural Networks}
\label{subsection:ch1-training_compact_neural_networks}


As a first contribution, we use circulant matrices, which are a particular type of matrix from the Toeplitz family, to devise a new compact architecture replacing fully connected neural networks.
More precisely, we study deep diagonal-circulant neural networks, which are deep neural networks in which weight matrices are replaced by the product of diagonal and circulant ones.
Besides making a theoretical analysis of their expressivity, we introduce principled techniques for training these models: we devise an initialization scheme and propose a smart use of nonlinearity functions in order to train deep diagonal-circulant networks.
Furthermore, we show that these networks outperform recently introduced deep networks with other types of structured layers.
We conduct a thorough experimental study to compare the performance of these networks with state-of-the-art models. 
We show that our models achieve better accuracy than other structured approaches while requiring 2x fewer weights than the next best approach.
Finally, we train accurate deep diagonal-circulant networks on a real-world video classification dataset with over 3.8 million training examples.

\begin{mdframed}[
    topline=true,
    bottomline=true,
    rightline=true,
    linecolor=mydarkblue,
    linewidth=3pt,
    backgroundcolor=white,
    innerleftmargin=0.02\textwidth,
    innerrightmargin=0.02\textwidth,
    skipabove=0.5cm,
    skipbelow=0.5cm
  ]
  This contribution has been the subject of the following publications:
  \begin{itemize}[topsep=0pt,leftmargin=12pt]
    \setlength\itemsep{-0.3em}
    \item
      \emph{Training Compact Deep Learning Models for Video Classification using Circulant Matrices}
      in the \textbf{European Conference on Computer Vision Workshops on Video Classification}
    \item
      \emph{Understanding and Training Deep Diagonal Circulant Neural Networks} in the 
      \textbf{24th European Conference on Artificial Intelligence}.
  \end{itemize}
\end{mdframed}


\subsection{Training Robust Neural Networks}
\label{subsection:ch1-training_robust_neural_networks}

As a second contribution, we build robust neural networks by studying the properties of the structure of convolutions.
We devise a new upper bound on the largest singular value of convolution layers that is both tight and easy to compute.
Our work is based on the result of~\citet{gray2006toeplitz} which states that an upper bound on the singular value of Toeplitz matrices can be computed from the inverse Fourier transform of the characteristic sequence of these matrices.
From our analysis immediately follows an algorithm for bounding the Lipschitz constant of a convolution layer, and by extension the Lipschitz constant of the whole network.
Finally, we illustrate our approach to adversarial robustness.
Recent work has shown that empirical methods such as adversarial training offer poor generalization~\cite{schmidt2018adversarially,rice2020overfitting} and can be improved by applying Lipschitz regularization~\cite{farnia2018generalizable}.
To illustrate the benefit of our new method, we train neural networks with Lipschitz regularization and show that it offers a significant improvement over adversarial training alone.

Additional joint contributions have also been made on the topic of robust neural networks.
A first work studied the effectiveness of noise injection at training and inference time in neural networks to protect against adversarial attacks.
In this work, we have shown that noise drawn from the Exponential family offers a provable protection against adversarial attacks. 
A follow-up work conducts a geometrical analysis of defense mechanisms designed to protect neural networks against several types of attacks.
This work shows that neural networks designed to be robust against one type of adversarial example offer poor protection against other types of attacks.

\begin{mdframed}[
    topline=true,
    bottomline=true,
    rightline=true,
    linecolor=mydarkblue,
    linewidth=3pt,
    backgroundcolor=white,
    innerleftmargin=0.02\textwidth,
    innerrightmargin=0.02\textwidth,
    skipabove=0.5cm,
    skipbelow=0.5cm
  ]
  The contribution on adversarial robustness has been the subject of the following publications:
  \begin{itemize}[topsep=0pt,leftmargin=12pt]
    \setlength\itemsep{-0.3em}
    \item
      \emph{On Lipschitz Regularization of Convolutional Layers using Toeplitz Matrix Theory}
      in the \textbf{\nth{35} AAAI Conference on Artificial Intelligence}
    \item
      \emph{Theoretical evidence for adversarial robustness through randomization} in the 
      \textbf{Advances in Neural Information Processing Systems}.
    \item
      \emph{Advocating for Multiple Defense Strategies against Adversarial Examples} in the
      \textbf{European Conference on Machine Learning Workshop for CyberSecurity}
  \end{itemize}
\end{mdframed}

\section*{Outline of the Thesis}
\label{section:ch1-outline_of_the_thesis}

This thesis is organized in six chapters.
First, \Cref{chapter:ch2-background} gives an introduction to the theory of Toeplitz matrices and on supervised learning and neural networks.
This chapter presents the necessary technical tools we will need for presenting the related work and for our contributions.
\Cref{chapter:ch3-related_work} is dedicated to discussing the state-of-the-art approaches related to our contributions.
The chapter is divided into two parts.
First, we review some techniques to build compact neural networks with an important focus on techniques that use structured matrices.
The second part focuses on presenting regularization methods for improving the robustness of neural networks.
\Cref{chapter:ch4-diagonal_circulant_neural_network} and \Cref{chapter:ch5-lipschitz_bound} constitute our main contributions.
\Cref{chapter:ch4-diagonal_circulant_neural_network} presents results on compact neural networks built from diagonal and circulant matrices.
\Cref{chapter:ch5-lipschitz_bound} presents our new regularization scheme to improve the robustness of neural networks based on the properties of doubly-block Toeplitz matrices.
\Cref{chapter:ch6-conclusion} proposes a discussion and some perspectives on our contributions.
Appendix~\ref{appendix:ap2-diagonal_circulant_neural_networks_for_video_classification} constitutes some complements to~\Cref{chapter:ch4-diagonal_circulant_neural_network}.
It provides additional experiments on video classification with compact neural networks.
Finally, Appendices~\ref{appendix:ap3-theoretical_evidence_for_adversarial_robustness_through_randomization} and ~\ref{appendix:ap4-advocating_multiple_defense_strategies_against_adversarial_examples} provide further work on the robustness of neural networks done during this Ph.D. thesis.

%% file: figures/main/ch1-introduction/example_structure_matrices.tex
\begin{figure}[t]
   \centering
   \begin{subfigure}[t]{0.24\textwidth}
       \centering
       \begin{equation*}
	  \leftmatrix
	    a &   &   &   \\
	      & b &   &   \\
	      &   & c &   \\
	      &   &   & d
	  \rightmatrix
       \end{equation*}
       \caption*{diagonal}
   \end{subfigure}
   \hfill
   \begin{subfigure}[t]{0.24\textwidth}
       \centering
       \begin{equation*}
	  \leftmatrix
	    a & b & c & d \\
	    e & a & b & c \\
	    f & e & a & b \\
	    g & f & e & a
	  \rightmatrix
       \end{equation*}
       \caption*{Toeplitz}
   \end{subfigure}
   \hfill
   \begin{subfigure}[t]{0.24\textwidth}
       \centering
       \begin{equation*}
	  \leftmatrix
	    ae & af & ag & ah \\
	    be & bf & bg & bh \\
	    ce & cf & cg & ch \\
	    de & df & dg & dh
	  \rightmatrix
       \end{equation*}
       \caption*{Low Rank}
   \end{subfigure}
   \hfill
   \begin{subfigure}[t]{0.24\textwidth}
       \centering
       \begin{equation*}
	  \leftmatrix
	    a & a^2 & a^3 & a^4 \\
	    b & b^2 & b^3 & b^4 \\
	    c & c^2 & c^3 & c^4 \\
	    d & d^2 & d^3 & d^4
	  \rightmatrix
       \end{equation*}
       \caption*{Vandermonde}
   \end{subfigure}
  \caption{Examples of structured matrices.}
  \label{figure:ch1-example_structure_matrices}
\end{figure}

%% file: sources/main/ch2-background.tex
\chapter{Background}
\label{chapter:ch2-background}
\localtoc

\section*{}

This chapter gives an overview on the theory of Toeplitz matrices and on supervised learning with neural networks.
The first section describes the mathematical properties of Toeplitz matrices and some known theorems that we use in this thesis.
A Toeplitz matrix, named after Otto Toeplitz, is a matrix in which each descending diagonal, from left to right, is constant.
This simple property has led to interesting theoretical results and numerous applications.
We will use a number of these results in the context of neural networks.
The second section of this chapter is divided into four parts.
First, we review notions of supervised learning which refer to the problem of optimizing the parameters of a function in order to map an input to an output based on a series of input-output pairs.
Then, we formally define neural networks and recall some of their properties.
We pursue by introducing the concept of adversarial examples which we will use in \Cref{chapter:ch5-lipschitz_bound}.
Finally, we present some recent theoretical results on neural networks that allow a better understanding of the contributions of this thesis.


\section{A Primer on Circulant and Toeplitz Matrices}
\label{section:ch2-a_primer_on_circulant_and_toeplitz_matrices}

\input{sources/main/ch2-background_toeplitz}

\section{Supervised Learning and Neural Networks}
\label{section:ch2-supervised_learning_neural_networks}
\input{sources/main/ch2-background_neural_networks}

\section{Summary of the Chapter}
\label{section:ch2-summary_of_the_background}

As explained in the Introduction (\Cref{chapter:ch1-introduction}), our contributions lie at the intersection between neural networks and structured matrices.
In this chapter, we have reviewed the necessary concepts to present our contributions and some related works.

First, \Cref{section:ch2-a_primer_on_circulant_and_toeplitz_matrices} introduced circulant and Toeplitz matrices which are the main mathematical objects used in this thesis.
Circulant and Toeplitz matrices are structured matrices in which each descending diagonal, from left to right, is constant.
These structured matrices are the building blocks of our contribution on compact neural networks (\Cref{chapter:ch4-diagonal_circulant_neural_network}) and enable fast approximation of the Lipschitz constant of convolution layers leading to a new regularization scheme (\Cref{chapter:ch5-lipschitz_bound}).

Finally, in \Cref{subsection:ch2-introduction_on_supervised_learning}, we gave a quick overview of the concept of supervised learning, which presents the mathematical tools for optimizing a parameterized function in order to map an input to an output based on a series of input-output pairs.
Although the statistical learning framework considers generic hypothesis space, in this work we use a class of functions called neural networks presented in \Cref{subsection:ch2-preliminaries_on_neural_networks}.
We also presented, in \Cref{subsection:ch2-adversarial_attacks_robustness_of_neural_networks}, the concept of adversarial attacks and robustness of neural networks.
We showed how a neural network can be sensitive to small perturbations to its input and thus vulnerable to adversarial examples.
Reducing the sensitivity and therefore increasing the robustness of neural networks is the central theme of our second contribution presented in \Cref{chapter:ch5-lipschitz_bound}.
Finally, in \Cref{subsection:ch2-recent_results_on_the_theory_of_neural_networks}, we have presented some recent results on the theory of neural networks.
These results give us insights on how neural networks generalize and a theoretical justification of the regularization scheme that we propose in \Cref{chapter:ch5-lipschitz_bound}.

%% file: sources/main/ch2-background_toeplitz.tex

\subsection{Properties of Circulant Matrices}
\label{subsection:ch2-properties_of_circulant_matrices}

A circulant matrix is a matrix in which each descending diagonal, from left to right, is constant and each row of the matrix is a cyclic right shift of the previous one:
\begin{equation}
  \Cmat =
  \leftmatrix
    c_0 & c_{n-1} & c_{n-2} & \cdots & \cdots & c_{1} \\
    c_{1} & c_0 & c_{n-1} & \ddots & & \vdots \\
    c_{2} & c_{1} & \ddots & \ddots & \ddots & \vdots \\
    \vdots & \ddots & \ddots & \ddots & c_{n-1} & c_{n-2} \\
    \vdots & & \ddots & c_{1} & c_{0} & c_{n-1} \\
    c_{n-1} & \cdots & \cdots & c_{2} & c_{1} & c_0
  \rightmatrix \enspace.
\end{equation}
\noindent
The $n \times n$ circulant matrix $\Cmat$ is fully determined by the sequence of scalars $\{c_h\}_{h \in \Iset^+_n}$ where $\Iset^+_n= \{0, \ldots, n-1\}$.
Furthermore, the $(j,k)$ entry of $\Cmat$ is given by
\begin{equation}
  \leftmat \Cmat \rightmat_{j,k} = c_{\left(k-j\right) \mod n} \enspace.
\end{equation}


\begin{algorithm}[htb]
  \begin{algorithmic}[1]
    \Procedure{CIRCMUL}{$\cvec, \xvec$} \Comment{first column of the circulant matrix $\Cmat$, vector $\xvec$}
      \State $\tilde{\xvec} \gets \textbf{FFT}(\xvec)$
      \State $\tilde{\cvec} \gets \textbf{FFT}(\cvec)$
      \State $\yvec \gets \textbf{IFFT}(\tilde{\xvec} \odot \tilde{\cvec})$ \Comment{element-wise vector-vector product}
      \State \textbf{return} $\yvec$ \Comment{return the result of the product $\Cmat \xvec$}
    \EndProcedure
  \end{algorithmic}
  \caption{Matrix-vector product with a circulant matrix}
  \label{algorithm:ch2-matrix_vector_product_circulant_matrix}
\end{algorithm}

In linear algebra, circulant matrices are important due to their numerous properties.
Indeed, circulant matrices can be compactly represented in memory using only $n$ values instead of $n^2$ values required for arbitrary matrices.
In addition, algorithms exist to speed-up the matrix-vector product operation from $\bigO(n^2)$ to $\bigO(n \log n)$. 
Finally, circulant matrices commute and are closed under the sum and products.
All these properties can be demonstrated with the special diagonalization of circulant matrices with the matrix expansion of the Discrete Fourier Transform (DFT), \ie, Fourier matrix, and an explicit formula of their eigenvalues.
The Fourier matrix is of the form:
\begin{definition}[Fourier Matrix]
  The Fourier matrix of order $n$ is defined as follows:
  \begin{equation} \label{definition:ch2-fourier_matrix}
    \Umat_n = 
    \leftmatrix
      1      & 1         & 1            & \cdots & 1                \\
      1      & z_n       & z_n^2        & \cdots & z_n^{n-1}        \\
      1      & z_n^2     & z_n^4        & \cdots & z_n^{2(n-1)}     \\
      \vdots & \vdots    & \vdots       &        & \vdots           \\
      1      & z_n^{n-1} & z_n^{2(n-1)} & \cdots & z_n^{(n-1)(n-1)}
    \rightmatrix \enspace,
  \end{equation}
  where $z_n = e^{-\frac{2\pi\ci}{n}}$ is an $n^{\text{th}}$ root of unity.
\end{definition}
\noindent
The diagonalization of circulant matrices is given by the following theorem:
\begin{theorem}[\citet{davis1979circulant}] \label{theorem:ch2-diagonalization_circulant_matrix}
  The eigenvalues $\lambda_k$ and the eigenvectors $\yvec^{(k)}$ of a circulant matrix $\Cmat = \circulant(\cvec)$ with $\cvec \in \Rbb^n$ are as follows:
  \begin{equation}
    \lambda_k = \sum_{j \in \Iset^+_n} c_j e^{-\frac{2 \pi \ci}{n} jk} \quad \Leftrightarrow \quad \lambda_k = \left( \Umat_n \cvec \right)_k \enspace,
  \end{equation}
  and
  \begin{equation}
    \yvec^{(k)} = \frac{1}{\sqrt{n}} \leftmatrix 1, e^{-\frac{2 \pi \ci k}{n}}, \dots, e^{-\frac{2 \pi \ci k(n-1)}{n}} \rightmatrix^\top \enspace.
  \end{equation}
  Furthermore, the circulant matrix $\Cmat$ can be expressed in the form 
  \begin{equation} \label{equation:ch2-diagonalization_circulant_matrix}
    \Cmat = \frac{1}{n} \Umat_n^* \diag(\Umat_n \cvec) \Umat_n \enspace.
  \end{equation}
\end{theorem}

\begingroup
\allowdisplaybreaks

\vspace{-2mm}

\noindent
Based on this decomposition, we can state several 
properties of circulant matrices:
\begin{itemize}[leftmargin=13pt]
  \item \textbf{Matrix-vector product}: Let $\xvec \in \Rbb^n$ an arbitrary vector then the product $\Cmat \xvec$ can be expanded as follows:
  \begin{align}
    \Cmat \xvec &= \frac{1}{n} \Umat_n^* \diag(\Umat_n \cvec) \Umat_n \xvec  \\
    &= \frac{1}{n} \Umat_n^* \left( \big(\Umat_n \cvec \big) \odot \big( \Umat_n \xvec \big) \right)
  \end{align}
  where $\odot$ is the element-wise vector multiplication.
  Thus, the matrix-vector product $\Cmat \xvec$ can be reduced to an element-wise multiplication between the characteristic vector $\cvec$ and the vector $\xvec$ in the Fourier domain.
  Furthermore, the multiplication between the Fourier matrix $\Umat_n$ and a vector can be efficiently computed with the \emph{Fast Fourier Transform} (FFT) algorithms~\cite{cooley1965algorithm}.
  \Cref{algorithm:ch2-matrix_vector_product_circulant_matrix} details the steps required to perform the $\bigO(n \log n)$ multiplication between  a circulant matrix and a vector.
\item \textbf{Closeness under sum}: Let $\xvec, \yvec \in \Rbb^n$, $\Xmat= \circulant(\xvec)$ and $\Ymat = \circulant(\yvec)$ then, $\Zmat  = \Xmat + \Ymat$ is also a circulant matrix with $\Zmat = \circulant(\xvec + \yvec)$:
    \begin{align}
      \Xmat + \Ymat &= \left( \frac{1}{n} \Umat_n^* \diag(\Umat_n \xvec) \Umat_n \right) + \left( \frac{1}{n} \Umat_n^* \diag(\Umat_n \yvec) \Umat_n \right) \\
      &= \frac{1}{n}  \Umat_n^* \left( \diag(\Umat_n \xvec) \Umat_n  + \diag(\Umat_n \yvec) \Umat_n \right) \\
      &= \frac{1}{n}  \Umat_n^* \left( \diag(\Umat_n \xvec) + \diag(\Umat_n \yvec) \right) \Umat_n  \\
      &= \frac{1}{n}  \Umat_n^* \left( \diag(\Umat_n (\xvec + \yvec)) \right) \Umat_n  \\
      &= \circulant(\xvec + \yvec)
    \end{align}
  \item \textbf{Closeness under product}: Let $\xvec, \yvec \in \Rbb^n$, $\Xmat= \circulant(\xvec)$ and $\Ymat = \circulant(\yvec)$ then, $\Zmat  = \Xmat \Ymat$ is also a circulant matrix with $\Zmat = \circulant(\xvec \odot \yvec)$:
    \begin{align}
      \Xmat \Ymat &= \left( \frac{1}{n} \Umat_n^* \diag(\Umat_n \xvec) \Umat_n \right) \left( \frac{1}{n} \Umat_n^* \diag(\Umat_n \yvec) \Umat_n \right) \\
      &= \frac{1}{n^2}  \Umat_n^* \diag(\Umat_n \xvec) \Umat_n \Umat_n^* \diag(\Umat_n \yvec) \Umat_n  \\
      &= \frac{1}{n^2}  \Umat_n^* \diag(\Umat_n \xvec) (n \Imat) \diag(\Umat_n \yvec) \Umat_n  \\
      &= \frac{1}{n}  \Umat_n^* \diag(\Umat_n \xvec) \diag(\Umat_n \yvec) \Umat_n  \\
      &= \frac{1}{n}  \Umat_n^* \diag(\Umat_n (\xvec \odot \yvec)) \Umat_n  \\
      &= \circulant(\xvec \odot \yvec)
    \end{align}
\end{itemize}

\endgroup

In this thesis, we also make use of specific type of circulant matrices called \emph{$f$-circulant matrices} which are one of the building blocks of \emph{low displacement rank operators} presented in \Cref{subsection:ch2-general_frameworks_for_structured_matrices} and also enjoy compact representation and fast matrix-vector product.
An $f$-unit-circulant matrix is defined as follows:

\begin{definition}[$f$-circulant matrix] \label{definition:ch2-f_circulant_matrix}
  Given a vector $\xvec$ and a scalar $f$, the $f$-circulant matrix, $\Zmat_f(\xvec)$, is defined as follows:
  \begin{equation}
    \Zmat_f(\xvec) \triangleq
    \leftmatrix
      \xvec_0 & $f$ \xvec_{n-1} & $f$ \xvec_{n-2} & \cdots & \cdots & $f$ \xvec_{1} \\
      \xvec_{1} & \xvec_0 & $f$ \xvec_{n-1} & \ddots & & \vdots \\
      \xvec_{2} & \xvec_{1} & \ddots & \ddots & \ddots & \vdots \\ 
      \vdots & \ddots & \ddots & \ddots & $f$ \xvec_{n-1} & $f$ \xvec_{n-2} \\
      \vdots & & \ddots & \xvec_{1} & \xvec_{0} & $f$ \xvec_{n-1} \\
      \xvec_{n-1} & \cdots & \cdots & \xvec_{2} & \xvec_{1} & \xvec_0
    \rightmatrix \enspace.
  \end{equation}
\end{definition}

\noindent
We denote $\Zmat_f$ the $f$-unit-circulant, defined by the vector $\left(0, 1, \dots, 0 \right)^\top$, a matrix of the form:
\begin{equation}
  \Zmat_f = 
    \leftmatrix
      0      & 0      & 0      & \cdots & \cdots & f      \\
      1      & 0      & 0      & \ddots &        & \vdots \\
      0      & 1      & \ddots & \ddots & \ddots & \vdots \\ 
      \vdots & \ddots & \ddots & \ddots & 0      & 0      \\
      \vdots &        & \ddots & 1      & 0      & 0      \\
      0      & \cdots & \cdots & 0      & 1      & 0
    \rightmatrix \enspace.
\end{equation}

\noindent
The matrix-vector product $\Zmat_f \xvec$ scales the last element by $f$ and makes a circular shift on the components of the vector $\xvec$ by one resulting in $\Zmat_f \xvec = \leftmat f \xvec_{n-1}, \xvec_0, \dots, \xvec_{n-2} \rightmat^\top$.

\subsection{A Fourier Representation of Toeplitz Matrices}
\label{subsection:ch2-a_fourier_representation_of_toeplitz_matrices}

Toeplitz matrices generalize circulant matrices by relaxing the cyclic right shift on the rows.
Therefore, a Toeplitz matrix is a matrix in which each descending diagonal, from left to right, is constant, \ie, a matrix of the form:

\begin{equation}
  \Amat =
  \leftmatrix
    a_{0}   & a_{-1} & a_{2}  & \cdots & \cdots & a_{-n+1} \\
    a_{1}   & a_{0}  & a_{1}  & \ddots &        & \vdots   \\
    a_{2}   & a_{1}  & \ddots & \ddots & \ddots & \vdots   \\
    \vdots  & \ddots & \ddots & \ddots & a_{-1} & a_{-2}   \\
    \vdots  &        & \ddots & a_{1}  & a_{0}  & a_{-1}   \\
    a_{n-1} & \cdots & \cdots & a_{2}  & a_{1}  & a_{0}
  \rightmatrix \enspace.
\end{equation}

\noindent
The $n \times n$ Toeplitz matrix $\Amat$ is fully determined by a two-sided sequence of scalars $\{a_h\}_{h \in \Iset_n}$ where $\Iset_n = \{-n+1, \dots, n-1\}$ and the $(j,k)$ entry of $\Amat$ is given by
\begin{equation}
  \leftmat \Amat \rightmat_{j,k} = a_{k-j} \enspace. \vspace{-2mm}
\end{equation}
\noindent
Similarly to their circulant counterpart, Toeplitz matrices can be represented compactly in memory using only $2n-1$ values instead of $n^2$ values required for arbitrary ones.
Toeplitz matrices have been extensively studied in the context of operator and spectral theory \cite{grenander1958toeplitz,widom1965toeplitz,bottcher2012introduction}.
One important result regarding Toeplitz matrices is Szeg\"{o}'s theorem \cite{szego1915grenzwertsatz} which describes the asymptotic behavior of the determinant of large Toeplitz matrices.
Because Toeplitz matrices do not have a closed-form expression for their eigenvalues, studying their spectrum is not as straightforward as their circulant counterpart.
In order to devise results on Toeplitz matrices, \citet{grenander1958toeplitz} introduced a representation based on the Fourier transform.
Indeed, Toeplitz matrices can be generated from a 2$\pi$-periodic function where the values of the Toeplitz matrix are the Fourier coefficients of this \emph{generating function}.
The spectrum of Toeplitz matrices can be described precisely from the properties of their generating functions.
This representation of Toeplitz matrices has been widely studied in contexts such as signal processing, trigonometric moment problems, integral equations and elliptic partial differential equations with boundary conditions, etc. \cite{serra1997extension,parter1961extreme,avram1988bilinear,widom1965toeplitz,tilli1997asymptotic,tyrtyshnikov1998spectra,tilli1998singular,tilli1997asymptotic}

%
%
%
%

The Fourier representation of Toeplitz matrices can be described as follows.
Let $\{a_h\}_{h \in \Iset_n}$ be the characteristic sequence of the Toeplitz matrix $\Amat \in \Rbb^{n\times n}$.
Then, the trigonometric polynomial $f: \Rbb \rightarrow \Cbb$ of the form
\begin{equation}
  f(\omega) = \sum_{h \in \Iset_n} a_h e^{\ci h \omega} \vspace{-2mm}
\end{equation}
is the \emph{inverse Fourier transform} of the sequence $\{a_h\}_{h \in \Iset_n}$.
From this function, one can recover the sequence $\{a_h\}_{h \in \Iset_n}$ using the standard Fourier transform:
\begin{equation}
  a_h = \frac{1}{2\pi} \int_0^{2\pi} e^{-\ci h \omega} f(\omega) \,\diff \omega \enspace. \vspace{-2mm}
\end{equation}
\noindent
We can, now, define an operator $\Tmat$ mapping integrable functions to Toeplitz matrices:
\begin{equation} \label{equation:ch2-toeplitz_operator}
  \Tmat_n(f) \triangleq \leftmat\frac{1}{2\pi} \int_{0}^{2\pi} e^{-\ci(i-j)\omega}f(\omega) \,\diff \omega \rightmat_{i,j \in \Iset^+_n} \enspace. \vspace{-2mm}
\end{equation}
In the following, when it is clear from context, we will write $\Tmat(f)$ instead of $\Tmat_n(f)$.

\subsection{Block Circulant, Block Toeplitz and the Convolution Operator}
\label{subsection:ch2-block_toeplitz_and_block_circulant_matrices}

\subsubsection{Block Toeplitz and Block Circulant Matrices}
\label{subsubsection:ch2-block_circulant_and_block_toeplitz_matrices}

We can adapt the structure of circulant matrices and their properties to block matrices.
A block circulant matrix is a matrix where each block is repeated identically along diagonals and each row of blocks is a cyclic right shift of the previous one.
Therefore, an $nm \times nm$ block circulant matrix $\Amat$ is fully determined by a sequence of blocks $\{\Amatsf^{(h)}\}_{h \in \Iset^+_n}$ and where each block $\Amatsf^{(h)}$ is an $m \times m$ matrix.
The block circulant matrix $\Amat = \leftmat \Amatsf^{((k-j) \mod n)} \rightmat_{j,k \in \Iset^+_n} $ is given by
\begin{equation}
  \Amat = 
  \leftmatrix
    \Amatsf^{(0)}   & \Amatsf^{(n-1)} & \Amatsf^{(n-2)} & \cdots        & \cdots          & \Amatsf^{(1)}   \\
    \Amatsf^{(1)}   & \Amatsf^{(0)}   & \Amatsf^{(n-1)} & \ddots        &                 & \vdots        \\
    \Amatsf^{(2)}   & \Amatsf^{(1)}   & \ddots          & \ddots        & \ddots          & \vdots        \\ 
    \vdots          & \ddots          & \ddots          & \ddots        & \Amatsf^{(n-1)} & \Amatsf^{(n-2)} \\
    \vdots          &                 & \ddots          & \Amatsf^{(1)} & \Amatsf^{(0)}   & \Amatsf^{(n-1)} \\
    \Amatsf^{(n-1)} & \cdots          & \cdots          & \Amatsf^{(2)} & \Amatsf^{(1)}   & \Amatsf^{(0)}
  \rightmatrix \enspace.
\end{equation}

\noindent
The diagonalization of circulant matrices can be extended to block circulant matrices where the diagonalization is done by blocks and the unit matrix is the Kronecker product of the Fourier matrix with the identity.
The following theorem describes this block diagonalization:
\begin{theorem}[\citet{gutierrez2012block}]
  Let $\Amat$ be an $n^2 \times n^2$ block circulant matrix defined by the sequence of blocks $\{\Amatsf^{(h)}\}_{h \in \Iset^+_n}$, then:
  \begin{equation} \label{equation:ch2-block_diagonalization_block_circulant}
    \Amat = \frac{1}{n} (\Umat_n \otimes \Imat_n)^* \bdiag(\mathsf{\Psi}^{(0)}, \cdots, \mathsf{\Psi}^{(n-1)}) (\Umat_n \otimes \Imat_n) \enspace,
  \end{equation}
  where $\otimes$ is the Kronecker product, $\bdiag$ is the block diagonal operator, $\Umat_n$ is the Fourier matrix of size $n \times n$ and $\mathsf{\Psi}^{(0)}, \dots, \mathsf{\Psi}^{(n-1)}$ are blocks determined as follows:
  \begin{equation}
    \leftmatrix
      \mathsf{\Psi}^{(0)} \\
      \mathsf{\Psi}^{(1)} \\
      \vdots \\
      \mathsf{\Psi}^{(n-1)} \\
    \rightmatrix = 
    (\Umat_n \otimes \Imat_n)
    \leftmatrix
      \Amatsf^{(0)} \\
      \Amatsf{(1)} \\
      \vdots \\
      \Amatsf^{(n-1)} \\
    \rightmatrix \enspace.
  \end{equation}
  \removespace
\end{theorem}
\noindent
One can remark that when the blocks are of size $1 \times 1$, \ie, scalars, this theorem coincides with \Cref{equation:ch2-diagonalization_circulant_matrix} of \Cref{theorem:ch2-diagonalization_circulant_matrix}.
Although interesting, this representation does not provide a closed form expression of the eigenvalues of the block circulant matrix.
However, in the special case where the blocks are also circulant matrices -- the matrix is called a doubly-block circulant matrix -- then we can extend the diagonalization and get a closed form of the eigenvalues of doubly-block circulant matrices.
First, we can remark that if the blocks $\Amatsf^{(0)}, \dots, \Amatsf^{(n-1)}$ are circulant matrices then the blocks $\mathsf{\Psi}^{(0)}, \dots, \mathsf{\Psi}^{(n-1)}$ are also circulant matrices because they are linear combinations of circulant matrices which are closed under the sum and products.
Therefore, using \Cref{theorem:ch2-diagonalization_circulant_matrix} for each block of the block diagonal independently, we have:
\begin{equation} \label{equation:ch2-diagonalization_circulant_with_bdiag}
  \bdiag\leftmat \mathsf{\Psi}^{(0)}, \cdots, \mathsf{\Psi}^{(n-1)} \rightmat = \frac{1}{n} (\Imat \otimes \Umat_n)^* \boldsymbol{\Lambda} (\Imat \otimes \Umat_n) \enspace,
\end{equation}
where $\boldsymbol{\Lambda} = \diag\left((\Umat_n \mathsf{\psi}^{(0)}, \dots, \Umat_n \mathsf{\psi}^{(n-1)})\right)$ and the vectors $\mathsf{\psi}^{(0)}, \dots, \mathsf{\psi}^{(n-1)}$ are the characteristic vectors of the circulant matrices $\mathsf{\Psi}^{(0)}, \dots, \mathsf{\Psi}^{(n-1)}$ respectively.
By combining \Cref{equation:ch2-block_diagonalization_block_circulant} and \Cref{equation:ch2-diagonalization_circulant_with_bdiag}, we obtain the eigenvalues decomposition of a doubly-block circulant matrix.
Given a doubly-block circulant matrix $\Amat$, we have:
\begin{equation} \label{equation:ch2-diagonalization_doubly_block_circulant_matrix}
  \Amat = \frac{1}{n^2} (\Umat_n \otimes \Umat_n)^* \boldsymbol{\Lambda} (\Umat_n \otimes \Umat_n)  \enspace.
\end{equation}
This decomposition makes it possible to express the eigenvalues of a doubly-block circulant matrix with the characteristic vectors of the circulant matrices composing it.
Furthermore, one can note that the eigenvectors are independent of the values of the matrix and can be expressed with the Fourier matrix.

Akin to circulant and block circulant matrices, we can extend the block structure to Toeplitz matrices.
An $nm \times nm$ block Toeplitz matrix $\Bmat$ is fully determined by a two-sided sequence of blocks $\{\Bmatsf^{(j)}\}_{h \in \Iset_n}$ and where each block $\Bmatsf^{(h)}$ is an $m \times m$ matrix.
The block Toeplitz matrix $\Bmat = \leftmat\Bmatsf^{(k-j)} \rightmat_{j,k \in \Iset^+_n}$ is given by
\begin{equation} \label{equation:ch2-block_toeplitz_matrices}
  \Bmat = 
  \leftmatrix
    \Bmatsf^{(0)}   & \Bmatsf^{(-1)} & \Bmatsf^{(-2)} & \cdots         & \cdots         & \Bmatsf^{(-n+1)} \\
    \Bmatsf^{(1)}   & \Bmatsf^{(0)}  & \Bmatsf^{(-1)} & \ddots         &                & \vdots           \\
    \Bmatsf^{(2)}   & \Bmatsf^{(1)}  & \ddots         & \ddots         & \ddots         & \vdots           \\ 
    \vdots          & \ddots         & \ddots         & \ddots         & \Bmatsf^{(-1)} & \Bmatsf^{(-2)}   \\
    \vdots          &                & \ddots         & \Bmatsf^{(1)} & \Bmatsf^{(0)}   & \Bmatsf^{(-1)}   \\
    \Bmatsf^{(n-1)} & \cdots         & \cdots         & \Bmatsf^{(2)} & \Bmatsf^{(1)}   & \Bmatsf^{(0)}
  \rightmatrix \enspace.
\end{equation}

\noindent
Block Toeplitz and doubly-block Toeplitz matrices (block Toeplitz matrix where the blocks are also Toeplitz) do not have a block diagonalization nor a closed-form expression for their eigenvalues.
However, the Toeplitz operator defined in~\Cref{equation:ch2-toeplitz_operator} can be extended to block Toeplitz and doubly-block Toeplitz matrices. 
For block Toeplitz matrices, the trigonometric polynomial that \emph{generates} the block Toeplitz matrix $\Bmat$ can be defined as follows:
\begin{equation}
  F_{\Bmat}(\omega) \triangleq \sum_{h \in \Iset_n} \Bmatsf^{(h)} e^{\ci h \omega} \enspace.
\end{equation}
The function $F_{\Bmat}$ is said to be the \emph{generating function} of the block matrix $\Bmat$.
To recover the block Toeplitz matrix from its generating function, we use the Toeplitz operator defined in \Cref{equation:ch2-toeplitz_operator}; therefore by construction, we have $\Tmat_n(F_\Bmat) = \Bmat$.



\subsubsection{Relation with the Convolution Operator}
\label{subsubsection:ch2-relation_with_the_convolution_operator}

\begin{figure}[ht]
  \centering
  \includegraphics[width=0.23\textwidth]{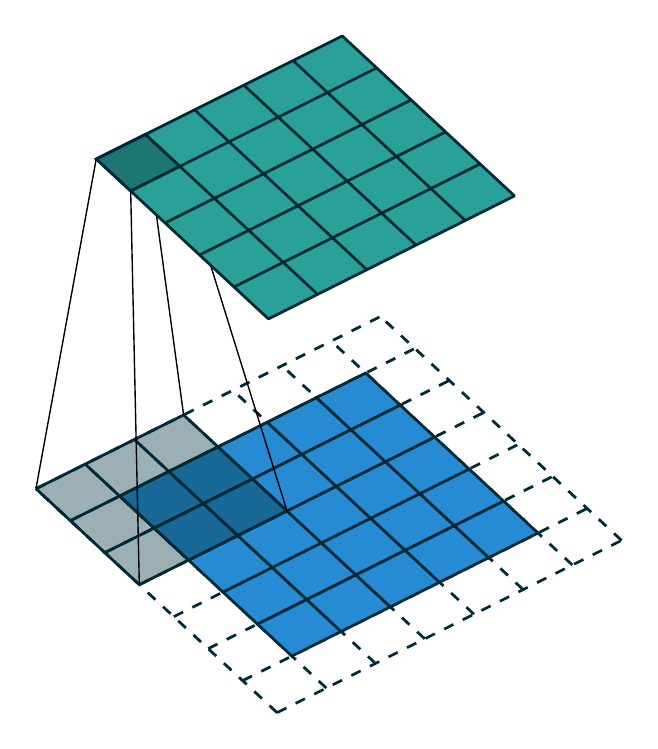}
  \hfill
  \includegraphics[width=0.23\textwidth]{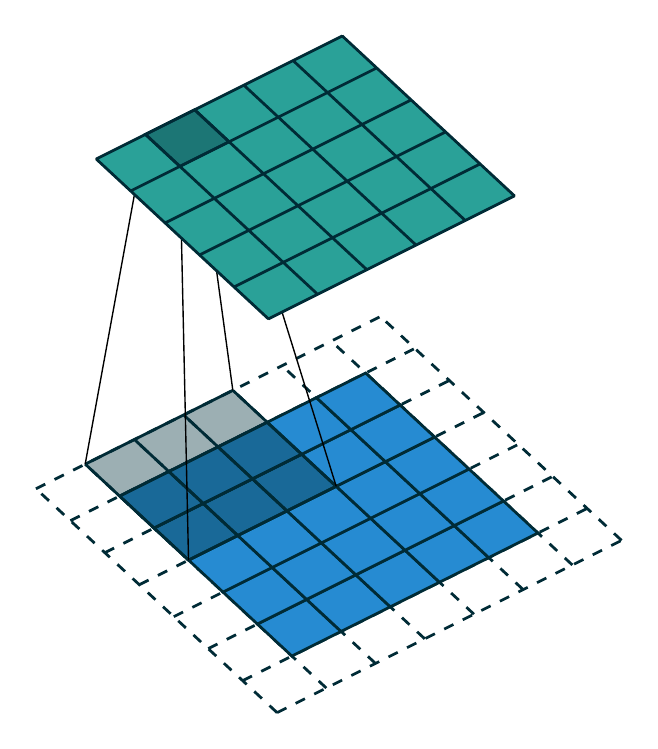}
  \hfill
  \includegraphics[width=0.23\textwidth]{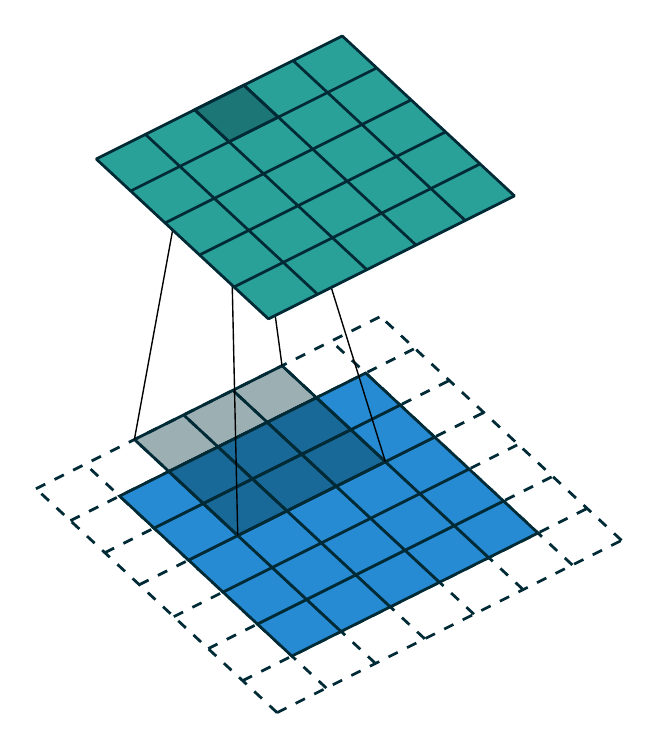}
  \hfill
  \includegraphics[width=0.23\textwidth]{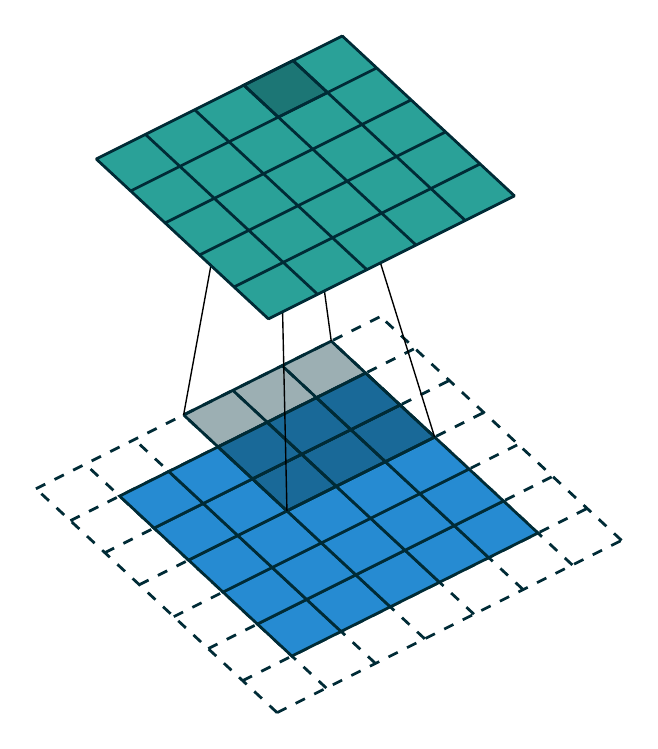}
  \caption{A convolution: a kernel sliding over an image and acting as a filter. \\Illustration taken from~\citet{dumoulin2016guide}.}
  \label{figure:illustration_convolution}
\end{figure}

%
%
%

Doubly-block circulant and doubly-block Toeplitz matrices are interesting structure due to their relation to 2-dimensional convolutions.
We recall that a discrete convolution can be seen as a kernel sliding over the image and acting as a filter.
\Cref{figure:illustration_convolution} illustrates the convolution operation with the image (blue), the kernel (gray) and the resulting operation (green).
It has been shown by~\citet{jain1989fundamentals} that the result of a discrete 2d-convolution can be obtained by applying a doubly-block Toeplitz matrix to the input signal.
A doubly-block Toeplitz matrix is a block Toeplitz matrix where the blocks are also Toeplitz.
To illustrate, we consider a discrete convolution between a 2-dimensional signal $\xvec$ and a kernel $\Kmat$ where the kernel is defined as follows:
\begin{equation*}
  \Kmat = \leftmatrix
    k_0 & k_1 & k_2 \\
    k_3 & k_4 & k_5 \\
    k_6 & k_7 & k_8 
  \rightmatrix
\end{equation*}
then, the doubly-block Toeplitz matrix $\Mmat$ that performs the convolution can be represented as:
\begin{equation*}
  \Mmat = \leftmatrix
    \Tmatsf^{(0)} & \Tmatsf^{(1)} &             &  0          \\
    \Tmatsf^{(2)} & \Tmatsf^{(0)} & \ddots      &             \\
                  & \ddots        & \Tmatsf^{(0)} & \Tmatsf^{(1)} \\
    0             &               & \Tmatsf^{(2)} & \Tmatsf^{(0)}
  \rightmatrix \enspace.
\end{equation*}
where $\Tmatsf^{(j)}$ are Toeplitz matrices and the values of the kernel $\Kmat$ are distributed in the Toeplitz blocks as follows:

\begin{equation*}
  \Tmatsf^{(0)} = \leftmatrix
    k_4 & k_3 &         &         & 0       \\
    k_5 & k_4 & k_3 &         &         \\
            & k_5 & \ddots  & \ddots  &         \\
            &         & \ddots  & k_4 & k_3 \\
    0       &         &         & k_5 & k_4
  \rightmatrix \quad \quad
  \hfill
  \Tmatsf^{(1)} = \leftmatrix
    k_7 & k_6 &         &         & 0       \\
    k_8 & k_7 & k_6 &         &         \\
            & k_8 & \ddots  & \ddots  &         \\
            &         & \ddots  & k_7 & k_6 \\
    0       &         &         & k_8 & k_7 \\
    \rightmatrix \vspace{5mm}
\end{equation*}

\begin{equation*}
  \Tmatsf^{(2)} = \leftmatrix
    k_1 & k_0 &         &         & 0       \\
    k_2 & k_1 & k_0 &         &         \\
            & k_2 & \ddots  & \ddots  &         \\
            &         & \ddots  & k_1 & k_0 \\
    0       &         &         & k_2 & k_1 \\
  \rightmatrix \vspace{5mm}
\end{equation*}

In practice, the signal often can have multiple channels (\eg, images have 3 channels corresponding to the colors red, green and blue).
Let us denote $\cin$ and $\cout$ the number of the input and output channels respectively.
Then the input signal has a dimension of $\cin \times n \times n$, performed by a kernel of size $\cout \times \cin \times k \times k$, and outputs a signal of size $\cout \times m \times m$ with $m = n - k + 2p + 1$ where $p$ corresponds to the padding.
The matrix for the multi-channel convolution is the concatenation of $\cout \cdot \cin$ doubly-block Toeplitz matrices.

\subsection{LDR: General Framework for Structured Matrices}
\label{subsection:ch2-general_frameworks_for_structured_matrices}

In this thesis, we mainly focus on structured transforms from the Toeplitz family.
The properties of matrices from the Toeplitz family presented in the previous section make them good candidates for applications in the context of signal processing and deep neural network.
However, other structured matrices with other properties have been considered.
In this section, we briefly present these families of structured matrices and introduce LDR, a more general framework to capture all structured matrices.
Below a description of some known structured matrices: 
\begin{itemize}
  \item \textbf{Hankel matrix}: A Hankel matrix has constant values along each of its anti-diagonals.
  \item \textbf{Vandermonde matrix}: A Vandermonde matrix is a matrix where each term follows a geometric progression.
    A very important special case is the complex matrix associated with the Discrete Fourier transform (DFT) presented in \Cref{definition:ch2-fourier_matrix} which has a Vandermonde structure.
  \item \textbf{Cauchy matrix}: A Cauchy matrix is an $m \times n$ matrix with elements $a_{ij}$ such that $a_{ij} = (\uvec_i - \vvec_j)^{-1}$ with $\uvec_i - \vvec_j \neq 0$, $i \in \{0,\dots,m-1\}$ and $j \in \{0,\dots,n-1\}$.
\end{itemize}
\Cref{figure:ch2-example_structure_matrices} shows the representation of the parameters sharing of Hankel, Vandermonde and Cauchy matrices.
The study of these matrices with those from the Toeplitz family can be unified thought to the concept of \emph{Low Displacement Rank} (LDR) initially proposed by~\citet{kailath1979displacement}.
Although these matrices appear to have very different kinds of structure, they can be all associated with a specific displacement operator $\triangleopdown_{\Amat, \Bmat}: \Rbb^{m \times n} \rightarrow \Rbb^{m \times n}$ which takes a matrix, $\Mmat$, and outputs a low rank matrix $\triangleopdown_{\Amat, \Bmat}(\Mmat)$ such that $\rank(\triangleopdown_{\Amat, \Bmat}(\Mmat)) \ll \min(m,n)$.

\begin{figure}[t]
   \centering
   \begin{subfigure}[b]{0.32\textwidth}
       \centering
       \begin{equation*}
	  \leftmatrix
	     \hvec_{n-1} & \cdots     & \hvec_1 & \hvec_0      \\
	     \vdots      & \ddots     & \hvec_0 & \hvec_{-1}   \\
	     \hvec_1     & \ddots     & \ddots  & \vdots       \\
	     \hvec_0     & \hvec_{-1} & \cdots  & \hvec_{-n+1} \\
	  \rightmatrix
       \end{equation*}
       \caption*{Hankel}
   \end{subfigure}
   \hfill
   \begin{subfigure}[b]{0.32\textwidth}
       \centering
       \begin{equation*}
	  \leftmatrix
	    1 & \vvec_0     & \cdots & \vvec_0^{n-1} \\
	    1 & \vvec_1     & \cdots & \vvec_1^{n-1} \\
	    1 & \vdots      &        & \vdots        \\
	    1 & \vvec_{n-1} & \cdots & \vvec_{n-1}^{n-1}
	  \rightmatrix
       \end{equation*}
       \caption*{Vandermonde}
   \end{subfigure}
   \hfill
   \begin{subfigure}[b]{0.32\textwidth}
       \centering
       \begin{equation*}
	  \leftmatrix
	  \frac{1}{\uvec_0 - \vvec_{0}}     & \cdots & \frac{1}{\uvec_0 - \vvec_{n-1}} \\
	  \frac{1}{\uvec_1 - \vvec_{0}}     & \cdots & \frac{1}{\uvec_1 - \vvec_{n-1}} \\
	  \vdots                            & \cdots & \vdots                          \\
	  \frac{1}{\uvec_{n-1} - \vvec_{0}} & \cdots & \frac{1}{\uvec_{n-1} - \vvec_{n-1}}
	  \rightmatrix
       \end{equation*}
       \caption*{Cauchy}
   \end{subfigure}
   \caption{Representation of Hankel, Vandermonde and Cauchy matrices}
  \label{figure:ch2-example_structure_matrices}
\end{figure}

More formally, two displacement operators can be defined as follows (for simplification, we consider $m = n$):
\begin{definition}[\emph{Sylvester} \& \emph{Stein} displacement operators]
  Let $\Amat, \Bmat \in \Rbb^{n \times n}$, the \emph{Sylvester} displacement operator $\triangleopdown_{\Amat, \Bmat}: \Rbb^{n \times n} \rightarrow \Rbb^{n \times n}$ is defined as follows:
  \begin{equation}
    \triangleopdown_{\Amat, \Bmat} (\Mmat) \triangleq \Amat \Mmat - \Mmat \Bmat
  \end{equation}
  The \emph{Stein} displacement operator $\triangleopup_{\Amat, \Bmat}: \Rbb^{n \times n} \rightarrow \Rbb^{n \times n}$ is defined as follows:
  \begin{equation}
    \triangleopup_{\Amat, \Bmat} (\Mmat) \triangleq \Mmat - \Amat \Mmat \Bmat
  \end{equation}
  where $\triangleopdown_{\Amat, \Bmat} (\Mmat) = \Amat \triangleopup_{\Amat^{-1}, \Bmat}(\Mmat)$ if the operator matrix $\Amat$ is non-singular, and $\triangleopdown_{\Amat, \Bmat}(\Mmat) = -\triangleopup_{\Amat, \Bmat^{-1}}(\Mmat) \Bmat$ if the operator matrix $\Bmat$ is non-singular.
\end{definition}

\begin{table}[t]
  \centering
  \begin{tabular}{c|c|c|c}
    \toprule
    \multicolumn{2}{c|}{\textbf{Operator Matrices}} & \textbf{Class of structured} & \textbf{Rank of } \\
    \textbf{A} & \textbf{B} & \textbf{matrices M} & $\triangleopdown_{\Amat, \Bmat}(\Mmat)$ \\
    \midrule
    $\Zmat_1$                & $\Zmat_{-1}$             & Toeplitz               & $\leq 2$ \\
    $\Zmat_1$                & $\Zmat_0^\top$           & Hankel                 & $\leq 2$ \\
    $\Zmat_0 + \Zmat_0^\top$ & $\Zmat_0 + \Zmat_0^\top$ & Toeplitz + Hankel      & $\leq 4$ \\
    $\diag(\vvec)$           & $\Zmat_0$                & Vandermonde            & $\leq 1$ \\
    $\Zmat_0$                & $\diag(\vvec)$           & Inverse of Vandermonde & $\leq 1$ \\
    $\diag(\uvec)$           & $\diag(\vvec)$           & Cauchy                 & $\leq 1$ \\
    $\diag(\vvec)$           & $\diag(\uvec)$           & Inverse of Cauchy      & $\leq 1$ \\
    \bottomrule
  \end{tabular}
  \caption{Displacing Matrices Associated with Families of Structured Matrices.}
  \label{table:ch2-displacing_matrices}
\end{table}

\noindent
Based on this definition, if $\Mmat$ is a structured matrix, there exist operator matrices $\Amat$ and $\Bmat$ such that $\triangleopdown_{\Amat, \Bmat} (\Mmat)$ is low rank.
In particular, $\Amat$ and $\Bmat$ can be chosen to be diagonal or $f$-unit-circulant matrices (see \Cref{definition:ch2-f_circulant_matrix}) for several classes of structured matrices.
\Cref{table:ch2-displacing_matrices} shows some specific choices of operators for the four classes of structured matrices presented above as well as other types of structured matrices from the same family.
We now define the matrices that can be considered structured with respect to the Sylvester or Stein operator.
\begin{definition}[L-like matrices \citet{pan2001structured}]
  For an $n \times n$ matrix $\Mmat$ and an associated operator $\triangleopdown_{\Amat, \Bmat}$ (or $\triangleopup_{\Amat, \Bmat}$), the value $r = \rank(\triangleopdown_{\Amat, \Bmat}(\Mmat))$ (or $r = \rank(\triangleopup_{\Amat, \Bmat}(\Mmat))$) is called the \emph{displacement rank}.
    If the value of $r$ is small relative to $n$ as $n$ grows large, then we call the matrix $\Mmat$ \emph{L-like} having a structure of type $L$.
    For example, in the case where the operator $\triangleopdown_{\Amat, \Bmat}$ is associated with Toeplitz matrices (\ie, $\Amat = \Zmat_1$ and $\Bmat = \Zmat_{-1}$, see~\Cref{table:ch2-displacing_matrices}), we call the matrix $\Mmat$, \emph{Toeplitz-Like}.
  \label{definition:ch2-l_like_matrices}
\end{definition}

An important result allows us to express structured matrices with low-displacement rank directly as a function of their low displacement generators.
This result can then be used to decompose structured matrices and define efficient algorithms for matrix-vector products.

\begin{theorem}[Krylov Decomposition~\citet{pan2003inversion,sindhwani2015structured}] \label{theorem:ch2-krylov_decomposition}
  If an $n \times n$ matrix $\Mmat$ is such that $\triangleopup_{\Amat, \Bmat}(\Mmat) = \Gmat \Hmat^\top$ where 
  $\Gmat = (\gvec^{(1)} \ldots \gvec^{(r)}), \Hmat = (\hvec^{(1)} \ldots \hvec^{(r)}) \in \Rbb^{n \times r}$ 
  and the operator matrices satisfy: $\Amat^n = a \Imat$, $\Bmat^n = b \Imat$ for some scalars $a, b$, then $\Mmat$ can be expressed as: 
  \begin{equation} \label{equation:ch2-krylov_decomposition}
    \Mmat = \frac{1}{1 - ab} \sum_{j=1}^{r} ~\krylov(\Amat, \gvec^{(j)}) ~\krylov(\Bmat^\top, \hvec^{(j)})^\top
  \end{equation}
  where $\krylov(\Amat, \vvec)$ is defined by:
  \begin{equation}
    \krylov(\Amat, \vvec) = [\vvec~~\Amat\vvec~~\Amat^2 \vvec~~\ldots~~\Amat^{n-1} \vvec]
  \end{equation}
  \removespace
\end{theorem}

\noindent
In the case of \emph{Toeplitz-like matrices}, the above theorem can be simplified as follows:
\begin{theorem}[Toeplitz-like matrix decomposition \citet{pan2001structured}] \label{theorem:ch2-toeplitz_like}
  If an $n \times n$ matrix $\Mmat$ satisfies $\triangleopdown_{\Zmat_1, \Zmat_{-1}}(\Mmat) = \Gmat \Hmat^\top$ $(\Mmat \text{ is Toeplitz-like})$ where $\Gmat = (\gvec ^{(1)} \ldots \gvec^{(r)}), \Hmat = (\hvec^{(1)} \ldots \hvec^{(r)}) \in \Rbb^{n \times r}$, then $\Mmat$ can be written as: 
  \begin{equation} \label{equation:ch2-toeplitz_like_matrix_decomposition}
    \Mmat = \frac{1}{2} \sum_{j=1}^{r} \Zmat_1(\gvec^{(j)}) \Zmat_{-1}(\Jmat_n \hvec^{(j)})
  \end{equation}
  where $\Jmat_n$ is the reflection of the $n \times n$ identity matrix and $\Zmat_1$ and $\Zmat_{-1}$ are $f$-unit-circulant matrices (see \Cref{definition:ch2-f_circulant_matrix}).
\end{theorem}
\noindent
In the next chapter, we will see how this general framework have been used in the context of compact neural networks.

%% file: sources/main/ch2-background_neural_networks.tex
\subsection{Introduction to Supervised Learning}
\label{subsection:ch2-introduction_on_supervised_learning}

Supervised learning consists in learning a function that maps an input to an output based on input-output pairs.
For example, one could learn to ``predict'' if a fruit will be tasty based on its features (\eg, size, weight, color, consistency, etc.).
These features are used as inputs to the function and the function outputs a value characterizing the taste of the fruit.

In the following, we will formalize the learning problem described above with the \emph{statistical learning framework}.
First, let us define the domain space $\Xset$ which corresponds to the set of inputs that we wish to label.
Let us denote the label space $\Yset$ and a finite sequence of pairs $\Sset = \left\{ \left(\xvec^{(1)}, y^{(1)} \right) \dots \left( \xvec^{(m)}, y^{(m)} \right) \right\}$ in $\Xset \times \Yset$.
Such pairs \ie, labeled examples, are called \emph{training examples} and the set $\Sset$ is called the \emph{training set}.
We denote $\Dset$ the \emph{joint distribution} over $\Xset \times \Yset$.
The main objective of the task at hand is to output a function $h: \Xset \rightarrow \Yset$ that maps the input $\xvec \in \Xset$ to the output $y \in \Yset$.
This function is called the \emph{hypothesis} or the \emph{classifier}.
Given the probability distribution $\Dset$, we aim to measure how \emph{likely} the hypothesis $h$ makes an error when labeled points are randomly drawn from the distribution $\Dset$.
Let us define the true error or \emph{risk} of the hypothesis $h$ that we wish to minimize:
\begin{equation} \label{equation:ch2-risk}
  R_{\Dset}(h) \triangleq \Ebb_{(\xvec, y) \sim \Dset} \left[ L\big( h(\xvec), y \big) \right] \enspace.
\end{equation}
where $L: \Yset \times \Yset \rightarrow \Rbb_{+}$ is a \emph{loss function} which measures the correctness of the hypothesis.
For example, for classification problems, we can use the 0-1 loss defined as: 
\begin{equation}
  L(h(\xvec), y) \triangleq \mathds{1}_{\big[ h(\xvec) \neq y \big]}
\end{equation}

However, in practice, the joint probability distribution $\Dset$ is unknown; therefore, the true error is not directly available to the learner.
The learner only has access to the training data, $\Sset$, and can calculate the \emph{empirical error}, \ie, the error over the training samples.
We define the \emph{empirical risk} as follows:
\begin{equation} \label{equation:ch2-empirical_risk}
  R_{\Sset}(h) \triangleq \frac{1}{|\Sset|} \sum_{(\xvec, y) \in \Sset} L\big( h(\xvec), y \big) \enspace.
\end{equation}
The learning paradigm which consists in minimizing this value is called \emph{Empirical Risk Minimization} denoted ERM.

We use the ERM paradigm as a surrogate to find a hypothesis $h$ that minimizes the true risk $R_\Dset$.
However, all hypotheses that minimize the empirical error do not necessarily minimize the true risk.
For example, consider the following function:
\begin{equation} \label{equation:ch2-perfect_function}
  h_c(\xvec) =
  \begin{cases}
    y^{(i)} &\quad \text{if }\exists i \in [m] \text{ s.t. } \xvec^{(i)} = \xvec \\
    c &\quad \text{otherwise}
  \end{cases}
\end{equation}
Clearly, this function, for any training, $\Sset$ and a binary target, will have $R_\Sset(h_1) = R_\Sset(h_2) = 0$, whereas one the two functions will have a true risk $\geq \frac{1}{2}$ (under the reasonable assumption that $\Sset$ is a negligible set with respect to $\Dset$).
The phenomenon, called \emph{overfitting}, happens when the classifier fits the training data ``too well'' but will likely have a high error on unseen data.
One possible solution to this phenomenon is to apply ERM with a restricted search space to prevent the learning algorithm to output a function such as $h_c$ in \Cref{equation:ch2-perfect_function}.
We call this set the \emph{hypothesis class} and is denoted $\Hset$.
Each $h \in \Hset$ is a function mapping from $\Xset$ to $\Yset$.
We call $\mathrm{ERM}_{\Hset}$, the set of learned hypotheses that uses the $\mathrm{ERM}$ paradigm over the hypothesis class $\Hset$ and a training data $\Sset$.
Formally,
\begin{equation}
  \mathrm{ERM}_{\Hset}(\Sset) = \argmin_{h \in \Hset} R_{\Sset}(h) \enspace.
\end{equation}

\begin{figure}[t]
  \centering
  \begin{subfigure}[b]{0.32\textwidth}
    \includegraphics[width=0.98\textwidth]{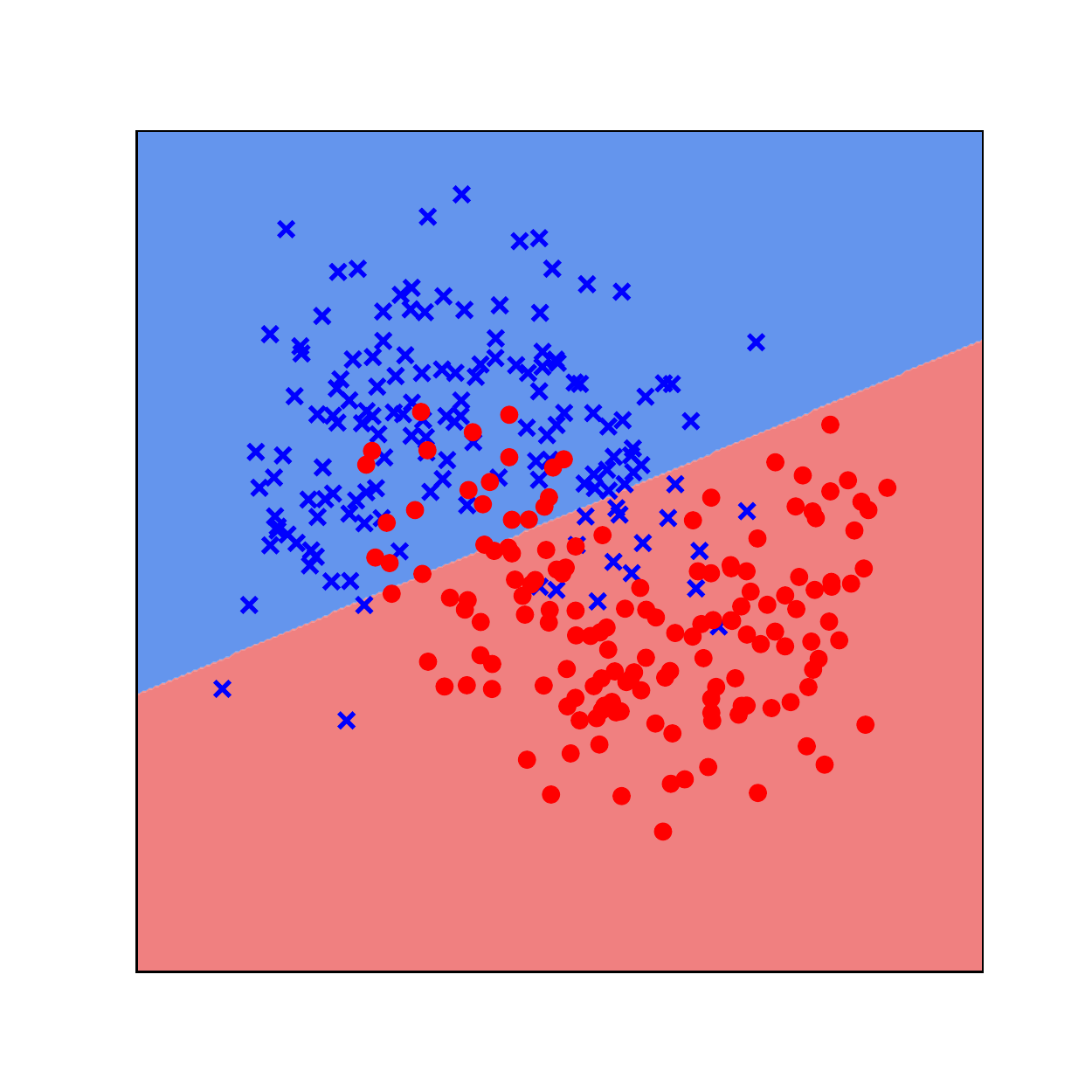}
    \caption{Underfitting}
    \label{figure:ch2-fitting_points_a}
  \end{subfigure}
  \hfill
  \begin{subfigure}[b]{0.32\textwidth}
    \includegraphics[width=0.98\textwidth]{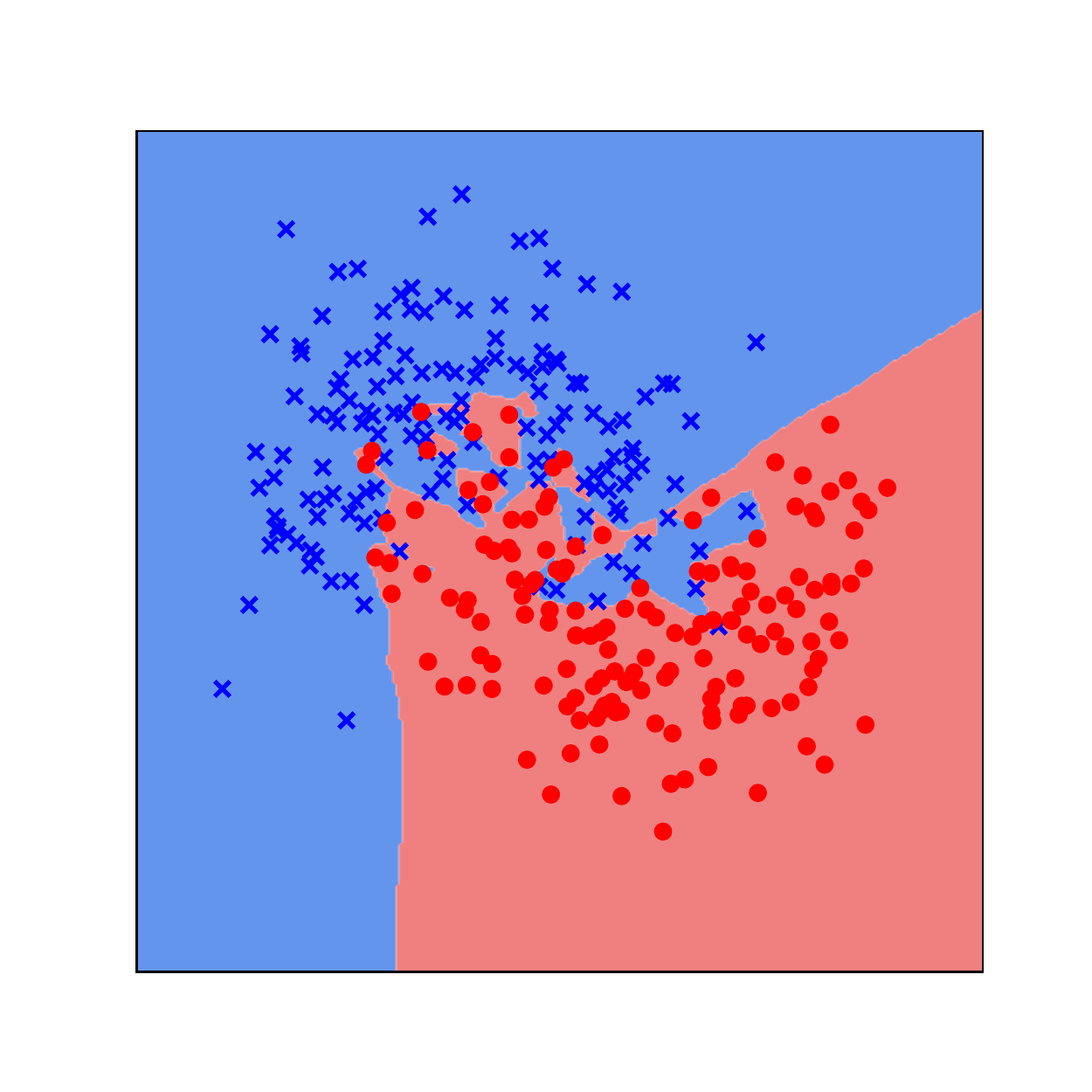}
    \caption{Overfitting}
    \label{figure:ch2-fitting_points_b}
  \end{subfigure}
  \hfill
  \begin{subfigure}[b]{0.32\textwidth}
    \includegraphics[width=0.98\textwidth]{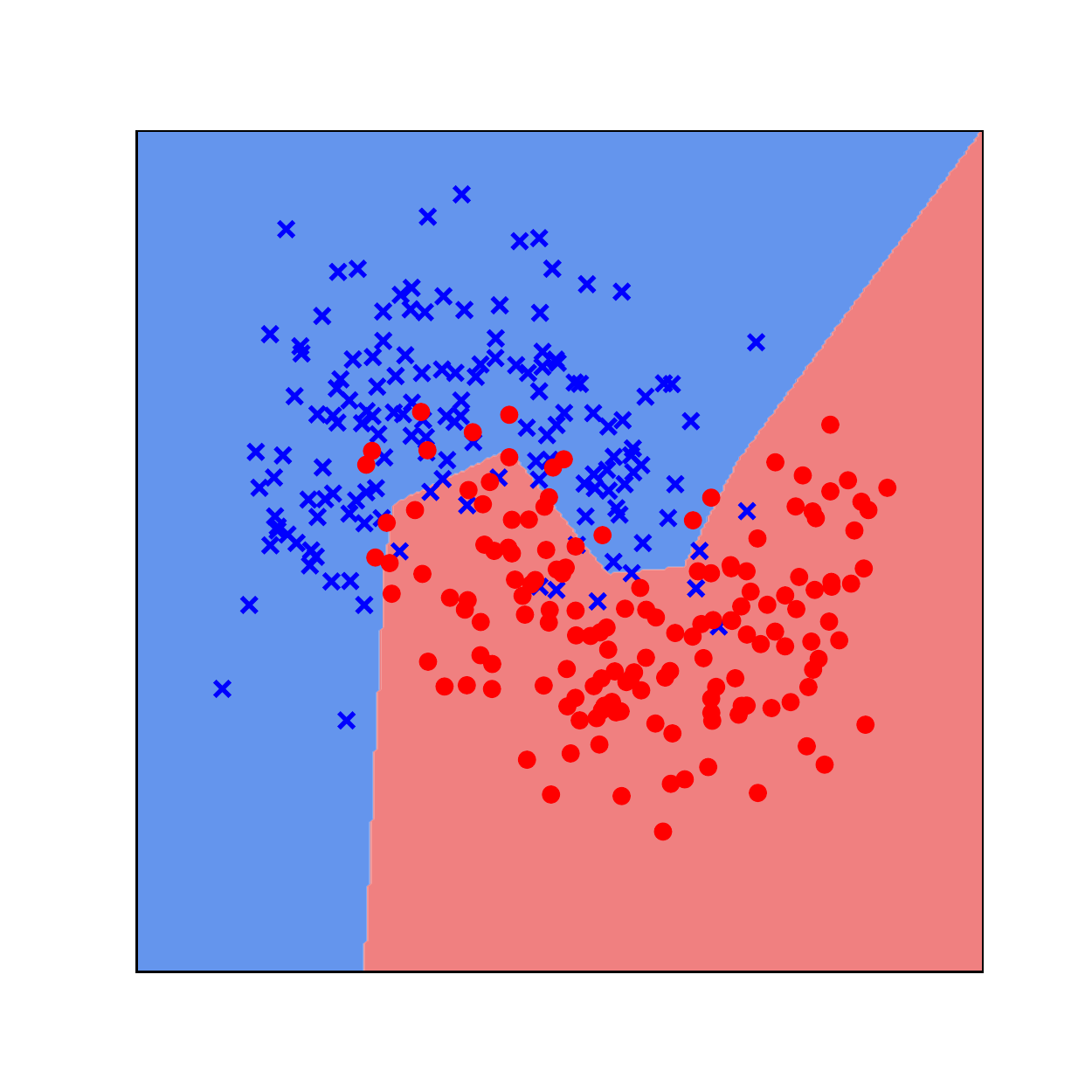}
    \caption{Good fit}
    \label{figure:ch2-fitting_points_c}
  \end{subfigure}
  \caption{
    Decision boundary of three classifiers with different complexity for the same set of samples.
  }
  \label{figure:ch2-fitting_points}
\end{figure}

For a training sample $\Sset$, we denote $h_\Sset \in \mathrm{ERM}_{\Hset}(\Sset)$, one solution of applying $\mathrm{ERM}_\Hset$ on the set $\Sset$, if there exist multiple hypotheses with minimal error on the training sample, then the minimization problem returns an arbitrary one.
In practice, the hypothesis class is chosen on the basis of an assumption about the relationship between the data and its label.
For example, if the relation between the data and its label is supposedly linear, then the hypothesis class can be the set of all linear functions.
This kind of restriction is called the \emph{inductive bias} because the learner is \emph{biased} towards a particular set of predictors.

The ERM paradigm assumes that a hypothesis $h_\Sset$ that minimizes the risk $R_\Sset$ will also minimize the true risk $R_\Dset$.
To verify that this assumption is correct, we need to ensure that all hypotheses in the hypothesis class $\Hset$ are good approximators of their true risk.
We say that a hypothesis class has the \emph{uniform convergence property} if there exists a function $m_\Hset:(0, 1)^2 \rightarrow \Nbb$ such that for every $\epsilon, \delta \in (0, 1)$, if $\Sset$ is a sample of size $m \geq m_\Hset(\epsilon, \delta)$ examples drawn \emph{independently and identically} according to $\Dset$, then, with probability of at least $1 - \delta$:
\begin{equation} \label{equation:ch2-eps_representative}
  \forall h \in \Hset, \quad |R_\Dset(h) - R_\Sset(h)| \leq \epsilon \enspace,
\end{equation}
where the function $m_\Hset$, called \emph{the sample complexity}, measures the minimal number of examples needed to ensure that with probability of at least $1 - \delta$, \Cref{equation:ch2-eps_representative} holds.
The i.i.d. assumption is common in statistical learning theory.
It is easy to see that the error between $R_\Sset(h)$ and $R_\Dset(h)$ is dependent on the representativeness of the sample $\Sset$ with respect to the underlying distribution $\Dset$.
Therefore, the parameter $\delta$ characterizes the probability of having a nonrepresentative sample.
The quantity $1 - \delta$ is the \emph{confidence parameter} of the prediction.


The following result gives the sample complexity measure the maximum value for which the hypothesis class has the \emph{uniform convergence property}.
\begin{theorem}[\citet{shalev2014understanding}] \label{theorem:ch2-sample_complexity_bound}
  Let $\Hset$ be a \emph{finite} hypothesis class, then $\Hset$ enjoys the uniform convergence property with sample complexity:
  \begin{equation}
    m_\Hset(\epsilon, \delta) \leq \left\lceil \frac{\log(2 |\Hset|	/ \delta)}{2 \epsilon^2} \right\rceil \enspace,
  \end{equation}
  \removespace
\end{theorem}
\noindent
where $|\Hset|$ is the cardinal of the set $\Hset$.
Note that we can also write the bound from~\Cref{theorem:ch2-sample_complexity_bound} as follows: for all $h \in \Hset$, we have
\begin{equation}
  \underbrace{\vphantom{\sqrt{\frac{\log(|\Hset|/\delta)}{2 m_\Hset(\epsilon, \delta)}}} R_\Dset(h)}_{\text{Error}} \leq \underbrace{\vphantom{\sqrt{\frac{\log(|\Hset|/\delta)}{2 m_\Hset(\epsilon, \delta)}}} R_\Sset(h)}_{\text{Estimation Error}} + \underbrace{\sqrt{\frac{\log(|\Hset|/\delta)}{2 m_\Hset(\epsilon, \delta)}}}_{\text{Complexity penality}} \enspace,
\end{equation}
with probability $1 - \delta$.
This bound is called a \emph{generalization bound} and consists in bounding the true error by the empirical error and a complexity penalty.
In the above theorem, the condition on the finiteness of the hypothesis class might by too strong.
For example, if we consider the set of linear functions parameterized by a set of real-valued parameters the hypothesis class is infinite and the theorem above does not apply.
To characterize the learnability of infinite hypothesis classes, several complexity measures have been proposed.
One of the first, discovered by~\citet{vapnik2015uniform}, relies on a combinatorial notion called the Vapnik-Chervonenkis dimension (VC-dimension). They showed that having a finite VC-dimension is a necessary and sufficient condition for the uniform convergence property.
In the same vein, the Rademacher complexity~\cite{koltchinskii2000rademacher}, measures the richness of the class of real-valued functions with respect to a probability distribution.
In \Cref{subsection:ch2-recent_results_on_the_theory_of_neural_networks}, we will study recent generalization bounds specific to neural networks where the complexity penalty is dependent on the Lipschitz constant of the weights matrices.

A fundamental question of the ERM paradigm remains: \emph{how to choose the correct hypothesis class for which $\text{ERM}_\Hset$ will not lead to overfitting?} 
We answer this question by decomposing the true risk into two different components as follows: 
\begin{equation} \label{equation:ch2-bias_complexity_tradeoff}
  R_\Dset (h_\Sset) = 
  \underbrace{\left[ \min_{h \in \Hset} R_\Dset(h) \right]}_{\text{\scriptsize Approximation Error}} + \quad 
  \underbrace{\left[ R_\Dset(h_\Sset) - \min_{h \in \Hset} R_\Dset(h) \right]}_{\text{\scriptsize Estimation Error}} 
\end{equation}
\begin{itemize}
  \item \textbf{Approximation Error}: The approximation error corresponds to the minimum risk achievable by a classifier in the given hypothesis class.
  Intuitively, this error measures the quality of the hypothesis class and therefore the quality of the prior knowledge.
  Enlarging the hypothesis class, \ie, allowing more complex functions, can decrease the approximation error.
  \item \textbf{Estimation Error}: The estimation error is the difference between the approximation error and the error made by the ERM predictor.
  Recall that the empirical risk is only an estimate of the true risk.
  This error is dependent on the sample size and the complexity of the hypothesis class.
\end{itemize}
Recall that the main goal is to minimize the true risk $R_\Dset (h_\Sset)$, however, \Cref{equation:ch2-bias_complexity_tradeoff} shows a trade-off called the \emph{bias-complexity trade-off}.
The trade-off is as follows: if we choose a large and complex hypothesis space, we reduce the approximation error but at the same time we can increase the estimation error because a complex hypothesis space might lead to overfitting.
Conversely, choosing a small hypothesis space might reduce the estimation error but increase the approximation error leading to \emph{underfitting}.
We can illustrate the \emph{overfitting} and \emph{underfitting} phenomenons with \Cref{figure:ch2-fitting_points} which shows the decision boundary of 3 classifiers for the same set of samples.
\Cref{figure:ch2-fitting_points_a} shows a classifier which \emph{underfits} the data, meaning the decision boundary is not complex enough to separate the data correctly.
\Cref{figure:ch2-fitting_points_b} shows a classifier that almost perfectly fits the training data but is likely to have a higher error rate on the unseen data.
Finally, \Cref{figure:ch2-fitting_points_c} shows a classifier that seems to have a good compromise between the two.

As seen above, defining a small hypothesis class might lead to underfitting and a large hypothesis class might lead to overfitting.
A good way to balance the trade-off would be to minimize the empirical risk while also minimizing the complexity of the hypothesis class.
Let us define a \emph{regularization} function $r: \Hset \rightarrow \Rbb$ which takes a hypothesis as input and measures the ``complexity'' of the hypothesis.
We could now update the learning rule as follows:
\begin{equation}
  \argmin_{h \in \Hset} \left[ R_\Sset(h) + r(h) \right]
\end{equation}
This learning rule minimizes the empirical risk $R_\Sset(h)$ and a well chosen regularization function $r$. If $r(\ \cdot\ )$ is carefully chosen, this prevent overfitting and improve generalization on unseen data.
This learning rule is closely related to \emph{Structural Minimization Paradigm} (SRM) \cite{shalev2014understanding}.
In the next section, we will present a classical regularization function for neural networks and we will introduce a new regularization scheme in \Cref{chapter:ch5-lipschitz_bound}.

\subsection{Preliminaries on Neural Networks}
\label{subsection:ch2-preliminaries_on_neural_networks}

%


Neural networks, which find their roots in the work of \citet{mcculloch1943logical,rosenblatt1958perceptron}, can be analytically described as a composition of linear functions interlaced with nonlinear functions (also called activation functions).
A feedforward neural network can be defined as follows:

\begin{definition}[Neural Network] \label{definition:ch2-neural_networks}
  Given a depth $\depth \in \Nbb$, 
  let $\dimw = \{ \dimw^{(i)} \}_{i \in [\depth+1]}$ and $\dimb = \{ \dimb^{(i)} \}_{i \in [\depth]}$ be sequences of integers, $\weights = \left\{ \left( \Wmat^{(i)}, \bvec^{(i)} \right) \right\}_{i \in [\depth]}$ a set of weights matrices and bias vectors 
  such that $\Wmat^{(i)} \in \Rbb^{\dimw^{(i)} \times \dimw^{(i+1)}}$ and $\bvec^{(i)} \in \Rbb^{\dimb^{(i)}}$ and a sequence of activation functions $\act = \{\act_i \}_{i \in [\depth]}$.
  Let $\Xset \subset \Rbb^{\dimw^{(1)}}$ and $\Yset \subset \Rbb^{\dimw^{(\depth+1)}}$ be the input and output spaces respectively.
	$\dimw^{(1)}$ and $\dimw^{(\depth)}$ refer to the input and output dimension respectively.
  A neural network is a function $\nn^\act_\weights : \Xset \rightarrow \Yset$ such that
  \begin{equation}
    \nn^\act_{\weights}(\xvec) \triangleq \layer^{\act_\depth}_{\Wmat^{(\depth)}, \bvec^{(\depth)}} \circ \cdots \circ \layer^{\act_1}_{\Wmat^{(1)}, \bvec^{(1)}}(\xvec)
  \end{equation}
  where $\layer^{\act_i}_{\Wmat^{(i)},\bvec^{(i)}}: \Rbb^{w^{(i)}} \rightarrow \Rbb^{w^{(i+1)}}$ (also called layer) is a function parameterized by the weight matrix $\Wmat^{(i)}$, the bias vector $\bvec^{(i)}$ and the activation function $\act_i$.
  $\layer^{\act_i}_{\Wmat^{(i)},\bvec^{(i)}}:$  is defined as follows: 
  \begin{equation}
    \layer^{\act_i}_{\Wmat^{(i)},\bvec^{(i)}} (\xvec) \triangleq \act_i \left(\Wmat^{(i)}\xvec + \bvec^{(i)}\right) \enspace,
  \end{equation}
  and $\rho_\depth$ is identity function.
\end{definition}

\noindent
Based on this definition, for a given training set $\Sset \subset \Xset \times \Yset$, a set of activation functions $\act$, a set of weights and biases $\weights$ and a loss function $L: \Yset \times [k] \rightarrow \Rbb_+$, the ERM learning paradigm for neural networks is given by
\begin{equation} \label{equation:ch2-erm_neural_network}
  \argmin_{\weights} \frac{1}{|\Sset|} \sum_{(\xvec, y) \in \Sset} L(N^\act_\weights(\xvec), y) 
\end{equation}
For classification problems, the zero-one loss is non-convex, and finding a near optimal solution is an NP-hard problem~\cite{feldman2012agnostic,bendavid2003difficulty}.
Instead, a common approach is to use a surrogate loss such as the logistic loss multiclass function and estimate the parameters by maximizing the \emph{likelihood} over the data.
This loss $L:\Yset \times [k]$, is defined as follows:
\begin{equation}
  L(N^\rho_\Omega(\xvec), y) = -\log
    \left(
      \frac
        {e^{\left(N^\rho_\Omega(\xvec)\right)_y}}
	{\sum_{j\in[k]} e^{\left(N^\rho_\Omega(\xvec)\right)_j}}
    \right)
\end{equation}
The generic approach for minimizing the empirical risk in \Cref{equation:ch2-erm_neural_network} is by \emph{gradient descent} with the \emph{backpropagation} algorithm ~\cite{rumelhart1986learning} which consists in computing the gradient with the chain rule.


As seen in the previous section, the SRM paradigm minimizes two terms, the empirical risk and a weight function measuring the ``complexity'' of the hypothesis.
It has been shown that the $\ell_2$ norm of the weights of a network can be used as a measure of complexity of the network \cite{hinton1987learning}.
This \emph{regularization}, also called \emph{weight decay}, prevents weights from growing too large.
The SRM learning algorithm can then be expressed as follows:
\begin{equation}
  \argmin_{\weights} \frac{1}{|\Sset|} \sum_{(\xvec, y) \in \Sset} L(N^\act_\weights(\xvec), y) + \lambda \sum_{(\Wmat, \bvec) \in \weights} \left( \norm{\Wmat}_\mathrm{F} + \norm{\bvec}_\mathrm{2} \right)
\end{equation}
where $\lambda > 0$ is the regularization parameter.

\begin{figure}[ht]
  \centering
  \begin{subfigure}[b]{0.32\textwidth}
    \includegraphics[width=0.98\textwidth]{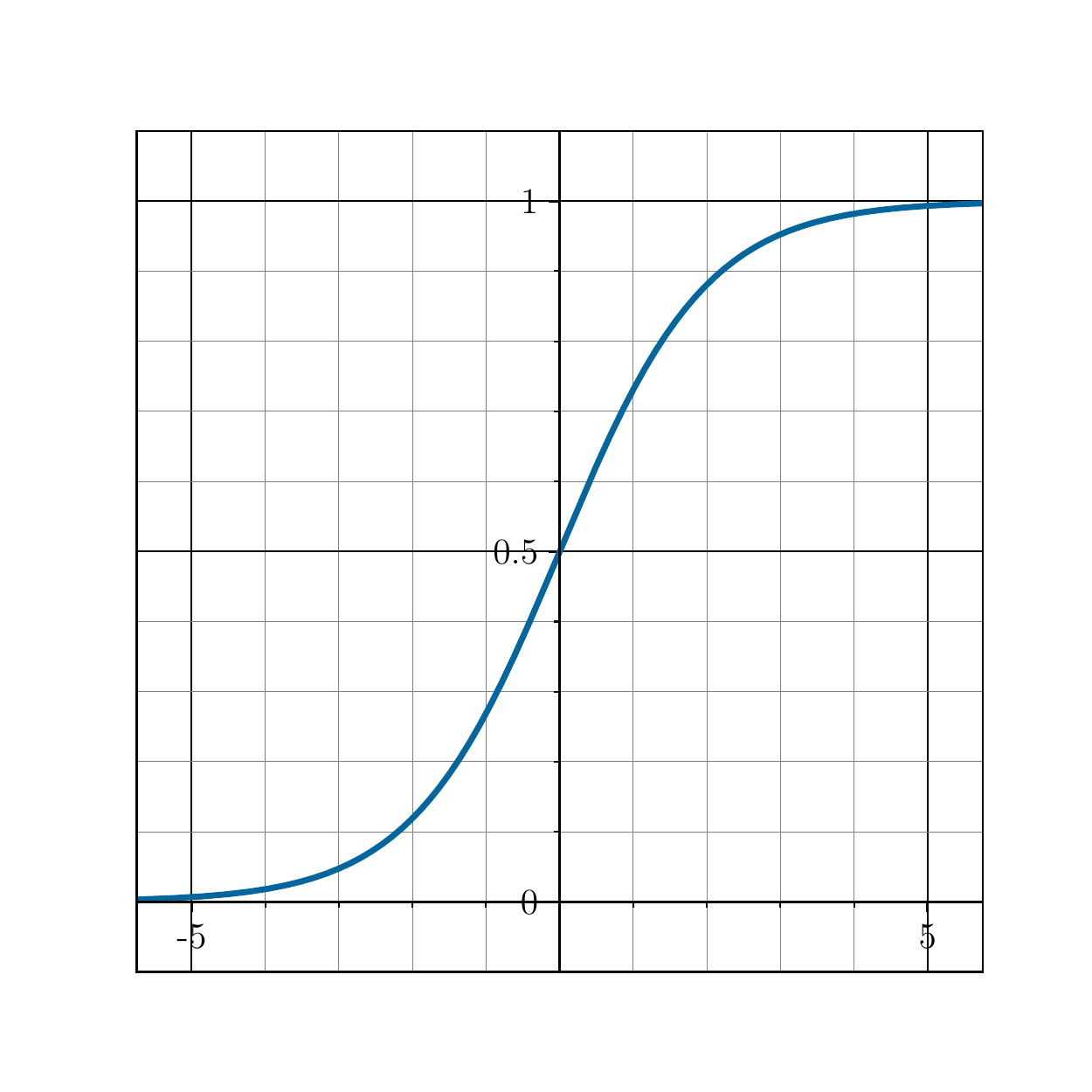}
    \caption{Sigmoid Activation}
  \end{subfigure}
  \hfill
  \begin{subfigure}[b]{0.32\textwidth}
    \includegraphics[width=0.98\textwidth]{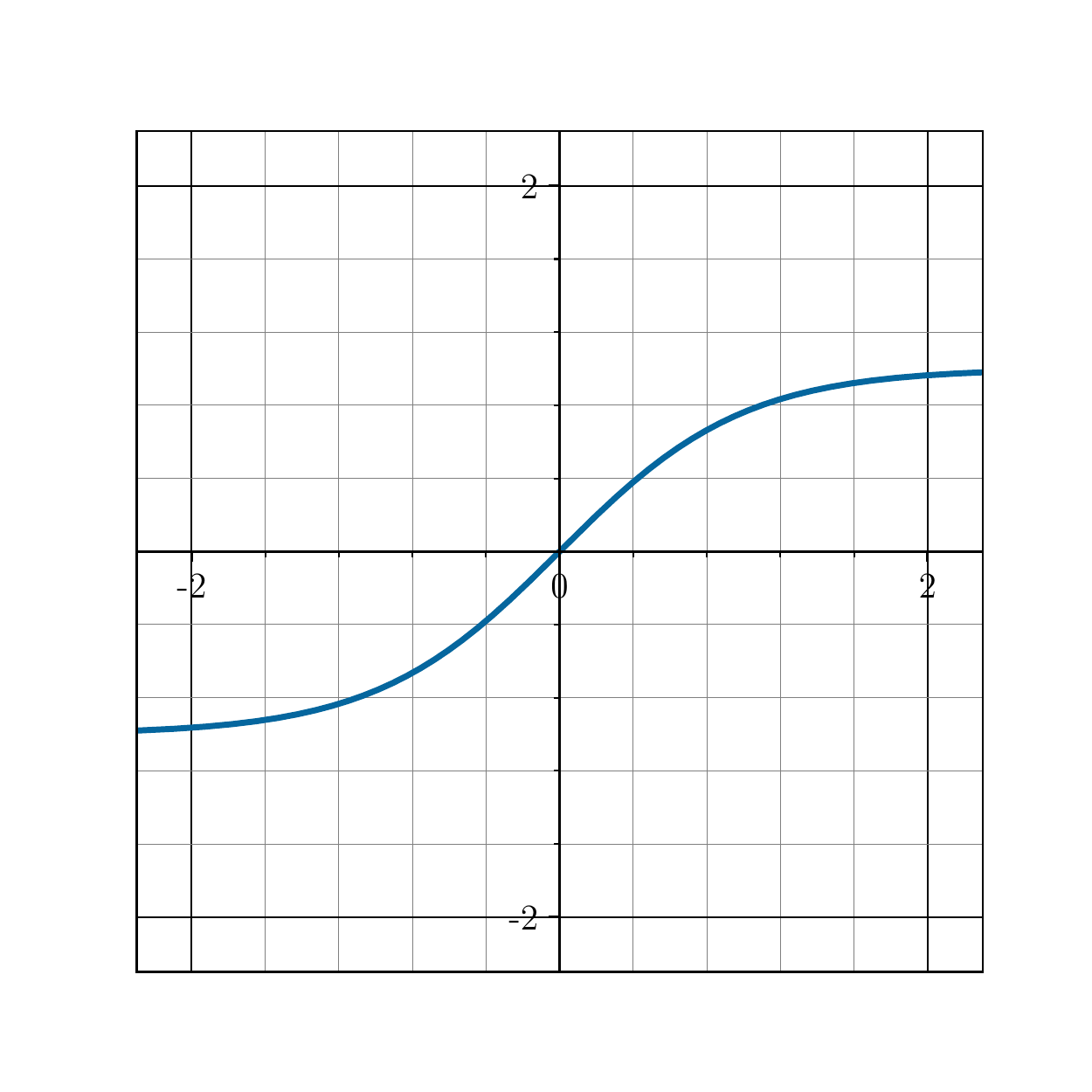}
    \caption{Tanh Activation}
  \end{subfigure}
  \hfill
  \begin{subfigure}[b]{0.32\textwidth}
    \includegraphics[width=0.98\textwidth]{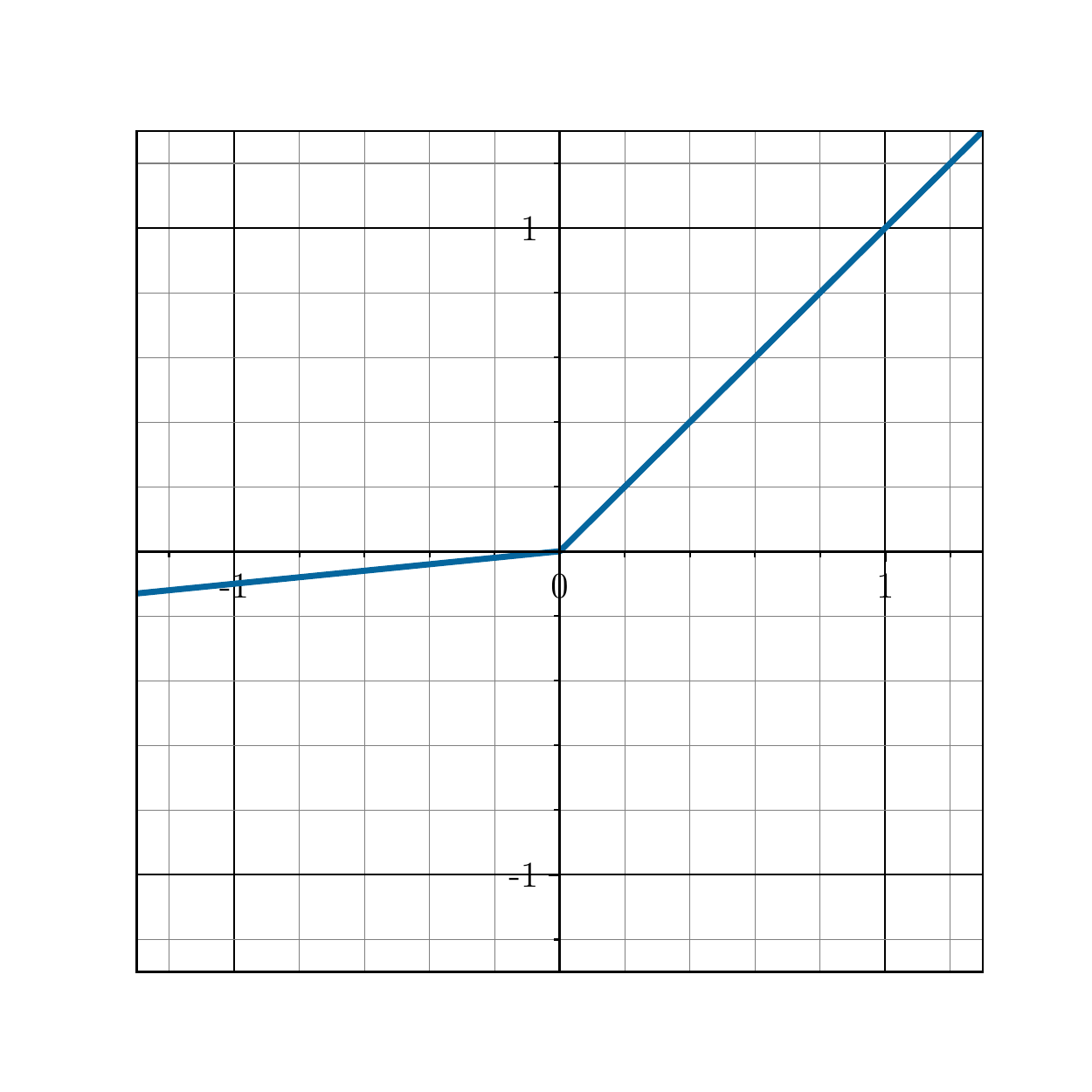}
    \caption{Leaky-ReLU Activation}
  \end{subfigure}
  \caption{Graphical representation of three common activation functions}
  \label{figure:ch2-activation_functions}
\end{figure}

Choosing the right activation function has been an active area of research.
Hereafter, we present three common activation functions used by practitioners.
\begin{itemize}
  \item \textbf{Sigmoid activation}
    \begin{equation*}
      \act(x) = \frac{1}{1+e^{-x}} 
    \end{equation*}
    The sigmoid activation function is one of the first continuous nonlinear functions to be used in the context of neural networks.
		It takes a real value as input and outputs another value between 0 and 1.
  \item \textbf{Hyperbolic Tangent activation}
    \begin{equation*}
      \act(x) = \frac{e^x - e^{-x}}{e^x + e^{-x}}
    \end{equation*}
    The hyperbolic tangent activation function is similar to the sigmoid activation function but instead of returning between 0 and 1, the function returns values between -1 and 1.
  \item \textbf{Rectified Linear activation (ReLU)} \cite{nair2010rectified}
    \begin{equation*}
      \act(x) = \max(0, x)
    \end{equation*}
    The ReLU activation was proposed to avoid the \emph{vanishing gradient problem}.
    The vanishing gradient problem, discovered by \citet{bengio1994learning}, occurs with hyperbolic tangent or sigmoid activation when the magnitude of the input values are almost saturated at $-1$ or $1$, in this case the gradient is close to $0$ and difficulties of optimization and convergence occur.
    The ReLU activation addresses this problem due to the simple values of its gradients which are either 0 or 1 on $\Rbb_-$ or $\Rbb_+$ respectively.
    Furthermore, it has the advantage to be less computationally expensive than tanh and sigmoid functions because it involves simpler mathematical operations.
  \item \textbf{Leaky Rectified Linear activation (Leaky-ReLU)} \cite{maas2013rectifier}
    \begin{equation*}
      \act(x) = \max(\alpha, x)
    \end{equation*}
    More recently, the Leaky-ReLU ($\alpha > 0$) activation function was proposed.
    It introduces the parameter $\alpha$ which characterizes the slope on $\Rbb_-$.
    The advantage of Leaky-ReLU over the ReLU nonlinear activation is that it prevents sparse gradient, which facilitates convergence of neural networks.
\end{itemize}

\noindent
\Cref{figure:ch2-activation_functions} presents the graphical representation of the activation functions presented above.
In this thesis, we will use the Leaky-ReLU function with different $\alpha$ when we train deep neural networks.
We simplify the notation $\nn^\act_\weights$ with $\nn_\weights$.

%
%
%

%
%

\subsection{Adversarial Attacks \& Robustness of Neural Networks}
\label{subsection:ch2-adversarial_attacks_robustness_of_neural_networks}

As seen in the introduction (\Cref{chapter:ch1-introduction}), deep neural networks achieve state-of-the-art performances in a variety of domains such as natural language processing~\cite{radford2019language}, image recognition~\cite{he2016deep} and speech recognition~\cite{hinton2012deep}.
However, it has been shown that such neural networks are vulnerable to \emph{adversarial examples}, \ie, imperceptible variations of the natural examples, crafted to deliberately mislead the models~\cite{globerson2006nightmare,biggio2013evasion,szegedy2013intriguing}.
Because it is difficult to characterize the space of visually imperceptible variations of a natural image, existing adversarial attacks use $\ell_p$ norms as surrogate measures.
We can formally define an adversarial example as follows:
\begin{definition}[Adversarial Pertubation]
  Given an example $\xvec$ and its predicted label $y$, $k$ number of classes, a trained neural network $\nn_\weights$ with $\argmax_{i \in [k]} \big( \nn_\weights(\xvec) \big)_i = y$ and a radius $\epsilon \in \Rbb$, an adversarial perturbation is a vector $\adv \in \Xset$ such that:
  \begin{align}
    &\argmax_{i \in [k]} \big( \nn_\weights (\xvec + \adv) \big)_i \neq y \enspace, \\
    &\st\ \norm{\adv}_p \leq \epsilon \notag
  \end{align}
  where $\epsilon$ is a small value defined by the attacker.
\end{definition}

Note that this definition assumes that the attacker (the person crafting the attack) has access to the parameters of the model.
Typically, an attack method is either \emph{white-box} (complete knowledge of the model and its parameters) or \emph{black-box} (no knowledge of the model).
It is possible to consider that the white-box setting admits too strong assumptions because a model and its parameters could very well be hidden from the public.
In general, it is safer to assume that the adversary has complete knowledge of the model and its defense.
This principle is known in the field of security as the Shannon’s maxim~\cite{shannon1949communication}.
Therefore, in this thesis, we only consider defenses against white-box attacks. 

\subsubsection{Implementing Adversarial Attacks}
\label{subsubsection:ch2-adversarial_attacks}

Since the discovery of adversarial perturbations, a variety of procedures, \aka \emph{adversarial attacks}, have been developed to generate adversarial examples.
FGSM \cite{goodfellow2014explaining}, PGD \cite{madry2018towards} and \cite{carlini2017towards} to name a few, are the most popular ones.
To find the best perturbation $\adv$, existing attacks can adopt one of the two following strategies:

\paragraph{Loss maximization.}
In this scenario, the procedure maximizes the loss objective function $L(\nn_\weights(\xvec + \adv), y)$, under the constraint that the $\lp$ norm of the perturbation remains bounded by some value $\epsilon$, as follows:
\begin{equation} \label{equation:ch2-lossmax}
  \argmax_{\adv:\norm{\adv}_p \leq \epsilon} L(\nn_\weights(\xvec + \adv), y) \enspace.
\end{equation}
The typical value of $\epsilon$ depends on the norm $\norm{\ \cdot\ }_p$ considered in the problem setting.
The current state-of-the-art method to solve \Cref{equation:ch2-lossmax} is based on a projected gradient descent (PGD)~\cite{madry2018towards} of radius~$\epsilon$.
Given a budget $\epsilon$, it recursively computes
\begin{equation} \label{equation:ch2-projectionPGD}
  \xvec^{(t+1)} = \prod_{\mathcal{B}_p(\xvec,\epsilon)}\left(\xvec^{(t)}
    + \alpha \argmax_{\adv: \norm{\adv}_p \leq 1} \adv^\top \nabla_{\xvec^{(t)}} L\left( \nn_\weights \big(\xvec^{(t)} \big), y \right)
\right)
\end{equation}
where $\Bset_p(\xvec,\epsilon) = \{ \xvec + \adv:\norm{\adv}_p \leq \epsilon\}$ is the ball of norm $p$ with radius $\epsilon$, centered at $\xvec$, $\alpha$ is a gradient step size, and $\prod_\Bset$ is the projection operator on the ball $\Bset$.
The PGD attack is currently used in the literature with $p=2$ and $p=\infty$.
The attack with the norm $p=\infty$ is state-of-the-art for the loss maximization problem.

\paragraph{Perturbation minimization.}
This type of procedure searches for the perturbation with the minimal $\lp$ norm, under the constraint that $L(\nn_\weights(\xvec + \adv), y)$ is bigger than a given bound $c$:
\begin{align}
  &\argmin_{\adv} \norm{\adv}_p \label{equation:ch2-normmin} \enspace, \\
  &\st\ L(\nn_\weights(\xvec + \adv), y) \geq c \notag
\end{align}
The value of $c$ is typically chosen depending on the loss function $L$.
For example, if $L$ is the 0-1 loss, any $c > 0$ is acceptable.
\Cref{equation:ch2-normmin} has been tackled by~\citet{carlini2017towards}, leading to the following method, denoted C\&W attack in the rest of the chapter.
It aims at solving the following Lagrangian relaxation of \Cref{equation:ch2-normmin}:
\begin{equation}
  \argmin_{\adv} \norm{\adv}_p + \lambda g(\xvec+\adv) \enspace,
\end{equation}
where $g(\xvec + \adv)<0$ if and only if $L(\nn_\weights(\xvec + \adv),y) \geq c$.
The authors use a binary search to optimize the constant $c$, and gradient descent to compute an approximated solution.
The C\&W attack is currently used in the literature with $p \in \{1, 2, \infty \}$ and is state-of-the-art with $p=2$ for the perturbation minimization problem if we consider that the attacker has unlimited computing power.


\subsubsection{Defending against Adversarial Attacks}
\label{subsubsection:ch2-defending_against_adversarial_attacks}

Given the security risks that adversarial attacks pose, it is important to design defenses to protect neural networks against these kinds of attacks.
Adversarial Training was introduced by~\citet{goodfellow2014explaining} and later improved by~\citet{madry2018towards} as a first defense mechanism to train robust neural networks.
It consists in augmenting training batches with adversarial examples generated during the training procedure.
The structural risk minimization paradigm is thus replaced by the following $\min$ $\max$ problem, where the classifier tries to minimize the expected loss under the maximum perturbation of its input:
\begin{equation}
  \min_\weights \max_{\adv: \norm{\adv} \leq \epsilon} \frac{1}{|\Sset|} \sum_{(\xvec, y) \in \Sset} L\left( \nn_\weights \left(\xvec + \adv \right), y \right) + \lambda \sum_{(\Wmat, \bvec) \in \weights} \left( \norm{\Wmat}_\mathrm{F} + \norm{\bvec}_\mathrm{2} \right) \enspace.
\end{equation}
Although adversarial training lacks formal guarantees, it is one of the few techniques that proves to be empirically very effective.


%
%
%



\subsection{Recent Results on the Theory of Neural Networks}
\label{subsection:ch2-recent_results_on_the_theory_of_neural_networks}

In this section, we give recent generalization bounds for neural networks.
Neural networks have the astonishing property of providing a low error rate on unseen data although they have more parameters than the number of training samples and therefore have the capabilities to fit random labels \cite{zhang2016understanding}.
In this context, traditional approaches of statistical learning  fail to explain why large neural networks generalize well in practice.

\citet{harvey2017nearly} have introduced a generalization bound of neural networks with the VC-dimension as a complexity measure of the hypothesis class.
They improved over previous bounds~\cite{bartlett1998almost,anthony1999neural} by showing that the VC-dimension of a $\depth$-layer feedforward neural network is equal to the depth times the number of parameters.
Unfortunately, this kind of bound with such a complexity measure is of little help to better understand the generalization capabilities of neural networks.

More recently, \citet{bartlett2017spectrally} have proposed to use a \emph{scale-sensitive} complexity measure instead of combinatorial ones (\ie, VC-dimension) which can work with real-valued function classes and are sensitive to their magnitudes.
They proposed to use the product of the spectral norms of the weight matrices (\ie, the Lipschitz constant of the weight matrices) of the network to this scale-sensitive complexity measure.
In addition, they investigated the margins and show that normalizing these Lipschitz constants by the margin allows to better control their excess risk (the test error minus the training error) across
training epochs.
The margin has been previously studied in relationship to generalization by \citet{langford2002pac} and more recently by~\citet{neyshabur2018pacbayesian}.

In what follows, we present generalization bounds for neural networks which are independent of the number of parameters of the network and use as a complexity penalty the Lipschitz constant of the weight matrices.
In addition, following the \Cref{subsection:ch2-adversarial_attacks_robustness_of_neural_networks} on Adversarial Attacks, we present the work of~\citet{farnia2018generalizable} which introduce \emph{adversarial risk} and \emph{empirical adversarial risk} and present an \emph{adversarial generalization bound} similar to the one proposed by \citet{bartlett2017spectrally}.

First, let us formally define the Lipschitz constant of a function as well as the spectral norm of a matrix.
We will use these notions in the following and later in the thesis.
Formally, the Lipschitz constant of a function is defined as follows: 
\begin{definition}[Lipschitz Constant] \label{definition:ch2-lipschitz_constant}
  The Lipschitz constant with respect to the $\ell_p$-norm of a Lipschitz continuous function $f: \Rbb^n \rightarrow \Rbb^m$ is defined as follows:
  \begin{equation}
    \lipp{p}{f} \triangleq \sup_{\substack{\xvec, \yvec \in \Rbb^n \\ \xvec \neq \yvec}} \frac{\norm{f(\xvec) - f(\yvec)}_p}{\norm{\xvec - \yvec}_p} \enspace.
  \end{equation}
  \removespace
\end{definition}
\noindent
In the following of this thesis, we denote $\lipp{2}{f}$ by $\lip{f}$ for simplicity and if $\lipp{p}{f}=k$, we denote the function $f$ as $k$-Lipschitz.
The spectral norm of a matrix $\Wmat$, which is equivalent to the Lipschitz constant of the function $\xvec \mapsto \Wmat \xvec$, is defined as follows:
\begin{definition}[Spectral norm] \label{defintion:ch2-spectral_norm}
Given a matrix $\Wmat$, the spectral norm of $\Wmat$ denoted $\norm{\Wmat}_2$ is defined as:
  \begin{equation}
    \norm{\Wmat}_2 \triangleq \sup_{\substack{\xvec \in \Rbb^n \\ \xvec \neq \zerovec{n}}} \frac{\norm{\Wmat \xvec}_2}{\norm{\xvec}_2} \enspace.
  \end{equation}
  \removespace
\end{definition}
\noindent
Note that the spectral norm also corresponds to the largest singular value of the matrix denoted $\sigma_1(\Wmat)$.

Before presenting the bound from \citet{bartlett2017spectrally}, let us introduce and recall some notations.
Let $N_\Omega$ be a neural network parameterized by $\Omega$ as in the Definition~\ref{definition:ch2-neural_networks}.
Let us recall the risk with respect to the neural network $N_\Omega$ and a distribution $\Dset$ as in~\Cref{equation:ch2-risk}:
\begin{equation}
  R_\Dset(N_\Omega) = \Pbb_{(\xvec, y) \sim \Dset} \left[ \argmax_{i \in [k]} \left( N_\Omega(\xvec) \right)_i \neq y \right] \enspace.
\end{equation}
\citet{bartlett2017spectrally} extended the notion of risk with a \emph{margin operator} $\margin: \Rbb^k \times [k] \rightarrow \Rbb$ defined as $\margin(\vvec, j) \triangleq \vvec_j - \max_{i \neq j} \vvec_i$ and an extension to the 0-1 loss called the \emph{ramp loss} $L_\gamma: \Rbb \rightarrow \Rbb_+$ as:
\begin{equation}
  L_\gamma(r) \triangleq 
  \begin{cases}
    0 &r< -\gamma, \\
    1 + r/\gamma &r \in [-\gamma,0], \\
    1 & r > 0,
  \end{cases}
\end{equation}
Now, we can define the \emph{margin risk} as 
\begin{equation} \label{equation:ch2-margin_risk}
  R_{\gamma,\Dset} (N_\Omega) \triangleq \Ebb_{(\xvec, y) \sim \Dset} \left[ L_\gamma \big(-\margin(N_\Omega(\xvec), y) \big) \right] \enspace,
\end{equation}
and the \emph{empirical margin risk} as 
\begin{equation} \label{equation:ch2-emprical_margin_risk}
  R_{\gamma,\Sset} (N_\Omega) \triangleq \frac{1}{|\Sset|} \sum_{(\xvec, y) \in \Sset} L_\gamma \big(-\margin(N_\Omega(\xvec), y) \big) \enspace.
\end{equation}
Note that the margin risk and the empirical margin risk upper bound the risk and empirical risk.
The generalization bound proposed by~\citet{bartlett2017spectrally} for neural networks is stated as follows:
\begin{theorem}[\citet{bartlett2017spectrally}]
  Let $(\rho^{(1)}, \dots, \rho^{(\depth)})$ be nonlinearities where $\forall i \in [\depth], \lip{\rho^{(i)}} < \infty$ and $\rho^{(i)}(0) = 0$.
  Let $\dimw^{(1)}, \dots, \dimw^{(p+1)}$ be integers such that $\Wmat^{(i)} \in \Rbb^{\dimw^{(i)} \times \dimw^{(i+1)}}$ and let $W = \max_i w^{(i)}$.
  Let $\Xmat$ a matrix where the rows of $\Xmat$ are the input data $\xvec^{(1)}, \dots, \xvec^{(m)} \in \Sset$.
  Let $N_\Omega: \Rbb^{w^{(1)}} \rightarrow \Rbb^{w^{(\depth+1)}}$ be a neural network parameterized by $\Omega$ as in the Definition~\ref{definition:ch2-neural_networks} where $\left( \Wmat^{(1)}, \dots, \Wmat^{(\depth)} \right)$ are the weights matrices.
  Then for every margin $\gamma > 0$, the following bound applies:
  \begin{equation}
    R_\Dset(N_\Omega) \leq R_{\gamma,\Sset} (N_\Omega) + \widetilde{\bigO}\left( \frac{\norm{\Xmat}_\fro \Rset_{\Omega}}{\gamma |\Sset|} \ln(W) + \sqrt{\frac{\ln(1/\delta)}{|\Sset|}} \right)
  \end{equation}
  with probability at least $1 - \delta$, where the \emph{spectral complexity} $\Rset_{\Omega}$ is defined as 
  \begin{equation}
    \Rset_{\Omega} = \left(\prod_{i = 1}^{\depth} \lip{\rho^{(i)}} \norm{\Wmat^{(i)}}_2 \right) \left( \sum_{i=1}^{\depth} \frac{\norm{\Wmat^{(i)\top}}_{2,1}^{2/3}}{\norm{\Wmat^{(i)}}_2^{2/3}} \right)^{3/2}
  \end{equation}
  where $\widetilde{\bigO}(\ \cdot\ )$ ignores logarithmic factors and the norm $\norm{\ \cdot\ }_{p,q}$ is defined by $\norm{\Wmat}_{p,q} \triangleq \left( \sum _{j=1}^{n} \left(\sum _{i=1}^{m} |a_{ij}|^{p} \right)^{\frac {q}{p}}\right)^{\frac {1}{q}}$.
\end{theorem}

As we stated in the introduction, the generalization of neural networks is important but should not be the only metric to consider.
Indeed, a neural network that performs well on natural data could be vulnerable to adversarial attacks.
We saw in the previous section that \emph{adversarial training} is a technique that successfully improves the robustness of neural networks by learning on adversarial examples instead of natural ones.
In the following, we present recent results devised by~\citet{farnia2018generalizable} on the generalization capabilities of neural networks trained with adversarial training.
First, let us define the \emph{adversarial margin risk} as:
\begin{equation}
  R^{\text{adv}}_{\gamma,\Dset} (N_\Omega) \triangleq \Ebb_{(\xvec, y) \sim \Dset} \ L_\gamma \big(-\margin(N_\Omega(\xvec + \adv^{\text{adv}}_\Omega(\xvec)), y) \big) \enspace,
\end{equation}
where $\adv^{\text{adv}}_\Omega(\xvec)$ is an adversarial perturbation following the loss maximization strategy presented above.
The \emph{adversarial empirical margin risk} $R^{\text{adv}}_{\gamma,\Sset}$ is defined similarly as in~\Cref{equation:ch2-emprical_margin_risk}.
The adversarial generalization bound is stated as follows:
\begin{theorem}[\citet{farnia2018generalizable}]
  Let $(\rho^{(1)}, \dots, \rho^{(\depth)})$ be nonlinearities where $\forall i \in [\depth], \lip{\rho^{(i)}} = 1$ and $\rho^{(i)}(0) = 0$.
  Let $b = \max_{\xvec \in \Sset} \norm{\xvec}_2$.
  Let $\Xset \subset \Rbb^{\dimw^{(1)}}$ and $\Yset \subset \Rbb^{\dimw^{(\depth+1)}}$ be the input and output spaces respectively.
  Let $N_\Omega: \Xset \rightarrow \Yset$ be a neural network parameterized by $\Omega$ of depth $\depth$ and of largest width $W = \max_i \dimw^{(i)}$ following the~\Cref{definition:ch2-neural_networks}.
  Assume that for a constant $c_1 \geq 1$ the weights matrices satisfy:
  \begin{equation}
    \forall i, \quad \frac{1}{c_1} \leq \frac{\norm{\Wmat^{(i)}}_2}{\prod_{j=1}^\depth \left(\norm{\Wmat^{(j)}}_2\right)^{1/\depth}} \leq c_1
  \end{equation}
  and that $c_2 \leq \norm{\nabla_\xvec L(N_\Omega(\xvec), y)}_2$ holds for a constant $c_2 > 0$, any $y \in \Yset$ and any $\xvec \in \Bset_2(\xvec, \epsilon)$.
  Let us consider an attack such that $\norm{\adv^{\text{\emph{adv}}}_\Omega(\xvec)}_2 \leq \epsilon$ with $r$ iterations, and a stepsize $\alpha$.
  Then for every margin $\gamma > 0$, the following bound applies:
  \begin{equation}
    R_{0, \Dset}^{\text{\emph{adv}}}(N_\Omega) \leq R_{\gamma, \Sset}^{\text{\emph{adv}}}(N_\Omega) + \bigO \left( \sqrt{\frac{(b + \epsilon)^2 \depth^2 W \log(\depth W) \Rset^{\text{\emph{adv}}}_{N_\Omega} + \depth \log(\frac{r \depth |\Sset| \log(M)}{\delta})}{\gamma^2 |\Sset|}} \right)
  \end{equation}
  with probability at least $1 - \delta$, where the \emph{adversarial spectral complexity} $\Rset^{\text{\emph{adv}}}_{\Omega}$ is defined as
  \begin{equation}
    \Rset^{\text{\emph{adv}}}_{N_\Omega} \triangleq \left[ \prod_{i=1}^\depth \norm{\Wmat^{(i)}}_2 \left( 1 + (\alpha / c_2) \frac{1 - (2\alpha/c_2)^r \Phi_\Omega^r }{1 - (2\alpha/c_2) \Phi_\Omega } \Phi_\Omega  \right) \right]^2 \sum_{i = 1}^\depth \frac{\norm{\Wmat^{(i)}}_\fro^2}{\norm{\Wmat^{(i)}}_2^2}
  \end{equation}
  and where $\Phi_\Omega$ is defined as: $\Phi_\Omega \triangleq \left(\prod_{i=1}^\depth \norm{\Wmat^{(i)}}_2 \right) \sum_{i=1}^\depth \prod_{j=1}^{i} \norm{\Wmat^{(j)}}_2$
\end{theorem}

These generalization bounds for neural networks give us a theoretical justification for \emph{Lipschitz regularization}.
However, computing the spectral norm of the weights matrices is a difficult task.
In the next chapter, we will review some known techniques and we propose in \Cref{chapter:ch5-lipschitz_bound} new efficient method for regularizing the Lipschitz constant of neural networks.

%% file: sources/main/ch3-related_work.tex
\chapter{Related Work}
\label{chapter:ch3-related_work}
\localtoc

\section*{}

This chapter, divided into two parts, is intended to provide an overview of the state of the art related to our contributions.
First, we present current methods for building compact neural networks.
Since the scope of application of these techniques is broad, we have chosen to focus mainly on work that uses linear algebra tools and more particularly structured matrices. 
We present in the first subsection an overview of general techniques for building compact neural networks.
In the next subsection, we present in more detail the current methods for building compact neural networks with structured matrices.
Finally, we discuss these techniques with respect to our contribution to compact neural networks.

The second part of this chapter presents current methods for regularizing the Lipschitz constant of neural networks with the aim of improving their robustness.
This section is divided into four parts.
First, we present techniques that focus on the computation of the Lipschitz constant of neural networks.
Although theoretically and empirically interesting, we will see how these techniques do not scale and therefore cannot be applied to current neural network architectures.
The following subsection presents the approach of Lipschitz regularization via the Lipschitz constant of individual layers of the networks.
We describe the advantages and disadvantages of this approach.
Moreover, in the third subsection, we focus our presentation on current techniques that compute the singular values of convolution layers.
Finally, we discuss these techniques with respect to our contribution on Lipschitz regularization of convolutional neural networks.



\section{Related Work on Compact Neural Networks}
\label{section:ch3-related_work_on_compact_neural_networks}
\input{sources/main/ch3-related_work_structured}

\subsection{Discussion}

In this section, we have shown current methods and techniques for designing compact neural networks with structured matrices. 
Our contributions on \emph{Deep Diagonal Circulant Neural Networks} are a direct follow-up to the work of~\citet{cheng2015exploration,sindhwani2015structured,moczulski2016acdc,thomas2018learning} focusing on compact neural networks with \emph{structured matrices}.
More precisely, we extend the work of \citet{moczulski2016acdc} by training \emph{fully structured networks} (\ie, networks with structured layers only) hence demonstrating that diagonal-circulant layers are able to model complex relations between inputs and outputs.
Although, this diagonal-circulant layers fit in the low displacement rank framework, we demonstrate much better performances in practice.
Indeed, thanks to a solid theoretical analysis and thorough experiments, we were able to train deep (up to 40 layers) circulant neural networks, and apply, for the first time, this structured architecture in the context of large-scale video classification.
This contrasts with previous experiments in which only one or a few dense layers were replaced inside a large redundant network such as VGG~\cite{simonyan2014very}.


\section{Related Work on Lipschitz Regularization}
\label{section:ch3-related_work_on_lipschitz_regularization}
\input{sources/main/ch3-related_work_lipschitz}
\subsection{Discussion}

We have presented state-of-the-art methods for regularizing the Lipschitz constant of neural networks with the aim to improve their robustness against adversarial attacks.
The power method~\cite{golub2000eigenvalue} is a popular technique for approximating the maximal singular value of a matrix.
Recent works in deep learning use this method in a wide variety of settings, for example, robustness \cite{farnia2018generalizable,tsuzuku2018lipschitz}, generalization~\cite{yoshida2017spectral,gouk2018regularisation} or to stabilize the training of Generative Adversarial Networks (GANs) \cite{miyato2018spectral}.
Despite a number of interesting results, using the power method is expensive and results in prohibitive training times. 
Other approaches to regularize the Lipschitz constant of neural networks have been proposed by~\citet{sedghi2018singular} and ~\citet{singla2019bounding}.
These methods exploit the properties of circulant matrices to approximate the maximal singular value of a convolution layer.
Although interesting, theses method results in a loose approximation of the maximal singular value.
Our work is positioned at the intersection between these works, we will introduce a new approach for regularizing the Lipschitz constant of neural networks, that is more efficient than the power method and more accurate than methods relying on the structure of convolutions.

%% file: sources/main/ch3-related_work_structured.tex
\subsection{General Techniques to Build Compact Neural Networks}

As seen in the Introduction (\Cref{chapter:ch1-introduction}), scaling up networks can lead to better accuracy~\cite{tan2019efficientnet,brown2020language}.
However, large neural networks lead to difficult and expensive training and after observing that a lot of parameters in large neural networks were redundant~\cite{dai2018compressing,frankle2018lottery}, an important question arises: \emph{do neural networks need to be over-parameterized? And if not, how to build accurate and compact neural networks?} 

Numerous other directions have been investigated to build compact and cost-effective neural networks without impacting the accuracy.
For example \citet{gupta2015deep,micikevicius2018mixed} have proposed to represent weights with limited numerical precision to reduce training time and memory requirements.
They used half-precision floating-point format instead of single-precision floating-point format which uses 32 bits of computer memory.
In the same direction, \citet{courbariaux2015binaryconnect} have proposed a method to train neural networks with binary weights without an important loss in the accuracy.

An important idea in model compression, proposed by~\citet{bucilua2006model}, is based on the observation that the model used for training is not required to be the same as the one used for inference.
Indeed, models compressed after training can be deployed on smartphones or IoT devices.
Based on this idea, multiple post-processing techniques have been developed: a quantization procedure which consists in converting the weights into a binary or integer formats \emph{after} the training phase~\cite{mellempudi2017ternary,rastegariECCV16}, pruning techniques~\cite{dai2018compressing,han2015deep,lin2017runtime} or sparsity regularizers~\cite{collins2014memory,dai2018compressing,liu2015sparse} which consists in removing redundant weights after training and taking advantage of the sparse structure of the weight matrices.

Sparse neural networks have also been extensively studied since the seminal work of \citet{frankle2018lottery} in which they propose the \emph{Lottery Ticket Hypothesis}. 
This hypothesis states that there exists a sparse subnetwork of a dense neural network that when trained in isolation can match the test accuracy of the original dense network after training for at most the same number of iterations. 
This hypothesis led to a series of works on sparse neural networks \cite{zhou2019deconstructing,malach2019proving,evci2020rigging}.

Moreover, \citet{ba2014deep} have empirically demonstrated that shallow neural networks can learn the complex functions previously learned by other deep neural networks.
This result led \citet{hinton2015distilling} to propose a technique called \emph{model distillation} which consists in training a large complex model using all the available data and resources to be as accurate as possible, then a smaller and more compact model is trained to approximate the first model.
Although interesting for deployment purposes, this approach still requires to train one large network and one shallow, which entails a significant training cost.

\begin{figure}[t]
  \centering
  \includegraphics[width=0.80\textwidth]{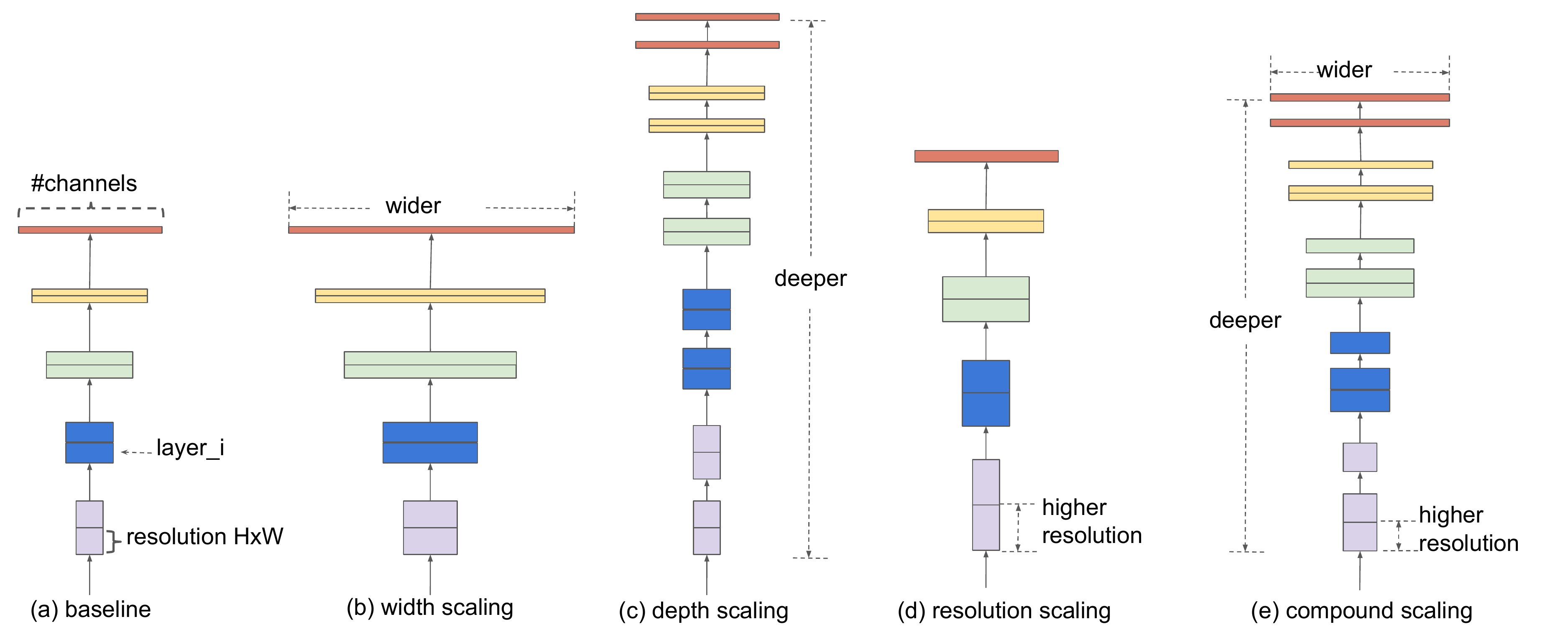}
  \caption{Illustration of the scaling of the EfficientNet architecture.}
  \label{figure:p1-ch3-illustration_efficientnet}
\end{figure}

More recently, \citet{zoph2018learning,real2019regularized} have designed algorithms that automatically tune the width and depth of neural network architectures to obtain the best trade-off between compactness and accuracy.
With this approach, \citet{tan2019efficientnet} found a new compound scaling method that uniformly scales network width and depth leading to efficient and compact architecture.
\Cref{figure:p1-ch3-illustration_efficientnet} illustrates the different scaling proposed by~\citet{tan2019efficientnet}.

\subsection{Building Compact Neural Networks with Structured Matrices}
\label{subsection:ch3-building_compact_neural_networks_with_structured_matrices}






An effective method to build compact neural networks is to constrain the hypothesis space by imposing a \emph{structure} on the weight matrices which constitute the different layers of the network.

\paragraph{Structured Neural Networks with Low Rank Approximation} ~\\

\noindent
For example, \citet{sainath2013lowrank} were among the first to use low-rank matrices in deep learning contexts followed by the work of~\citet{jaderberg2014speeding,yu2017compressing}.
Their work consists in replacing the weight matrices of size $n \times m$ by the product of two rectangular matrices of size $n \times r$ and $r \times m$, where $r$ corresponds to the rank of the new matrix. 
In order to reduce the number of parameters, the rank $r$ is chosen to be small such that $r \ll \min(m, n)$.
By representing the weight matrices with a low-rank decomposition, one can reduce the storage from $mn$ parameters to $(mr + nr)$ and accelerate the matrix-vector product from $\bigO(mn)$ to $\bigO(mr + rn)$.
To enforce the low-rank constraint, reduced storage and computation time during training, the authors trained the coefficients of the two rectangular matrices directly. 
Formally, let $\Wmat \in \Rbb^{n \times m}$ be a weight matrix and let $\widetilde{\Wmat}$ be the low-rank approximation of rank $r$ of the matrix $\Wmat$.
Then, the low-matrix $\widetilde{\Wmat}$ can be decomposed by the product of two rectangular matrices $\Umat \in \Rbb^{n \times r}$ and $\Vmat \in \Rbb^{r \times m}$ such that $\widetilde{\Wmat} = \Umat \Vmat$.
Therefore, a neural network layer with low-rank approximation can be expressed as follows:
\begin{equation}
  \layer_{\Umat, \Vmat, \bvec} (\xvec) = \act\left( \Umat \Vmat \xvec + \bvec \right) .
\end{equation}
The scalar $r$ defining the size of the two rectangular matrices becomes an hyper-parameter and controls the trade-off between the expressivity and compactness of the layer. 

In the same vein, \citet{oseledets2011tensor} have proposed the \emph{Tensor Train decomposition} (TT-decomposition), which is based on the tensor rank decomposition (Tucker decomposition) proposed by~\citet{hitchcock1927expression} and named after \citet{tucker1966some}.
The TT-decomposition is defined as follows.
Let $\boldsymbol{\Aset} \in \Rbb^{n_1 \times n_2 \times \dots \times n_{d-1} \times n_d}$ be a $d$-dimensional tensor.
The Tensor-Train Decomposition factorizes $\boldsymbol{\Aset}$ in a product of third-order tensors and it is given by: 
\begin{equation}
  (\boldsymbol{\Aset})_{(i_1,\dots,i_d)} = (\Gmat^{(1)})_{(i_1, :)} (\boldsymbol{\Gset}^{(2)})_{(:, i_2, :)} (\boldsymbol{\Gset}^{(3)})_{(:, i_3, :)} \dots (\Gmat^{(d)})_{(:, i_d)}
\end{equation}
where $\Gmat^{(i)}$ are matrices and $\boldsymbol{\Gset}^{(i)}$ are third-order tensors of size $r_{i} \times r_{i+1}$ called \emph{TT-cores}.
The sequence $\{r_k\}_{k=0}^d$ is referred to as the ranks of the TT-representation.
The above equation can be equivalently rewritten as a sum of elements of the TT-cores:
\begin{equation}
  (\boldsymbol{\Aset})_{(i_1,\dots,i_d)} = \sum_{\alpha_1, \dots, \alpha_{d-1}} (\Gmat^{(1)})_{(i_1, \alpha_1)} (\boldsymbol{\Gset}^{(2)})_{(\alpha_1, i_2, \alpha_2)} \dots (\Gmat^{(d)})_{(\alpha_{d-1}, i_d)}
\end{equation}
\citet{oseledets2011tensor} have shown that for an arbitrary tensor $\boldsymbol{\Aset}$, several TT-representations exist with different ranks.
The TT-decomposition can be very efficient in terms of memory requirement if the ranks are small.
Indeed, the tensor $\boldsymbol{\Aset}$ has $\prod_{k=1}^{d} n_k$ values compared with $\sum_{k=1}^d n_k r_{k-1} r_k$ values.

The TT-decomposition has been extensively used in the context of deep learning.
\citet{novikov2015tensorizing} was one of the first to use this technique to reduce the number of parameters of neural networks by using the decomposition to replace the fully connected layer of the VGG architecture~\cite{simonyan2014very}.
They reported a compression factor of the dense weight matrix up to 200000 times leading to the compression factor of the whole network up to 7 times with only 0.3 point drop of TOP-5 accuracy on ImageNet~\cite{deng2009imagenet}.
With this work, \citet{novikov2015tensorizing} have demonstrated that the TT-decomposition allows an important reduction of the number of parameters while preserving the expressive power of the layers.
Later, the TT-decomposition was used in other types of architectures.
\citet{garipov16ttconv} used it to compress convolution layers as well as fully connected layers.
\citet{yang2017tensor} used it in the context of video classification, \citet{tjandra2017compressing} compressed the layers of recurrent neural networks and finally, \citet{xindian2019tensorized} developed a compact architecture based on TT-decomposition for Language Modeling.

However, the Tensor-Train decomposition has some limitations.
Although it can reduce the number of parameters when the ranks are low, finding the best alignment of the tensor dimensions in order to find the best optimized TT-cores remains a challenging problem, as stated by~\citet{pan2019compressing}.


\paragraph{Neural Networks with Diagonal and Circulant Matrices} ~\\

\noindent

\citet{cheng2015exploration} proposed to replace the weight matrix of a fully connected layer by the product of a circulant and a diagonal matrix leading to following structured layer:
\begin{equation}
  \layer_{\Dmat, \Cmat, \bvec} (\xvec) = \act \left( \Dmat \Cmat \xvec + \bvec \right) \enspace,
\end{equation}
where the circulant matrix is learned by a gradient-based optimization algorithm and the diagonal matrix entries are sampled at random in $\{-1, 1\}$. 
The idea of replacing dense matrices with circulant ones comes from their use in dimensionality reduction with the \emph{fast Johnson-Lindenstrauss transform}~\cite{hinrichs2011johnson,vybiral2011variant}, binary embedding~\cite{yu2014circulant}, and kernel approximation~\cite{yu2015compact}, etc.
Circulant matrices exhibit several interesting properties from the perspective of numerical computations.
Recall from \Cref{theorem:ch2-diagonalization_circulant_matrix} that circulant matrices can be diagonalized with the Fourier Transform as follows:
\begin{equation}
  \Cmat = \frac{1}{n} \Umat_n^* \diag(\Umat_n \cvec) \Umat_n \enspace.
\end{equation}
where the vector $\cvec$ corresponds to the first columns of the matrix $\Cmat$.
This decomposition allows a compact representation in memory ($n$ values instead of $n^2$) and efficient matrix-vector product with the FFT algorithm (see \Cref{algorithm:ch2-matrix_vector_product_circulant_matrix}).
Despite the reduction of expressivity, \citet{cheng2015exploration} demonstrated good empirical results using only a fraction of the original weights (90\% reduction).

\citet{moczulski2016acdc} built upon the work of~\citet{cheng2015exploration} and \citet{huhtanen2015factoring} and introduced two \emph{Structured Efficient Linear Layers} (SELL) based on the Fourier and cosine transforms.
First, by observing that the DC transform cannot express an arbitrary linear operator they proposed to apply the result of \citet{huhtanen2015factoring} which states that almost all matrices can be decomposed as a product of DC transforms.
\begin{theorem}[Reformulation from \citet{huhtanen2015factoring}] \label{theorem:ch3-huhtanen}
  For every matrix $\Mmat \in \Cbb^{n \times n}$, for any $\epsilon > 0$, there exists a sequence of matrices $\{ \Amat^{(i)} \}_{i \in [2n-1]}$ where $\Amat^{(i)}$ is a circulant matrix if $i$ is odd, and a diagonal matrix otherwise, such that $\norm{\Amat^{(1)} \ldots \Amat^{(2n-1)} - \Mmat} < \epsilon$.
\end{theorem}
\noindent
Based on this result, they proposed to parameterize the layers of a neural network with $k$ products of diagonal and circulant matrices as follows:
\begin{equation} \label{equation:acdc_layer}
  \layer_{\Dmat, \Cmat, \bvec} (\xvec) = \act \left( \left(\prod_{i = 1}^{k} \Dmat^{(i)} \Cmat^{(i)} \right) \xvec + \bvec \right)
\end{equation}
where $\Dmat$ and $\Cmat$ are sequences of $k$ diagonal and circulant matrices respectively.
This structured layer is therefore parameterized by $n(2k+1)$ values and the value $k$ becomes a hyper-parameter controlling the trade-off between compactness and expressivity. 
By diagonalizing the circulant matrix, the layer in \Cref{equation:acdc_layer} can be expressed as a product of diagonal matrices and the Fourier transform as follows:
\begin{equation} \label{equation:ch3-laye_acdc}
  \layer^\act_{\dvec, \cvec, \bvec} (\xvec) = \act \left(\frac{1}{n^k} \left(\prod_{i = 1}^{k} \diag\left(\dvec^{(i)}\right) \Umat_n^* \diag\left(\Umat_n \cvec^{(i)}\right) \Umat_n \right) \xvec + \bvec \right)
\end{equation}
\noindent
Although interesting and demonstrating good empirical results, the work of~\citet{moczulski2016acdc} suffers from multiple limitations. 
First, the result from~\citet{huhtanen2015factoring} is expressed with respect to $n$, the size of the matrices $\Amat$.
Therefore, the theorem does not provide any insights regarding the expressive power of $k$ factors when $k$ is much lower than $2n-1$ as it is the case in most practical scenarios they consider.
Finally, in order to stay in the real domain, they replaced the Fourier transform in~\Cref{equation:ch3-laye_acdc} with the cosine transform thus learning a different kind of linear transform (see the work of~\citet{sanchez1995diagonalizing} which characterizes the matrices diagonalizable by the cosine transform).
Furthermore, because the cosine transform does not diagonalize circulant matrices, \Cref{theorem:ch3-huhtanen} no longer applies.


\paragraph{General Representation of Structured Linear Maps: LDR and K-Matrices} ~\\

\vspace{-0.5cm}

General frameworks for structured matrices that reduce the memory footprint but also accelerate matrix-vector product operations have been used to build compact neural networks.
\citet{sindhwani2015structured} have used the notion of low displacement rank presented in \Cref{subsection:ch2-general_frameworks_for_structured_matrices} to learn a broad family of structured matrices.
Recall from~\Cref{theorem:ch2-toeplitz_like} that all matrices expressed as the following sum of products are called \emph{Toeplitz-like} matrices:
\begin{equation}
    \Mmat = \frac{1}{2} \sum_{j=1}^{r} \Zmat_1(\gvec^{(j)}) \Zmat_{-1}(\Jmat_n \hvec^{(j)})
\end{equation}
where $\Zmat_f$ is an $f$-circulant matrix defined in \Cref{definition:ch2-f_circulant_matrix}, $\Gmat = \leftmat \gvec^{(1)} \ldots \gvec^{(r)} \rightmat, \Hmat = \leftmat \hvec^{(1)} \ldots \hvec^{(r)} \rightmat \in \Rbb^{n \times r}$ with $r \ll n$ and $\Jmat_n$ is the reflection of the $n \times n$ identity matrix.
More precisely, \citet{sindhwani2015structured} proposed to learn Toeplitz-like matrices by learning the factors $\Gmat$ and $\Hmat$. 
Therefore, they proposed the following parameterized layer:
\begin{equation}
  \layer_{\Gmat, \Hmat, \bvec}(\xvec) = \act \left( \left(\sum_{j=1}^{r} \Zmat_1\left(\gvec^{(j)}\right) \Zmat_{-1}\left(\Jmat_n \hvec^{(j)}\right) \right) \xvec + \bvec \right)
\end{equation}
where the rank $r$ is a hyper-parameter and controls the number of parameters of the layer.
In addition to offer fast matrix-vector product, they have showed that this class of layers is very rich from a modeling perspective.
More precisely, they characterize the expressivity of the layer as follows: 
\begin{theorem}[LDR expressivity \citet{pan2001structured,sindhwani2015structured}] ~\\
  The set of all $n \times n$ matrices that can be written as, $\frac{1}{2} \sum_{i=1}^{r} \Zmat_1(\gvec^{(i)}) \Zmat_{-1}(\Jmat_n \hvec^{(i)})$
  for some $\Gmat = \leftmat \gvec^{(1)} \ldots \gvec^{(r)} \rightmat,
  \Hmat = \leftmat \hvec^{(1)} \ldots \hvec^{(r)} \rightmat \in \Rbb^{n \times r}$ contains:
  \begin{compactitem}
    \item All $n \times n$ Circulant and Skew-Circulant matrices for $r \geq 1$.
    \item All $n \times n$ Toeplitz matrices for $r \geq 2$.
    \item Inverses of Toeplitz matrices for $r \geq 2$.
    \item All products of the form $\Amat^{(1)} \ldots \Amat^{(t)}$ for any $r \geq 2t$.
    \item All linear combinations of the form $\sum_{i=1}^p \beta_i \Amat^{(1, i)} \ldots \Amat^{(t, i)}$ for any $r \geq 2pt$.
    \item All $n\times n$ matrices for $r=n$.
  \end{compactitem}
  where each $\Amat^{(i)}$ above is a Toeplitz matrix or the inverse of a Toeplitz matrix. 
\end{theorem}
\noindent

In the same line of work, \citet{thomas2018learning} have proposed neural network layers directly form the Krylov decomposition presented in~\Cref{theorem:ch2-krylov_decomposition} which encompasses an even larger family of structured matrices including Toeplitz-like, Vandermonde-like, Cauchy-like ones.
Despite being elegant and general, we found that the LDR framework suffers from several limits which are inherent to its generality and makes it difficult to use in the context of deep neural networks.
As acknowledged by the authors, the number of parameters required to represent a given structured matrix (a Toeplitz matrix) in practice is unnecessarily high (higher than required in theory) making the training very hard.

More recently, another type of generalization of structured linear maps has been proposed by~\citet{dao2019learning,dao2020kaleidoscope}.
They introduced a family of matrices called \emph{kaleidoscope matrices} (K-matrices) which are the product of sparse matrices with specific predefined sparsity patterns.
They showed that this type of matrices can capture any sparse matrix with near-optimal space (parameter) and time (arithmetic operation) complexity.
The authors claim that their structured linear maps can capture more common structures with a few numbers of parameters than the displacement operators presented above.
More precisely, their representation is based on products of a particular building block known as a butterfly matrix introduced by~\citet{parker1995random}.
Butterfly matrices have been extensively used in numerical linear algebra~\cite{parker1995random,li2015butterfly} and machine learning~\cite{mathieu2014fast,jing2017tunable,munkhoeva2018quadrature,dao2019learning,choromanski2019unifying}.

%% file: sources/main/ch3-related_work_lipschitz.tex
\subsection{The Global Lipschitz Constant of Neural Networks}
\label{subsection:ch3-the_global_lipschitz_constant_of_neural_networks}

\noindent
The regularization of the Lipschitz constant of neural networks has seen a growing interest in the last few years.
Indeed, numerous results have shown that neural networks with a low Lipschitz constant exhibit better generalization~\cite{bartlett2017spectrally} and higher robustness to adversarial attacks~\cite{szegedy2013intriguing,tsuzuku2018lipschitz, farnia2018generalizable}.

The Lipschitz constant, defined in~\Cref{definition:ch2-lipschitz_constant}, is a measure of the stability of the network.
If the Lipschitz constant is high, the network will tend to be more sensitive to input perturbations, meaning, if the input changes by $\epsilon$, the output changes by at most $k\epsilon$.
The Lipschitz constant of a function can also be expressed using the differential operator as follows:
\begin{theorem}[Rademacher's Theorem] \label{theorem:ch3-lipschitz_differential_op}
  If $f: \Rbb^n \rightarrow \Rbb^m$ is a Lipschitz continuous function, then $f$ is differentiable almost everywhere.
  Moreover, if $f$ is Lipschitz continuous, then
  \begin{align}
    \lip{f} = \sup_{\xvec \in \Rbb^n} \norm{\mathrm{D}_\xvec f(\xvec)}_2
  \end{align}
  where $\mathrm{D}_\xvec$ is the differential operator of $f$ at $\xvec$.
\end{theorem}

\citet{tsuzuku2018lipschitz} have studied the relationship between the robustness and the Lipschitz constant and the margin of neural networks. 
By the definition of the Lipschitz constant, we have the following:
\begin{equation}
  \norm{N_\Omega(\xvec) - N_\Omega(\xvec + \adv)}_2 \leq \lip{N_\Omega} \norm{\adv}_2
\end{equation}
Recall the margin operator $\margin: \Rbb^k \times [k] \rightarrow \Rbb$ from~\Cref{subsection:ch2-recent_results_on_the_theory_of_neural_networks} defined as:
\begin{equation}
  \margin(\vvec, j) \triangleq \vvec_j - \max_{i \neq j} \vvec_i
\end{equation}
Then, we have the following proposition which characterizes the robustness of a neural network with respect to its margin and Lipschitz constant.
\begin{proposition}[\citet{tsuzuku2018lipschitz}]
  \begin{equation} \label{equation:ch3-margin_guarded_area}
    \margin \big( N_\Omega(\xvec), y \big) \geq \sqrt{2} \lip{N_\Omega} \norm{\adv}_2 \quad \Longrightarrow \quad \margin \big( N_\Omega(\xvec + \adv), y \big) \geq 0
  \end{equation}
  \removespace
\end{proposition}
\noindent
If the inequality on the right-hand side of \Cref{equation:ch3-margin_guarded_area} is verified then the adversarial margin is positive, \ie, the network correctly predicts the label. 
From this proposition, we can conclude that for a given neural network with specific margins, a lower Lipschitz constant allows for an increase in robustness. 
Note that the margin is already maximized in a multi-class setting with the cross-entropy loss as stated in~\citet{hein2017formal}.
A multitude of work have tried to reduce the Lipschitz constant in order to improve adversarial robustness.
However, \citet{scaman2018lipschitz} have shown that computing the exact Lipschitz constant of a neural network is NP-hard.
The following theorem shows that, even for shallow neural networks, exact Lipschitz computation is not achievable in polynomial time:
\begin{theorem}[\citet{scaman2018lipschitz}] \label{theorem:ch3-lipschitz_computation}
  Let us define the problem associated with the exact computation of the Lipschitz constant of a $2$-layer neural network with $\relu$ activation:
  \begin{itemize}
    \item[] \textbf{Input:} Two matrices $\Wmat^{(1)} \in \Rbb^{l \times n}$ and $\Wmat^{(2)} \in \Rbb^{m \times l}$, and a constant $c \geq 0$.
    \item[] \textbf{Question:} Let $N = \Wmat^{(2)} \circ \rho \circ \Wmat^{(1)}$ where $\rho$ is the $\relu$ activation function. \emph{Is the Lipschitz constant $\lip{N} \leq c$ ?}
  \end{itemize}
  Then, assuming that $\textbf{P} \neq \textbf{NP}$, the problem above is NP-hard. 
\end{theorem}

\noindent
To overcome this difficulty, researchers have relied on devising a tight upper bound of the Lipschitz constant.
For example, \citet{scaman2018lipschitz} have shown that the Lipschitz constant of a neural network $N$ can be explicitly formulated using \Cref{theorem:ch3-lipschitz_differential_op} and the chain rule:
\begin{equation} \label{equation:ch3-decomposition_jacobian_lipschitz}
  \lip{\nn} = \sup_{x \in \Rbb^n} \norm{\Wmat^{(p)} \diag(\rho'_\depth(\theta_\depth)) \dots \Wmat^{(2)} \diag(\rho'_1(\theta_1)) \Wmat^{(1)}}_2,
\end{equation}
where $\theta_i = \layer^{\act_i}_{\Wmat^{(i)}, \bvec^{(i)}} \circ \cdots \circ \layer^{\act_1}_{\Wmat^{(1)}, \bvec^{(1)}}(\xvec)$ is the intermediate output after $i$ layers and $\rho'_i$ is the derivative of $\rho_i$.
The Lipschitz of the neural network $N$ can then be upper bounded as follows:
\begin{align}
  \lip{\nn} &\leq \max_{\forall i,\ \sigma_i \in [0, 1]^{w^{(i+1)}}} \norm{\Wmat^{(\depth)} \diag(\sigma_{\depth-1}) \dots \diag(\sigma_1) \Wmat^{(1)}}_2 \notag \\
  &\leq \max_{\forall i,\ \sigma_i \in [0, 1]^{w^{(i+1)}}} \norm{ \pmb{\Sigma}^{(\depth)} \Vmat^{(\depth)\top} \diag(\sigma_{\depth-1}) \dots \diag(\sigma_1) \Umat^{(1)} \pmb{\Sigma}^{(1)}}_2 \notag \\
  &\leq \prod_{i=1}^{\depth-1} \max_{\sigma_i \in [0, 1]^{w^{(i+1)}}} \norm{\widetilde{\pmb{\Sigma}}^{(i+1)} \Vmat^{(i+1)\top} \diag(\sigma_{i+1}) \Umat^{(i)} \widetilde{\pmb{\Sigma}}^{(i)}}_2 
\end{align}
where $\widetilde{\pmb{\Sigma}}^{(i)} = \pmb{\Sigma}^{(i)}$ if $i \in \{1, \depth\}$ and $\widetilde{\pmb{\Sigma}}^{(i)} = {\pmb{\Sigma}^{(i)}}^{1/2}$ otherwise.
The first inequality is due to the fact that the derivatives of the activation functions are bounded, \ie, $\rho_i(\xvec) \in [0, 1]^{w^{(i+1)}}$, the second inequality is obtained by decomposing each weight matrix $\Wmat^{(i)}$ with the \emph{Singular Value Decomposition} such that $\Wmat^{(i)} = \Umat^{(i)} \pmb{\Sigma}^{(i)} \Vmat^{(i)\top}$; and finally, the last inequality is due to the submultiplicativity of the operator norm.
Although accurate, this bound is still computationally expensive to compute due to the singular value decomposition and the optimization for each layer. 
In the same line of research, recent work~\cite{fazlyab2019safety,fazlyab2019efficient,latorre2020lipschitz} has proposed a tight bound on the Lipschitz constant of the full network with the use of semi-definite programming.
More precisely, \citet{fazlyab2019efficient} have demonstrated the following result:
\begin{theorem}[Lipschitz bounds \citet{fazlyab2019efficient}] \label{theorem:ch3-lipschite_semidefinite_programming}
  Consider a neural network $N: \Rbb^n \rightarrow \Rbb^m$ such that $N(\xvec) = \Wmat^{(2)} \rho(\Wmat^{(1)} \xvec + \bvec^{(1)}) + \bvec^{(2)}$.
  Suppose the activation function $\rho$ is \emph{slope-restricted} in the sector $[\alpha,\beta]$, \ie,
  \begin{equation}
    \alpha \leq \frac{\rho(y) - \rho(x)}{y-x} \leq \beta \quad \forall x,y \in \Rbb. 
  \end{equation}
  Define the set $\mathcal{T}_{n}$ as the following:
  \begin{equation*}
    \mathcal{T}_n = \{\Tmat \in \Sbb^n \mid \Tmat = \sum_{i=1}^{n} \lambda_{ii} \evec^{(i)} \evec^{(i)\top} + \sum_{1 \leq i<j \leq n} \lambda_{ij}(\evec^{(i)} - \evec^{(j)})(\evec^{(i)}-\evec^{(j)})^\top, \lambda_{ij} \geq 0 \}.
  \end{equation*}
  where $\Sbb^n$ is the set of all symmetric matrices of size $n \times n$.
  Suppose there exists a constant $c>0$ such that the matrix inequality
  \begin{align}
    \Mmat(c,\Tmat) \triangleq
      \leftmatrix
      -2\alpha \beta \Wmat^{(1)\top} \Tmat \Wmat^{(1)} - c \Imat_n & (\alpha+\beta) \Wmat^{(1)\top} \Tmat  \\
      (\alpha+\beta) \Tmat \Wmat^{(1)} & -2\Tmat+\Wmat^{(2)\top} \Wmat^{(2)}
      \rightmatrix
      \leq 0,
  \end{align}
  holds for some $\Tmat \in \mathcal{T}_{n}$. Then $\norm{N(\xvec)-N(\yvec)}_2 \leq \sqrt{c} \norm{\xvec-\yvec}_2$ for all  $\xvec,\yvec \in \Rbb^n$.
\end{theorem}
\noindent
From \Cref{theorem:ch3-lipschite_semidefinite_programming}, the constant $c$ is an upper bound on the Lipschitz constant of the network.
The authors proposed to find the tightest bound by solving the following optimization problem (Semidefinite Program):
\begin{align}
  \textrm{minimize} \quad c \quad \text{ subject to} \quad \Mmat(c,\Tmat) \leq 0 \quad \text{and} \quad \Tmat \in \mathcal{T}_{n},
\end{align}
where the decision variables are $(c,\Tmat) \in \Rbb_+ \times \mathcal{T}_n$.
Note that $\Mmat(c,\Tmat)$ is linear in $c$ and $\Tmat$ and the set $\mathcal{T}_n$ is convex.
Although, these works on devising a global bound on the Lipschitz constant of a neural network are theoretically interesting, they lack scalability
They can only be computed on small networks and cannot be used during the training of large neural networks for regularization purposes.

\subsection{Lipschitz Constant of Individual Layers}
\label{subsection:ch3-lipschitz_constant_of_individual_layers}

\noindent
Instead of regularizing the ERM using the global Lipschitz constant, researchers have devised techniques to reduce the Lipschitz constant of \emph{individual layers} instead. 
The global Lipschitz of a neural network can easily be upper bounded by the product of the spectral norm of each weight matrix as follows:
\begin{proposition}[\citet{scaman2018lipschitz}] \label{proposition:ch3-naive_upper_bound_lipschitz}
  Let $N$ be a neural network of $\depth$ layers with 1-Lipschitz activation functions (\eg ReLU,
  Leaky ReLU, Tanh, Sigmoid, etc.), then, the Lipschitz constant of the neural network can be upper bounded as follows:
  \begin{equation} \label{equation:ch3-naive_upper_bound_lipschitz}
    \lip{N} \leq \prod_{i=1}^\depth \norm{\Wmat^{(i)}}_2 \enspace,
  \end{equation}
  where $\Wmat^{(i)}$ are the weights matrices of the neural network.
\end{proposition}

\begin{remark}
  The Lipschitz constant of a layer $\layer_{\Wmat, \bvec}^\rho$ (with a 1-Lipschitz activation function) is equal to the spectral norm of the matrix $\Wmat$ (largest singular value).
  Let $\layer_{\Wmat, \bvec}^\rho: \Rbb^n \rightarrow \Rbb^m$ such that $\layer_{\Wmat, \bvec}^\rho = \rho(\Wmat \xvec + \bvec)$ then by definition of the Lipschitz constant (see \Cref{definition:ch2-lipschitz_constant}) and of the operator norm, we have:
  \begin{equation}
    \lip{\layer_{\Wmat, \bvec}^\rho} = \sup_{\substack{\xvec \in \Rbb^n \\ \xvec \neq 0}} \frac{\norm{\Wmat \xvec}_2}{\norm{\xvec}_2} = \norm{\Wmat}_2
  \end{equation}
  \removespace
\end{remark}

The trivial bound given by the product of layer-wise Lipschitz constants in \Cref{equation:ch3-naive_upper_bound_lipschitz} is known to be loose and pessimistic.
Furthermore, we can show that reducing the Lipschitz constant of each layer independently does not imply that the global Lipschitz constant of the network will be reduced.

\begin{proposition} \label{proposition:ch3-limit_bound_lipschitz}
  Let $N$ be a neural network, then decreasing the Lipschitz constant of one or more layers does not imply reducing the Lipschitz constant of the network, \ie, $\lip{N}$.
\end{proposition}



\begin{proof}[\Cref{proposition:ch3-limit_bound_lipschitz}]
  Let us prove this claim with a counter-example.
  Let $N_1(\xvec) = \Amat^{(2)} \rho(\Amat^{(1)} \xvec)$ and $N_2(\xvec) = \Bmat^{(2)} \rho(\Bmat^{(1)} \xvec)$ where $\rho$ is the $\relu$ activation function.
  Let
  \begin{align*}
    \Amat^{(1)} &= \leftmatrix 
      \phantom{+}0 & -1 \\ -1 & \phantom{+}0
    \rightmatrix \quad
    \Amat^{(2)}  = \leftmatrix
      -1 & -1 \\ -1 & \phantom{+}0
    \rightmatrix \\
    \Bmat^{(1)} &= \leftmatrix
      \phantom{+}0 & \phantom{+}0 \\ \phantom{+}0 & -1
    \rightmatrix \quad
    \Bmat^{(2)} = \leftmatrix
      -1 & -1 \\ -1 & -1
    \rightmatrix
  \end{align*}
  then: \vspace{-0.5cm}
  \begin{equation*}
    \norm{\Amat^{(1)}}_2 = 1,\ \norm{\Amat^{(2)}}_2 = \sqrt{2}
    \quad \text{and} \quad
    \norm{\Bmat^{(1)}}_2 = 1,\ \norm{\Bmat^{(2)}}_2 = 2
  \end{equation*}
  From \Cref{theorem:ch3-lipschitz_differential_op} and the chain rule, the Lipschitz constant of the networks $N_1$ and $N_2$ can be expressed as follows:
  \begin{align*}
    \lip{N_1} &= \sup_{\xvec \in [0, 1]^2} \norm{\Amat^{(2)} \diag\left(\xvec\right) \Amat^{(1)}}_2 \\
    \lip{N_2} &= \sup_{\xvec \in [0, 1]^2} \norm{\Bmat^{(2)} \diag\left(\xvec\right) \Bmat^{(1)}}_2
  \end{align*}
  It is easy to verify that:
  \begin{equation*}
    \lip{N_1} = \frac{1 + \sqrt{5}}{2} \approx 1.618 \quad \text{and} \quad \lip{N_2} = \sqrt{2} \approx 1.414
  \end{equation*}
  which concludes the proof.
\end{proof}
\noindent
While we cannot have a guarantee that the global Lipschitz will be reduced, we could still have an idea of the value of the global Lipschitz with the upper bound presented in~\Cref{equation:ch3-naive_upper_bound_lipschitz}.

\citet{huster2018limitations} have demonstrated several limitations on the expressive power of neural networks where the product of layer-wise Lipschitz constants is constrained.
In the same vein, \citet{couellan2019coupling} empirically showed that Lipschitz Regularization offers a trade-off between adversarial robustness and expressivity of the network.
However, the bound in \Cref{equation:ch3-naive_upper_bound_lipschitz} appears in multiple generalization bound~\cite{neyshabur2017,bartlett2017spectrally,golowich2018} and adversarial generalization~\cite{farnia2018generalizable} (see \Cref{chapter:ch2-background}) which could suggest that reducing the bound would improve the generalization capabilities of neural networks and its robustness.

Based on this theoretical insight, researchers have developed several techniques to constrain the Lipschitz constant of each layer in order to improve the generalization and robustness of neural networks.
A technique to enforce 1-Lipschitz layers is to impose or promote an orthogonality constrain of the weight matrices.
A square orthogonal matrix $\Mmat$ is a matrix whose columns and rows are orthogonal unit vectors and all eigenvalues are equal to 1.
\citet{cisse2017parseval} and more recently \citet{wang2020orthogonal,huang2020controllable} have proposed to minimize the following term:
\begin{equation} \label{equation:ch3-orthogonality_constraint}
  \frac{\beta}{2} \norm{\Wmat^\top \Wmat - \Imat}_2  \enspace, 
\end{equation}
to promote the orthogonality constraint, in addition to the usual loss function:
In the above equation, the hyper-parameter $\beta$ controls the constraint.
A higher $\beta$ would lead to a better orthogonality constraint and therefore, a Lipschitz constant ``almost'' equal to 1 for all the layers.

On the other hand, \citet{anil2019sorting} proposed to enforce the orthogonality of weight matrices by directly optimizing on the Stiefel
manifold (\ie, the manifold of orthogonal matrices, see~\citet{absil2009optimization}).
To perform this optimization, they made use of an iterative algorithm first introduced by~\citet{bjorck1971iterative}.
For a given matrix $\Wmat = \Wmat^{(0)}$, the algorithm finds the closest orthonormal matrix by computing the following term:
\begin{equation}
  \Wmat^{(k+1)} = \Wmat^{(k)} \left( \Imat + \frac{1}{2} \Vmat^{(k)} + \cdots + (-1)^r {-\frac{1}{2} \choose r}  \left(\Vmat^{(k)}\right)^r \right)
\end{equation}
where $\Vmat = \Imat - \Wmat^{(k)\top} \Wmat^{(k)}$.
Although this algorithm works well on dense matrices, it can be difficult to apply it to convolutions. 
\citet{li2019preventing} built upon this idea and proposed an algorithm to enforce the orthogonality of convolution layers.
They used the orthogonal projection proposed by \citet{kautsky1994matrix} and \citet{xiao2018dynamical} to build convolutional neural networks with orthogonal convolutions.

All techniques that impose an orthogonality constraint on the weights matrices successfully reduce the Lipschitz constant of the layers of the networks.
Moreover, when the Lipschitz constants of all the layers are low, we could have an idea of the value of the global Lipschitz with the upper bound of~\Cref{equation:ch3-naive_upper_bound_lipschitz} (\ie, if Lipschitz constant of all the layers are equal to 1, then, the network will have a Lipschitz constant below 1).
However, enforcing the orthogonality constraint, either by regularizing with the term of~\Cref{equation:ch3-orthogonality_constraint} or by optimizing on the Stiefel manifold, is the costly operation which make it difficult to scale on large neural networks.

\begin{algorithm}[tb]
  \caption{Power method for producing the largest singular value, $\sigma_1$, of a non-square matrix, $\Wmat$ \cite{gouk2018regularisation,golub2000eigenvalue}}
  \begin{algorithmic}[1]
    \Require{affine function $f(\xvec) = \Wmat \xvec + \bvec$, number of iteration $N$}
    \Ensure{approximation of the Lipschitz constant $\lip{f}$}
    \State Randomly initialise $\xvec$
    \For{$i = 1$ \textbf{to} $N$}
      \State $\xvec \gets \Wmat^\top \Wmat \xvec / \norm{\xvec}_2$
    \EndFor
    \State \textbf{return} $\norm{\Wmat \xvec}_2 / \norm{\xvec}_2$
  \end{algorithmic}
  \label{algorithm:ch3-power_method}
\end{algorithm}

\begin{algorithm}[tb]
  \caption{Convolutional power method \cite{farnia2018generalizable}}
  \begin{algorithmic}[1]
    \Require{2d-convolution function $f: \Rbb^{n \times n} \rightarrow \Rbb^{m \times m}$ with kernel $k$, 2d-convolution-transpose function $g: \Rbb^{n \times n} \rightarrow \Rbb^{m \times m}$ with kernel $k$ number of iteration $N$}
    \Ensure{approximation of the Lipschitz constant $\lip{f}$}
    \State Initialize $\xvec$ with a random vector matching the shape of the convolution input
    \For{$i = 1$ \textbf{to} $N$}
      \State $\xvec \gets f(\xvec) / \norm{f(\xvec)}_2 $
      \State $\xvec \gets g(\xvec) / \norm{g(\xvec)}_2$
    \EndFor
    \State \textbf{return} $\norm{f(\xvec)}_2 / \norm{\xvec}_2$
  \end{algorithmic}
  \label{algorithm:ch3-power_method_generic}
\end{algorithm}

Another technique, called \emph{Spectral Normalization}, consists in normalizing each weight matrix by its largest singular value, thus imposing each layer to be 1-Lipschitz.
As with the orthogonality constraint, this technique leads the network to have a global Lipschitz constant at most 1.
\citet{yoshida2017spectral} were the first to propose this method to improve the generalization of neural networks followed by~\cite{miyato2018spectral,gouk2018regularisation,farnia2018generalizable} for improving generalization and robustness against adversarial attacks.
In order to perform spectral normalization, they divided the values of each weight matrix by an approximation of its largest singular value.
The approximation of the largest singular was computed using the power method~\cite{golub2000eigenvalue}.


The power method is an iterative eigenvalue algorithm (also known as the Von Mises iteration \cite{mises1929praktische}).
Given a matrix $\Wmat$ and a random vector $\bvec^{(0)}$, the eigenvector associated with the largest eigenvalue of the matrix $\Wmat$ can be computed with the following recurrence relation:
\begin{equation}
  \bvec^{(k+1)} = \frac{\Wmat \bvec^{(k)}}{\norm{\bvec^{(k)}}_2}  
\end{equation}
Then, the largest eigenvalue (when we talk about ``largest eigenvalue'' we mean in absolute value) can be obtained with the \emph{Rayleigh quotient}:
\begin{equation}
  \sigma_1\left( \Wmat \right) = \frac{\bvec^{(k)\top} \Wmat \bvec^{(k)}}{\bvec^{(k)\top} \bvec^{(k)}}
\end{equation}
With a sufficient number of iterations, the algorithm provably converges to the largest eigenvalue of the matrix.
To find the largest singular value, we can leverage the relation between eigenvalues and singular values:
\begin{equation}
  \sigma \left( \Wmat \right) = \sqrt{ \lambda \left( \Wmat^\top \Wmat \right) }
\end{equation}
The rate of convergence of the algorithm depends on the ratio between the second-largest eigenvalue and the largest eigenvalue.
Indeed, a ratio close to one can lead to slow convergence.
The pseudocode of the power method is given in \Cref{algorithm:ch3-power_method}.
Altough, \Cref{algorithm:ch3-power_method} needs explicit matrix for computing the largest singular value, \citet{farnia2018generalizable,ryu2019plug} extended the power method to convolution layers where the matrix $\Wmat$ is not explicitly constructed.
The pseudocode of their method is presented in \Cref{algorithm:ch3-power_method_generic}. 

In the context of deep learning and spectral normalization, the largest singular value needs to be computed for each layer of the network at each step of the training. 
Given that current state-of-the-art architecture have between 50 and 100 layers \cite{he2016deep,tan2019efficientnet}, using the power method \emph{until convergence} is prohibitive.
In~\Cref{chapter:ch5-lipschitz_bound}, we propose a new regularization scheme for reducing the Lipschitz constant of individual layers.
We will shown in~\Cref{subsection:ch5-comparison_of_lipbound_with_other_state-of-the-art_approaches} that our approach is more efficient that the power method even with a small number of iterations.

\subsection{Singular Values of Convolutional Layers}
\label{subsection:ch3-singular_values_of_convolutional_layers}

The power method is not the only technique available for approximating the largest singular value (Lipschitz constant) of a convolution layer.
Several works have devised bounds or approximations on the largest singular value of convolution layers by exploiting the \emph{structure} of the convolution operation \cite{sedghi2018singular,bibi2019deep,singla2019bounding,jia2017improving}.

%


To approximate the singular values of a convolution layer, \citet{sedghi2018singular} have exploited the properties of doubly-block circulant matrices (\ie, a circulant block matrix where each block is also a circulant matrix).
Indeed, a doubly-block circulant matrix is the matrix representation of a convolution with circulant padding.
In their work, \citet{sedghi2018singular} assume that the properties of doubly-block circulant matrices are `close' to the properties of a doubly-block Toeplitz matrix.

To compute the singular values of doubly-block circulant matrices, \citet{sedghi2018singular} have demonstrated the following result:
\begin{theorem}[Theorem 5 from \citet{sedghi2018singular}] \label{theorem:ch3-singular_values_doubly_block_circulant}
  Let $\Amat$ be a doubly-block circulant matrix such that:
  \begin{equation*}
    \Amat = \leftmatrix
      \Cmatsf^{(0)}   & \Cmatsf^{(n-1)} & \Cmatsf^{(n-2)} & \cdots        & \cdots          & \Cmatsf^{(1)}   \\
      \Cmatsf^{(1)}   & \Cmatsf^{(0)}   & \Cmatsf^{(n-1)} & \ddots        &                 & \vdots          \\
      \Cmatsf^{(2)}   & \Cmatsf^{(1)}   & \ddots          & \ddots        & \ddots          & \vdots          \\
      \vdots          & \ddots          & \ddots          & \ddots        & \Cmatsf^{(n-1)} & \Cmatsf^{(n-2)} \\
      \vdots          &                 & \ddots          & \Cmatsf^{(1)} & \Cmatsf^{(0)}   & \Cmatsf^{(n-1)} \\
      \Cmatsf^{(n-1)} & \cdots          & \cdots          & \Cmatsf^{(2)} & \Cmatsf^{(1)}   & \Cmatsf^{(0)}
    \rightmatrix
  \end{equation*}
  where $\Cmatsf^{(i)} = \circulant({\cvec_i}),\ \forall i \in \Iset_n^+$.
  Let $\Kmat = \leftmat \cvec_0, \cvec_1, \cdots, \cvec_{n-1} \rightmat^\top$ then, the singular values of the doubly-block circulant matrix $\Amat$ are the modulus of the entries of $\Umat_n^\top \Kmat \Umat_n$.
\end{theorem}

\noindent
To prove \Cref{theorem:ch3-singular_values_doubly_block_circulant}, \citet{sedghi2018singular} used the diagonalization of doubly-block circulant matrices (see~\Cref{chapter:ch2-background}, \Cref{equation:ch2-diagonalization_doubly_block_circulant_matrix}).
The main advantage of this approach is that the singular values of a doubly-block circulant matrix can be computed with the Fast Fourier Transform algorithm (see \Cref{section:ch2-a_primer_on_circulant_and_toeplitz_matrices}) which offers a reduced complexity compared to classical approaches for computing the singular values of a matrix.
However, this approach exhibits several limitations.
First, this method results in a loose approximation of the maximal singular value of a convolution layer which does not use the circulant padding, which is often the case in practical settings.
Also, the complexity of their algorithm is dependent on the size of the input which can be high for large datasets.
Finally, for multi-channel convolution, their method requires the computation of the spectral norm of $n^2$ matrices each of size $\cin \times \cout$ as stated in the following theorem:

\begin{theorem}[Theorem 6 from \citet{sedghi2018singular}]
  Let $\Mmat$ be the matrix encoding the linear transform computed by a multi-channel convolution layer.
  Let $\Kmat \in \Rbb^{\cin \times \cout \times n \times n}$ such that $(\Kmat)_{i,j}$ for all $i,j \in [n]^2$ be constructed as in \Cref{theorem:ch3-singular_values_doubly_block_circulant}, 
  Let $\widetilde{\Kmat}_{i,j} = \Umat_n^\top \leftmat \Kmat \rightmat_{i,j} \Umat_n $ and define the following operator matrix 
  \begin{equation}
    \Pmat(i,j) = \leftmatrix 
    \leftmat (\widetilde{\Kmat})_{(0,0)} \rightmat_{i, j} & \cdots & \leftmat (\widetilde{\Kmat})_{(0, \cout-1)} \rightmat_{i, j} \\
    \vdots & & \vdots \\
    \leftmat (\widetilde{\Kmat})_{(\cin-1, 0)} \rightmat_{i, j} & \cdots & \leftmat (\widetilde{\Kmat})_{(\cin-1, \cout-1)} \rightmat_{i, j}
    \rightmatrix
  \end{equation}
  Then
  \begin{equation}
    \sigma(\Mmat) = \bigcup_{i, j = 0}^{n-1} \sigma \left(  \Pmat(i,j) \right).
  \end{equation}
  \removespace
\end{theorem}


In the same vein, \citet{singla2019bounding} have used the properties of convolutions to devise several bounds on the singular values of convolution layers.
Recall from \Cref{subsubsection:ch2-relation_with_the_convolution_operator} that a convolution kernel is a 4 dimensional tensor of size $\cout \times \cin \times k_1 \times k_2$.
\citet{singla2019bounding} have demonstrated that the largest singular value of a convolution layer $\layer_\Kmat$ parameterized by a kernel $\Kmat$ can be upper-bounded as follows:

\begin{theorem}[Reformulation of Theorem 1 from \citet{singla2019bounding}]
  Let $\Kmat \in \Rbb^{\cout \times \cin \times k_1 \times k_2}$ be the kernel of a convolution layer $\layer_\Kmat$, then,
  \begin{equation}
    \lip{\layer_\Kmat} \leq \min \left\{ \sqrt{k_1 k_2} \norm{\Rmat}_2, \sqrt{k_2 k_2} \norm{\Smat}_2 \right\}
  \end{equation}
  where $\Rmat$ and $\Smat$ are matrices of size $k_1 \cout \times k_2 \cin$ and $k_2 \cout \times k_1 \cin$ defined as follows:
  \begin{align}
    \Rmat &= \leftmatrix
      (\Kmat)_{0,0}       & \cdots & (\Kmat)_{0,\cin-1} \\
      (\Kmat)_{1,0}       & \cdots & (\Kmat)_{1,\cin-1} \\
      \vdots              & \ddots & \vdots             \\
      (\Kmat)_{\cout-1,0} & \cdots & (\Kmat)_{\cout-1,\cin-1}
    \rightmatrix \\[0.5cm]
    \Smat &= \leftmatrix
      (\Kmat)_{0,0}^\top       & \cdots & (\Kmat)_{0,\cin-1}^\top \\
      (\Kmat)_{1,0}^\top       & \cdots & (\Kmat)_{1,\cin-1}^\top \\
      \vdots                   & \ddots & \vdots                  \\
      (\Kmat)_{\cout-1,0}^\top & \cdots & (\Kmat)_{\cin-1,\cout-1}^\top
    \rightmatrix
  \end{align}
  \removespace
\end{theorem}

In order to prove this result, \citet{singla2019bounding} built upon the work of \citet{sedghi2018singular} and have also only considered circulant convolutions (performed by doubly-block circulant matrices).
Instead of proposing a method to compute \emph{all} singular values of the equivalent doubly-block circulant matrix, their method is an upper-bound on the largest singular value of the Jacobian of the convolution. 
Because this method is independent of the input dimension, the computational complexity is substantially reduced compared to the approach of \citet{sedghi2018singular}, however, the reduction in computational complexity is at the expense of accuracy as we will show in~\Cref{chapter:ch5-lipschitz_bound}.

%% file: sources/main/ch4-diagonal_circulant.tex
\chapter{Diagonal and Circulant Matrices for Compact Neural Networks}
\label{chapter:ch4-diagonal_circulant_neural_network}
\localtoc


\section{Introduction}
\label{chapter:ch4-introduction}


As seen in the previous chapters, structured matrices are at the very core of most of the work on compact networks.
Despite substantial efforts  (\eg, \citet{cheng2015exploration,moczulski2016acdc}), the performance of compact models is still far from achieving an acceptable accuracy motivating their use in real-world scenarios.
This raises several questions about the effectiveness of such models and about our ability to train them.
In particular two main questions call for investigation:
\begin{enumerate}
    \item \emph{What is the expressive power of structured layers compared to dense layers?}
    \item \emph{How to efficiently train deep neural networks with a large number of structured layers?}
\end{enumerate}
We aim at answering these questions by studying deep diagonal-circulant neural networks (\aka DCNNs), which are deep neural networks in which weight matrices are the product of diagonal and circulant ones.

To answer the first question, we propose an analysis of the expressivity of DCNNs by extending the results obtained by~\citet{huhtanen2015factoring} which states that any matrix can be decomposed into the product of $2n-1$ alternating diagonal and circulant matrices.
We introduce a new bound on the number of diagonal-circulant products required to approximate a matrix that depends on its rank.
Building on this result, we demonstrate that a DCNN with bounded width and small depth can approximate any dense neural networks with ReLU activations. 

To answer the second question, we first describe a theoretically sound initialization procedure for DCNN which allows the signal to propagate through the network without vanishing or exploding.
Furthermore, we provide a number of empirical insights to explain the behavior of DCNNs and show the impact on the number of nonlinearities in the network on the convergence rate and the accuracy of the network. 
By combining all these insights, we are able (for the first time) to train large and deep DCNNs and demonstrate the good performance of these networks on a large-scale application (the \yt video classification problem) and obtain very competitive accuracy. 

The chapter is organized as follows:
\Cref{section:ch4-diagonal_and_circulant_matrices_for_matrix_decomposition} introduces our new result extending the one from \citet{huhtanen2015factoring}.
\Cref{section:ch4-analysis_of_diagonal_circulant_neural_networks} proposes a theoretical analysis on the expressivity of DCNNs.
\Cref{section:ch4-how_to_train_deep_diagonal_circulant_neural_networks} describes two efficient techniques for training deep diagonal-circulant neural networks.
\Cref{section:ch4-experiments} presents extensive experiments to compare the performance of deep diagonal-circulant neural networks in different settings with respect to other state-of-the-art approaches.
Finally, \Cref{section:ch4-concluding_remarks} provides concluding remarks.

\section{Diagonal and Circulant Matrices for Matrix Decomposition}
\label{section:ch4-diagonal_and_circulant_matrices_for_matrix_decomposition}

As seen in the Background (\Cref{chapter:ch2-background}), circulant matrices exhibit several interesting properties from the perspective of numerical computations.
Most importantly, any $n \times n$ circulant matrix $\Cmat$ can be represented using only $n$ coefficients instead of the $n^2$ coefficients required to represent classical unstructured matrices.
In addition, the matrix-vector product is simplified from $\bigO(n^2)$ to $\bigO(n \log n)$ using the  convolution theorem.
As we will show in this chapter, circulant matrices combined with diagonal matrices can have a strong expressive power.
They can also be used as building blocks to represent any linear transform~\cite{schmid2000decomposing,huhtanen2015factoring} with an arbitrary precision.

We are interested in the relation between the product of diagonal and circulant matrices and the expressivity of low-rank matrices.
\citet{huhtanen2015factoring} were able to bound the number of factors that is required to approximate any matrix $\Amat$ with arbitrary precision.
We recall this result in \Cref{theorem:ch4-huhtanen} as it is the starting point of our theoretical analysis.

\begin{theorem}[Reformulation from \citet{huhtanen2015factoring}] \label{theorem:ch4-huhtanen}
  For every matrix $\Mmat \in \Cbb^{n \times n}$, there exists a sequence of matrices $\Amat^{(1)} \ldots \Amat^{(2n-1)}$ where $\Amat^{(i)}$ is a circulant matrix if $i$ is odd, and a diagonal matrix otherwise, such that for any $\epsilon > 0$, we have
  \begin{equation}
    \norm{\Amat^{(1)} \ldots \Amat^{(2n-1)} - \Mmat}_\fro < \epsilon \enspace.
  \end{equation}
  \removespace
\end{theorem}

Unfortunately, this theorem is of little use to understand the expressive power of diagonal-circulant matrices when they are used in deep neural networks for two reasons:
\begin{enumerate}
  \item the bound only depends on the dimension of the matrix $\Mmat$, not on the matrix itself;
  \item the theorem does not provide any insights regarding the expressive power of $m$ diagonal-circulant factors when $m$ is much lower than $2n - 1$ as it is the case in most practical scenarios we consider in this chapter. 
\end{enumerate}

In the following theorem, we improve the result of~\citet{huhtanen2015factoring} by expressing the number of factors required to approximate $\Mmat$, \emph{as a function of the rank of $\Mmat$}.
This is useful when one deals with low-rank matrices, which is common in machine learning problems. 
Note that in this chapter, our results hold for complex matrices.
This is due to the fact that they are based on \Cref{theorem:ch4-huhtanen} which holds for complex matrices.

\begin{maintheorem}[Rank-based diagonal-circulant decomposition] \label{theorem:ch4-rank-decomposition}
  Let $\Mmat \in \Cbb^{n \times n}$ be a matrix of rank at most $k$.
  Assume that $n$ can be divided by $k$.
  There exists a sequence of $4k+1$ matrices $\Amat^{(1)} \ldots \Amat^{(4k+1)}$, where $\Amat^{(i)} \in \Cbb^{n \times n}$ is a circulant matrix if $i$ is odd, and a diagonal matrix otherwise, such that for any $\epsilon > 0$, we have
  \begin{equation}
    \norm{\Amat^{(1)} \ldots \Amat^{(4k+1)} - \Mmat}_{\fro} < \epsilon \enspace.
  \end{equation}
  \removespace
\end{maintheorem}

\begin{proof}[\Cref{theorem:ch4-rank-decomposition}]
Let $\Umat \mathbf{\Sigma} \Vmat^*$ be the SVD decomposition of $\Mmat$ where $\Umat,\Vmat$ and $\mathbf{\Sigma}$ are $n \times n$ matrices.
Because $\Mmat$ is of rank $k$, the last $n-k$ columns of $\Umat$ and $\Vmat$ are null.
In the following, we will first decompose $\Umat$ into a product of matrices $\Wmat\Rmat\Omat$, where $\Rmat$ and $\Omat$ are respectively circulant and diagonal matrices, and $\Wmat$ is a matrix which will be further decomposed into a product of diagonal and circulant matrices.
Then, we will apply the same decomposition technique to $\Vmat$.
Ultimately, we will get a product of $4k+1$ matrices alternatively diagonal and circulant.

Let $\Rmat = \circulant(r_1\ldots r_n)$. Let $\Omat$ be an $n \times n$ diagonal matrix where $\leftmat \Omat \rightmat_{i,i} = 1$ if $i \le k$ and $0$ otherwise.
The $k$ first columns of the product $\Rmat\Omat$ will be equal to that of $\Rmat$, and the $n-k$ last columns of $\Rmat\Omat$ will be zeros. For example, if $k=2$, we have: 
\begin{equation}
  \Rmat\Omat = \leftmatrix
  r_{1} & r_{n} & 0 & \cdots & 0\\
  r_{2} & r_{1}\\
  r_{3} & r_{2} & \vdots &  & \vdots\\
  \vdots & \vdots\\
  r_{n} & r_{n-1} & 0 & \cdots & 0
  \rightmatrix
\end{equation}

\noindent
Let us define $k$ diagonal matrices $\Dmat^{(i)} = \diagonal(d_1^{(i)} \ldots d_n^{(i)})$ for $i \in [k]$.
For now, the values of $d_{j}^{(i)}$ are unknown, but we will show how to compute them.
Let $\Zmat_{1, k}$ and $\Zmat_{1, n}$ be $1$-unit-circulant matrix respectively of size $k \times k$ and $n \times n$ as defined in~\Cref{section:ch2-a_primer_on_circulant_and_toeplitz_matrices} and let $\Wmat = \sum_{i=1}^{k} \Dmat^{(i)} \Zmat_{1,n}^{i-1}$.
Note that the $n-k$ last columns of the product $\Wmat\Rmat\Omat$ will be zeros.
For example, with $k=2$, we have: 

\begin{equation}
  \Wmat = \leftmatrix
  d_{1}^{(1)} &  &  &  & d_{1}^{(2)} \\
  d_{2}^{(2)} & d_{2}^{(1)} \\
   & d_{3}^{(2)} & \ddots \\
   &  & \ddots & \ddots \\
   &  &  & d_{n}^{(2)} & d_{n}^{(1)}
  \rightmatrix
\end{equation}

\begin{equation}
  \Wmat\Rmat\Omat = \leftmatrix
  r_{1} d_{1}^{(1)} + r_{n}   d_{1}^{(2)} & r_{n}   d_{1}^{(1)} + r_{n-1} d_{1}^{(2)} & 0 & \cdots & 0 \\
  r_{2} d_{2}^{(1)} + r_{1}   d_{2}^{(2)} & r_{1}   d_{2}^{(1)} + r_{n}   d_{2}^{(2)} & 0 & \cdots & 0 \\
  \vdots & \vdots & \vdots &  & \vdots \\
  r_{n} d_{n}^{(1)} + r_{n-1} d_{n}^{(2)} & r_{n-1} d_{n}^{(1)} + r_{n-2} d_{n}^{(2)} & 0 & \cdots & 0
  \rightmatrix
\end{equation}

\noindent
We want to find the values of $d_{j}^{(i)}$ such that $\Wmat \Rmat \Omat = \Umat$.
We can formulate this as a linear equation system.
In case $k=2$, we get:

\begin{equation}
  \leftmatrix
  r_{n} & r_{1}\\
  r_{n-1} & r_{n}\\
   &  & r_{1} & r_{2}\\
   &  & r_{n} & r_{1}\\
   &  &  &  & r_{2} & r_{3}\\
   &  &  &  & r_{1} & r_{2}\\
   &  &  &  &  &  & \ddots\\
   &  &  &  &  &  &  & \ddots
  \rightmatrix \times \leftmatrix
  d_{1}^{(2)} \\
  d_{1}^{(1)} \\
  d_{2}^{(2)} \\
  d_{2}^{(1)} \\
  d_{3}^{(2)} \\
  d_{3}^{(1)} \\
  \vdots\\
  \vdots
  \rightmatrix = \leftmatrix
  (\Umat)_{0,0} \\
  (\Umat)_{0,1} \\
  (\Umat)_{1,0} \\
  (\Umat)_{1,1} \\
  \\
  \\
  \vdots\\
  \\
  \rightmatrix
\end{equation}

\noindent
The $i^{th}$ block of this block-diagonal matrix is a Toeplitz matrix induced by a contiguous subsequence of length $k+1$ of $(r_1,\ldots r_n,r_1 \ldots r_n)$.
Set $r_{j}=1$ for all $j\in\{k,2k,3k,\ldots n\}$ and set $r_{j}=0$ for all other values of $j$.
Then it is easy to see that each block is equal to $\Zmat^\alpha_{1,k}$ for some $\alpha$.
Note that the matrices $\Zmat^\alpha_{1,k}$ are invertible.
This entails that the block diagonal matrix above is also invertible.
So by solving this set of linear equations, we could find $d_1^{(1)} \ldots d_n^{(k)}$ such that $\Wmat\Rmat\Omat=\Umat$.
We can apply the same idea to factorize $\Vmat = \Wmat' \Rmat \Omat$ for some matrix $\Wmat'$.
Finally, we get 
\begin{equation}
  \Amat = \Umat \mathbf{\Sigma} \Vmat^* = \Wmat\Rmat\Omat \mathbf{\Sigma} \Omat^* \Rmat^* \Wmat^{'*}
\end{equation}

\noindent
Note that the matrix $\Wmat\Rmat$ can be decomposed as follows:
\begin{equation}
  \Wmat\Rmat = \left( \sum_{i=1}^{k} \Dmat^{(i)} \Zmat_{1,n}^{(i-1)} \right) \Rmat
\end{equation}
where the last matrix $\Zmat_{1,n}^{(k-1)} \Rmat$ is a circulant matrix because both matrices are circulant (see \Cref{theorem:ch2-diagonalization_circulant_matrix}).
The same reasoning can be applied with the matrix $\Rmat^*\Wmat'^*$.
Therefore, by construction, the matrices $\Wmat\Rmat$ and $\Rmat^*\Wmat'^*$ can both be factorized by $2k$ circulant and diagonal matrices.
Also note that $\Omat \mathbf{\Sigma} \Omat^*$ is a diagonal matrix, because $\Omat$ and $\mathbf{\Sigma}$ are diagonal matrices.
Overall, $\Amat$ can be represented with a product of $4k+1$ matrices, alternatively diagonal and circulant.
\end{proof}

A direct consequence of \Cref{theorem:ch4-rank-decomposition}, is that if the number of diagonal-circulant factors is set to a value $K$, we can represent all linear transforms $\Mmat$ whose rank is $\frac{K - 1}{4}$.
Compared to \citet{huhtanen2015factoring}, this result shows that structured matrices with fewer than $2n$ diagonal-circulant matrices (as it is the case in practice) can still represent a large class of matrices.

In the following section, we will analyze the expressivity of neural networks based on diagonal and circulant matrices.
In order to characterize the expressivity, we will decompose the matrices of a dense neural network with diagonal and circulant matrices based on \Cref{theorem:ch4-rank-decomposition}.


 \section{Analysis of Diagonal Circulant Neural Networks}
\label{section:ch4-analysis_of_diagonal_circulant_neural_networks}

\citet{zhao2017theoretical} have shown that circulant networks with 2 layers and unbounded width are universal approximators.
However, results on unbounded networks offer weak guarantees and two important questions have remained open until now: 
\begin{enumerate}
  \item \emph{Can we approximate any function with a bounded-width diagonal-circulant network?}
  \item \emph{What function can we approximate with a diagonal-circulant neural network that has a bounded width and a small depth?}
\end{enumerate}
We answer these two questions in this section.
First, we present two lemmas that establish a link between the matrix decomposition presented in \Cref{theorem:ch4-rank-decomposition} and DCNNs and allow us to present our answer to the first question (\Cref{corollary:ch4-universal_approximation}).
Then, we analyze the expressive power of small depth diagonal-circulant neural networks by comparing them to dense neural networks. 
As in the previous section, we still work in the complex domain.
Therefore, we need to extend the definition of neural networks to the complex domain.
First, let us introduce an extension of the \relu function.

\begin{definition}[Complex ReLU function \citet{trabelsi2018deep}] \label{definition:relu_function}
  Let us define the complex \relu function $\rho: \Cbb^n \rightarrow \Cbb^n$ by: $\rho(\zvec) = \max\left(0, \mathfrak{R}(\zvec)\right) + \ci \max\left(0, \mathfrak{I}(\zvec) \right)$
\end{definition}

\begin{definition}[Dense Neural Network]
  Given a depth $\depth \in \Nbb$,
  let us define $\Omega = \left\{ \left( \Wmat^{(i)}, \bvec^{(i)} \right) \right\}_{i \in [\depth]}$ a set of weights matrices and bias vectors 
  such that $\Wmat^{(i)} \in \Cbb^{n \times n}$ and $\bvec^{(i)} \in \Cbb^n$. 
  Let $\Xset \subset \Cbb^n$ and $\Yset \subset \Cbb^n$ be the input space and output space respectively. 
  A dense neural network is a function $\nn_\Omega: \Xset \rightarrow \Yset$ such that
  \begin{equation}
    \nn^\act_\Omega (\xvec) \triangleq \layer^\act_{\Wmat^{(\depth)}, \bvec^{(\depth)}} \circ \cdots \circ \layer^\act_{\Wmat^{(1)}, \bvec^{(1)}}(\xvec)
  \end{equation}
  where $\rho$ is the complex \relu function, $\layer^{\act}_{\Wmat^{(i)},\bvec^{(i)}}: \Cbb^n \rightarrow \Cbb^n$ is a layer parameterized by the weight matrix $\Wmat^{(i)}$ and the bias vector $\bvec^{(i)}$, which can be expressed as follows: 
  \begin{equation}
    \layer^{\act}_{\Wmat^{(i)},\bvec^{(i)}} (\xvec) \triangleq \act \left(\Wmat^{(i)}\xvec + \bvec^{(i)}\right) \enspace,
  \end{equation}
  \removespace
\end{definition}

\begin{definition}[Diagonal-Circulant Neural Network]
  Given a depth $\depth \in \Nbb$,
  let us define $\Pi = \left\{ \left( \Dmat^{(i)}, \Cmat^{(i)}, \bvec^{(i)} \right) \right\}_{i \in [\depth]}$ a set of weight matrices and bias vectors 
  such that $\Dmat^{(i)} \in \Cbb^{n \times n}$ is diagonal, $\Cmat^{(i)} \in \Cbb^{n \times n}$ is circulant and $\bvec^{(i)} \in \Cbb^n$. 
  Let $\Xset \subset \Cbb^n$ and $\Yset \subset \Cbb^n$ be the input space and output space respectively. 
  Let us denote the product of $\Dmat^{(i)}$ and $\Cmat^{(i)}$ by $\Dmat\Cmat^{(i)}$.
  A diagonal-circulant neural network is a function $\nn_\Pi: \Xset \rightarrow \Yset$ such that
  \begin{equation}
    \nn^\act_\Pi (\xvec) \triangleq \layer^\act_{\Dmat\Cmat^{(\depth)}, \bvec^{(\depth)}} \circ \cdots \circ \layer^\act_{\Dmat\Cmat^{(1)}, \bvec^{(1)}}(\xvec)
  \end{equation}
  where $\layer^{\act}_{\Dmat\Cmat^{(i)},\bvec^{(i)}}: \Cbb^n \rightarrow \Cbb^n$ is a layer parameterized by the weight matrix $\Dmat\Cmat^{(i)}$, the bias vector $\bvec^{(i)}$ and can be expressed as follows: 
  \begin{equation}
    \layer^{\act}_{\Dmat\Cmat^{(i)},\bvec^{(i)}} (\xvec) \triangleq \act \left(\Dmat\Cmat^{(i)}\xvec + \bvec^{(i)}\right) \enspace,
  \end{equation}
  where $\rho$ is the complex \relu function.
\end{definition}

Diagonal-circulant neural networks are compact due to the layer being parameterized by diagonal and circulant matrices.
Indeed, diagonal and circulant matrices of size  $n \times n$ can be represented with only $n$ values.
Therefore, the layer $\layer^{\act}_{\Dmat\Cmat^{(i)},\bvec^{(i)}}$ is parameterized by $3n$ complex values.

Diagonal-circulant neural networks can have more parameters than a dense neural networks but their depth need to be scaled accordingly.
Let $\depth_1$ and $\depth_2$ be the depth of a dense neural network and a diagonal-circulant neural network respectively, then $\depth_2$ needs to be higher than $\depth_1 \frac{n+1}{3}$ to have more parameters than the dense network.

\subsection{From Matrix Decomposition to Neural Networks}
\label{subsection:ch4-from_matrix_decomposition_to_neural_networks}

The purpose of this section is to extend the matrix decomposition presented in \Cref{theorem:ch4-rank-decomposition} to neural networks (\Cref{lemma:ch4-product_of_mat_to_DNN}) and show that bounded-width diagonal-circulant neural networks can approximate any dense neural network (\Cref{lemma:ch4-dcnn_approx_neural_network}).

\begin{lemma} \label{lemma:ch4-product_of_mat_to_DNN}
  Let $\Wmat^{(1)} \ldots \Wmat^{(\depth)} \in \Cbb^{n\times n}$, $\bvec \in \Cbb^n$ and let $\Xset \subset \Cbb^n$ be a bounded set.
  There exists $\cvec^{(1)} \ldots \cvec^{(\depth)} \in \Cbb^n$ such that for all $\xvec \in \Xset$ we have 
  \begin{equation}
    \rho\left(\Wmat^{(\depth)} \ldots \Wmat^{(1)} \xvec + \bvec \right) = \layer^\rho_{\Wmat^{(\depth)},\cvec^{(\depth)}} \circ \ldots \circ \layer^\rho_{\Wmat^{(1)},\cvec^{(1)}}(\xvec)
  \end{equation}
  where $\layer^\rho_{\Wmat^{(i)},\cvec^{(i)}} = \rho( \Wmat^{(i)} \xvec + \cvec^{(i)} )$ and $\rho$ is the complex \relu function.
\end{lemma}

\begin{proof}[\Cref{lemma:ch4-product_of_mat_to_DNN}]
  Let $\pmb{\mathcal{W}}(j) = \prod_{k=1}^{j} \Wmat^{(k)}$ and let us define the following set:
  \begin{equation}
    \mathcal{S} = \left\{ \left(\pmb{\mathcal{W}}(j) \xvec \right)_{t} \mid \ \xvec \in \Xset, t \in [n], j \in [\depth] \right\}
  \end{equation}
  and let $\Xi = \max\left\{ |\mathfrak{R}(v)|: v \in \Sset \right\} + \ci \max\left\{ |\mathfrak{I}(v)|:v \in \Sset \right\}$.
  Intuitively, the real and imaginary parts of $\Xi$ are the largest any activation in the network can have.
  Define $\psi_{\Wmat^{(i)}, \cvec^{(i)}}(\xvec) = \Wmat^{(i)} \xvec + \cvec^{(i)}$. Let $\cvec^{(1)} = \Xi \onevec{n}$.
  Clearly, for all $\xvec \in \Xset$ we have $\psi_{\Wmat^{(1)}, \cvec^{(1)}}(\xvec) \ge 0$, so $\act(\psi_{\Wmat^{(1)}, \cvec^{(1)}}(\xvec)) = \psi_{\Wmat^{(1)}, \cvec^{(1)}}(\xvec)$ where $\act$ is the complex \relu function.
  \noindent
  More generally, for all $j < \depth-1$ define $\cvec^{(j+1)} = \onevec{n} \Xi - \Wmat^{(j+1)} \cvec^{(j)}$.
  It is easy to see that for all $j < \depth$ we have 
  \begin{equation}
    \psi_{\Wmat^{(j)}, \cvec^{(j)}}\circ \ldots \circ \psi_{\Wmat^{(1)}, \cvec^{(1)}}(\xvec) = \pmb{\mathcal{W}}(j) \xvec + \onevec{n} \Xi \enspace.
  \end{equation}
  This guarantees that for all $j < \depth$,
  \begin{equation}
    \psi_{\Wmat^{(j)}, \cvec^{(j)}} \circ \ldots \circ \psi_{\Wmat^{(1)}, \cvec^{(1)}}(\xvec) = \act \circ \psi_{\Wmat^{(j)}, \cvec^{(j)}} \circ \ldots \circ \act \circ \psi_{\Wmat^{(1)}, \cvec^{(1)}}(\xvec) \enspace.
  \end{equation}
  Finally, define $\cvec^{(\depth)} = \bvec - \Wmat^{(\depth)} \cvec^{(\depth-1)}$.
  We have,
  \begin{equation}
    \act \circ \psi_{\Wmat^{(\depth)}, \cvec^{(\depth)}} \circ \ldots \circ \act \circ \psi_{\Wmat^{(1)}, \cvec^{(1)}}(\xvec) = \act \big( \pmb{\mathcal{W}}(\depth) \xvec + \bvec \big) \enspace,
  \end{equation}
  which concludes the proof.
\end{proof}
  
The following lemma is our first result on the expressivity of diagonal-circulant neural networks.
It states that a diagonal-circulant neural network with bounded width and depth can approximate any dense neural network.
To prove this result, we use the matrix decomposition from \Cref{theorem:ch4-huhtanen} and \Cref{lemma:ch4-product_of_mat_to_DNN} to decompose the dense matrices of the layers of a dense network and unfold it.

\begin{lemma} \label{lemma:ch4-dcnn_approx_neural_network}
  Let $\nn_\Omega$ be a dense neural network of width $n$ and depth $\depth$,
  and let $\Xset \subset \Cbb^n$ be a bounded set.
  There exists a diagonal-circulant neural network $\nn_\Pi$ of width $n$ and of depth $(2n-1)\depth$ such that for any $\epsilon > 0$, we have
  \begin{equation}
    \norm{\nn_\Omega(\xvec) - \nn_\Pi(\xvec)}_2 < \epsilon, \quad \forall \xvec \in \Xset \enspace.
  \end{equation}
  \removespace
\end{lemma}

\begin{proof}[\Cref{lemma:ch4-dcnn_approx_neural_network}]
  Let us assume $\nn_\Omega = \layer_{\Wmat^{(\depth)}, \bvec^{(\depth)}} \circ \ldots \circ \layer_{\Wmat^{(1)}, \bvec^{(1)}}$.
  By \Cref{theorem:ch4-huhtanen}, for any $\epsilon' > 0$, any matrix $\Wmat^{(i)}$, there exists a sequence of $2n-1$ diagonal, $\{\Dmat^{(i, j)}\}_{i \in [\depth], j \in [2n-1]}$, and circulant matrices, $\{\Cmat^{(i, j)}\}_{i, \in [\depth], j \in [2n-1]}$, such that for all $i \in [\depth]$,
  \begin{equation}
    \norm{\prod_{j=1}^{2n-1} \Dmat^{(i,2n-j)} \Cmat^{(i,2n-j)} - \Wmat^{(i)}}_\fro < \epsilon' \enspace.
  \end{equation}
  For simplicity, let us denote the product of the two matrices $\Dmat^{(i,j)} \Cmat^{(i,j)}$ by $\Dmat\Cmat^{(i,j)}$.
  By \Cref{lemma:ch4-product_of_mat_to_DNN}, we know that there exists a sequence of bias vectors $\left\{ \cvec^{(i,j)} \right\}_{i \in [\depth], j \in [2n-1]}$ such that for all $i\in[\depth]$, 
  \begin{equation}
    \layer^\act_{\Dmat\Cmat^{(i, 2n-1)}, \cvec^{(i, 2n-1)}} \circ \ldots \circ \layer^\act_{\Dmat\Cmat^{(i, 1)}, \cvec^{(i, 1)}}(\xvec) = \act \left(\Dmat\Cmat^{(i,2n-1)} \ldots \Dmat\Cmat^{(i,1)} \xvec + \bvec^{(i)} \right).
  \end{equation}
  Now if $\epsilon'$ tends to zero,
  \begin{equation}
    \norm{ \layer^\act_{\Dmat\Cmat^{(i,2n-1)},\cvec^{(i,2n-1)}} \circ \ldots \circ \layer^\act_{\Dmat\Cmat^{(i,1)}, \cvec^{(i,1)}} - \act \left(\Wmat^{(i)} \xvec + \bvec^{(i)} \right)}_2
  \end{equation}
  will also tend to zero for any $\xvec \in \Xset$, because the \relu function is continuous and $\Xset$ is bounded.
  Let $\nn_\Pi = \layer^\act_{\Dmat\Cmat^{(\depth,2n-1)},\cvec^{(\depth,2n-1)}} \circ \ldots \circ \layer^\act_{\Dmat\Cmat^{(1,1)}, \cvec^{(1,1)}}$, because all functions are continuous, for all $\xvec \in \Xset$, $\norm{\nn_\Omega(\xvec) - \nn_\Pi(\xvec)}_2$ tends to zero as $\epsilon'$ tends to zero which concludes the proof.
\end{proof}



Now that we know that diagonal-circulant neural networks can approximate any dense neural networks with arbitrary precision, we can extend is result to any function, thus demonstrating that they are universal approximators.
First, let us present universal approximation results for neural networks.
\citet{cybenko1989approximation,hornik1989multilayer} have shown that neural networks with a single hidden layer and sigmoid activation can approximate any function if the hidden layer is allowed to be arbitrary large. 
However, arbitrary large neural networks lack practical applications.

More recently, the universal approximation results have been extended to bounded width neural networks with arbitrary depth~\cite{lu2017expressive,hanin2017universal}.
More formally, we have the following result for neural networks with $\relu$ activations:

\begin{theorem}[Universal Approximation Theorem for Neural Network \citet{hanin2017universal}]
  \label{theorem:ch3-universal_approximation_theorem_for_neural_network}
  For any continuous function $f:[0,1]^{n} \rightarrow \Rbb_+$ of bounded supremum norm, for any $\epsilon>0$, there exists a neural network $\nn_\weights$ parameterized by $\weights$ with an input layer of width $n$, an output layer of width $1$, hidden layers of width $n+3$ and ReLU activations such that 
  \begin{equation}
    \forall x \in [0,1]^n, \quad \left| f(\xvec) - \nn_\weights(\xvec) \right| < \epsilon \enspace.
  \end{equation}
  \removespace
\end{theorem}

From \Cref{lemma:ch4-product_of_mat_to_DNN,lemma:ch4-dcnn_approx_neural_network} and \Cref{theorem:ch3-universal_approximation_theorem_for_neural_network} by~\citet{hanin2017universal} which states that dense neural networks are universal approximators, we can prove that bounded-width diagonal-circulant neural networks are also universal approximators.



\begin{maincorollary}
  \label{corollary:ch4-universal_approximation}
  Diagonal circulant neural networks with bounded width are universal approximators in the following sense:
  for any continuous function $f:[0,1]^n \rightarrow \Rbb_+$ of bounded supremum norm, for any $\epsilon > 0$, there exists a complex-valued diagonal-circulant neural network $\nn_\Pi$ of width $n+3$ such that $\forall \xvec \in [0,1]^{n+3}$, $\left| f(\xvec_{1} \ldots \xvec_{n}) - \left( \nn_\Pi \left( \xvec \right) \right)_0 \right| < \epsilon$ where $|\cdot|$ refer to the complex modulus. 
\end{maincorollary}

\begin{proof}[\Cref{corollary:ch4-universal_approximation}]
  From \Cref{theorem:ch3-universal_approximation_theorem_for_neural_network}, we know that there exists a dense neural network $\nn_\Omega$ with an input layer of width $n$, an output layer of width $1$, hidden layers of width $n+3$ and ReLU activations such that $\forall \xvec \in [0,1]^n, \left| f(\xvec) - \nn_\Omega (\xvec) \right| < \epsilon$.
From $\nn_\Omega$, we can easily build a dense neural networks $\widetilde{\nn}_\Omega$ of width exactly $n+3$, such that $\forall \xvec \in [0,1]^{n+3}$, $\left| f(\xvec_{1} \ldots \xvec_{n}) - \left(\widetilde{\nn}_\Omega (\xvec) \right)_{0}\right| < \epsilon$.
Thanks to \Cref{lemma:ch4-dcnn_approx_neural_network}, this last network can be approximated arbitrarily well by a diagonal-circulant neural network of width $n+3$.
Note that the matrices in the diagonal-circulant neural network are complex, even though we are approximating a real-valued function.
\end{proof}

The previous result shows that diagonal-circulant neural networks are universal approximators of real-valued functions.
However the depth needed is in $\bigO(n)$ where $n$ is the width of the network (size of the input).
The depth needed to reach universal approximation is not small, in our experiments, $n$ can be over 3000.
Nonetheless, \citet{cheng2015exploration} have provided empirical evidence that diagonal-circulant neural networks with small depth can offer good performance.
In the following subsection, we study the theoretical expressivity of diagonal-circulant neural networks with bounded-width and small depth.
This study allows us to better understand why DCNNs show good empirical performances with limited depth.



\subsection{The Expressive Power of Diagonal-Circulant Neural Networks}
\label{subsection:ch4-on_the_expressive_power_of_diagonal-circulant_neural_networks}



In this subsection, we study the expressive power of diagonal-circulant neural networks with small depth.
To assess the expressivity of DCNNs, we compare the depth needed to approximate dense neural networks with low \emph{total rank} which we define as the sum of ranks of each weights matrix.
With the concept of total rank, we present in the following,  our result on the expressive power of DCNNs with respect to the total rank of dense neural networks.


\begin{definition}[Total Rank]
  The total rank $\kappa(\nn_{\Omega})$ of the neural network $\nn_{\Omega}$ corresponds to the sum of the ranks of the matrices $\Wmat^{(1)} \ldots \Wmat^{(\depth)}$ as follows
  \begin{align}
    \kappa(\nn_{\Omega}) \triangleq \sum_{i \in [\depth]} \rank(\Wmat^{(i)}) \enspace.
  \end{align}
  \removespace
\end{definition}

\begin{maintheorem}[Rank-based expressive power of DCNNs] \label{theorem:ch4-low_rank_nn}
  Let $\nn_\Omega$ be a dense neural network of width $n$, depth $\depth$ and a total rank $K$, and assume $n$ is a power of $2$.
  Let $\Xset \subset \Cbb^n$ be a bounded set.
  Then, for any $\epsilon > 0$, there exists a diagonal-circulant neural network $\nn_\Pi$ of width $n$ such that $\norm{\nn_\Omega(\xvec) - \nn_\Pi(\xvec)}_2 < \epsilon$ for all $\xvec \in \Xset$ and the depth of $\nn_\Pi$ is bounded by $5K$.
\end{maintheorem}

\begin{proof}[\Cref{theorem:ch4-low_rank_nn}]
  Let $\nn_\Omega$ be a dense neural networks parameterized by $\Omega = \{( \Wmat^{(i)}, \bvec^{(i)} \}_{i \in [\depth]}$ of width $n$, depth $\depth$.
  Let $k^{(1)} \ldots k^{(\depth)}$ be the ranks of matrices $\Wmat^{(1)} \ldots \Wmat^{(\depth)}$, which are $n \times n$ matrices. Assume that $\forall i$, $n$ can be divided by $k^{i}$.
  By \Cref{theorem:ch4-rank-decomposition}, for any $\epsilon > 0$, any matrix $\Wmat^{(i)}$ of rank $k^{(i)}$, there exists a sequence of diagonal matrices $\{\Dmat^{(i, j)}\}_{i \in [\depth], j \in [4k^{i}+1]}$ and circulant matrices, $\{\Cmat^{(i, j)}\}_{i, \in [\depth], j \in [4k^{i}+1]}$, such that for all $i \in [\depth]$,
  \begin{equation}
    \norm{\prod_{j=1}^{4k^{i}+1} \Dmat^{(i,4k^{i}+2-j)} \Cmat^{(i,4k^{i}+2-j)} - \Wmat^{(i)}}_\fro < \epsilon' \enspace.
  \end{equation}
  Using the exact same technique as in \Cref{lemma:ch4-dcnn_approx_neural_network}, we can build a diagonal-circulant neural network $\nn_\Pi$, such that
  \begin{equation}
    \norm{ \nn_\Omega (\xvec) - \nn_\Pi (\xvec)}_2 < \epsilon, \quad \forall \xvec \in \Xset,
  \end{equation}
  for which the total number of layers is bounded as follows:
  \begin{equation}
    \sum_{i \in [\depth]} \left( 4 k^{(i)} + 1 \right) \le \depth + 4 \sum_{i \in [\depth]} k^{(i)} \le \depth + 4 \kappa(\nn_\Omega) \le 5 \kappa(\nn_\Omega) \enspace. 
  \end{equation}
  where $\kappa(\nn_\Omega)$ is the total rank of the dense neural network $\nn_\Omega$.
\end{proof}

\noindent
Remark that in the theorem, we require that $n$ is a power of $2$.
We conjecture that the result still holds even without this condition.
This result refines \Cref{lemma:ch4-dcnn_approx_neural_network} and answers our second question: does DCNN of bounded width and small depth can approximate a dense neural network of low total rank?
Note that the converse is not true because an $n \times n$ circulant matrix can be of full rank, therefore approximating a DCNN of depth $1$ can require a dense network of total rank equal to $n$.

Finally, what if we choose to use diagonal-circulant networks with a small depth to approximate a dense neural network whose matrices are not of lower rank? 
To answer this question, we present three results. First, we characterize the negative impact of replacing matrices by their low rank approximation.
Then, we extend this result to neural networks and bound the error between a dense neural network with full total rank and one with low total rank.
Finally, \Cref{corollary:relu_to_circ} presents our result which bounds the error between a dense neural network with full total rank and a diagonal-circulant neural network.

\begin{lemma} \label{lemma:ch4-bound_one_layer}
  Let $\Wmat \in \Cbb^{n \times n}$ with singular values $\sigma_1 \ldots \sigma_n$, and let $\bvec, \xvec, \yvec \in \Cbb^n$.
  Let $\widetilde{\Wmat}$ be the matrix obtained by an SVD approximation of rank $k$ of matrix $\Wmat$.
  Then we have:
  \begin{equation}
    \norm{ \act \big( \Wmat\xvec + \bvec \big) - \act \big( \widetilde{\Wmat}\yvec + \bvec \big)}_2 \le \sigma_{1} \norm{\xvec - \yvec}_2 + \sigma_{k+1} \norm{\xvec}_2 
  \end{equation}
  \removespace
\end{lemma}

\begin{proof}[\Cref{lemma:ch4-bound_one_layer}]
  Let us denote $\sigma_j$ be the $j^{th}$ singular value of $\Wmat$ and recall that $\sigma_1(\Wmat) = \norm{\Wmat}_2$ by the definition of the spectral norm.
  Furthermore, we have $\sigma_1(\Wmat) = \sigma_1(\widetilde{\Wmat})$ because the greatest singular values are equal for both $\Wmat$ and $\widetilde{\Wmat}$.
  Also, note that $\norm{\Wmat - \widetilde{\Wmat}}_{2} = \sigma_{k+1}$.
  First, let us bound the formula without ReLUs:

  \begin{align}
    \norm{\big(\Wmat\xvec+\bvec\big) - \big(\widetilde{\Wmat} \yvec + \bvec \big)}_2 &= \norm{\Wmat\xvec - \widetilde{\Wmat} \xvec - \widetilde{\Wmat} \big(\yvec - \xvec \big)}_2 \\
     &\le \norm{\big(\Wmat - \widetilde{\Wmat}\big)\xvec}_2 + \norm{\widetilde{\Wmat}}_{2} \norm{\xvec - \yvec}_2 \\
     &\le \norm{\xvec}_2 \sigma_{k+1} + \sigma_1 \norm{\xvec - \yvec}_2 
  \end{align}
  Finally, it is easy to see that for any pair of vectors $\xvec,\yvec \in \Cbb^n$, we have
  \begin{equation}
    \norm{ \act(\xvec) - \act(\yvec)}_2 \le \norm{\xvec - \yvec}_2 \enspace,
  \end{equation}
  because the complex \relu function is $1$-Lipschitz.
  This concludes the proof.
\end{proof}

The lemma above bound the error between a linear transform and its equivalent low rank approximation.
In the following, we extend this result to neural networks.

\begin{proposition} \label{proposition:ch4-approximation_network_dense_to_low_rank}
  Let $\nn_\Omega: \Cbb^n \rightarrow \Cbb^n$ be a dense neural network, with \relu activation, parameterized by $\Omega = \left\{ \big(\Wmat^{(i)},\bvec^{(i)} \big) \right\}_{i \in [\depth]}$ with $\Wmat^{(i)} \in \Cbb^{n \times n}, \bvec^{(i)} \in \Cbb^n$ for all $i \in [\depth]$ and $\nn_\Omega = \layer_{\Wmat^{(\depth)},\bvec^{(\depth)}} \circ \ldots \circ \layer_{\Wmat^{(1)},\bvec^{(1)}}$ of depth $\depth$ and width $n$.
  Let $\widetilde{\Omega} = \left\{ \big( \widetilde{\Wmat}^{(i)},\bvec^{(i)} \big) \right\}_{i \in [\depth]}$ where $\widetilde{\Wmat}^{(i)}$ is the matrix obtained by the SVD approximation of rank $k$ of matrix $\Wmat^{(i)}$. 
  Define the network $\nn_{\widetilde{\Omega}}$ and let $\sigma_{j}^{(i)}$ be the $j^{th}$ singular value of $\Wmat^{(i)}$ and denote $\sigma_j^{(\max)} = \max_i \sigma_j^{(i)}$, the largest $j^{th}$ singular value across layers.
  Then, for any $\xvec \in \Cbb^n$, we have:
  \begin{itemize}
    \item if $\sigma_{1}^{(\max)} = 1$:
      \begin{equation}
	\norm{ \nn_\Omega(\xvec) - \nn_{\widetilde{\Omega}}(\xvec)}_2 \le \depth \left(  R \sigma_{k+1}^{(\max)} \right) \enspace.
      \end{equation}
    \item if $\sigma_{1}^{(\max)} \neq 1$:
      \begin{equation}
	\norm{ \nn_\Omega(\xvec) - \nn_{\widetilde{\Omega}}(\xvec)}_2 \le \frac{\left( \big(\sigma_{1}^{(\max)}\big)^\depth - 1 \right) R \sigma_{k+1}^{(\max)}}{\sigma_{1}^{(\max)} - 1}
      \end{equation}
  \end{itemize}
  where $R$ is an upper bound on the norm of the output of any layer in $\nn_\Omega$.
\end{proposition}

\begin{proof}[\Cref{proposition:ch4-approximation_network_dense_to_low_rank}]
  Let $\xvec^{(0)} \in \Cbb^n$ and $\yvec^{(0)} = \xvec^{(0)}$.
  For all $i \in [\depth]$, define $\xvec^{(i)} = \act \left(\Wmat^{(i)} \xvec^{(i-1)} + \bvec^{(i)} \right)$ and $\yvec^{(i)} = \act \left( \widetilde{\Wmat}^{(i)} \yvec^{(i-1)} + \bvec^{(i)} \right)$.
  We aim to upper bound the difference in norm of $\xvec^{(i)}$ and $\yvec^{(i)}$.
  First, let us consider the linear transform within $\xvec^{(i)}$ and $\yvec^{(i)}$:
  The difference in norm between $\xvec^{(i)}$ and $\yvec^{(i)}$ can be upper bounded as follows:
  \begin{align}
    \norm{ \xvec^{(i)} - \yvec^{(i)}}_2 &\le \sigma_{1}^{(i)} \norm{ \xvec^{(i-1)} - \yvec^{(i-1)}}_2 + \sigma_{k+1}^{(i)} \norm{ \xvec^{(i-1)}}_2 \\
    &\leq \sigma_1^{(\max)} \norm{ \xvec^{(i-1)} - \yvec^{(i-1)}}_2 + \sigma_{k+1}^{(\max)} R
  \end{align}
  where the first inequality stems from \Cref{lemma:ch4-bound_one_layer} and the second by setting $\sigma_j^{(\max)} = \max_i \sigma_j^{(i)}$ where $\sigma_j^{(\max)}$ is the largest $j^{th}$ singular value across layers and $R = \max_i \norm{\xvec^{(i)}}_2$.
  From there, we need to consider two cases:
  \begin{itemize}
    \item If $\sigma_1^{(\max)} = 1$: we have a recurrence relation of the form $a_n = a_{n-1} + s$ with $a_0 = 0$ which can unfold as follows: $a_n = ns$.
      We can apply this formula to bound our error as follows:
      \begin{equation}
	\norm{ \xvec^{(\depth)} - \yvec^{(\depth)}}_2 \le \depth \left(  R \sigma_{k+1}^{(\max)} \right) \enspace.
      \end{equation}
    \item If $\sigma_1^{(\max)} \neq 1$: we have a recurrence relation of the form $a_n = r a_{n-1} + s$ with $a_0= 0$ which can unfold as follows: $a_n = \frac{s \left(r^n - 1\right)}{r-1}$.
      We can apply this formula to bound our error as follows:
      \begin{equation}
	\norm{ \xvec^{(\depth)} - \yvec^{(\depth)} }_2 \le \frac{ \left( \big(\sigma_1^{(\max)}\big)^\depth - 1 \right) R \sigma_{k+1}^{(\max)}}{\sigma_1^{(\max)}-1} \enspace,
      \end{equation}
  \end{itemize}
  which concludes the proof.
\end{proof}

\noindent
\Cref{proposition:ch4-approximation_network_dense_to_low_rank} bounds the error between a dense neural network and a neural network whose matrices are the low rank approximations of the first ones.
Two cases are presented. If the largest singular value across the network is equal to 1, then the error is polynomial with the depth of the network.
In the case where the largest singular value across the network is different from 1, the error is exponential with respect to the depth of the network.

Now, we can easily extend \Cref{proposition:ch4-approximation_network_dense_to_low_rank} to diagonal-circulant neural networks. 
By \Cref{theorem:ch4-low_rank_nn}, we can replace the layers with low rank approximation by the product of diagonal and circulant matrices leading to a diagonal-circulant neural network with a higher depth.

\begin{maincorollary} \label{corollary:relu_to_circ}
  Let $\nn_\Omega$ be a dense neural network of depth $\depth$ and width $n$ and parameterized by $\Omega = \left\{(\Wmat^{(i)},\bvec^{(i)} ) \right\}_{i \in [\depth]}$.
  Let $\sigma_1^{(i)}$ be the largest singular value of $\Wmat^{(i)}$.
  Let $\Xset \subset \Cbb^n$ be a bounded set.
  Let $k$ be an integer dividing $n$.
  There  exists a diagonal-circulant neural network $\nn_\Pi$ of width $n$ and of depth $m = (4k+1)\depth$, parameterized by $\Pi = \left\{(\Dmat^{(i)}\Cmat^{(i)},\cvec^{(i)})\right\}_{i \in [m]}$ such that, for any $\xvec \in \Cbb^n$, we have:
  \begin{itemize}
    \item if $\sigma_{1}^{(\max)} = 1$:
      \begin{equation}
	\norm{ \nn_\Omega(\xvec) - \nn_{\Pi}(\xvec)}_2 \le \depth \left(  R \sigma_{k+1}^{(\max)} \right) \enspace.
      \end{equation}
    \item if $\sigma_{1}^{(\max)} \neq 1$:
      \begin{equation}
	\norm{ \nn_\Omega(\xvec) - \nn_{\Pi}(\xvec)}_2 \le \frac{\left( \big(\sigma_{1}^{(\max)}\big)^\depth - 1 \right) R \sigma_{k+1}^{(\max)}}{\sigma_{1}^{(\max)} - 1} \enspace,
      \end{equation}
  \end{itemize}
  where $R$ is an upper bound on the norm of the output of any layer in $\nn_\Omega$.
\end{maincorollary}

\begin{proof}[\Cref{corollary:relu_to_circ}]
  Let $\nn_{\widetilde{\Omega}}$ be a dense neural network of depth $\depth$ and width $n$ and let $\widetilde{\Omega} = \left\{ \big(\widetilde{\Wmat}^{(i)}, \bvec^{(i)} \big) \right\}_{i \in [\depth]}$ such that $\widetilde{\Wmat}^{(i)}$ is the matrix obtained by the SVD approximation of rank $k$ of matrix $\Wmat^{(i)}$.
  With \Cref{proposition:ch4-approximation_network_dense_to_low_rank}, we have an error bound on $\norm{ \nn_\Omega (\xvec) - \nn_{\widetilde{\Omega}} (\xvec)}_2$.
  Now each matrix $\widetilde{\Wmat}^{(i)}$ can be replaced by a product of $4k+1$ diagonal-circulant matrices.
  By \Cref{theorem:ch4-low_rank_nn}, this product yields a diagonal-circulant neural network of depth $m = (4k+1)\depth$, strictly equivalent to $\nn_{\widetilde{\Omega}}$ on $\Xset$.
  This concludes the proof.
\end{proof}

\begin{figure}[htb]
    \begin{center}
      \input{figures/main/ch4-diagonal_circulant/illustration_properties.tex}
    \end{center}
    \caption{Illustration of the expressivity of diagonal-circulant neural networks.}
    \label{figure:ch4-circfig}
\end{figure}
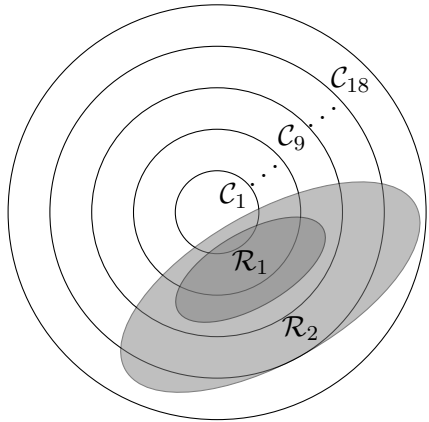

We highlight the significance of these results with the two following properties.
\begin{properties}
  Let $\Rset_{K}$ be the set of all functions $\nn_\Omega:\Cbb^n \rightarrow \Cbb^n$ for all $n$, representable by a dense neural network with complex \relu activation of total rank at most $K$ and let $\Cset_\depth$ be the set of all functions $\nn_\Pi:\Cbb^n \rightarrow \Cbb^n$ for all $n$, representable by deep diagonal-circulant networks of depth at most $\depth$, then:
  \begin{align}
    \label{property:eq1} \forall K,\exists \depth \, &\quad \Rset_{K} \subsetneq \Cset_{\depth} \\
    \label{property:eq2} \forall \depth,\nexists K\, &\quad \Cset_{\depth} \subseteq \Rset_{K}
  \end{align}
  \removespace
\end{properties}

\noindent
We illustrate the meaning of these properties using \Cref{figure:ch4-circfig}.
As we can see, the set $\mathcal{R}_K$ of all the functions representable by a dense neural network of total rank $K$ is strictly included in the set $\mathcal C_{9K}$ of all diagonal-circulant neural networks of depth $9K$ (as by \Cref{theorem:ch4-low_rank_nn}).
These properties are interesting for many reasons. 
First, \Cref{property:eq2} shows that diagonal-circulant networks are \emph{strictly more expressive} than networks with low total rank. 
Second and most importantly, in standard deep neural networks, it is known that the most of the singular values are close to zero. This phenomenon is called \emph{rank collapse} and it has been first observed by~\citet{saxe2013exact} and confirmed later by~\citet{sedghi2018singular,arora2019implicit}.
\Cref{property:eq1} shows that these networks can efficiently be approximated by diagonal-circulant networks.
Finally, several publications have shown that neural networks can be trained explicitly to have low-rank weight matrices \cite{chong18eccv, goyal2019compression}.
This opens the possibility of learning compact and accurate diagonal-circulant networks.

\section{How to Train Deep Diagonal Circulant Neural Networks?}
\label{section:ch4-how_to_train_deep_diagonal_circulant_neural_networks}

Training diagonal-circulant neural networks has revealed to be a challenging problem.
Indeed, as discussed earlier, the expressivity of diagonal-circulant neural networks scale with depth.
In the following, we devise two techniques to facilitate the training of deep diagonal-circulant neural networks.
First, we propose an initialization procedure which guarantees the signal is propagated across the network without vanishing nor exploding.
Secondly, we study the behavior of DCNNs with different nonlinearity functions and determine the best parameters for different settings. 
Note that we choose to train diagonal-circulant neural networks with real matrices instead of complex ones.
Indeed, the complex version of diagonal-circulant neural networks would have twice the number of parameters.

\subsection{Initialization Scheme of Diagonal-Circulant Neural Networks}
\label{subsection:ch4-initialization_scheme_of_diagonal-circulant_neural_networks}

In order to facilitate the training of deep diagonal-circulant neural networks, we extend the Xavier initialization \cite{glorot2010understanding} which is an initialization scheme proposed for dense and convolutional neural networks.
First, for each circulant matrix $\Cmat = \circulant (\cvec)$ with $\cvec \in \Rbb^n$, each $\cvec_i$ is randomly drawn from $\mathcal{N} \left(0,\alpha^2\right)$, with $\alpha = \sqrt{\frac{2}{n}}$.
Next, for each diagonal matrix $\Dmat = \diagonal (\dvec)$ with $\dvec \in \Rbb^n$, each $\dvec_{i}$ is drawn randomly and uniformly from $\{-1,1\}$ for all $i$.
Finally, all biases in the network are randomly drawn from $\mathcal{N}\left(0,\alpha'^{2}\right)$, for some small value of $\alpha'$.

\begin{lemma} \label{lemma:ch4-covariance}
  Let $\cvec, \dvec, \bvec$ be random variables in $\Rbb^n$ such that $\cvec \sim \mathcal{N}(0, \Imat_n\alpha^2)$, $\bvec \sim \mathcal{N}(0, \Imat_n\alpha'^2)$ and $\dvec_i \sim \{-1, 1\}, \forall i$ uniformly.
  Define $\Cmat = \circulant(\cvec)$ and $\Dmat = \diagonal(\dvec)$
  Define $\yvec = \Dmat \Cmat \uvec + \bvec$ for some vector $\uvec$ in $\Rbb^n$.
  Then, for all $i$, the probability density function (\pdf) of $\yvec_{i}$ is symmetric.
  Also:
  \begin{itemize}
    \item Assume $\uvec_0, \ldots, \uvec_{n-1}$ are fixed. Then, we have for $i \neq i'$:
      \begin{equation}
	\Cov(\yvec_i, \yvec_{i'}) = 0 \quad \text{and} \quad
	\Var(\yvec_i) = \alpha'^2 + \sum_{j} \uvec_{j}^{2} \alpha^{2}
      \end{equation}
    \item Let $\xvec$ be random variables in $\Rbb^n$ such that the \pdf of $\xvec_i$ is symmetric for all $i$, and let $\uvec_i = \rho(\xvec_i)$.
      We have for $i \neq i':$
      \begin{equation}
	\Cov(\yvec_i, \yvec_{i'}) = 0 \quad \text{and} \quad
	\Var(\yvec_i) = \alpha'^2 + \frac{1}{2} \sum_{j} \Var(\xvec_i) \alpha^2
      \end{equation}
  \end{itemize}
\end{lemma}

\begin{proof}[\Cref{lemma:ch4-covariance}]
  By an abuse of notation, we write $\cvec_0 = \cvec_n, \cvec_{-1} = \cvec_{n-1}$ and so on.
  First, note that: $\yvec_i = \sum_{j=0}^{n-1} \cvec_{j-i} \uvec_j \dvec_j + \bvec_i$.
  Observe that the term $\cvec_{j-i}\uvec_j\dvec_j$ has symmetric \pdf because of $\dvec_j$.
  Thus, $\yvec_i$ has a symmetric \pdf.
  Now let us compute the covariance.
  \begin{align}
    \Cov(\yvec_i,\yvec_{i'}) &= \left( \sum_{j,j'=0}^{n-1} \Cov\left(\cvec_{j-i}\uvec_j\dvec_j,\cvec_{j'-i'}\uvec_{j'}\dvec_{j'}\right) \right) + \Cov(\bvec_i, \bvec_{i'}) \\
      &= \sum_{j,j'=0}^{n-1} \Ebb\left[\cvec_{j-i}\uvec_j\dvec_j\cvec_{j'-i'}\uvec_{j'}\dvec_{j'}\right] - \Ebb\left[\cvec_{j-i}\uvec_j\dvec_j\right] \Ebb\left[ \cvec_{j'-i'}\uvec_{j'}\dvec_{j'} \right]
  \end{align}
  Observe that $\Ebb\left[\cvec_{j-i}\uvec_j\dvec_j\right] = \Ebb\left[\cvec_{j-i}\uvec_j\right] \Ebb\left[\dvec_j\right] = 0$ because $\dvec_j$ is independent from $\cvec_{j-i}\uvec_j$.
  Also, observe that if $j\neq j'$ then $\Ebb\left[\dvec_j\dvec_{j'}\right]=0$ and
  \begin{equation}
    \Ebb\left[\cvec_{j-i} \uvec_j \dvec_j \cvec_{j'-i'} \uvec_{j'} \dvec_{j'} \right] = \Ebb\left[ \dvec_j \dvec_{j'} \right] \Ebb\left[ \cvec_{j-i} \uvec_j \cvec_{j'-i'} \uvec_{j'} \right] = 0 \enspace.
  \end{equation}
  Therefore, the only non null terms are those for which $j = j'$.
  We get:
  \begin{align}
    \Cov(\yvec_i,\yvec_{i'}) & =\sum_{j=0}^{n-1} \Ebb\left[\cvec_{j-i}\uvec_j\dvec_j\cvec_{j-i'}\uvec_j\dvec_j\right] \\
     & =\sum_{j=0}^{n-1} \Ebb\left[\cvec_{j-i}\cvec_{j-i'}\uvec_j^{2}\right]
 \end{align}
  Assume $\uvec$ is a fixed vector.
  Then, $\Var(\yvec_i) = \sum_{j=0}^{n-1} \uvec_j^{2} \alpha^2$ and $\Cov(\yvec_i,\yvec_{i'}) = 0$ for $i\neq i'$ because $\cvec_{j-i}$ is independent from $\cvec_{j-i'}$.
  Now assume that $\uvec_j = \rho(\xvec_j)$ where $\xvec_j$ is a random variables in $\Rbb^n$.
  Clearly, $\uvec_j^2$ is independent from $\cvec_{j-i}$ and $\cvec_{j-i'}$, thus, we have:
  \begin{align}
    \Cov(\yvec_i,\yvec_{i'}) &= \sum_{j=0}^{n-1} \Ebb \left[ \cvec_{j-i}\cvec_{j-i'}\right] \Ebb\left[ \uvec_j^2 \right] \enspace.
  \end{align}
  For $i \neq i'$, then $\cvec_{j-i}$ and $\cvec_{j-i'}$ are independent, we have $\Ebb\left[\cvec_{j-i}c_{j-i'}\right] = \Ebb\left[\cvec_{j-i}\right] \Ebb\left[\cvec_{j-i'}\right] = 0$ and $\Cov(\yvec_i,\yvec_{i'}) = 0$ if $i \neq i'$.
  Let us compute the variance.
  We get $\Var(\yvec_i) = \sum_{j=0}^{n-1} \Var(\cvec_{j-i}) \Ebb\left[\uvec_j^2\right]$.
  Because the \pdf of $\xvec_j$ is symmetric, $\Ebb\left[\xvec_j^2\right] = 2\Ebb\left[\uvec_j^2\right]$ and $\Ebb\left[\xvec_j\right] = 0$.
  Thus, 
  \begin{equation}
    \Var(\yvec_i) = \frac{1}{2} \sum_{j=0}^{n-1} \Var(\cvec_{j-i}) \Ebb\left[\xvec_j^2\right] = \frac{1}{2} \sum_{j=0}^{n-1} \Var(\cvec_{j-i}) \Var(\xvec_j)
  \end{equation}
  which concludes the proof.
\end{proof}

Now we can state our result on the initialization of diagonal-circulant neural networks.
The following proposition states that the covariance matrix at the output of any layer in a diagonal-circulant neural network is constant.
Moreover, note that the result of this proposition is independent of the depth of the network.

\begin{proposition}[Initialization of Diagonal-Circulant Neural Networks] \label{proposition:ch4-initialization_dcnn}
  Let $\nn_\Pi$ be a diagonal-circulant neural network of depth $\depth$ initialized according to our procedure, with $\alpha'=0$.
  Assume that all layers $1$ to $\depth-1$ have \relu activation functions, and that the last layer has the identity activation function.
  Then, for any $\xvec \in \Rbb^n$, the covariance matrix of $\nn_\Pi(\xvec)$ is $\frac{2}{n} \Imat_n \norm{\xvec}_2^2$.
\end{proposition}

\begin{proof}[\Cref{proposition:ch4-initialization_dcnn}]
  Let $\nn_\Pi \triangleq \layer_{\Dmat^{(\depth)}, \Cmat^{(\depth)}} \circ \ldots \circ \layer_{\Dmat^{(1)}, \Cmat^{(1)}}$ be a $\depth$ layer diagonal-circulant neural network.
  All matrices are initialized as described in the statement of the proposition.
  Let $\yvec = \Dmat^{(1)} \Cmat^{(1)} \xvec$.
  \Cref{lemma:ch4-covariance} shows that $\Cov(\yvec_{i}, \yvec_{i'}) = 0$ for $i \neq i'$ and $\Var(\yvec_{i}) = \frac{2}{n}\norm{x}_{2}^{2}$.
  For any $j \le \depth$, define 
  \begin{equation}
    \zvec^{(j)} = \layer_{\Dmat^{(j)}, \Cmat^{(j)}} \circ \ldots \circ \layer_{\Dmat^{(1)}, \Cmat^{(1)}}(\xvec) \enspace.
  \end{equation}
  By a recursive application of \Cref{lemma:ch4-covariance}, we get
  \begin{equation}
    \Cov(\zvec_i^{(j)}, \zvec_{i'}^{(j)}) = 0 \quad \text{and} \quad  \Var(\zvec_i^{(j)}) = \frac{2}{n} \norm{\xvec}_2^2
  \end{equation}
   which concludes the proof.
\end{proof}

The effect of \Cref{proposition:ch4-initialization_dcnn} is that the signal and the gradient will not vanish during the training, facilitating the convergence.
The fact that the result of \Cref{proposition:ch4-initialization_dcnn} is independent of the depth of the network allows us to train very deep diagonal-circulant neural networks.

\subsection{Analysis of the Use of Nonlinearities}
\label{subsection:ch4-analysis_on_the_use_of_nonlinearities}

\begin{figure}[htb]
  \centering
  \includegraphics[width=\scalefigure\textwidth]{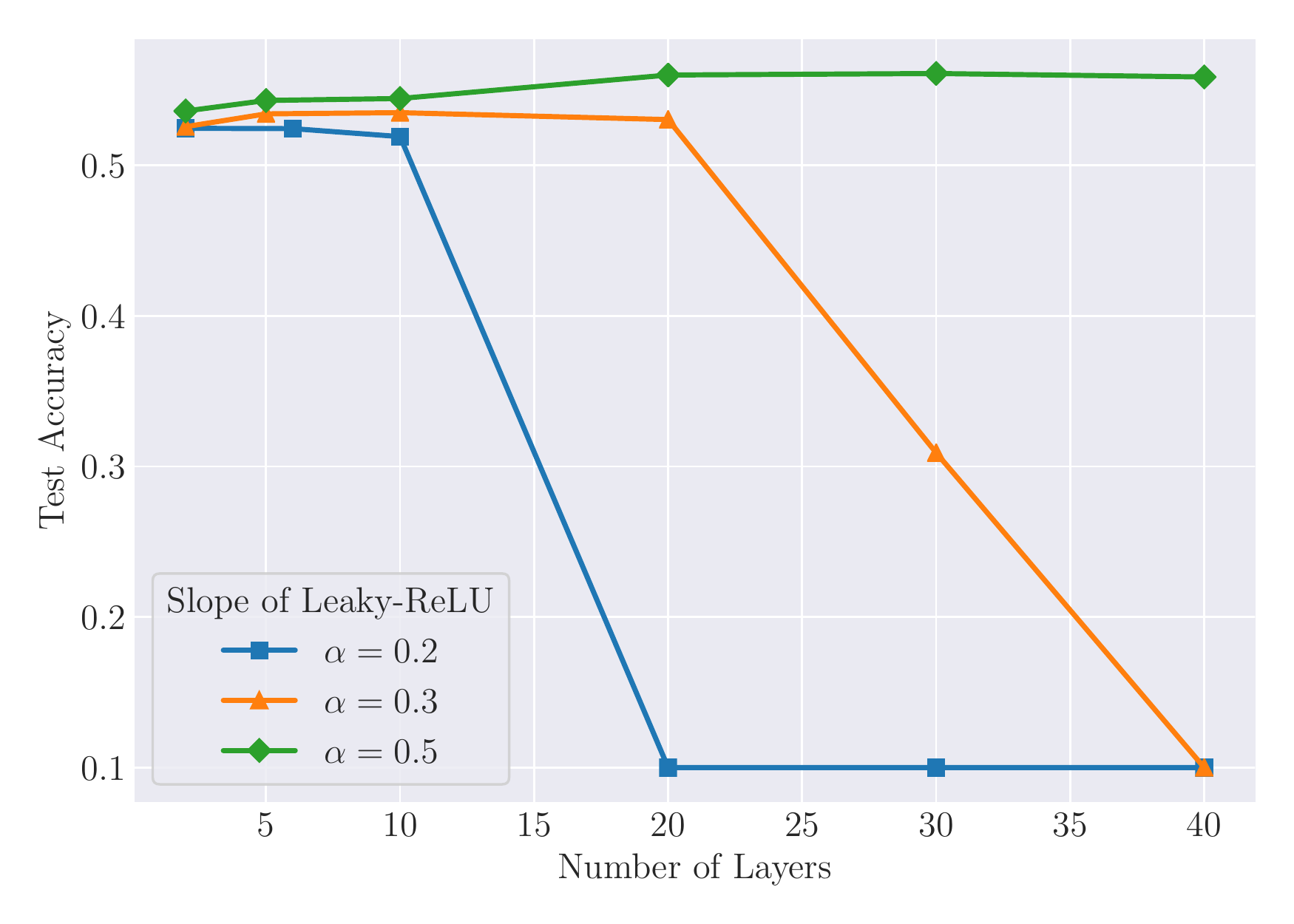}
  \caption{Impact of increasing the slope of a Leaky-ReLU in DCNNs.}
  \label{figure:ch4-cifar10_leaky_relu}
\end{figure}

We empirically found that the ReLU activations had an impact on the training of deep diagonal-circulant neural networks on CIFAR10 dataset.
Indeed, the deeper the network, the more nonlinear it is, which makes convergence difficult.
In an experiment, we replace the ReLU activations with Leaky-ReLU activations (\cf \Cref{subsection:ch2-preliminaries_on_neural_networks}) and vary the slope of the Leaky-ReLU (a higher slope means an activation function that is closer to a linear function).
The results of this experiment are presented in~\Cref{figure:ch4-cifar10_leaky_relu}.
In this experiment, we try different slopes for the Leaky-ReLU activation and train diagonal-circulant neural networks with different depth.
We can observe that a higher slope (making the network more linear) facilitates convergence, allowing us to train deeper networks.
This is an interesting result, since  we can use this technique to adjust the number of parameters in the network (increasing depth), without facing training difficulties.
We hence rely on this setting in the experimental section.

\section{Experiments}
\label{section:ch4-experiments}

This experimental section aims at answering the following questions:
\begin{enumerate}
    \itshape
    \item[Q1] How do DCNNs compare to other approaches such as ACDC, LDR or other structured approaches?
    \item[Q2] How do DCNNs compare to other compression based techniques?
    \item[Q3] How do DCNNs perform in the context of large-scale real-world machine learning applications?  
\end{enumerate}

\subsection{Comparison with Other Structured Approaches (Q1)}
\label{subsection:ch4-comparison_with_other_structured_approches}


\begin{figure}[ht]
   \centering
   \includegraphics[width=\scalefigure\textwidth]{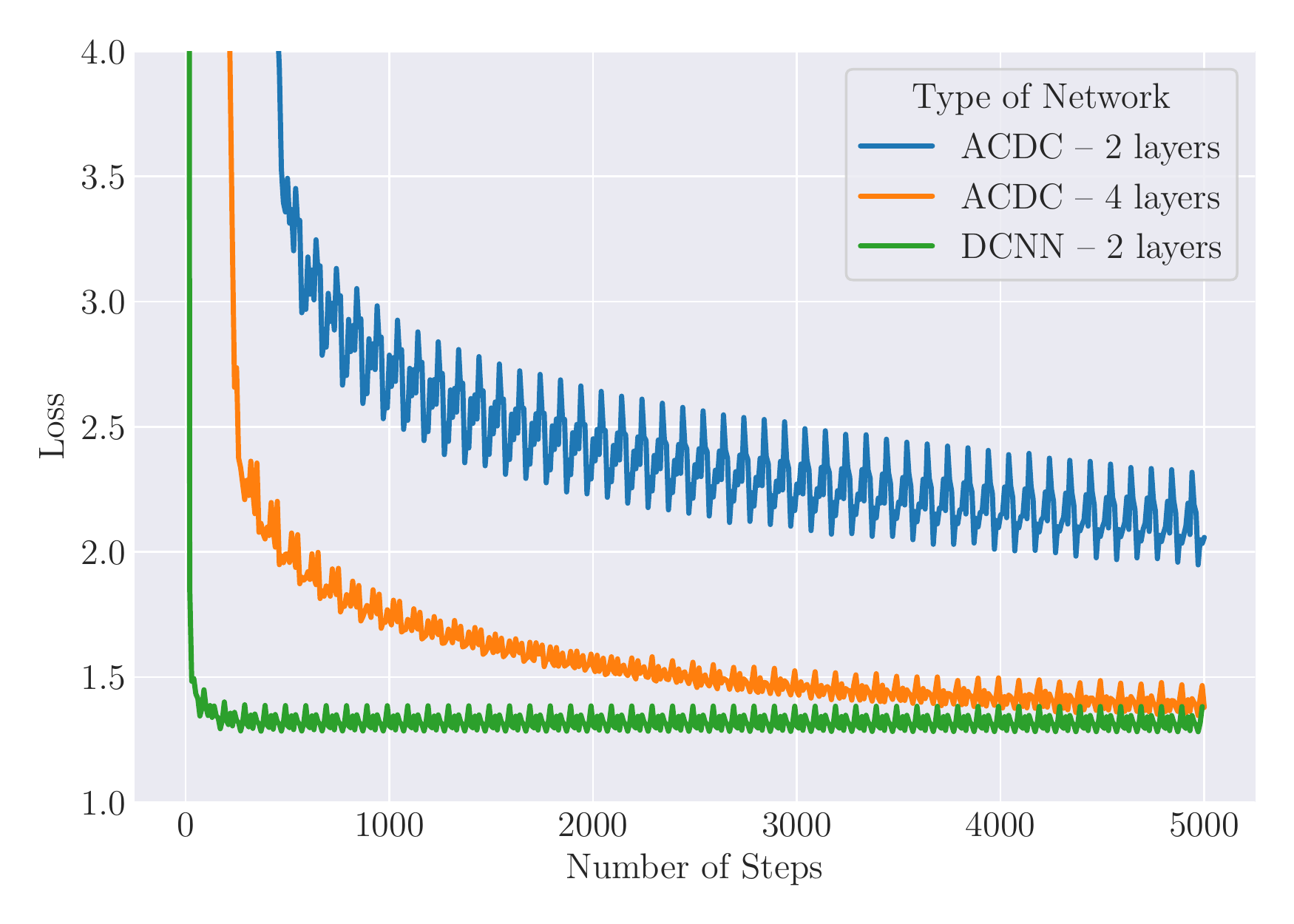}
   \caption{Comparison of the training loss of DCNNs and ACDC networks on a regression task with synthetic data.}
   \label{figure:ch4-acdc_regression}
\end{figure}

\begin{figure}[ht]
   \centering
   \includegraphics[width=\scalefigure\textwidth]{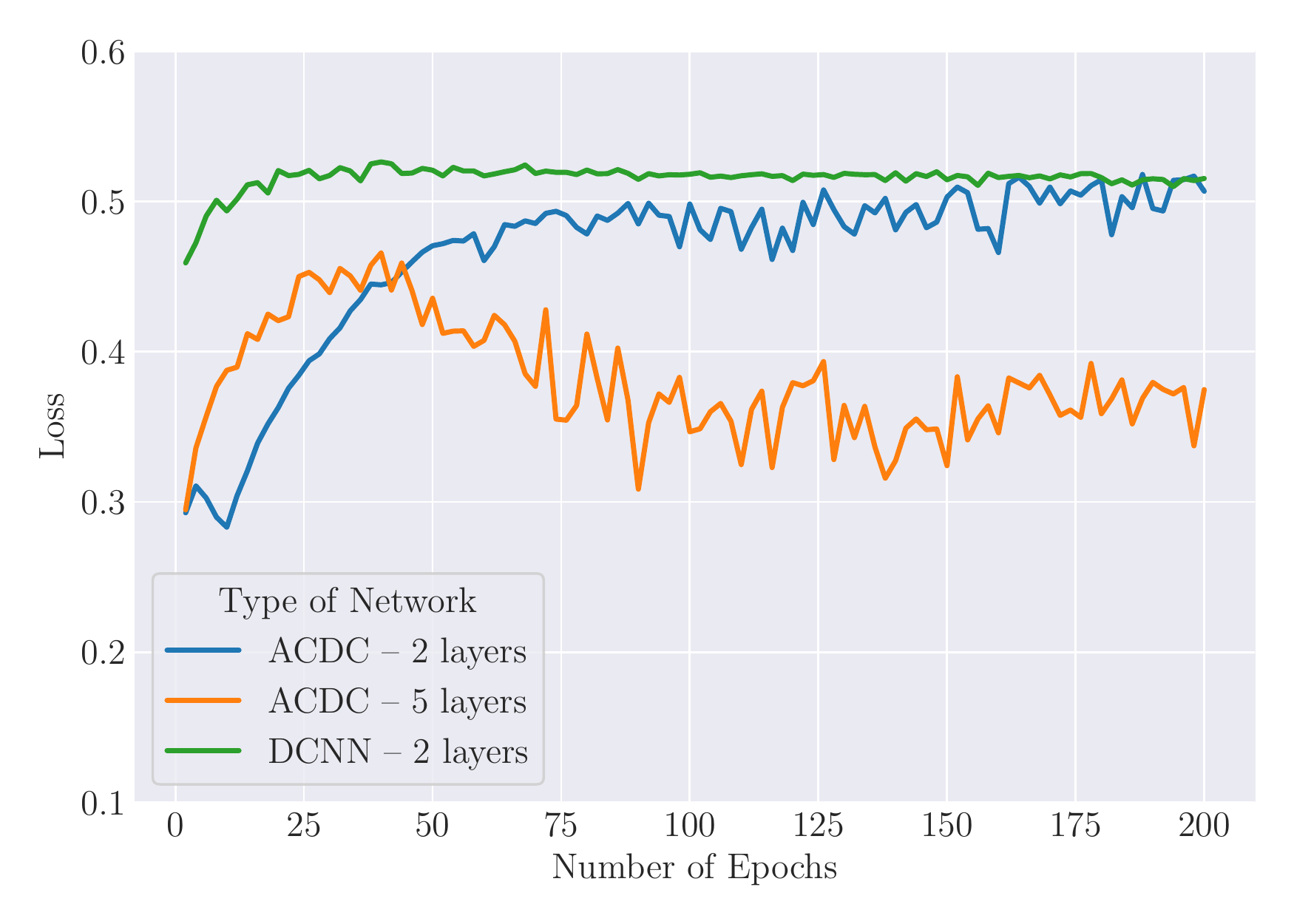}
   \caption{Comparison of the training loss of DCNNs and ACDC networks on a CIFAR-10 dataset.}
   \label{figure:ch4-acdc_cifar10}
\end{figure}

\subsubsection{Comparison with ACDC}

In \Cref{chapter:ch3-related_work}, we have discussed the differences between the ACDC framework and our approach from a theoretical perspective.
In this section, we conduct experiments to compare the performance of DCNNs with neural networks based on ACDC layers. 
We first reproduce the experimental setting from \citet{moczulski2016acdc}, and compare both approaches using only linear networks (\ie, networks without any nonlinear activation).
The synthetic dataset has been created in order to reproduce the experiment on the regression linear problem proposed by~\citet{moczulski2016acdc}.
We draw $\Xmat$ and $\Wmat$ from a uniform distribution between [-1, +1] and $\epsilon$ from a normal distribution with mean 0 and variance $0.01$.
The relationship between $\Xmat$ and $\Ymat$ is defined by $\Ymat = \Xmat\Wmat + \epsilon$. 
The results are presented in \Cref{figure:ch4-acdc_regression}.
In this simple setting, while both architectures demonstrate good performance, we can observe that DCNNs offer a better convergence rate.
In \Cref{figure:ch4-acdc_cifar10}, we compare neural networks with ReLU activations on CIFAR-10. 

We found that networks which are based only on ACDC layers are difficult to train and offer poor accuracy on CIFAR-10 (we have tried different initialization schemes including the one from the original paper, and the one we introduce in this chapter).
\citet{moczulski2016acdc} managed to train a large VGG network  however these networks are generally highly redundant and the contribution of the structured layer is difficult to quantify. 
We also observe that adding a single dense layer improves the convergence rate of ACDC in the linear case, which explains the good results of \citet{moczulski2016acdc}.
However, it is difficult to characterize the true contribution of the ACDC layers when the network has a large number of expressive layers.
In contrast, deep DCNNs can be trained and offer good performance without additional dense layers (these results are in line with our experiments on the \yt dataset).


\begin{figure}
   \centering
   \includegraphics[width=\scalefigure\textwidth]{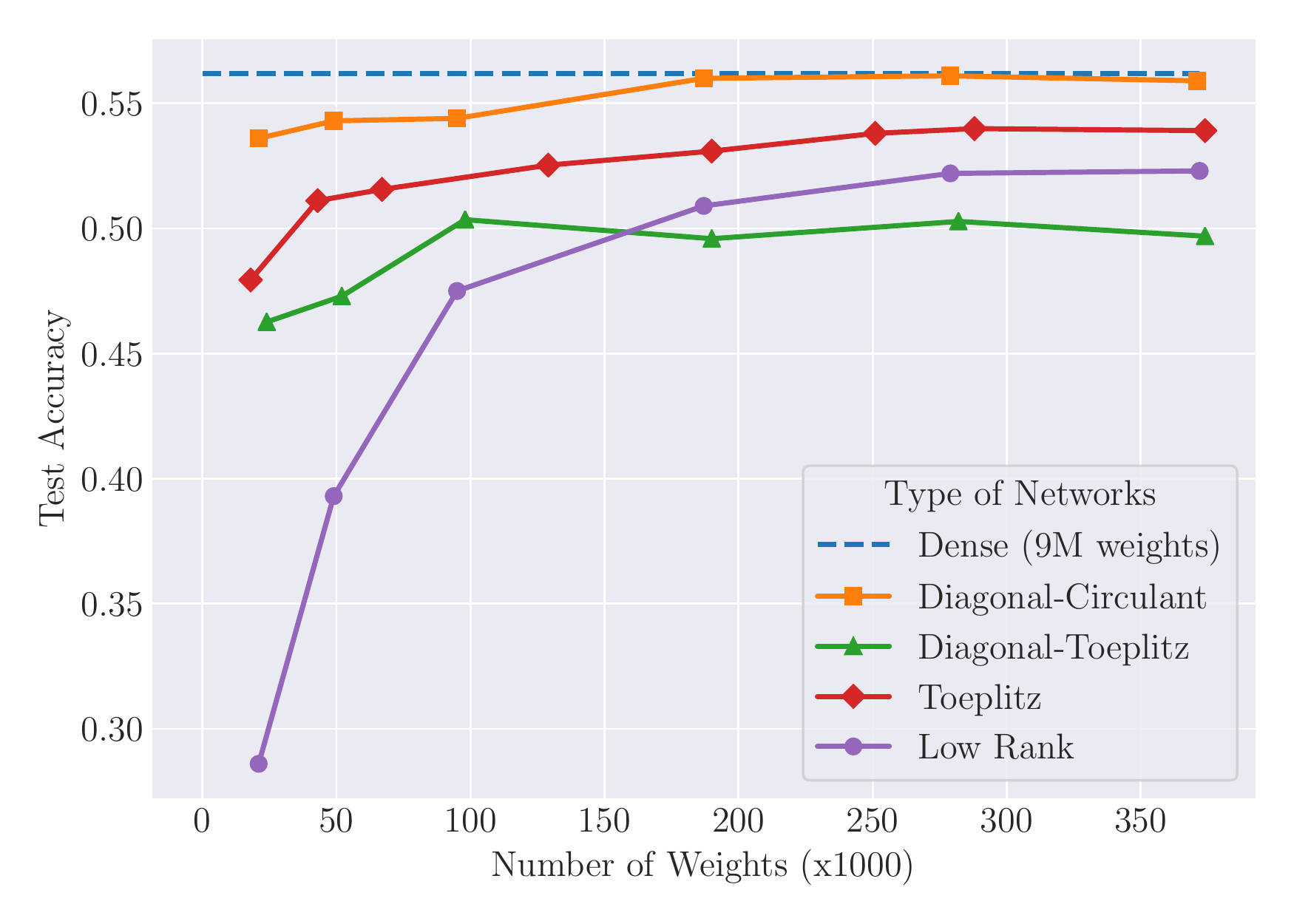}
   \caption{Network size \vs Accuracy compared on Dense networks, DCNNs, DTNNs, neural networks based on Toeplitz matrices and neural networks based on Low Rank-based matrices.} 
   \label{figure:ch4-cifar10_type}
\end{figure}

\begin{figure}
   \centering
   \includegraphics[width=\scalefigure\textwidth]{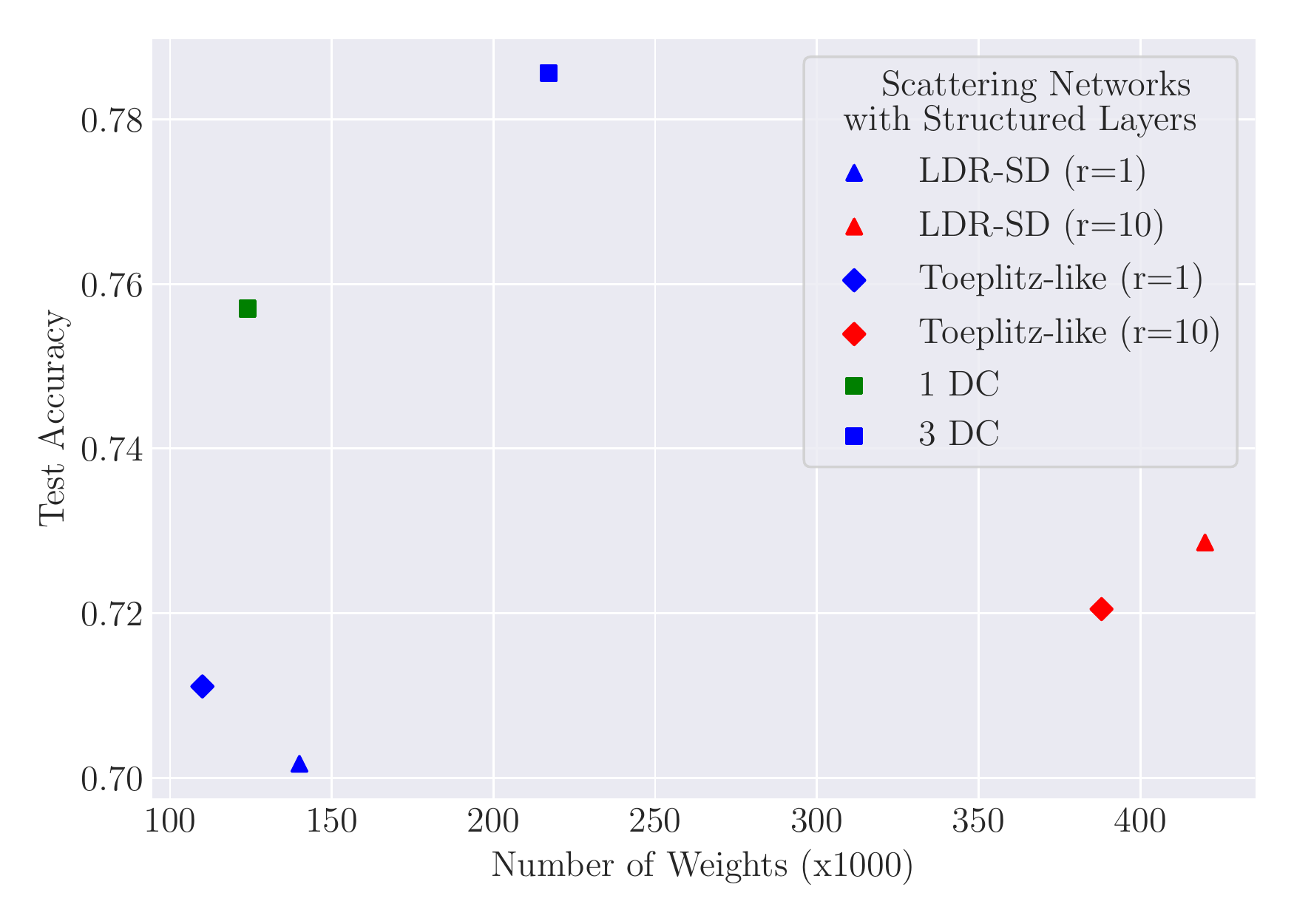}
   \caption{Accuracy of different structured architecture given the number of trainable parameters.}
   \label{figure:ch4-cifar10_with_channels_xp}
\end{figure}

\subsubsection{Comparison with Dense Networks, Toeplitz Networks and Low Rank Networks}

We now compare DCNNs with other state-of-the-art structured networks by measuring the accuracy on a flattened version of the CIFAR-10 dataset.
Our baseline is a dense feed-forward network with a fixed number of weights (9 million weights).
We compare with DCNNs and with DTNNs (see below), Toeplitz networks, and Low-Rank networks~\cite{yu2017compressing}.
We first consider Toeplitz networks which are stacked Toeplitz matrices interleaved with ReLU activations since Toeplitz matrices are closely related to circulant matrices.
However, Toeplitz networks have a different structure than DCNNs (they do not include diagonal matrices), therefore, we also experiment using DTNNs, a variant of DCNNs where all the circulant matrices have been replaced by Toeplitz matrices.
Finally we conduct experiments using networks based on low-rank matrices as they are also closely related to our work.
For each approach, we report the accuracy of several networks with a varying depth ranging from 1 to 40 (DCNNs, Toeplitz networks) and from 1 to 30 (from DTNNs).
For low-rank networks, we used a fixed depth network and increased the rank of each matrix from 7 to 40.
We also tried to increase the depth of low rank matrices, but we found that deep low-rank networks are difficult to train so we do not report the results here.
We compare all the networks based on the number of weights from 21K (0.2\% of the dense network) to 370K weights (4\% of the dense network) and we report the results in \Cref{figure:ch4-cifar10_type}. 
First we can see that the size of the networks correlates positively with their accuracy which demonstrated successful training in all cases.
We can also see that the DCNNs achieves the maximum accuracy of 56\% with 20 layers ($\sim$ 200K weights) which is as good as the dense networks with only 2\% of the number of weights.
Other approaches also offer good trade-offs but they are not able to reach the accuracy of a dense network.

\begin{table}[htb]
  \centering
  \begin{tabular}{lcc}
    \toprule
    \textbf{Architectures} & \textbf{\#Parameters} & \textbf{Accuracy}  \\
    \midrule
    \textit{Dense} & \textit{9.4M}	 & \textit{0.562} \\
    \textbf{\textit{DCNN $(5\ layers)$}} & \textbf{49K}	& \textbf{0.543} \\
    \textbf{\textit{DCNN $(2\ layers)$}} & \textbf{21K} & \textbf{0.536} \\
    LDR--TD	$(r = 2)$	         & 64K	& 0.511 \\
    LDR--TD	$(r = 3)$	         & 70K	& 0.473 \\
    Toeplitz-like $(r=2)$	         & 46K	& 0.483 \\
    Toeplitz-like $(r =3)$	         & 52K  & 0.496 \\
    \bottomrule
    \end{tabular}
    \caption{Comparison of LDR and Diagonal-Circulant neural networks on a flattened version of CIFAR-10.} 
    \label{table:ch4-xp_ldr}
\end{table}

\subsubsection{Comparison with LDR networks}

We now compare DCNNs with the LDR framework using the network configuration proposed by~\citet{thomas2018learning}: a single LDR structured layer followed by a dense layer.
In the LDR framework, we can change the size of a network by adjusting the rank of the residual matrix, effectively capturing matrices with a structure that is close to a known structure but not exactly (in the LDR framework, Toeplitz matrices can be encoded with a residual matrix with rank=2, so a matrix that can be encoded with a residual of rank=3 can be seen as Toeplitz-like.).
The results are presented in \Cref{table:ch4-xp_ldr} and demonstrate that DCNNs outperform all LDR networks both in terms of size and accuracy.

\subsection{Comparison with Other Compression Based Approaches (Q2)}
\label{subsection:ch4-comparison_with_other_compression_based_approaches}

We provide a comparison with other compression based approaches such as HashNet \cite{chen2015compressing}, Dark Knowledge \cite{hinton2015distilling}. 
\Cref{table:ch4-mnist} shows the test error of DCNN against other known compression techniques on the MNIST datasets.
We can observe that DCNN outperforms HashNet \cite{chen2015compressing} and Dark Knowledge \cite{hinton2015distilling} with fewer number of parameters.

\subsection{Large-scale Video Classification on the \yt Dataset (Q3)}
\label{subsection:ch4-large_scale_video_classification}

\begin{table}[ht]
  \centering
    \begin{tabular}{lcrc}
    \toprule
    \multicolumn{1}{c}{\textbf{Architecture}} & \multicolumn{1}{c}{\textbf{\#Params}} & \textbf{Error (\%)} \\
    \midrule
    LeNet \cite{lecun1998gradient}             & 4.2M         & 0.61          \\
    HashNet \cite{chen2015compressing}         & 46K          & 2.79          \\
    Dark Knowledge \cite{hinton2015distilling} & 46K          & 6.32          \\
    \textbf{DCNN}                              & \textbf{25K} & \textbf{1.74} \\
    \bottomrule
    \end{tabular}%
    \caption{Comparison with compression based approaches.}
  \label{table:ch4-mnist}%
\end{table}%

To understand the performance of deep DCNNs on large-scale applications, we conducted experiments on the \yt video classification with 3.8 training examples introduced by~\citet{abu2016youtube}.
This section provides a summary of the results obtained on the \yt dataset, the full experimental analysis is reported in Appendix~\ref{appendix:ap2-diagonal_circulant_neural_networks_for_video_classification}.
Notice that we favor this experiment over ImageNet applications because modern image classification architectures involve a large number of convolution layers, and compressing convolution layers is out of our scope. 
Also, as mentioned earlier, testing the performance of DCNN architectures mixed with a large number of expressive layers makes little sense.
The \yt includes two datasets describing 8 million labeled videos.
Both datasets contain audio and video features for each video.
In the first dataset (\emph{aggregated}) all audio and video features have been aggregated every 300 frames.
The second dataset (\emph{full}) contains the descriptors for all the frames.
To compare the models we use the GAP score (Global Average Precision) proposed by~\citet{abu2016youtube}.
On the simpler \emph{aggregated} dataset we compared off-the-shelf DCNNs with a dense baseline with 5.7 million weights.
On the full dataset, we designed three new compact architectures based on the state-of-the-art architecture introduced by~\citet{abu2016youtube}. 

\paragraph{Experiments on the aggregated dataset with DCNNs}
We compared DCNNs with a dense baseline with 5.7 million weights.
The goal of this experiment is to discover a good trade-off between depth and model accuracy.
To compare the models we use the GAP score (Global Average Precision) following the experimental protocol proposed by~\citet{abu2016youtube}, to compare our experiments. 
\Cref{table:youtube_agg_xp} shows the results of our experiments on the \emph{aggregated} \yt dataset in terms of the number of weights and GAP score.
These results suggest that we can compress the baseline at the cost of a little decrease of GAP score.

\begin{table}
  \centering
  \begin{tabular}{lccc}
    \toprule
    \textbf{Architecture} & \textbf{\#Weights} & \textbf{GAP@20} \\
    \midrule
    Baseline    & 5.7M  & 0.773          \\
    4 DC        & 25K   & 0.599          \\
    32 DC       & 122K  & 0.685          \\
    4 DC + 1 FC & 4.46M & \textbf{0.747} \\
    \bottomrule
  \end{tabular}
  \caption{GAP score on the \yt dataset with DCNNs.}
  \label{table:youtube_agg_xp}
\end{table}

\begin{table}
  \centering
  \begin{tabular}{lccc}
  \toprule
  \textbf{Architecture} & \textbf{\#Weights} & \textbf{GAP@20} \\
  \midrule
  \textit{original} & \textit{45M} & \textit{0.846} \\
  DBoF with DC   & 36M (\textit{80}) & 0.838 \\
  FC with DC    & 41M (\textit{91}) & \textbf{0.845} \\
  MoE with DC   & 12M (\textit{\textbf{26}}) & 0.805 \\
  \bottomrule
  \end{tabular}
  \caption{GAP score on the \yt dataset with different layers represented with diagonal-circulant decomposition.}
  \label{table:youtube_full_xp}
\end{table}

\paragraph{Experiments with DCNNs Deep Bag-of-Frames Architecture:}
The Deep Bag-of-Frames architecture can be decomposed into three blocks of layers, as illustrated in \Cref{figure:ch4-archi_youtube}.
The first block of layers, composed of the Deep Bag-of-Frames embedding (DBoF), is meant to model an embedding of these frames in order to make a simple representation of each video.
A second block of fully connected layers (FC) reduces the dimensionality of the output of the embedding and merges the resulting output with a concatenation operation.
Finally, the classification block uses a combination of Mixtures-of-Experts (MoE)~\cite{jordan1993hierarchical,abu2016youtube} and Context Gating~\cite{miech2017learnable} to calculate the final class probabilities.
\Cref{table:youtube_full_xp} shows the results in terms of the number of weights, size of the model (MB) and GAP on the full dataset, replacing the DBoF block reduces the size of the network without impacting the accuracy.
We obtain the best compression ratio by replacing the MoE block with DCNNs (26\%) of the size of the original dataset with a GAP score of 0.805 (95\% of the score obtained with the original architecture).
We conclude that DCNN are both theoretically sound and of practical interest in real, large-scale applications.

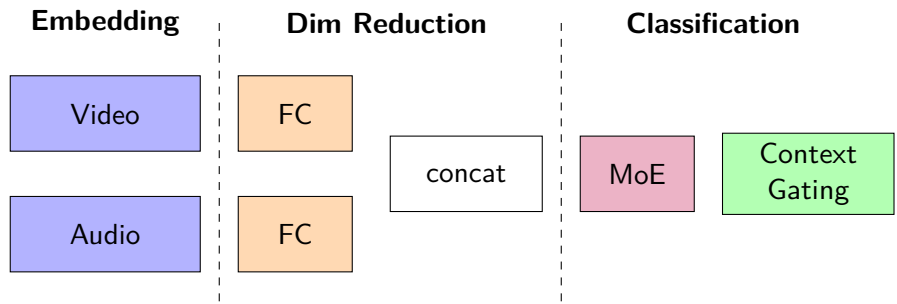
\begin{figure}[htb]
  \centering
  \input{figures/main/ch4-diagonal_circulant/archi_youtube.tex}
  \caption{Diagram of the state-of-the-art neural network architecture, initially proposed by~\citet{abu2016youtube} and later improved by~\citet{miech2017learnable}.}
  \label{figure:ch4-archi_youtube}
\end{figure}

\paragraph{Architectures \& Hyper-Parameters:} 
For the first set of our experiments (\emph{experiments on CIFAR-10}), we train all networks for 200 epochs, a batch size of 200, Leaky ReLU activation with a different slope.
We minimize the Cross Entropy Loss with Adam optimizer and use a piecewise constant learning rate of $5 \times 10^{-5}$, $2.5\times10^{-5}$, $5\times10^{-6}$ and $1\times10^{-6}$ after respectively 40K, 60K and 80K steps.
For the \yt dataset experiments, we built a neural network based on state-of-the-art architecture initially proposed by~\citet{abu2016youtube} and later improved by~\citet{miech2017learnable}.
Remark that no convolution layer is involved in this application since the input vectors are embeddings of video frames processed using state-of-the-art convolutional neural networks trained on ImageNet.
We trained our models with the CrossEntropy loss and used Adam optimizer with a 0.0002 learning rate and a 0.8 exponential decay every 4 million examples.
All fully connected layers are composed of 512 units.
DBoF, NetVLAD and NetFV are respectively 8192, 64 and 64 of cluster size for video frames and 4096, 32, 32 for audio frames.
We used 4 mixtures for the MoE Layer.
We used all the available 300 frames for the DBoF embedding.
In order to stabilize and accelerate the training, we used batch normalization before each nonlinear activation and gradient clipping.

\subsection{Exploiting Image Features}

\begin{table}[htb]
  \centering
  \begin{tabular}{lcc}
    \toprule
    \textbf{Architectures} & \textbf{\#Parameters} & \textbf{Accuracy}  \\
    \midrule
    \textbf{DC $(1\ layers)$} & \textbf{124K} & \textbf{0.757} \\
    \textbf{DC $(3\ layers)$} & \textbf{217K} & \textbf{0.785} \\
    LDR-SD $(r=1)$ & 140K & 0.701 \\
    LDR-SD $(r=10)$ & 420K & 0.728 \\
    Toeplitz-like $(r=1)$ & 110K & 0.711 \\
    Toeplitz-like $(r=10)$ & 388K & 0.720 \\
    \bottomrule
    \end{tabular}
    \caption{Accuracy of scattering models followed by LDR or DC layer on CIFAR-10 dataset.}
    \label{table:ch4-xp_ldr_scattering}
\end{table}

Dense layers and DCNNs are not designed to capture task-specific features such as the translation invariance inherently useful in image classification.
We can further improve the accuracy of such general-purpose architectures on image classification without dramatically increasing the number of trained parameters by stacking them on top of fixed (ie non-trained) transforms such as the scattering transform \cite{mallat2010recursive}.
In this section we compare the accuracy of various structured networks, enhanced with the scattering transform, on an image classification task, and run comparative experiments on CIFAR-10. 

Our test architecture consists of 2 depth scattering on the RGB images followed by a batch norm and LDR or DC layer.
To vary the number of parameters of Scattering+LDR architecture, we increase the rank of the matrix (stacking several LDR matrices quickly exhausted the memory).
The \Cref{figure:ch4-cifar10_with_channels_xp} and \Cref{table:ch4-xp_ldr_scattering} show the accuracy of these architectures given the number of trainable parameters.

First, we can see that the DCNN architecture very much benefits from the scattering transform and is able to reach a competitive accuracy over 78\%.
We can also see that scattering followed by a DC layer systematically outperforms scattering + LDR or scattering + Toeplitz-like with fewer parameters.

\section{Concluding Remarks}
\label{section:ch4-concluding_remarks}

This chapter dealt with the training and understanding of diagonal-circulant neural networks.
To the best of our knowledge, training such networks with a large number of layers had not been done before.
We also endowed this kind of architecture with theoretical guarantees, hence enriching and refining previous theoretical work from the literature.
More importantly, we showed that DCNNs outperform their competing structured alternatives, including the very recent general approach based on LDR networks.
Our results suggest that stacking diagonal-circulant layers with nonlinearities improves the convergence rate and the final accuracy of the network.
Formally proving these statements constitutes the future directions of this work.
We would like to generalize the good results of DCNNs to convolutional neural networks.
We also believe that circulant matrices deserve particular attention in deep learning because of their strong ties with convolutions: a circulant matrix operator is equivalent to the convolution operator with circular padding.
This fact makes any contribution to the area of circulant matrices particularly relevant to the field of deep learning with impacts beyond the problem of designing compact models.

%% file: figures/main/ch4-diagonal_circulant/illustration_properties.tex
\tikzset{%
  >={Latex[width=2mm,length=2mm]},
            base/.style = {rectangle, draw=black, text centered, font=\sffamily},
           other/.style = {base, fill=none,  minimum width=1.7cm, minimum height=0.7cm},
         ellipse/.style = {base}
}
\begin{tikzpicture}[every node/.style={fill=white, font=\sffamily}, align=center, scale=1.1]

    \draw (0,0) circle (2.5cm);
    \draw (0,0) circle (2.0cm);
    \draw (0,0) circle (1.5cm);
    \draw (0,0) circle (1.0cm);
    \draw (0,0) circle (0.5cm);
    
    \draw[ellipse, rotate=30, fill=gray, opacity=0.5] (0.1, -1.1) ellipse (2.0cm and 0.9cm);
    \draw[ellipse, rotate=30, fill=gray, opacity=0.5] (0.0, -0.8) ellipse (1.0cm and 0.45cm);

    \node[other, draw=none] at (0.20, 0.20) {$\mathcal{C}_{1}$};
    \node[other, draw=none] at (0.55, 0.55) {$\iddots$};
    \node[other, draw=none] at (0.90, 0.90) {$\mathcal{C}_{9}$};
    \node[other, draw=none] at (1.25, 1.25) {$\iddots$};
    \node[other, draw=none] at (1.60, 1.60) {$\mathcal{C}_{18}$};
    
    \node[other, draw=none] at (0.4, -0.6) {$\mathcal{R}_{1}$};
    \node[other, draw=none] at (1.0, -1.4) {$\mathcal{R}_{2}$};

\end{tikzpicture}

%% file: figures/main/ch4-diagonal_circulant/archi_youtube.tex
\tikzset{%
  >={Latex[width=2mm,length=2mm]},
            base/.style = {rectangle, draw=black, text centered, font=\sffamily},
             box/.style = {base, rounded corners, text depth=3cm, minimum height=4cm, minimum width=3cm},
     transparent/.style = {rectangle, draw=black},
       circulant/.style = {base, fill=yellow!30},
       embedding/.style = {base, fill=blue!30, minimum width=2.5cm, minimum height=1cm},
           other/.style = {base, fill=white!30,  minimum width=2cm, minimum height=1cm},
              fc/.style = {base, fill=orange!30, minimum width=1.5cm, minimum height=1cm},
          gating/.style = {base, fill=green!30, minimum width=2cm, text width=2cm, minimum height=1cm},
             moe/.style = {base, fill=purple!30, minimum width=1.5cm, minimum height=1cm},
}

\begin{tikzpicture}[every node/.style={fill=white, font=\sffamily}, align=center]

  \draw (0.0, +2.)  node [other, draw=none] {\textbf{Embedding}};
  \draw (+3.7, +2.)  node [other, draw=none] {\textbf{Dim Reduction}};
  \draw (+8.0, +2.)  node [other, draw=none] {\textbf{Classification}};

  \draw (0, +0.8)  node [embedding] {Video};
  \draw (0, -0.8)  node [embedding] {Audio};

  \draw (+2.5, +0.8)  node (fc) [fc] {FC};
  \draw (+2.5, -0.8)  node (fc) [fc] {FC};

  \draw (+4.75, 0)  node (fc) [other] {concat};
  \draw (+7.0, 0)  node (moe) [moe] {MoE};
  \draw (+9.25, 0)  node (gating2) [gating] {Context Gating};
 
  \draw (+1.5, +2) [dashed] -- (+1.5, -1.7);
  \draw (+6, +2) [dashed] -- (+6, -1.7);
  
\end{tikzpicture}

%% file: sources/main/ch5-lipschitz_regularization.tex
\chapter{Bound on the Lipschitz Constant of Convolution Layers}
\label{chapter:ch5-lipschitz_bound}
\localtoc

\section{Introduction}
\label{section:ch5-introduction}

In this chapter we introduce a new upper bound on the largest singular value of convolution layers that is both tight and easy to compute.
Instead of using the power method to iteratively approximate this value, we study the properties of \emph{doubly-block Toeplitz matrices} and its links with Fourier analysis. 
Our work is based on the result of~\citet{gray2006toeplitz} which states that an upper bound on the singular value of Toeplitz matrices can be computed from the inverse Fourier transform of the characteristic sequence of these matrices.
We first extend this result to doubly-block Toeplitz matrices (\ie, block Toeplitz matrices where each block is Toeplitz) and then to convolutional operators, which can be represented as stacked sequences of doubly-block Toeplitz matrices.
From our analysis immediately follows an algorithm for bounding the Lipschitz constant of a convolution layer, and by extension the Lipschitz constant of the whole network.
We theoretically study the approximation of this algorithm and show experimentally that it is more efficient and accurate than competing approaches.

Finally, we illustrate our approach on adversarial robustness.
Recent work has shown that empirical methods such as adversarial training (AT) offer poor generalization~\cite{schmidt2018adversarially}, and can be improved by applying Lipschitz regularization~\cite{farnia2018generalizable}.
To illustrate the benefit of our new method, we train a large Wide ResNet with Lipschitz regularization and show that it offers a significant improvement over adversarial training alone, and over other methods for Lipschitz regularization.
In summary, we make the three following contributions:
\begin{enumerate}
  \item We devise an upper bound on the singular values of the operator matrix of convolution layers by leveraging Toeplitz matrix theory and its links with Fourier analysis.
  \item We propose an efficient algorithm to compute this upper bound which enables its use in the context of Convolutional Neural Networks.
  \item We use our method to regularize the Lipschitz constant of neural networks for adversarial robustness and show that it offers a significant improvement over AT alone.
\end{enumerate}

\section{Results on the Spectrum of Matrices from the Toeplitz Family}
\label{section:ch5-results_on_the_spectrum_of_matrices_from_the_toeplitz_family}

\subsection{Upper-Bounds on the Largest Singular Value of Toeplitz and Block Toeplitz Matrices}
\label{subsection:ch5-upper_bounds_on_the_largest_singular_value_of_toeplitz_and_block_toeplitz_matrices}

Doubly-block Toeplitz matrices inherit the properties of Toeplitz and block Toeplitz matrices.
Recall that for Toeplitz and block Toeplitz matrices, there exist no closed-form expression to compute their eigenvalues.
However, we can represent Toeplitz and block Toeplitz matrices with a 2$\pi$-periodic function which can describe very precisely the spectrum of the matrices. 
Let $\{a_h\}_{h \in \Iset_n}$ be the characteristic sequence of a Toeplitz matrix $\Amat \in \Rbb^{n \times n}$ and let $\{\Bmatsf^{(h)}\}_{h \in \Iset_n}$ be the characteristic sequence of $m \times m$ blocks of a block Toeplitz matrix $\Bmat \in \Rbb^{nm \times nm}$ such that $\Amat = (a_{j-i})_{i,j\in \Iset_n}$ and $\Bmat = (\Bmatsf^{(j-i)})_{i,j \in \Iset_n}$ with $\Iset_n = \left\{ -n+1, \cdots, n-1 \right\}$.
Building on the results presented in the Background (\Cref{chapter:ch2-background}), we can define two trigonometric polynomials $f: \Rbb \rightarrow \Cbb$ and $F: \Rbb \rightarrow \Cbb^{m \times m}$ as follows:
\begin{equation}
  f(\omega) \triangleq \sum_{h \in \Iset_n} a_h e^{\ci h \omega} \quad \quad F(\omega) \triangleq \sum_{h \in \Iset_n} \Bmatsf^{(h)} e^{\ci h \omega} \enspace.
\end{equation}
$f(\omega)$  and $F(\omega)$ are the \emph{inverse Fourier transforms} of the sequences $\{a_h\}_{h \in \Iset_n}$ and $\{\Bmatsf^{(h)}\}_{h \in \Iset_n}$ respectively.
From there, inspired by the work done by~\citet{grenander1958toeplitz}, we recall from \Cref{subsection:ch2-a_fourier_representation_of_toeplitz_matrices} the operator $\Tmat$ mapping integrable functions to Toeplitz matrices:
\begin{equation} \label{equation:expression_toeplitz_matrix}
  \Tmat(f) \triangleq\leftmat\frac{1}{2\pi} \int_{0}^{2\pi} e^{-\ci(i-j)\omega} f(\omega) \,\diff \omega \rightmat_{i,j \in \Iset^+_n} \enspace,
\end{equation}
with this operator, we have $\Tmat(f) = \Amat$ and $\Tmat(F) = \Bmat$.

Now, we can state two known theorems which upper bound the largest singular value of Toeplitz and block Toeplitz matrices with respect to their generating functions.
\begin{theorem}[Bound on the singular values of Toeplitz matrices] \label{theorem:teoplitz_sup_singular}
  Let $f: \Rbb \rightarrow \Cbb$, be a continuous and $2\pi$-periodic function, then $\Tmat(f) \in \Rbb^{n \times n}$ is a Toeplitz matrix generated by the function $f$
  We can bound the largest singular value of the Toeplitz matrix $\Tmat(f)$ as follows:
  \begin{align}
    \sigma_1 \left( \Tmat(f) \right) \leq \sup_{\omega \in [0, 2\pi]} |f(\omega)|.
  \end{align}
  \removespace
\end{theorem}
\noindent
\Cref{theorem:teoplitz_sup_singular} is a direct application of Lemma 4.1 in~\citet{gray2006toeplitz} for real Toeplitz matrices. 

\begin{theorem}[Bound on the singular values of Block Toeplitz matrices ~\citet{gutierrez2012block}] \label{theorem:block_teoplitz_sup_singular}
  Let $F: \Rbb \rightarrow \Cbb^{m \times m}$ be a continuous and $2 \pi$-periodic matrix-valued function, then, $\Tmat(F) \in \Rbb^{mn \times mn}$ is a block Toeplitz matrix generated by the function $F$.
  We can bound the largest singular value of the block Toeplitz matrix $\Tmat(F)$ as follows:
  \begin{align}
    \sigma_1 \left( \Tmat(F) \right) \leq \sup_{\omega \in [0, 2\pi]} \sigma_1 \left( F\left( \omega \right) \right) .
  \end{align}
  \removespace
\end{theorem}

\subsection{Upper-Bound on the Largest Singular Value of Doubly-Block Toeplitz Matrices}
\label{subsection:ch5-bound_on_the_singular_value_of_doubly-block_toeplitz_matrices}

We extend the reasoning from Toeplitz and block Toeplitz matrices to doubly-block Toeplitz matrices (\ie block Toeplitz matrices where each block is also a Toeplitz matrix).
A doubly-block Toeplitz matrix can be generated by a function $f: \Rbb^2 \rightarrow \Cbb$ using the 2-dimensional inverse Fourier transform.
For this purpose, we define an operator $\Dmat$ which maps a function $f: \Rbb^2 \rightarrow \Cbb$ to a doubly-block Toeplitz matrix of size $nm \times nm$.
For the sake of clarity, the dependence of $\Dmat(f)$ on $m$ and $n$ is omitted.
Let $\Dmat(f) \triangleq \leftmat \Dmat_{i,j}(f)\rightmat_{i,j \in \Iset^+_n}$ where $\Dmat_{i,j}(f)$ is a $m \times m$ matrix defined as:
\begin{equation} \label{equation:doubly_block_toeplitz_operator}
  \Dmat_{i,j}(f) \triangleq \leftmat \frac{1}{4\pi^{2}} \int_{0}^{2\pi} \int_{0}^{2\pi} e^{- \ci \left((i-j)\omega_{1}+(k-l)\omega_{2}\right)}f(\omega_{1},\omega_{2}) \,\diff \omega_{1} \,\diff \omega_{2} \rightmat_{k,l \in \Iset^+_m} \enspace.
\end{equation}

We are now able to combine \Cref{theorem:teoplitz_sup_singular,theorem:block_teoplitz_sup_singular} to bound the largest singular value of doubly-block Toeplitz matrices with respect to their generating functions. 
Note that in the following, we only consider generating functions as trigonometric polynomials with real coefficients therefore the matrices generated by $\Dmat(f)$ are real.

\begin{maintheorem}[Bound on the largest singular value of a Doubly-Block Toeplitz Matrix] \label{theorem:doubly_block_teoplitz_sup_singular}
  Let $f: \Rbb^2 \rightarrow \Cbb$ be a multivariate trigonometric polynomial of the form:
  \begin{equation}\label{equation:muli_variate_poly_on_M}
    f(\omega_1, \omega_2) \triangleq \sum_{h_1 \in \Iset_n} \sum_{h_2 \in \Iset_m} d_{h_1, h_2} e^{\ci (h_1 \omega_1 + h_2 \omega_2)}.
  \end{equation}
  Then, $\Dmat(f) \in \Rbb^{nm \times nm}$ is a doubly-block Toeplitz matrix where $d_{h_{1},h_{2}}$ is the ${h_2}^\textrm{th}$ scalar of the ${h_1}^\textrm{th}$ block of the matrix.
  We can bound the largest singular value of the matrix $\Dmat(f)$ as follows:
  \begin{align}
    \sigma_{1} \left( \Dmat(f) \right) &\leq \sup_{\omega_1, \omega_2 \in [0, 2\pi]^2}|f(\omega_1,\omega_2)|
  \end{align}
  \removespace
\end{maintheorem}

%

\begin{proof}[\Cref{theorem:doubly_block_teoplitz_sup_singular}]
  By definition, a doubly-block Toeplitz matrix is a block matrix where each block is a Toeplitz matrix.
  Let $\Amat$ be a $mn \times mn$ doubly-block Toeplitz matrices determined by the sequence of blocks $\{\Amatsf^{(-n+1)}, \dots, \Amatsf^{(n-1)} \}$ of size $m \times m$ where the blocks $\Amatsf$ are Toeplitz matrices such that $\Amatsf^{(i)}$ is determined by the sequence of scalars $\{d_{i, -m+1}, \dots, d_{i, m-1} \}$. 
  Therefore the matrix $\Amat$ can be expressed with the operator $\Tmat$ with the matrix-valued generating function $F: \Rbb \rightarrow \Cbb^{n \times n}$ such that:
  \begin{equation}
    F(\omega_1) = \sum_{h_1 \in \Iset_n} \Amatsf^{(h_1)} e^{\ci h_1\omega_1}
  \end{equation}
  From \Cref{theorem:block_teoplitz_sup_singular} we have:
  \begin{equation} \label{equation:th_bound_block_toeplitz}
    \sigma_1\left(\Tmat(F) \right) \leq \sup_{\omega_1 \in [0,2\pi] } \sigma_{1}\left( F(\omega_1) \right)
  \end{equation}

  \noindent
  Note that because the function $F$ is a linear combination of the Toeplitz matrices $\Amatsf$ and that Toeplitz matrices are closed under addition and scalar product, $F(\omega_1)$ is also a Toeplitz matrix of size $m \times m$. 
  Therefore, we can define a function $f: \Rbb^2 \rightarrow \Cbb$ such that: 
  \begin{align}
    f(\omega_1,\omega_2) &= \sum_{h_2 \in \Iset_m} \left[ F(\omega_1) \right]_{h_2} e^{\ci h_2 \omega_2} \\
    f(\omega_1,\omega_2) &= \sum_{h_2 \in \Iset_m} \left[ \sum_{h_1 \in \Iset_n} \Amatsf^{(h_1)} e^{\ci h_1\omega_1} \right]_{h_2} e^{\ci h_2 \omega_2} \\
    f(\omega_1,\omega_2) &= \sum_{h_2 \in \Iset_m} \left[ \sum_{h_1 \in \Iset_n} \Amatsf^{(h_1)} \right]_{h_2} e^{\ci \left( h_1\omega_1 + h_2 \omega_2 \right)} \\
    f(\omega_1,\omega_2) &= \sum_{h_1 \in \Iset_n} \sum_{h_2 \in \Iset_m} d_{h_1, h_2} e^{\ci \left( h_1\omega_1 + h_2 \omega_2\right)} \enspace,
  \end{align}




  \noindent
  From \Cref{theorem:teoplitz_sup_singular}, we can write:
  \begin{align}
      \sigma_{1}\left( F(\omega_1) \right) &\leq \sup_{\omega_2 \in [0,2\pi]} \left| f(\omega_1, \omega_2) \right| \\
      \Rightarrow \sup_{\omega_1 \in [0,2\pi]} \sigma_{1}\left( F(\omega_1) \right) &\leq  \sup_{\omega_1, \omega_2 \in [0,2\pi]^2} \left| f(\omega_1, \omega_2) \right| \\
      \Rightarrow \sigma_1\left(\Tmat(F) \right) &\leq \sup_{\omega_1, \omega_2 \in [0,2\pi]^2} \left| f(\omega_1, \omega_2) \right|
  \end{align}
  Because the function $f(\omega_1,\ \cdot\ )$ is the generating function of $F(\omega_1)$ it is easy to show that the function $f$ is also the generating function of the matrix $\Tmat(F)$. Therefore, $\Tmat(F) = \Dmat(f)$ which concludes the proof. 
\end{proof}

\section{Extending the Bound to Convolutional Layers}
\label{section:ch5-bound_on_the_singular_values_of_convolutional_layers}

From now on, without loss of generality, we will assume that $n=m$ to simplify notations.
A discrete convolution operation with a 2-dimensional kernel applied on a 2-dimensional signal (\eg, an image) is equivalent to a matrix multiplication with a doubly-block Toeplitz matrix~\cite{jain1989fundamentals}.
In practice, the input signal often has 3 or more dimensions called \emph{channels} (for example, RGB images have 3 channels, one for each color).
If we denote $\cin$, the number of channels of the input signal, then, the input signal is a tensor of size $\cin \times n \times n$.
Moreover, when we perform multiple convolutions on the same signal the output signal will have multiple channels denoted $\cout$. 
Therefore, the kernel is defined as a 4-dimensional tensor of size: $\cout \times \cin \times s \times s$. 
The operation performed by a 4-dimensional kernel on a 3-dimensional signal can be formulated as the concatenation (horizontally and vertically) of doubly-block Toeplitz matrices.
Hereafter, we bound the singular value of multiple vertically stacked doubly-block Toeplitz matrices which corresponds to the operation performed by a 3-dimensional kernel with $\cout = 1$ on a 3-dimensional signal.


\begin{theorem}[Bound on the largest singular value of stacked Doubly-block Toeplitz matrices] \label{theorem:bound_sv_stacked_dbt} 
  Consider doubly-block Toeplitz matrices $\Dmat(f_1), \dots, \Dmat(f_{\cin})$ where each $f_i: \Rbb^2 \rightarrow \Cbb$ is a multivariate polynomial of the same form as \Cref{equation:muli_variate_poly_on_M}.
  Construct a matrix $\Mmat$ with $\cin\times n^2$ rows and $n^2$ columns, as follows:
  \begin{equation}
    \Mmat \triangleq \leftmat \Dmat^\top(f_1), \dots, \Dmat^\top(f_{\cin}) \rightmat^\top .
  \end{equation}
  Then, we can bound the largest singular value of the matrix $\Mmat$ as follows:
  \begin{align} \label{equation:bound_asymptotic_equiv}
    \sigma_1 \left( \Mmat \right) &\leq \sup_{\omega_1, \omega_2 \in \left[0, 2\pi\right]^2} \sqrt{ \sum_{i=1}^{\cin} \left|f_i\right (\omega_1, \omega_2)|^2} \enspace.
  \end{align}
\end{theorem}

\noindent
In order to prove \Cref{theorem:bound_sv_stacked_dbt}, we will need the following lemmas:

\begin{lemma}[\citet{gutierrez2012block}] \label{theorem:properties_block_toeplitz}
  Let $f:\Rbb^2 \rightarrow \Cbb$ and $g:\Rbb^2 \rightarrow \Cbb$ be two continuous and $2\pi$-periodic functions. Let $\Dmat(f)$ and $\Dmat(g)$ be doubly-block Toeplitz matrices generated by the functions $f$ and $g$ respectively.
  Then:
  \begin{itemize}
      \item $\Dmat^\top(f) = \Dmat(f^*)$
      \item $\Dmat(f) + \Dmat(g) = \Dmat(f + g)$
  \end{itemize}
  \removespace
\end{lemma}

\begin{lemma}[\citet{serra1994preconditioning}] \label{theorem:block_toeplitz_hermitian}
  If the doubly-block Toeplitz matrix $\Dmat(f)$ is generated by a function $f: \Rbb^2 \rightarrow \Rbb$, then the matrix $\Dmat(f)$ is Hermitian. 
\end{lemma}

\begin{lemma}[\citet{serra1994preconditioning}] \label{theorem:block_toeplitz_positive_definite}
  If the doubly-block Toeplitz matrix $\Dmat(f)$ is generated by a non-negative function $f$ not identically zero, then the matrix $\Dmat(f)$ is positive definite. 
\end{lemma}

\begin{lemma}[\citet{zhang2011matrix}] \label{theorem:diff_positive_semidefinite_matrices}
Let $\Amat$ and $\Bmat$ be Hermitian positive semi-definite matrices. If $\Amat - \Bmat$ is positive semi-definite, then:
  \begin{equation}
      \lambda_1 \left( \Bmat \right) \leq \lambda_1 \left( \Amat \right)
  \end{equation}
  \removespace
\end{lemma}

We need now to extend the well known Widom identity \cite{widom1976asymptotic} which expresses the relation between Toeplitz and Hankel matrices to doubly-block Toeplitz and Hankel matrices.
Let us first generalize the doubly-block Toeplitz operator presented in~\Cref{subsection:ch5-bound_on_the_singular_value_of_doubly-block_toeplitz_matrices}.

Given a function $f:\Rbb^2 \rightarrow \Cbb$, let $\Gmat^{\alpha_p} (f)$ be a matrix such that $\Gmat^{\alpha_p} (f) = \leftmat \Gmat^{\alpha_p}_{i,j}(f) \rightmat_{i,j \in \Iset^+_n}$ where $\Gmat^{\alpha_p}_{i,j}$ is defined as:
\begin{equation}
  \Gmat^{\alpha_p}_{i,j}(f) =\leftmat \frac{1}{4\pi^{2}} \int_{0}^{2\pi} \int_{0}^{2\pi} e^{-\ci \alpha_p(i, j, k, l, \omega_1, \omega_2)} f(\omega_{1},\omega_{2}) \,\diff \omega_{1} \,\diff \omega_{2}
  \rightmat_{k,l \in \Iset^+_n} \enspace.
\end{equation}
Note that as with the operator $\Dmat(f)$ we only consider generating functions as trigonometric polynomials with real coefficients therefore the matrices generated by $\Gmat(f)$ are real. 
And as with the operator $\Dmat(f)$, the matrices generated by the operator $\Gmat^{\alpha_p}$ are of size $n^2 \times n^2$. 

\noindent
We will use the following $\alpha$ functions:
\begin{itemize}
    \item[] $\alpha_0(i, j, k, l, \omega_1, \omega_2) = (-j-i-1)\omega_1 + (k-l)\omega_2$
    \item[] $\alpha_1(i, j, k, l, \omega_1, \omega_2) = (i-j)\omega_1 + (-l-k-1)\omega_2$
    \item[] $\alpha_2(i, j, k, l, \omega_1, \omega_2) = (-j-i-1)\omega_1 + (-l-k-1)\omega_2$
    \item[] $\alpha_3(i, j, k, l, \omega_1, \omega_2) = (-j-i+n)\omega_1 + (-l-k-1)\omega_2$
\end{itemize}

\noindent
We now present the generalization of the Widom identity for Doubly-Block Toeplitz matrices below:
\begin{lemma}[Extension of Widom Identity to doubly-block operators] \label{lemma:ch5-widom_idenity}
  Let $f:\Rbb^2 \rightarrow \Cbb$ and $g:\Rbb^2 \rightarrow \Cbb$ be two continuous and $2\pi$-periodic functions. 
  Let $fg$ be the product of the functions $f$ and $g$ such that $(fg)(\omega_1, \omega_2) = f(\omega_1, \omega_2) g(\omega_1, \omega_2)$.
  We can decompose the Doubly-Block Toeplitz matrix $\Dmat(fg)$ as follows:
  \begin{equation}
      \Dmat(fg) = \Dmat(f)\Dmat(g) + \sum_{p=0}^3 \Gmat^{\alpha_p \top}(f^*) \Gmat^{\alpha_p}(g) + \Jmat_{n^2} \left( \sum_{p=0}^3 \Gmat^{\alpha_p \top}(f) \Gmat^{\alpha_p }(g^*) \right) \Jmat_{n^2}.
  \end{equation}
  where $\Jmat_{n^2}$ is the reflection of the identity matrix of size $n^2 \times n^2$.
\end{lemma}

\noindent
The proof of this Lemma is delayed to Appendix~\ref{appendix:ap1-proof_of_the_generalization_of_widom_identity}.
Now we have all the elements to prove \Cref{theorem:bound_sv_stacked_dbt} which bounds the largest singular value of vertically stacked doubly-block Toeplitz matrices with their generating functions. 

\begingroup
\addtolength{\jot}{1.5em}

\begin{proof}[\Cref{theorem:bound_sv_stacked_dbt}]

Consider doubly-block Toeplitz matrices $\Dmat(f_1), \dots, \Dmat(f_{\cin})$ where each $f_i: \Rbb^2 \rightarrow \Cbb$ is a multivariate polynomial of the same form as \Cref{equation:muli_variate_poly_on_M}.
Construct a matrix $\Mmat$ with $\cin\times n^2$ rows and $n^2$ columns, such that:
\begin{equation}
  \Mmat \triangleq \leftmat \Dmat^\top(f_1), \dots, \Dmat^\top(f_{\cin}) \rightmat^\top .
\end{equation}

\noindent
First, let us observe the following equality which relates the largest singular value of the matrix $\Mmat$ and the largest eigenvalue of the sum of the doubly-block Toeplitz matrices composing $\Mmat$:
\begin{equation}
    \sigma_1^2 \left( \Mmat \right) = \lambda_1 \left( \Mmat^{\top} \Mmat \right) = \lambda_1 \left( \sum_{i=1}^{\cin} \Dmat^{\top} \left(f_i \right) \Dmat (f_i) \right).
\end{equation}

\noindent
Secondly, let us bound the largest eigenvalue of the sum of doubly-block Toeplitz generated by $|f_i|^2$:
\begin{align}
  \lambda_1 \left( \sum_{i=1}^{\cin} \Dmat \left(|f_i|^2 \right) \right) \quad &= \quad \lambda_1 \left( \Dmat \left( \sum_{i=1}^{\cin} |f_i|^2 \right) \right) \\ 
  \quad &= \quad \sigma_1 \left( \Dmat \left( \sum_{i=1}^{\cin} |f_i|^2 \right) \right) \\
  \quad &\leq \quad \sup_{\omega_1, \omega_2 \in [0, 2\pi]^2} \sum_{i=1}^{\cin} |f_i(\omega_1, \omega_2)|^2.
\end{align}
where the first equality is due to \Cref{theorem:properties_block_toeplitz}, the second equality is due to \Cref{theorem:block_toeplitz_hermitian} and the last inequality is due to \Cref{theorem:doubly_block_teoplitz_sup_singular}.
To finalize the proof, we need to demonstrate that the following inequality holds:
\begin{equation} \label{equation:ch5-eq2}
    \lambda_1 \left( \sum_{i=1}^{\cin} \Dmat^{\top} \left(f_i \right) \Dmat (f_i) \right) \leq \lambda_1 \left( \Dmat \left( \sum_{i=1}^{\cin} |f_i|^2 \right) \right). 
\end{equation}

\noindent
In order to prove the inequality above, let us study the positive definiteness of the following matrix:
\begin{equation}
  \Dmat \left( \sum_{i=1}^{\cin} |f_i|^2 \right) - \sum_{i=1}^{\cin} \Dmat^{\top} \left(f_i \right) \Dmat (f_i),
  \label{equation:ch5-eq3}
\end{equation}

\noindent
One can observe that the term $\Dmat \left( \sum_{i=1}^{\cin} |f_i|^2 \right)$ of \Cref{equation:ch5-eq3} is a real symmetric positive definite matrix by \Cref{theorem:block_toeplitz_positive_definite,theorem:block_toeplitz_hermitian}. 
Furthermore, the term $\sum_{i=1}^{\cin} \Dmat^{\top} \left(f_i \right) \Dmat (f_i)$ of \Cref{equation:ch5-eq3} is a sum of positive semi-definite matrices.
Therefore, if the subtraction of the two is positive semi-definite, one could apply \Cref{theorem:diff_positive_semidefinite_matrices} to prove the \Cref{equation:ch5-eq2}. 
We know from \Cref{lemma:ch5-widom_idenity} that 
\begin{equation}
  \Dmat(fg) - \Dmat(f)\Dmat(g) = \sum_{p=0}^3 \Gmat^{\alpha_p \top}(f^*) \Gmat^{\alpha_p}(g) + \Jmat \left( \sum_{p=0}^3 \Gmat^{\alpha_p \top}(f) \Gmat^{\alpha_p}(g^*) \right) \Jmat.
\end{equation}
By choosing $f = f^*$, $g = f$ and with the use of \Cref{theorem:properties_block_toeplitz}, we obtain:
\begin{align} \label{equation:widom_identity_block_topelitz}
  \Dmat(f^* f) - \Dmat(f^*)\Dmat(f)
  &= \Dmat(|f|^2) - \Dmat^{\top}(f)\Dmat(f) \\
  &= \sum_{p=0}^3 \Gmat^{\alpha_p \top}(f)\Gmat^{\alpha_p}(f) + \Jmat \left( \sum_{p=0}^3 \Gmat^{\alpha_p \top}(f^*)\Gmat^{\alpha_p} (f^*) \right) \Jmat .
\end{align}

\noindent
From \Cref{equation:widom_identity_block_topelitz}, we can see that the matrix $\Dmat(|f|^2) - \Dmat^{\top}(f)\Dmat(f)$
is positive semi-definite because it can be decomposed into a sum of positive semi-definite matrices and because positive semi-definiteness is closed under addition, we have:
\begin{align}
    \sum_{i=1}^{\cin} \leftmat \Dmat \left( |f_i|^2 \right) - \Dmat^{\top} \left(f_i \right) \Dmat (f_i) \rightmat &\geq 0
\end{align}

\noindent
By re-arranging and with the use \Cref{theorem:properties_block_toeplitz}, we obtain:
\begin{align}
   \sum_{i=1}^{\cin} \Dmat \left( |f_i|^2 \right) - \sum_{i=1}^{\cin} \leftmat \Dmat^{\top} \left(f_i \right) \Dmat (f_i) \rightmat &\geq 0 \\
    \Dmat \left( \sum_{i=1}^{\cin} |f_i|^2 \right) - \sum_{i=1}^{\cin} \leftmat \Dmat^{\top} \left(f_i \right) \Dmat (f_i) \rightmat &\geq 0
\end{align}

\noindent
We can conclude that the \Cref{equation:ch5-eq2} is true and therefore by \Cref{theorem:diff_positive_semidefinite_matrices} we have:
\begin{align}
  \lambda_1 \left( \sum_{i=1}^{\cin} \Dmat^{\top} \left(f_i \right) \Dmat (f_i) \right) &\leq \lambda_1 \left( \Dmat \left( \sum_{i=1}^{\cin} |f_i|^2 \right) \right) \\
  \sigma_1^2 \left( \Mmat \right) &\leq \sup_{\omega_1, \omega_2 \in [0, 2\pi]^2} \sum_{i=1}^{\cin} |f_i(\omega_1, \omega_2)|^2 \\
  \sigma_1 \left( \Mmat \right) &\leq \sup_{\omega_1, \omega_2 \in [0, 2\pi]^2} \sqrt{ \sum_{i=1}^{\cin} |f_i(\omega_1, \omega_2)|^2 }
\end{align}
which concludes the proof. 
\end{proof}

\endgroup

To have a bound on the full convolution operation, we extend \Cref{theorem:bound_sv_stacked_dbt} to take into account the number of output channels.
The matrix of a full convolution operation is a block matrix where each block is a doubly-block Toeplitz matrix.
Below, we present our main result:

\begin{maintheorem}[Bound on the largest singular value on the discrete convolution operation] \label{theorem:ch5-bound_max_sv_convolution} 
  Let us define doubly-block Toeplitz matrices $\Dmat(f_{1, 1}), \dots, \Dmat(f_{\cin, \cout})$ where each $f_{i,j}: \Rbb^2 \rightarrow \Cbb$ is a multivariate polynomial of the same form as \Cref{equation:muli_variate_poly_on_M}.
  Construct a matrix $\Mmat$ with $\cin\times n^2$ rows and $\cout\times n^2$ columns.
  We can bound the largest singular value of the matrix $\Mmat$ as follows: 
  \begin{equation} \label{equation:lipbound_sv}
     \sigma_1 \left( \Mmat \right) \leq \sqrt{ \sum_{i=1}^{\cout} \sup_{\omega_1, \omega_2 \in [0, 2\pi]^2} \sum_{j = 1}^{\cin} \left|f_{ij}(\omega_1, \omega_2) \right|^2 } .
  \end{equation} 
\end{maintheorem}

First, in order to prove \Cref{theorem:ch5-bound_max_sv_convolution}, we will need the following lemma which bounds the singular values of a matrix constructed from the concatenation of multiple matrices.
\begin{lemma}[Bound on the singular values of concatenation of matrices] \label{theorem:ch2-bound_concatenation_matrices}
  Let us define matrices $\Amat^{(1)}, \dots, \Amat^{(p)}$ with $\Amat^{(i)} \in \Rbb^{n \times n}$. Let us construct the matrix $\Mmat \in \Rbb^{n \times pn}$ as follows:
  \begin{equation}
    \Mmat \triangleq \leftmat \Amat^{(1)}, \dots, \Amat^{(p)} \rightmat
  \end{equation}
  where $\leftmat\ \cdot\ \rightmat$ define the concatenation operation. Then, we can bound the singular values of the matrix $\Mmat$ as follows:
  \begin{equation}
    \sigma_1(\Mmat) \leq \sqrt{\sum_{i=1}^p \sigma_1(\Amat^{(i)})^2}
  \end{equation}
\end{lemma}

\begingroup
\allowdisplaybreaks
\addtolength{\jot}{1.5em}

\begin{proof}[\Cref{theorem:ch2-bound_concatenation_matrices}]
  \begin{align}
    \sigma_1\left(\Mmat\right)^2 &= \lambda_1\left(\Mmat \Mmat^\top\right) \\
    &= \lambda_1\left( \sum_{i=1}^p\Amat^{(i)} \Amat^{(i)\top}  \right) \\
    &\leq \sum_{i=1}^p \lambda_1\left( \Amat^{(i)} \Amat^{(i)\top}  \right) \\
    &\leq \sum_{i=1}^p \sigma_1\left( \Amat^{(i)} \right)^2 \\
    \Leftrightarrow \sigma_1\left(\Mmat\right) &\leq \sqrt{\sum_{i=1}^p \sigma_1(\Amat^{(i)})^2}
  \end{align}
  which concludes the proof.
\end{proof}

\noindent
Now, the proof of \Cref{theorem:ch5-bound_max_sv_convolution} is a combination of \Cref{theorem:ch2-bound_concatenation_matrices} and \Cref{theorem:bound_sv_stacked_dbt}.

\begin{proof}[\Cref{theorem:ch5-bound_max_sv_convolution}]
  Let us define the matrix $\Mmat^{(i)}$ as follows:
\begin{equation}
  \Mmat^{(i)} = \leftmat \Dmat(f_{1, i})^\top, \dots, \Dmat(f_{\cin, i})^\top \rightmat^\top .
\end{equation}
We can express the matrix $\Mmat$ as the concatenation of multiple $\Mmat^{(i)}$ matrices:
\begin{equation}
  \Mmat = \leftmat \Mmat^{(1)}, \dots, \Mmat^{(\cout)} \rightmat
\end{equation}
Then, we can bound the singular values of the matrix $\Mmat$ as follows:
\begin{align}
  \sigma_1\left(\Mmat\right) &\leq \sqrt{\sum_{i=1}^{\cout} \sigma_1(\Mmat^{(i)})^2} \\
  \sigma_1\left(\Mmat\right) &\leq \sqrt{\sum_{j=1}^{\cout} \sup_{\omega_1, \omega_2 \in [0, 2\pi]^2} \sum_{i=1}^{\cin} |f_{i,j}(\omega_1, \omega_2)|^2 }
\end{align}
where the first inequality is due to \Cref{theorem:ch2-bound_concatenation_matrices} and the second one is due to \Cref{theorem:bound_sv_stacked_dbt}.
This concludes the proof.
\end{proof}

\endgroup

\Cref{theorem:ch5-bound_max_sv_convolution} depends on the convolution matrix $\Mmat$, however, we can easily formulate the bound with the values of a 4-dimensional kernel.
Let us define a kernel $\Kmat \in \Rbb^{\cout \times \cin \times s \times s}$, a padding $p \in \Nbb$ and $d = \lfloor s / 2 \rfloor$ the degree of the trigonometric polynomial, then:
\begin{equation}
  f_{ij}(\omega_1, \omega_2) = \sum_{h_1 = -d}^d \sum_{h_2 = -d}^d k_{i, j, h_1,h_2} e^{\ci (h_1 \omega_1 + h_2 \omega_2)}.
\end{equation}
where $k_{i, j, h_1,h_2} = \leftmat \Kmat \rightmat_{i, j, a, b}$ with $a =  s - p - 1 + h_1$ and $b =  s - p - 1 + h_2$.

In the rest of the chapter, we will refer to the bound in \Cref{theorem:ch5-bound_max_sv_convolution} applied to a kernel as $\lipbound$ and we denote $\lipbound(\Kmat)$ the Lipschitz upper bound of the convolution performed by the kernel $\Kmat$.

\section{Computation and Performance Analysis of LipBound}
\label{section:ch5-computation_and_performance_analysis_of_lipbound}

This section aims at analyzing the bound introduced in \Cref{theorem:ch5-bound_max_sv_convolution}.
First, we present an algorithm to efficiently compute the bound, we analyze its tightness by comparing it against the true largest singular value.
Finally, we compare the efficiency and the accuracy of our bound against the state-of-the-art methods.

\subsection{The Maximum Modulus of a Trigonometric Polynomial}
\label{subsection:ch5-computing_the_maximum_modulus_of_a_trigonometric_polynomial}

In order to compute $\lipbound$ from \Cref{theorem:ch5-bound_max_sv_convolution}, we have to compute the maximum modulus of several trigonometric polynomials.
However, finding the maximum modulus of a trigonometric polynomial has been known to be NP-Hard~\cite{pfister2018bounding}, and in practice they exhibit low convexity (see \Cref{figure:contour_plot_trigonometric_polynomials}).
We found that for 2-dimensional kernels, a simple grid search algorithm such as PolyGrid (see \Cref{algorithm:ch5-polygrid}), works better than more sophisticated approximation algorithms (\eg ~\citet{green1999calculating,de2009finding}).
This is because the complexity of the computation depends on the degree of the polynomial which is equal to $\lfloor s / 2 \rfloor$ where $s$ is the size of the kernel and is usually small in most practical settings (\eg $s=3$).
Furthermore, the grid search algorithm can be parallelized effectively on CPUs or GPUs and runs within less time than alternatives with lower asymptotic complexity. 

\begin{figure}[htb]
  \centering
  \begin{subfigure}[b]{.49\textwidth}
    \centering
    \includegraphics[scale=0.35]{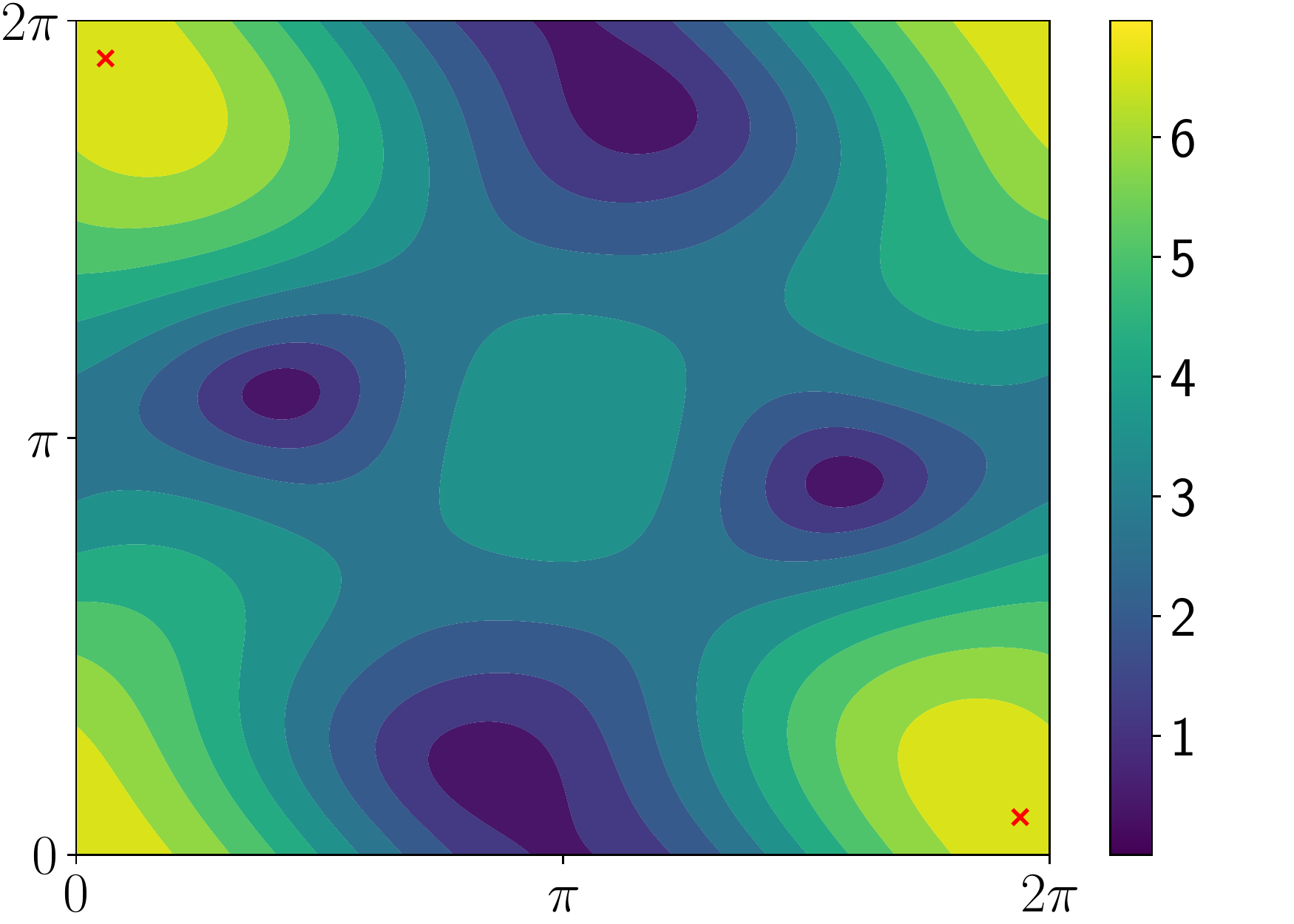}
    \caption{kernel $1\times3\times3$}
  \end{subfigure}
  \hfill
  \begin{subfigure}[b]{.49\textwidth}
    \centering
    \includegraphics[scale=0.35]{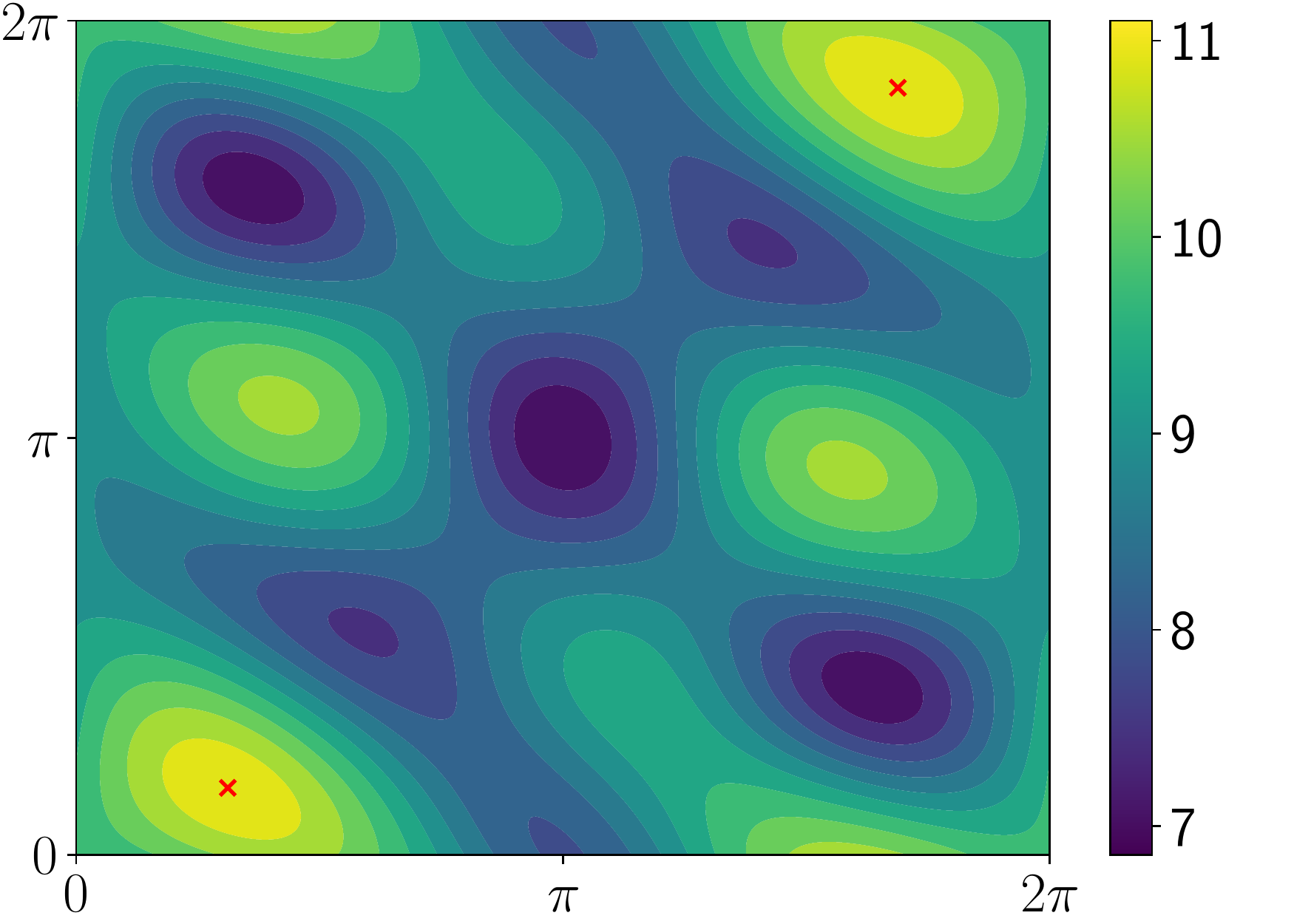}
    \caption{kernel $9\times3\times3$}
  \end{subfigure}
  \par\bigskip
  \begin{subfigure}[b]{.49\textwidth}
    \centering
    \includegraphics[scale=0.35]{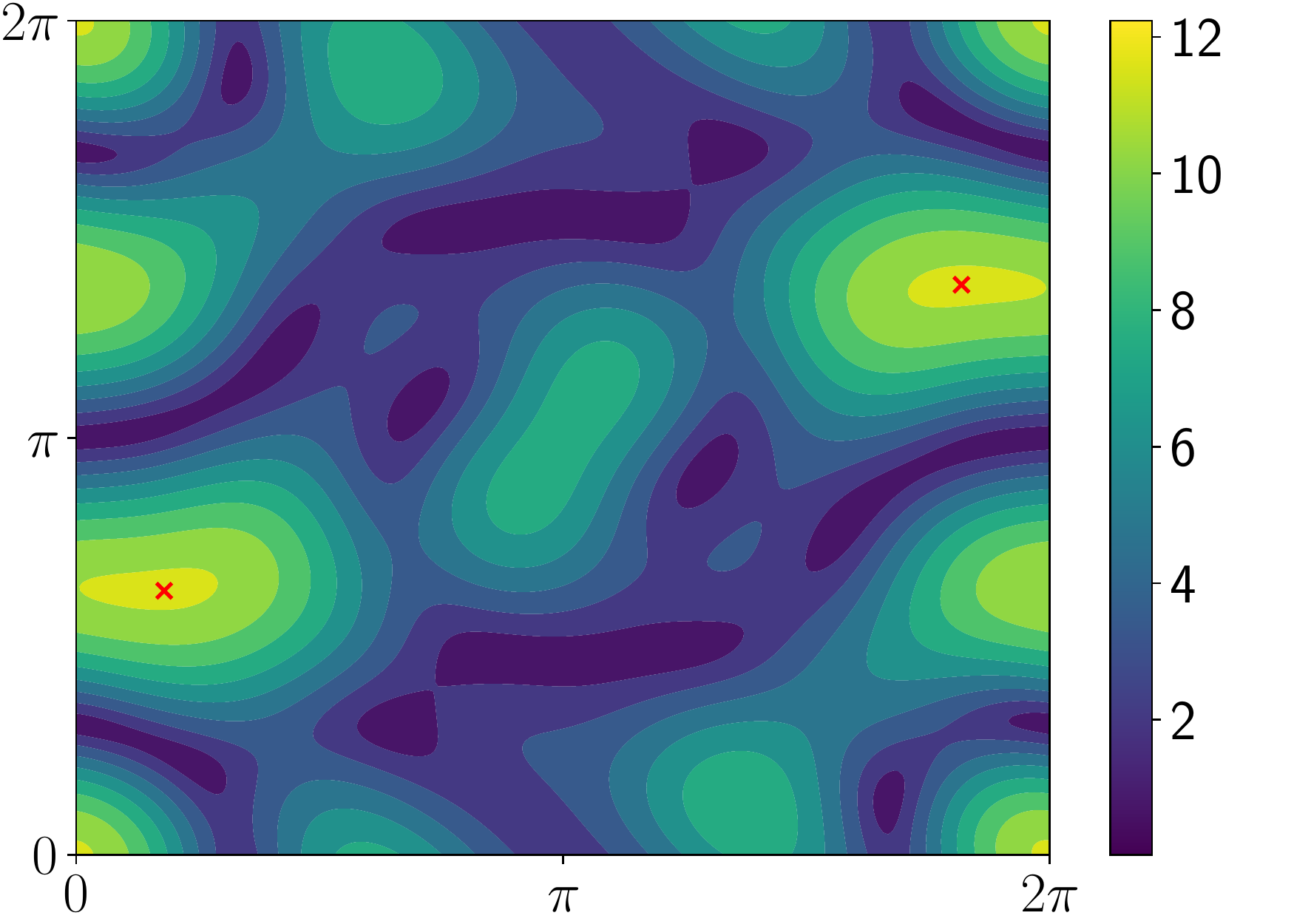}
    \caption{kernel $1\times5\times5$}
  \end{subfigure}
  \begin{subfigure}[b]{.49\textwidth}
    \centering
    \includegraphics[scale=0.35]{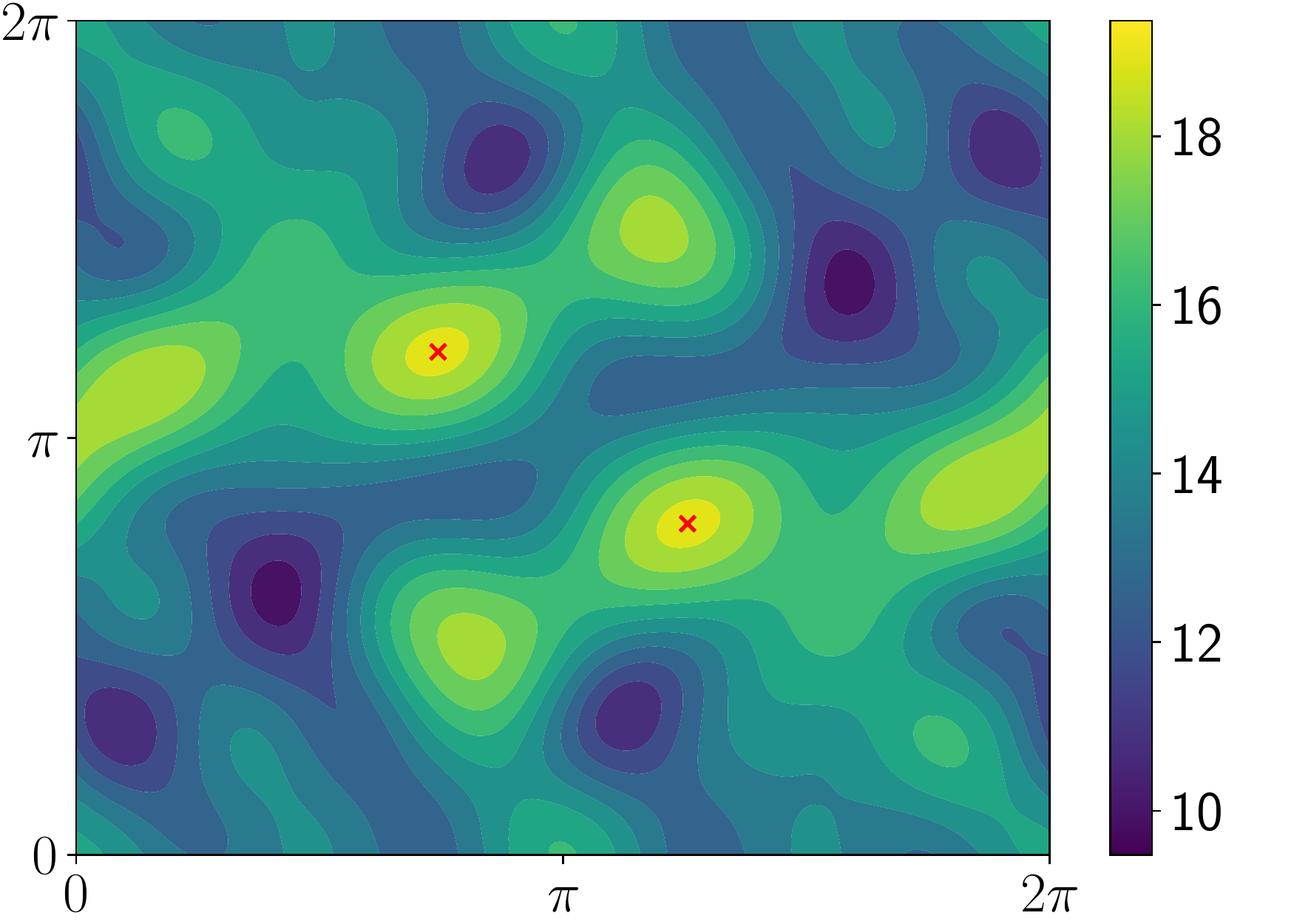}
    \caption{kernel $9\times5\times5$}
  \end{subfigure}
  \caption{Contour plots of multivariate trigonometric polynomials where the values of the coefficient are the values of a random convolutional kernel. The red dots in the figures represent the maximum modulus of the trigonometric polynomials.}
  \label{figure:contour_plot_trigonometric_polynomials}
\end{figure}%

\begin{algorithm}[htb]
  \begin{algorithmic}[1]
    \Procedure{PolyGrid}{$f, S$}\Comment{polynomial $f$, number of samples $S$}
      \State $\sigma \gets 0$, $\omega_1 \gets 0$, $\epsilon \gets \frac{2\pi}{S}$
      \For{$i=0$ \textbf{to} $S-1$}
        \State $\omega_2 \gets 0$
	\For{$j=0$ \textbf{to} $S-1$}
	  \State $\omega_2 \gets \omega_2 + \epsilon$
	  \State $\sigma \gets \max( \sigma, |f(\omega_1, \omega_2)|)$
	\EndFor
	\State $\omega_1 \gets \omega_1 + \epsilon$
      \EndFor
      \State \textbf{return} $\sigma$ \Comment{approximated maximum modulus of $f$}
    \EndProcedure
  \end{algorithmic}
  \caption{PolyGrid Algorithm}
  \label{algorithm:ch5-polygrid}
\end{algorithm}

To fix the number of samples $S$ in the grid search, we rely on the work of~\cite{pfister2018bounding}, who has analyzed the quality of the approximation depending on $S$.
Following this work we first define $\Theta_S$, the set of $S$ equidistant sampling points as follows:
\begin{equation}
  \Theta_S \triangleq \left\{ \omega \mid \omega = k \cdot \frac{2\pi}{S} \mbox{ with }  k = 0, \dots, S-1 \right\}.
\end{equation}
Then, for a trigonometric polynomial $f: [0, 2\pi]^2 \rightarrow \Cbb$, we have:
\begin{equation}
  \max_{\omega_1, \omega_2 \in [0,2\pi]^2} \left| f(\omega_1, \omega_2) \right| \leq (1 - \alpha)^{-1} \max_{\omega_1', \omega_2' \in \Theta_S^2} \left| f(\omega_1', \omega_2') \right|,
\end{equation}
where $d$ is the degree of the polynomial and $\alpha = 2d / S$.
For a $3\times3$ kernel which gives a trigonometric polynomial of degree 1, we use $S = 10$ which gives $\alpha = 0.2$.
Using this result, we can now compute $\lipbound$ for a convolution operator with $\cout$ output channels as per \Cref{theorem:bound_sv_stacked_dbt}.
 

\subsection{Analysis of the Tightness of the Bound}
\label{subsection:ch2-analysis_of_the_tightness_of_the_bound}

In this section, we study the tightness of the bound with respect to the dimensions of the doubly-block Toeplitz matrices.
For each $n \in \Nbb$, we define the matrix  $\Mmat^{(n)}$ of size $kn^2 \times n^2$ as follows:
\begin{equation}
  \Mmat^{(n)} \triangleq \textstyle \leftmat \Dmat^{(n)\top}(f_1), \dots, \Dmat^{(n)\top}(f_k) \textstyle \rightmat^\top
\end{equation}
where the matrices $\Dmat^{(n)}(f_i)$ are of size $n^2 \times n^2$. 
To analyze the tightness of the bound, we define the function $\Gamma$, which computes the difference between $\lipbound$ and the largest singular value of the function $\Mmat^{(n)}$:
\begin{equation} \label{equation:function_convergence}
  \Gamma(n) = \lipbound(\Kmat_{\Mmat^{(n)}}) - \sigma_1(\Mmat^{(n)})
\end{equation}
where $\Kmat_{\Mmat^{(n)}}$ is the convolution kernel associated with the matrix $\Mmat^{(n)}$.

\begin{figure}[ht]
  \centering
  \includegraphics[width=\scalefigure\textwidth]{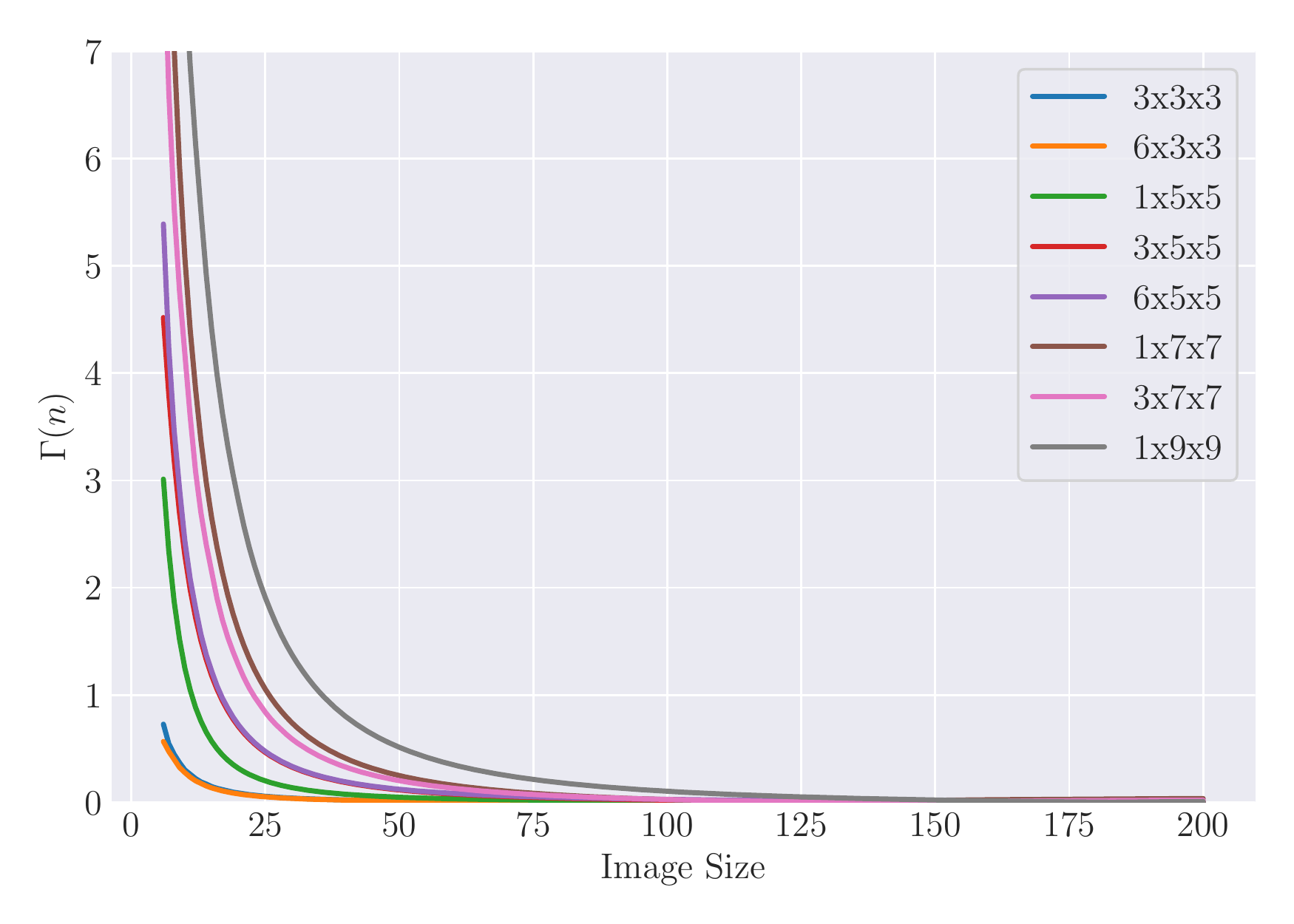}
  \caption{Representation of the function $\Gamma(n)$ defined for different kernel size.}
  \label{figure:convergence_bound}
\end{figure}

To compute a very close approximation of the exact largest singular value of $\Mmat^{(n)}$ for a specific $n$, we use the Implicitly Restarted Arnoldi Method (IRAM) ~\cite{lehoucq1996deflation} available in SciPy.
The results of this experiment are presented in \Cref{figure:convergence_bound}.
We observe that the difference between the bound and the actual value (approximation gap) quickly decreases as the input size increases.
For an input size of $50$, the approximation gap is as low as $0.012$ using a standard $6\times3\times3$ convolution kernel.
For a larger input size such as ImageNet ($224$), the gap is lower than $4.10^{-4}$.
Therefore $\lipbound$ gives an almost exact value of the largest singular value of the operator matrix for most realistic settings.

\subsection{Comparison of LipBound with State-of-the-Art Approaches}
\label{subsection:ch5-comparison_of_lipbound_with_other_state-of-the-art_approaches}

\begin{table}[ht]
  \centering
  \sisetup{%
    table-align-uncertainty=true,
    separate-uncertainty=true,
    detect-weight=true,
    detect-inline-weight=math
  }
  {\small
  \begin{tabular}
    {
      lr
      S[table-format=1.3]@{\,\( \pm \)\,}S[table-format=1.3]
      S[table-format=4.2]@{\,\( \pm \)\,}S[table-format=3.2]
      r
      S[table-format=1.3]@{\,\( \pm \)\,}S[table-format=1.3]
      S[table-format=4.2]@{\,\( \pm \)\,}S[table-format=3.2]
    }
    \toprule
      &   & \multicolumn{4}{c}{\textbf{1x3x3}} &   & \multicolumn{4}{c}{\textbf{32x3x3}} \\
    \cmidrule{3-6} \cmidrule{8-11}
    &   & \multicolumn{2}{c}{\textbf{Ratio}} & \multicolumn{2}{c}{\textbf{Time (ms)}}
    &   & \multicolumn{2}{c}{\textbf{Ratio}} & \multicolumn{2}{c}{\textbf{Time (ms)}} \\
    \midrule
    \citeauthor{sedghi2018singular} &   & 0.431 & 0.042 & 1088 & 251  & & 0.666 & 0.123 & 1729 & 399 \\
    \citeauthor{singla2019bounding} &   & 1.293 & 0.126 & 1.90 & 0.48 & & 1.441 & 0.188 & 1.90 & 0.46 \\
    \citeauthor{farnia2018generalizable} &   & 0.973 & 0.006 & 4.30 & 0.64 & & 0.972 & 0.004 & 4.93 & 0.67 \\
    \midrule
    \midrule
    LipBound &  & 0.992 & 0.012 & 0.49 & 0.05 & & 0.984 & 0.021 & 0.63 & 0.46 \\
    \bottomrule
  \end{tabular}%
  }
  \caption{Comparison of the accuracy of approximation methods for computing an approximation of the largest singular value of a convolution layer.}
  \label{table:ch5-compare_bounds}%
\end{table}

\begin{table}[h]
  \centering
  \sisetup{%
    table-number-alignment=center,
    table-align-uncertainty=true,
    separate-uncertainty=true,
    detect-weight=true,
    detect-inline-weight=math
  }
  \begin{tabular}
    {
      lr
      S[table-format=4.2,table-number-alignment=right]@{\,\( \pm \)\,}S[table-format=2.2,table-number-alignment=left]
      r
      S[table-format=6.2,table-number-alignment=right]@{\,\( \pm \)\,}S[table-format=3.2,table-number-alignment=left]
      r
      S[table-format=1.2]
    }
    \toprule
    \textbf{Network} & & \multicolumn{2}{c}{\textbf{LipBound (ms)}} & & \multicolumn{2}{c}{\textbf{Power Method (ms)}} & & \textbf{Ratio} \\
    \midrule
    AlexNet & & 4.75 & 1.10 & & 38.75 & 2.52 & & 8.14 \\
    \midrule
    ResNet 18 & & 29.88 & 1.73 & & 148.35 & 14.92 & & 4.96 \\
    ResNet 34 & & 54.73 & 3.62 & & 266.85 & 25.35 & & 4.87 \\
    ResNet 50 & & 60.77 & 4.62 & & 467.61 & 36.52 & & 7.69 \\
    ResNet 101 & & 102.72 & 11.53 & & 817.06 & 102.87 & & 7.95 \\
    ResNet 152 & & 158.80 & 20.84 & & 1373.57 & 164.37 & & 8.64 \\
    \midrule
    DenseNet 121 & & 125.55 & 14.59 & &  937.35 &  11.52 & & 7.46 \\
    DenseNet 161 & & 176.11 & 19.13 & & 1292.61 &  30.50 & & 7.33 \\
    DenseNet 169 & & 188.03 & 19.74 & & 1372.62 &  21.16 & & 7.29 \\
    DenseNet 201 & & 281.13 & 23.41 & & 1930.19 & 170.79 & & 6.86 \\
    \midrule
    VGG 11 & & 13.73 & 1.19 & &  81.78 & 4.45 & & 5.95 \\
    VGG 13 & & 14.96 & 1.99 & & 102.04 & 4.20 & & 6.82 \\
    VGG 16 & & 21.92 & 1.94 & & 132.29 & 5.99 & & 6.03 \\
    VGG 19 & & 29.05 & 0.66 & & 162.28 & 4.87 & & 5.58 \\
    \midrule
    WideResNet 50-2 & & 113.28 & 45.44 & & 468.74 & 6.54 & & 4.13 \\
    \midrule
    SqueezeNet 1-0 & & 18.44 & 5.93 & & 222.40 & 25.49 & & 12.05 \\
    SqueezeNet 1-1 & & 18.26 & 6.65 & & 209.80 &  3.59 & & 11.48 \\
    \bottomrule
  \end{tabular}%
  \caption{Efficiency of LipBound computation \vs the Power Method with 10 iterations on full networks.}
  \label{table:ch5-efficiency_lipbound_full_model}%
\end{table}%

In this section we compare our PolyGrid algorithm with the values obtained using alternative approaches.
We consider the 3 alternative techniques by~\citet{sedghi2018singular,singla2019bounding,farnia2018generalizable} which have been described in \Cref{chapter:ch3-related_work}, \Cref{section:ch3-related_work_on_lipschitz_regularization}.

To compare the different approaches, we extracted 20 kernels from a trained model.
For each kernel we construct the corresponding doubly-block Toeplitz matrix and compute its largest singular value.
Then, we compute the ratio between the approximation obtained with the considered approach and the approximated singular value obtained by IRAM, and average the ratios over the 20 kernels.
Thus good approximations result in approximation ratios that are close to 1.
The results of this experiment are presented in \Cref{table:ch5-compare_bounds}.
The comparison has been made on a Tesla V100 GPU.
The time was computed with the PyTorch CUDA profiler and we ``warmed'' up the GPU before starting the timer for caching purposes. 

The method introduced by~\citet{sedghi2018singular} and presented in~\Cref{subsection:ch3-singular_values_of_convolutional_layers} computes the largest singular value of convolution layers based on doubly-block circulant matrices.
Doubly-block circulant matrices perform a convolution with a ``wrapping around'' operation which do not correspond to the more general setting.
We can see in \Cref{table:ch5-compare_bounds} that the values differ by an important margin.
This technique is also computationally expensive as it requires computing the SVD of $n^2$ small matrices where $n$ is the size of inputs.
\citet{singla2019bounding} have shown that the singular value of the reshape kernel is a bound on the largest singular value of the convolution layer.
Their approach is very efficient but the approximation is loose and overestimate the real value.
As said previously, the power method provides a good approximation at the expense of the efficiency.
We also compare our approach to the power method with 10 iterations from ~\citet{farnia2018generalizable} (see~\Cref{algorithm:ch3-power_method_generic}).
The results show that our proposed technique: PolyGrid algorithm can get the best of both worlds.
It achieves a near perfect accuracy while being very efficient to compute.

The results of \Cref{table:ch5-compare_bounds} shows the performance for the computation for only one convolution layer.
However, during the training Lipbound or the power method need to be computed for every layer of the network and the computation time is dependent on the architecture of the network, for example, the size of the activations or the size of the kernels.
In ~\Cref{table:ch5-efficiency_lipbound_full_model}, we compare our approach method against the power method on the full architecture, \ie, the time needed for the computation on all the layers of the networks.
We measure on the following convolutional architectures: AlexNet \cite{krizhevsky2012imagenet}, ResNet \cite{he2016deep}, DenseNet \cite{huang2017densely}, VGG \cite{simonyan2014very}, WideResNet \cite{zagoruyko2016wide}, SqueezeNet \cite{iandola2016squeezenet}.
\Cref{table:ch5-efficiency_lipbound_full_model} shows that our approach is systematically faster than the power method by a factor up to 12 when considering all the layers of the networks.
This demonstrates the scalability of our method.



\section{Lipschitz Regularization for Adversarial Robustness}
\label{section:ch5-lipschitz_regularization_for_adversarial_robustness}

One promising application of Lipschitz regularization is in the area of adversarial robustness.
Empirical techniques to improve robustness against adversarial examples such as Adversarial Training only impact the training data,  and often show poor generalization capabilities~\cite{schmidt2018adversarially}.
\citet{farnia2018generalizable} have shown that the adversarial generalization error depends on the Lipschitz constant of the network, which suggests that the adversarial test error can be improved by applying Lipschitz regularization in addition to adversarial training.

In this section, we illustrate the usefulness of LipBound by training a Wide ResNet~\citep{zagoruyko2016wide} with Lipschitz regularization and adversarial training.
Our regularization scheme is inspired by the one used by \citet{yoshida2017spectral} but instead of using the power method, we use our \textbf{PolyGrid} algorithm presented in~\Cref{subsection:ch5-computing_the_maximum_modulus_of_a_trigonometric_polynomial} which efficiently computes an upper bound on the largest singular value of convolution layers.

We introduce the \textbf{AT+LipReg} loss to combine Adversarial Training and our Lipschitz regularization scheme in which layers with a large Lipschitz constant are penalized.
We consider a neural network $N_\Omega : \Xset \rightarrow \Yset$ with $\depth$ layers $\phi^{(1)}_{\Wmat^{(1)}, \bvec^{(1)}}, \dots, \phi^{(\depth)}_{\Wmat^{(\depth)}, \bvec^{(\depth)}}$ where $\Wmat^{(1)}, \dots, \Wmat^{(\depth)}$ are the weight matrices and $\Omega$ is the union of all the parameters as defined in~\Cref{definition:ch2-neural_networks}.
Given a distribution $\Dset$ over $\Xset \times \Yset$, we can train the parameters of the network by minimizing the AT+LipReg loss as follows:
\begin{equation} \label{equation:ch5-obj_function}
  \min_{\Omega} \Ebb_{\xvec,y \sim \Dset} \left[ L(N_\Omega(\xvec + \adv^{\text{adv}}_\Omega(\xvec)), y) + r(\Omega) \right]
\end{equation}
where $L$ is the cross-entropy loss function, $\adv^{\text{adv}}_\Omega(\xvec)$ is an adversarial perturbation following the loss maximization strategy presented in~\Cref{subsubsection:ch2-adversarial_attacks} and the regularization function $r$ is defined as follows:
\begin{equation}
  r(\Omega) =  C_1 \underbrace{\sum_{(\Wmat, \bvec) \in \Omega} \left( \norm{\Wmat}_\fro + \norm{\bvec}_2 \right)}_{
  \text{$\ell_2$ regularization}} + C_2 \underbrace{\vphantom{\sum_{(\Wmat, \bvec) \in \Omega}} \sum_{i=1}^{\depth-1} \log\left(\lipbound\left(\Kmat_{\Wmat^{(i)}}\right)\right)}_{\text{Lipschitz regularization}} 
\end{equation}
where $C_1$, $C_2$ are two user-defined hyper-parameters.
Note that regularizing the sum of logs is equivalent to regularizing the product of all the $\lipbound$ which is an upper bound on the global Lipschitz constant.
In practice, we also include the upper bound on the Lipschitz of the batch normalization since we can compute it very efficiently (see Appendix C.4.1 of ~\citet{tsuzuku2018lipschitz}) but we omit the last fully connected layer.

In this section, we compare the robustness of Adversarial Training~\cite{goodfellow2014explaining, madry2018towards} against the combination of Adversarial Training and Lipschitz regularization.
To regularize the Lipschitz constant of the network, we use the objective function defined in Equation~\ref{equation:ch5-obj_function}.
We train Lipschitz regularized neural networks with LipBound (see~\Cref{theorem:ch5-bound_max_sv_convolution}) implemented with PolyGrid (see~\Cref{algorithm:ch5-polygrid}) (AT+LipBound) with $S = 10$ or with the specific power method for convolutions introduced by~\citet{farnia2018generalizable} with 10 iterations (AT+PM).

Table~\ref{table:ch5-cifar_robustness} shows the gain in robustness against strong adversarial attacks across different datasets.
We can observe that both AT+LipBound and AT+PM offer a better defense against adversarial attacks and that AT+LipBound offers a further improvement over the Power Method.
\Cref{figure:ch5-attacks_pgd,figure:ch5-attacks_cw} show the Accuracy under attack with different numbers of iterations of the PGD algorithm.
Table~\ref{table:ch5-results_imagenet_dataset} presents our results on the ImageNet Dataset.
First, we can observe that the AT+LipReg trained networks offer a better generalization than with standalone Adversarial Training.
Secondly, we can observe the gain in robustness against strong adversarial attacks.
Network trained with Lipschitz regularization and Adversarial Training offer a consistent increase in robustness across $\ell_\infty$ and $\ell_2$ attacks with different $\epsilon$ value.
We can also note that increasing the regularization leads to an increase in generalization and robustness.

\afterpage{
\begin{figure}[p!]
   \centering
   \begin{subfigure}[b]{\textwidth}
     \centering
     \includegraphics[width=\scalefigure\textwidth]{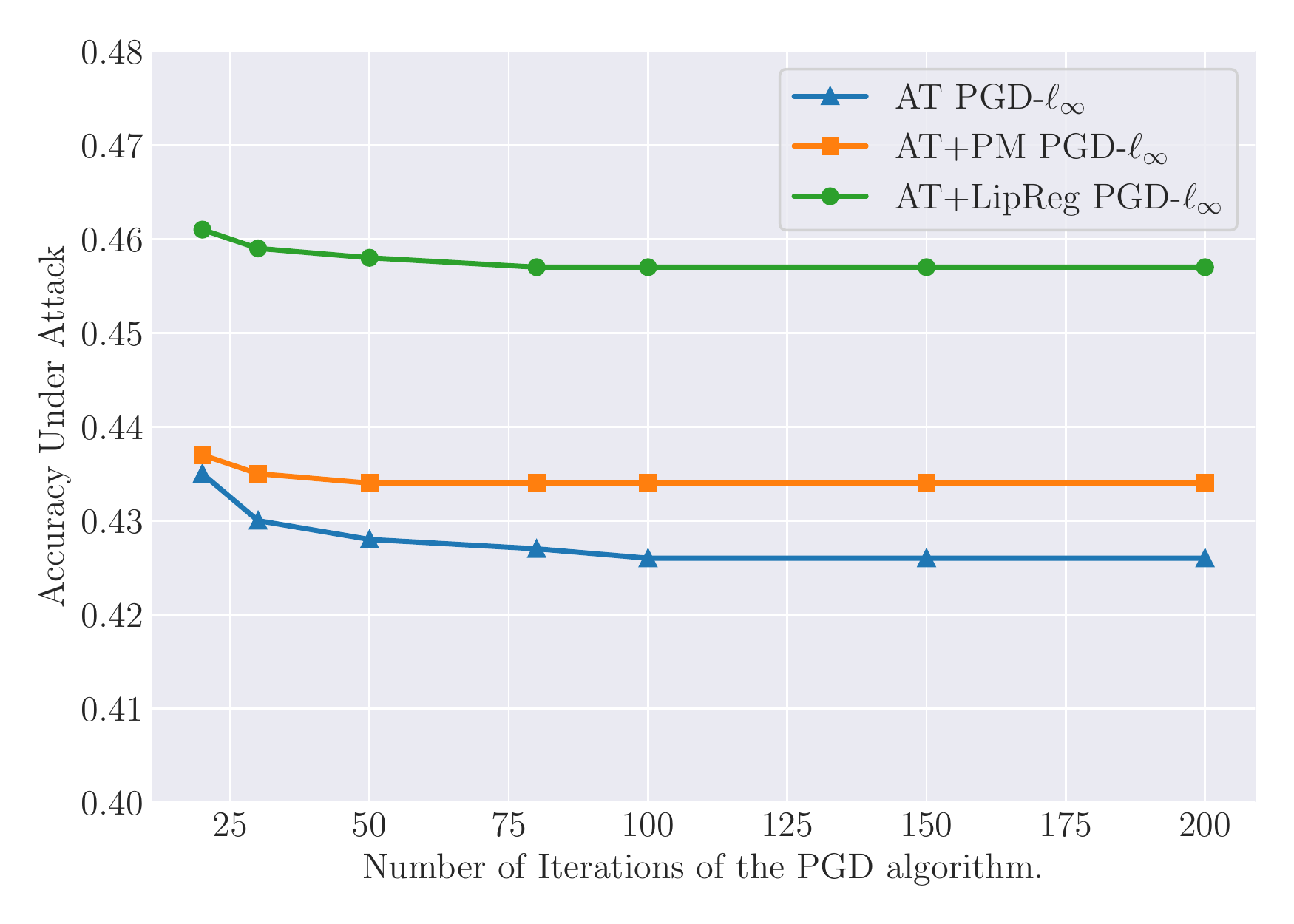}
     \caption{Robustness against $\ell_\infty$ attacks for different classifiers trained with Adversarial Training given the number of iterations.}
     \label{figure:ch5-attacks_pgd}
   \end{subfigure}
   ~\\[1cm]
   \begin{subfigure}[b]{\textwidth}
     \centering
     \includegraphics[width=\scalefigure\textwidth]{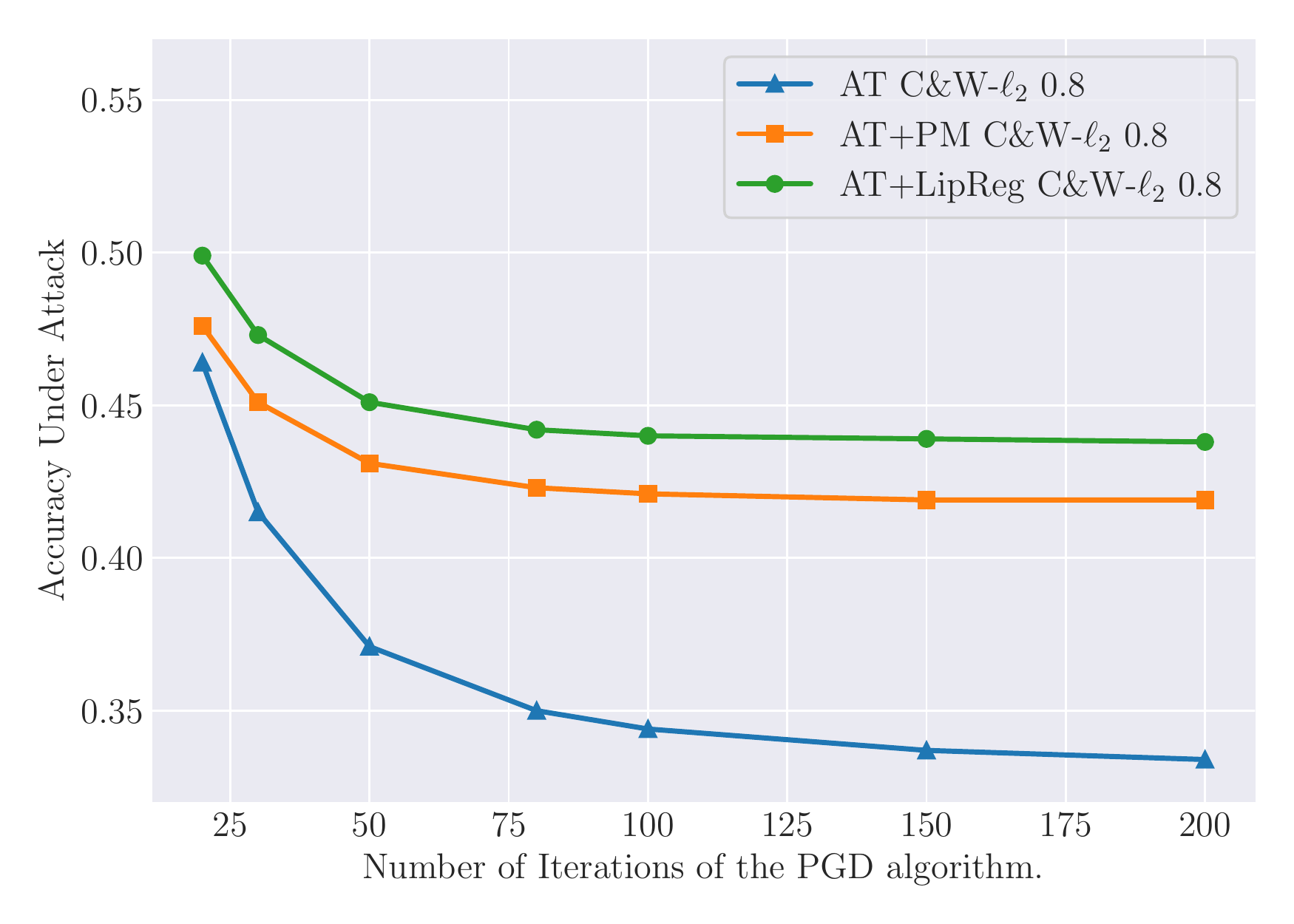}
     \caption{Robustness against $\ell_\infty$ attacks for different classifiers trained with Adversarial Training given the number of iterations.}
     \label{figure:ch5-attacks_cw}
   \end{subfigure}
   ~\\[1cm]
   \caption{Accuracy under attack on CIFAR10 test set with $\ell_\infty$ and $\ell_2$ attacks for several classifiers trained with Adversarial Training given the number of iterations.}
\end{figure}
\clearpage
}

\afterpage{
\begin{figure}[p!]
   \centering
   \begin{subfigure}[b]{\textwidth}
     \centering
     \includegraphics[width=\scalefigure\textwidth]{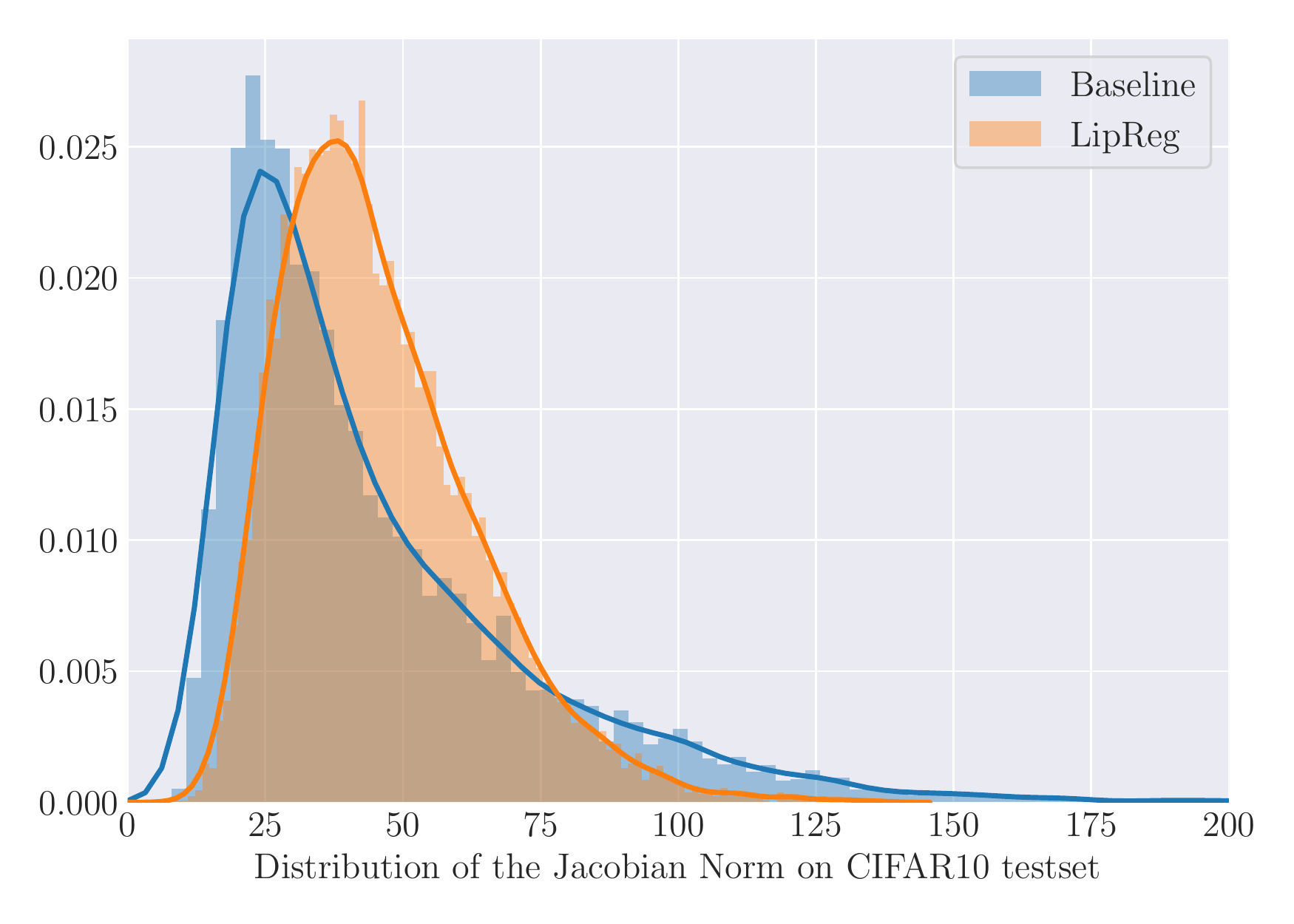}
     \caption{Comparison of the distribution of the norm of the Jacobian of the baseline model against the model trained with Lipschitz regularization.}
     \label{figure:ch5-jacobian_distribution_v1}
   \end{subfigure}
   ~\\[1cm]
   \begin{subfigure}[b]{\textwidth}
      \centering
      \includegraphics[width=\scalefigure\textwidth]{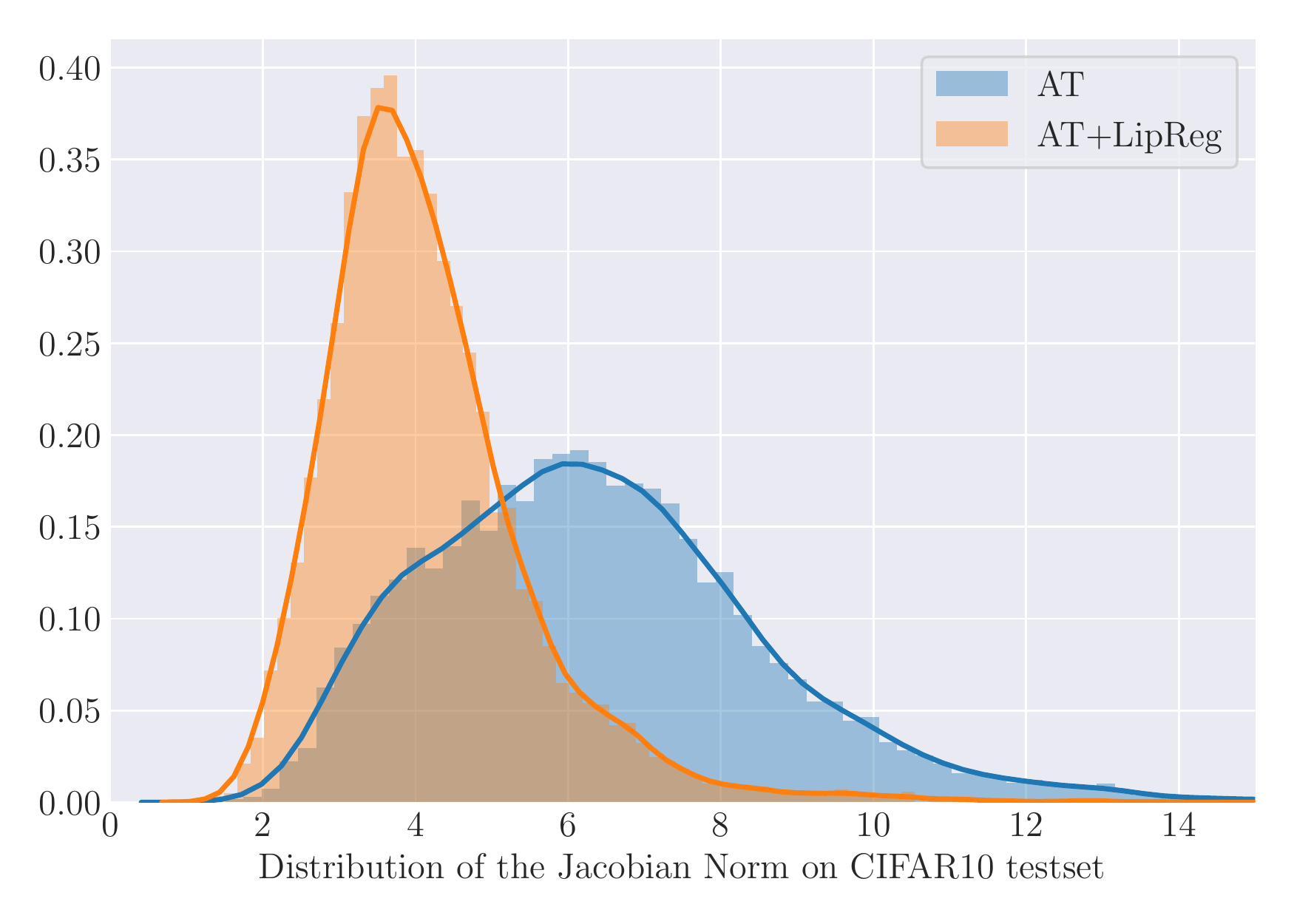}
      \caption{Comparison of the distribution of the norm of the Jacobian of the model trained with Adversarial training against the model trained with Adversarial training and Lipschitz regularization.}
      \label{figure:ch5-jacobian_distribution_v2}
   \end{subfigure}
   ~\\[1cm]
   \caption{Distribution of the norm of the Jacobian matrix with respect to the CIFAR10 test set from a Wide ResNet trained with different schemes.} 
\end{figure}
\clearpage
}

Finally, we also conducted an experiment to study the impact of the regularization on the gradients of the whole network by measuring the distributions of the norm of the Jacobian matrix with respect to the inputs from the test set.
The results of this experiment are presented in~\Cref{figure:ch5-jacobian_distribution_v1} and show more concentrated gradients with Lipschitz regularization.
Indeed, we can observe that while the median is higher, the regularization decreases the number of points with very high Lipschitz constant.
Although Lipschitz regularization is not a Jacobian regularization, we can observe a clear shift in the distribution.
This suggests that our method does not only work layer-wise, but also at the level of the entire network.
A second experiment, using Adversarial Training, presented in~\Cref{figure:ch5-jacobian_distribution_v2} demonstrates that the effect is even stronger when the two techniques are combined together.
It also demonstrates that Lipschitz regularization and Adversarial Training (or other Jacobian regularization techniques) are complementary.
Hence they offer an increased robustness to adversarial attacks as demonstrated above.

\begin{table}[t]
  \sisetup{%
    table-align-uncertainty=true,
    separate-uncertainty=true,
    detect-weight=true,
    detect-inline-weight=math
  }
  \centering
  \begin{subfigure}[b]{\textwidth}
    \begin{tabular}
      {
	l
        S[table-format=1.3]@{\,\( \pm \)\,}S[table-format=1.3]
        S[table-format=1.3]@{\,\( \pm \)\,}S[table-format=1.3]
        S[table-format=1.3]@{\,\( \pm \)\,}S[table-format=1.3]
        S[table-format=1.3]@{\,\( \pm \)\,}S[table-format=1.3]
      }
      \toprule
      \textbf{Model} & \multicolumn{2}{c}{\textbf{Accuracy}} & \multicolumn{2}{c}{\textbf{PGD-$\ell_\infty$}}
      & \multicolumn{2}{c}{\textbf{C\&W-$\ell_2$ 0.6}} & \multicolumn{2}{c}{\textbf{C\&W-$\ell_2$ 0.8}} \\
      \midrule
      \textbf{Baseline} & \textbf{0.953} & 0.001 & 0.000 & 0.000 & 0.002 & 0.000 & 0.000 & 0.000 \\
      \textbf{AT}       & 0.864 & 0.001 & 0.426 & 0.000 & 0.477 & 0.000 & 0.334 & 0.000 \\
      \textbf{AT+PM}    & 0.788 & 0.010 & 0.434 & 0.007 & 0.521 & 0.005 & 0.419 & 0.003 \\
      \textbf{AT+LipReg} & 0.808 & 0.022 & \textbf{0.457} & 0.002 & \textbf{0.547} & 0.022 & \textbf{0.438} & 0.020 \\
      \bottomrule
    \end{tabular}%
    \caption{Results on CIFAR10 dataset}
    \label{subfigure:ch5-results_cifar10_data}
  \end{subfigure}
  \par\bigskip
  \begin{subfigure}[b]{\textwidth}
    \begin{tabular}
      {
	l
        S[table-format=1.3]@{\,\( \pm \)\,}S[table-format=1.3]
        S[table-format=1.3]@{\,\( \pm \)\,}S[table-format=1.3]
        S[table-format=1.3]@{\,\( \pm \)\,}S[table-format=1.3]
        S[table-format=1.3]@{\,\( \pm \)\,}S[table-format=1.3]
      }
      \toprule
      \textbf{Model} & \multicolumn{2}{c}{\textbf{Accuracy}} & \multicolumn{2}{c}{\textbf{PGD-$\ell_\infty$}}
      & \multicolumn{2}{c}{\textbf{C\&W-$\ell_2$ 0.6}} & \multicolumn{2}{c}{\textbf{C\&W-$\ell_2$ 0.8}} \\
      \midrule
      \textbf{Baseline} & \textbf{0.792} & 0.000  & 0.000 & 0.000  & 0.001 & 0.000  & 0.000 & 0.000 \\
      \textbf{AT} & 0.591 & 0.000  & 0.199 & 0.000  & 0.263 & 0.000  & 0.183 & 0.000 \\
      \textbf{AT+LipReg} & 0.552 & 0.019  & \textbf{0.215} & 0.004  & \textbf{0.294} & 0.010  & \textbf{0.226} & 0.008  \\
      \bottomrule
    \end{tabular}%
    \caption{Results on CIFAR100 dataset}
    \label{subfigure:ch5-results_cifar100_data}
  \end{subfigure}
  \caption{Accuracy under $\ell_2$ and $\ell_\infty$ attacks of different training schemes on CIFAR10/100 datasets.} 
  \label{table:ch5-cifar_robustness}%
\end{table}%

\begin{table}[t]
  \centering
  \tabcolsep=1.9mm
  {\small
  \begin{tabular}{
    lc
    c
    cc
    c
    ccc
  }
    \toprule
    \multicolumn{1}{c}{\multirow{2}[4]{*}{\textbf{Model}}} & \multicolumn{1}{c}{\multirow{2}[4]{*}{\textbf{Natural}}} &  
    & \multicolumn{2}{c}{\textbf{PGD}-$\ell_\infty$} &   & \multicolumn{3}{c}{\textbf{C\&W}-$\ell_2$} \\
    \cmidrule{4-5}\cmidrule{7-9} 
    &  &  & \multicolumn{1}{c}{0.02} & \multicolumn{1}{c}{0.031} &   & \multicolumn{1}{c}{1.00} & \multicolumn{1}{c}{2.00} & \multicolumn{1}{c}{3.00} \\
    \midrule
    \textbf{Baseline} \cite{he2016deep} & \textbf{0.782} & & 0.000 & 0.000 & & 0.000 & 0.000 & 0.000 \\
    \textbf{AT} & 0.509 &   & 0.251 & 0.118 &   & 0.307 & 0.168 & 0.099 \\
    \textbf{AT+LipReg} ($C_2 = 0.0006$) & \textbf{0.515} &   & \textbf{0.255} & \textbf{0.121} &   & \textbf{0.316} & \textbf{0.177} & \textbf{0.105} \\
    \textbf{AT+LipReg} ($C_2 = 0.0010$) & \textbf{0.519} &   & \textbf{0.259} & \textbf{0.123} &   & \textbf{0.338} & \textbf{0.204} & \textbf{0.129} \\
    \bottomrule
  \end{tabular}%
  }
  \caption{Natural accuracy and accuracy under $\ell_2$ and $\ell_\infty$ attacks of different training schemes on the ImageNet dataset.} 
  \label{table:ch5-results_imagenet_dataset}
\end{table}%

\paragraph{Experimental Settings CIFAR10/100 Dataset.}
For all our experiments, we use the Wide ResNet architecture introduced by~\citet{zagoruyko2016wide} to train our classifiers.
We use Wide ResNet networks with 28 layers and a width factor of 10.
We train our networks for 200 epochs with a batch size of $200$.
We use Stochastic Gradient Descent with a momentum of $0.9$, an initial learning rate of $0.1$ with exponential decay of 0.1 (MultiStepLR gamma = 0.1) after the epochs $60$, $120$ and $160$.
For Adversarial Training ~\cite{madry2018towards}, we use Projected Gradient Descent with an $\epsilon = 8/255 (\approx 0.031)$, a step size of $\epsilon/5 (\approx 0.0062)$ and 10 iterations, we use a random initialization but run the attack only once.
To evaluate the robustness of our classifiers, we rigorously followed the experimental protocol proposed by~\citet{carlini2019evaluating,tramer2020adaptive}.
More precisely, as an $\ell_\infty$ attack, we use PGD with the same parameters ($\epsilon = 8/255$, a step size of $\epsilon/5$) but we increase the number of iterations up to 200 with 10 restarts.
For each image, we select the perturbation that maximizes the loss among all the iterations and the 10 restarts.
As $\ell_2$ attacks, we use a bounded version of the~\citet{carlini2017towards} attack.
We choose $0.6$ and $0.8$ as bounds for the $\ell_2$ perturbation.
Note that the $\ell_2$ ball with a radius of $0.8$ has approximately the same volume as the $\ell_\infty$ ball with a radius of $0.031$ for the dimensionality of CIFAR10/100.

\paragraph{Experimental Settings for ImageNet Dataset.}
For all our experiments, we use the ResNet-101 architecture \cite{he2016deep}.
We have used Stochastic Gradient Descent with a momentum of $0.9$, a weight decay of $0.0001$, label smoothing of $0.1$, an initial learning rate of $0.1$ with exponential decay of $0.1$ (MultiStepLR gamma = $0.1$) after the epochs $30$ and $60$.
We have used Exponential Moving Average over the weights with a decay of $0.999$.
We have trained our networks for 80 epochs with a batch size of $4096$.
For Adversarial Training, we have used PGD with 5 iterations, $\epsilon = 8/255 (\approx 0.031)$ and a step size of $\epsilon/5 (\approx 0.0062)$.
To evaluate the robustness of our classifiers on ImageNet Dataset, we have used an $\ell_\infty$ and an $\ell_2$ attacks.
More precisely, as an $\ell_\infty$ attack, we use PGD with an epsilon of 0.02 and 0.031, a step size of $\epsilon/5$) with a number of iterations to 30 with 5 restarts.
For each image, we select the perturbation that maximizes the loss among all the iterations and the 10 restarts.
As $\ell_2$ attacks, we use a bounded version of the~\citet{carlini2017towards} attack.
We have used $1$, $2$ and $3$ as bounds for the $\ell_2$ perturbation.

\section{Concluding Remarks}
\label{section:ch5-concluding_remarks}

In this chapter, we introduced a new bound on the Lipschitz constant of convolution layers that is both accurate and efficient to compute.
We used this bound to regularize the Lipschitz constant of neural networks and demonstrated its computational efficiency in training large neural networks with a regularized Lipschitz constant.
As an illustrative example, we combined our bound with adversarial training, and showed that this increases the robustness of the trained networks to  adversarial attacks.
The scope of our results goes beyond this application and can be used in a wide variety of settings, for example, to stabilize the training of Generative Adversarial Networks (GANs) and invertible networks, or to improve generalization capabilities of classifiers.

%% file: sources/main/ch6-conclusion.tex
\chapter{Conclusion}
\label{chapter:ch6-conclusion}
\localtoc

\section{Summary of the Contributions}
\label{section:ch6-summary_of_the_contributions}

State-of-the-art in a variety of domains, deep neural networks exhibit important limitations.
Indeed, current neural networks tend to be very large in terms of their number of parameters which make them difficult to train and to deploy in real-world applications.
Furthermore, they exhibit instability to small perturbations of their inputs which lead to adversarial attacks. 

In this thesis, we have used structured matrices from the Toeplitz family to make contributions to the field of deep learning.
Our contributions are twofold.
First, we studied deep diagonal-circulant neural networks, which are deep neural networks in which weight matrices are the product of diagonal and circulant ones.
Using diagonal and circulant matrices instead of dense ones allows for an important reduction in the number of parameters which make them more efficient and cost-effective.
In addition to being more compact than fully connected neural networks, diagonal-circulant neural networks have a high expressivity that makes them useful for numerous use cases.
In order to characterize the expressive power of diagonal-circulant neural networks, we build upon the work of~\citet{huhtanen2015factoring} which states that any matrix can be decomposed into a product of alternating diagonal and circulant matrices.
Based on this result, we have successfully demonstrated that neural networks with diagonal and circulant matrices are \emph{universal approximators} and characterized their expressive power with respect to their depth.
We also demonstrated the effectiveness of this class of compact neural networks to video classification with a real-world dataset.


%

Secondly, we studied the properties of doubly-block Toeplitz matrices which are equivalent to the convolution operation. 
Using the properties of this type of structured matrix and a Fourier representation introduced by~\citet{grenander1958toeplitz}, we devised an upper-bound on the singular values of convolution layers leading to a new regularization scheme that improves the robustness of neural networks against adversarial attacks.
In order to use this upper-bound in a large-scale setting, we introduced the PolyGrid algorithm (see \Cref{algorithm:ch5-polygrid}) which efficiently and accurately computes an approximation of this upper-bound.


%
%

\section{Perspectives and Future Works}
\label{section:ch6-perspectives_and_future_works}

\subsection{Designing Compact Transformers for Natural Language Processing}

In order to improve upon our work on compact neural networks, one idea follows naturally.
The race towards larger convolutional neural networks seemed to have slowed down following the work of \citet{tan2019efficientnet} which devised compact state-of-the-art neural networks for image recognition.
However, other types of architecture, \eg, \emph{Transformers} which rely heavily on dense matrices, have seen their number of parameters exploding in recent years.
The latest model which was designed by \citet{fedus2021switch} has 1 trillion parameters, 5.7 times than the second largest, proposed by \citet{brown2020language}, which had 175 billion parameters.

In~\Cref{chapter:ch4-diagonal_circulant_neural_network} and Appendix~\ref{appendix:ap2-diagonal_circulant_neural_networks_for_video_classification}, we have used the diagonal-circulant decomposition for compressing embedded layers in the context of video classification.
This decomposition could also be used to compress attention layers of Transformers networks~\cite{vaswani2017attention} where the attention layer is described as follows:
\begin{equation}
  \text{Attention}(\Qmat, \Kmat, \Vmat) = \text{softmax} \left( \frac{\Qmat \Kmat^\top}{\sqrt{d_k}} \right) \Vmat \enspace,
\end{equation}
where $\Qmat, \Kmat$ and $\Vmat$ are dense matrices.
Taking this layer as a building block leads to large neural networks as demonstrated by the GPT-3 architecture with 96 attention layers and 175 billion parameters.
Although, the diagonal-circulant decomposition could successfully reduce the number of parameters of attention layers, it may have limited impact on \emph{multi-head attention layers} which are a concatenation of small attention layers due to the reduced dimension of each matrix.

%
%
%
%
%
%

\subsection{Regularization on the Condition Number of Convolution Layers}

In \Cref{chapter:ch5-lipschitz_bound}, we have proposed an upper-bound on the largest singular value of convolution layers which allow us to regularize the Lipschitz constant of the network thus improving the robustness.
However, an important reduction of the Lipschitz constant seems to prevent the network from learning correctly. 
Indeed, as demonstrated by the work of \cite{zhang2019theoretically}, accuracy and robustness are actually at odds, meaning that improving the robustness (\ie in our case, reducing the expressivity) hurts the training and the natural accuracy of the network. 
In our experiments, the phenomenon called \emph{rank collapse} \cite{saxe2014exact} where the rank of the weights matrices tend to decrease during training combined with a strong Lipschitz regularization would prevent convergence.
An interesting solution would be to regularize the largest singular value and promoting the smallest in order to enforce orthogonality.
The following bound on the condition number of general matrices \citet{guggenheimer1995simple} could be studied:
\begin{equation}
  \kappa(\Wmat) \le \frac{2}{\abs{\det(\Wmat)}} \left(\frac{\norm{\Wmat}_\fro}{\sqrt{r}}\right)^r
\end{equation}
where $r$ is the rank of $\Wmat$.
As such, using the bound as a regularizer will enforce the orthogonality, just like a layer normalization.
Therefore, we could design some heuristic regularizers to encourage a smaller $\left(\frac{\norm{\Wmat}_\fro}{\sqrt{r}}\right)^r$ and larger $\abs{\det(\Wmat)}$ separately, as given by the following objective function: 
\begin{align} 
  \min_{\Omega} \Ebb_{\xvec, y \sim \Dset} \left[ L(N_\Omega(\xvec), y)  + C_1 \sum_{i=1}^\depth \norm{\Wmat^{(i)}}_\fro + C_2 \sum_{i=1}^{\depth} \log \abs{\det(\Wmat^{(i)})} \right]
\end{align}
where the determinant of doubly-block Toeplitz matrices under some assumption could be expressed with the Szeg\"{o} Theorem \cite{szego1915grenzwertsatz} and can be approximated with the help of \emph{Random Matrix Theory} \cite{basor2017asymptotics}.

%
%
%

\subsection{Going Beyond the Lipschitz Constant} 

Finally, in order to better understand the behavior of neural networks and the transformation they perform, it would be interesting to go beyond the Lipschitz constant and consider their full spectrum.
Indeed, the spectrum of a linear map is a set that contains the eigenvalues and can be seen as a description of the properties and behavior of the operator.
For example, for a linear operator $\Lmat: \Xset \rightarrow \Xset$, the spectrum gives precise information on the solvability of the following linear equation
\begin{equation}
  \lambda \xvec - \Lmat \xvec = \yvec
\end{equation}
It is natural to ask if we could define a spectrum that equivalently gives information on the following nonlinear equation
\begin{equation}
  \lambda \xvec - \Fmat(\xvec) = \yvec \enspace.
\end{equation}

In this line of research, \citet{kachurovskii1969regular} have defined a spectrum for nonlinear continuous Lipschitz operators which share important properties with the spectrum of linear operators.
More precisely, let $\Fmat: \Xset \rightarrow \Xset$ be a nonlinear continuous Lipschitz map, the \emph{Kachurovskij spectrum} of $\Fmat$ is given by
\begin{equation}
  \sigma(\Fmat) \triangleq \left\{ \lambda \in \Cbb \mid \lambda \Imat - \Fmat \text{ is not a lipeomorphism} \right\} 
\end{equation} 
where a nonlinear Lipschitz continuous operator is a \emph{lipeomorphism} if its inverse is also nonlinear Lipschitz continuous.
We can also define the complement of the spectrum, \ie, the \emph{Kachurovskij resolvent set} as follows: 
\begin{equation}
  \mu(\Fmat) \triangleq \Cbb \setminus \sigma(\Fmat) 
\end{equation}
The resolvent set can be seen as the set of complex numbers for which the operator is \emph{well behaved}.
The Kachurovskij spectrum is a compact subset of the complex plane but may be empty.
Kachurovskij have also shown that the emptiness of this spectrum can be prevented if we restrict ourselves to nonlinear continuous Lipschitz operators that admit a Fréchet-derivative $\Fmat'(\xvec_0)$ at some point $\xvec_0 \in \Xset$.
In this case, the Kachurovskij spectrum share all the properties of a linear operator which are: closed, compact, bounded and non-empty. 

Neural networks with differentiable nonlinearities are differentiable nonlinear Lipschitz continuous functions, therefore the study of the Kachurovskij spectrum could give important insight on their stability, invertibility and robustness to adversarial examples.

\section{Discussion}
\label{section:ch6-discussion}


Although our contributions offer concrete techniques for building compact and reliable neural networks, they also highlight some important difficulties in them.
First, if we discard techniques such as pruning or quantization for building compact neural networks due to the necessity of training a large neural network prior to compression, designing parameters-efficient neural networks that are compact \emph{by design} requires rethinking the whole architecture.
For computer vision tasks, the convolution operation is a compact and powerful transform, however, we still haven't found such equivalent transforms for other use cases.
Although the multi-head attention layer is more efficient than the attention layer for NLP tasks, using this type of transform in a neural network is still very parameter-hungry as demonstrated by the recent state-of-the-art for language models \cite{brown2020language}.

Secondly, defense techniques against adversarial attacks have shown great improvements in the last few years.
However, with current state-of-the-art techniques, it is still difficult to reach an accuracy higher than 60\% on CIFAR10 (which is considered a small dataset) and the accuracy decreases further on datasets with a larger dimensionality.
Consequently, building robust neural networks still remain very much an open question.
We believe that further breakthroughs in this area will come as a by-product on research on \emph{understanding neural networks}. 
Accordingly, we hope that our contribution to the understanding of diagonal-circulant and convolution neural networks is a small step in this direction.


%% file: sources/appendix/ap1-technical_proofs.tex
\chapter{Generalization of Widom Identity}
\label{appendix:ap1-proof_of_the_generalization_of_widom_identity}

This appendix aims at proving a generalization of Widom Identity for doubly-block Toeplitz operators.
The Widom identity, which states the relation between Toeplitz and Hankel operators, was introduced by Harold Widom in a \citeyear{widom1976asymptotic} seminal paper \cite{widom1976asymptotic}.
Let us define the semi-infinite Toeplitz and Hankel operators:
\begin{align}
  \Tmat_\infty(f) &\triangleq \leftmat\frac{1}{2\pi} \int_{0}^{2\pi} e^{-\ci(i-j)\omega}f(\omega) \,\diff \omega\rightmat_{i,j \in \{0, \dots, \infty\}} \\
  \Hmat_\infty(f) &\triangleq \leftmat\frac{1}{2\pi} \int_{0}^{2\pi} e^{-\ci(i+j+1)\omega}f(\omega) \,\diff \omega\rightmat_{i,j \in \{0, \dots, \infty\}}
\end{align}
Then, for $f$ and $g$ integrable functions, the Widom identity can be written as follows:
\begin{equation}
  \Tmat_\infty(fg) - \Tmat_\infty(f) \Tmat_\infty(g) = \Hmat_\infty(f) \Hmat_\infty(g^*)
\end{equation}
Note that Widom extends this identity from finite Toeplitz matrices:
\begin{equation} \label{equation:ap1-widom_identity}
  \Tmat_n(fg) - \Tmat_n(f) \Tmat_n(g) = \Hmat_n(f) \Hmat_n(g^*) - \Jmat_n \Hmat_n(f^*) \Hmat_n(g^*) \Jmat_n
\end{equation}
where $\Jmat_n$ is the anti-identity matrix, \ie, the reflexion matrix.

We would like to expend the identity presented in~\Cref{equation:ap1-widom_identity} to finite doubly-block Toeplitz operator.
We will need to generalize the doubly-block Toeplitz operator presented in~\Cref{subsection:ch5-bound_on_the_singular_value_of_doubly-block_toeplitz_matrices}.
Let $\Gmat^{\alpha_p} (f) = \leftmat \Gmat^{\alpha_p}_{i,j}(f) \rightmat_{i,j \in \Iset^+_n}$ where $\Gmat^{\alpha_p}_{i,j}$ is defined as:
\begin{equation}
  \Gmat^{\alpha_p}_{i,j}(f) =\leftmat \frac{1}{4\pi^{2}} \int_{0}^{2\pi} \int_{0}^{2\pi} e^{-\ci \alpha_p(i, j, k, l, \omega_1, \omega_2)}  f(\omega_{1},\omega_{2}) \,\diff \omega_{1} \,\diff \omega_{2})
  \rightmat_{k,l \in \Iset^+_n} \enspace.
\end{equation}
Note that as with the operator $\Dmat(f)$ we only consider generating functions as trigonometric polynomials with real coefficients therefore the matrices generated by $\Gmat(f)$ are real. 
And as with the operator $\Dmat(f)$, the matrices generated by the operator $\Gmat^{\alpha_p}$ are of size $n^2 \times n^2$. 

\noindent
We will use the following $\alpha$ functions:
\begin{itemize}
    \item[] $\alpha_0(i, j, k, l, \omega_1, \omega_2) = (-j-i-1)\omega_1 + (k-l)\omega_2$
    \item[] $\alpha_1(i, j, k, l, \omega_1, \omega_2) = (i-j)\omega_1 + (-l-k-1)\omega_2$
    \item[] $\alpha_2(i, j, k, l, \omega_1, \omega_2) = (-j-i-1)\omega_1 + (-l-k-1)\omega_2$
    \item[] $\alpha_3(i, j, k, l, \omega_1, \omega_2) = (-j-i+n)\omega_1 + (-l-k-1)\omega_2$
\end{itemize}

\noindent
We now present the generalization of the Widom identity for Doubly-Block Toeplitz matrices below:
\begin{lemma}[Generalization of Widom Identity] \label{lemma:ap1-widom_idenity}
  Let $f:\Rbb^2 \rightarrow \Cbb$ and $g:\Rbb^2 \rightarrow \Cbb$ be two continuous and $2\pi$-periodic functions. 
  We can decompose the Doubly-Block Toeplitz matrix $\Dmat(fg)$ as follows:
  \begin{equation}
    \Dmat(fg) = \Dmat(f)\Dmat(g) + \sum_{p=0}^3 \Gmat^{\alpha_p \top}(f^*) \Gmat^{\alpha_p}(g) + \Jmat_{n^2} \left( \sum_{p=0}^3 \Gmat^{\alpha_p \top}(f) \Gmat^{\alpha_p }(g^*) \right) \Jmat_{n^2}.
  \end{equation}
  where $\Jmat$ is the reflection of the identity matrix of size $n^2 \times n^2$.
\end{lemma}

\begin{proof}[\Cref{lemma:ap1-widom_idenity}]
Let $(i, j)$ be matrix indexes such $(\ \cdot\ )_{i, j}$ correspond to the value at the $i^\textrm{th}$ row and $j^\textrm{th}$ column, let us define the following notation:
\begin{align*}
    i_1 &= \left\lfloor i/n \right\rfloor \quad \quad &&j_1 = \left\lfloor j/n \right\rfloor \\
    i_2 &= i \mod n \quad \quad &&j_2 = j \mod n
\end{align*}

\noindent
Let us define $\hat{f}$ as the 2 dimensional Fourier transform of the function $f$. We refer to $\hat{f}_{h_1, h_2}$ as the Fourier coefficient indexed by $(h_1, h_2)$ where $h_1$ correspond to the index of the block of the doubly-block Toeplitz and $h_2$ correspond to the index of the value inside the block. More precisely, we have 
\begin{align}
    \leftmat \Dmat(f) \rightmat_{i, j} &= \hat{f}_{(\left\lfloor j/n \right\rfloor - \left\lfloor i/n \right\rfloor), ((j \mod n) - (i \mod n)))} \label{equation:expression_fourier} \\
    \leftmat \Gmat^{\alpha_0}(f) \rightmat_{i, j} &= \hat{f}_{(\left\lfloor j/n \right\rfloor + \left\lfloor i/n \right\rfloor + 1), ((j \mod n) - (i \mod n)))} \\
    \leftmat \Gmat^{\alpha_1}(f) \rightmat_{i, j} &= \hat{f}_{(\left\lfloor j/n \right\rfloor - \left\lfloor i/n \right\rfloor), ((j \mod n) + (i \mod n) + 1))} \\
    \leftmat \Gmat^{\alpha_2}(f) \rightmat_{i, j} &= \hat{f}_{(\left\lfloor j/n \right\rfloor - \left\lfloor i/n \right\rfloor), ((j \mod n) - (i \mod n)))} \\
    \leftmat \Gmat^{\alpha_3}(f) \rightmat_{i, j} &= \hat{f}_{(\left\lfloor j/n \right\rfloor + \left\lfloor i/n \right\rfloor + n), ((j \mod n) + (i \mod n) + 1))}
\end{align}

\noindent
We simplify the notation of the expressions above as follow:
\begin{align}
    \leftmat \Dmat(f) \rightmat_{i, j} &= \hat{f}_{(j_1 - i_1), (j_2 - i_2 )} \\
    \leftmat \Gmat^{\alpha_0}(f) \rightmat_{i, j} &= \hat{f}_{(j_1 + i_1 + 1), (j_2 - i_2 )} \\
    \leftmat \Gmat^{\alpha_1}(f) \rightmat_{i, j} &= \hat{f}_{(j_1 - i_1), (j_2 + i_2 + 1)} \\
    \leftmat \Gmat^{\alpha_2}(f) \rightmat_{i, j} &= \hat{f}_{(j_1 - i_1), (j_2 - i_2 )} \\
    \leftmat \Gmat^{\alpha_3}(f) \rightmat_{i, j} &= \hat{f}_{(j_1 + i_1 + n), (j_2 + i_2 + 1)}
\end{align}

\noindent
The convolution theorem states that the Fourier transform of a product of two functions is the convolution of their Fourier coefficients. Therefore, one can observe that the entry $(i, j)$ of the matrix $\Dmat(f g)$ can be express as follows:
\begin{equation*}
    \leftmat \Dmat(f g) \rightmat_{i, j} = \sum_{k_1 = -2n + 1}^{2n-1} \sum_{k_2 = -2n + 1}^{2n-1} \hat{f}_{(k_1-i_1),(k_2-i_2)} \hat{g}_{(j_1-k_1),(j_2-k_2)}. 
\end{equation*}

\noindent
By splitting the double sums and simplifying, we obtain:
\begin{align} \label{equation:split_double_sum}
  \left( \Dmat(f g) \right)_{i, j} &= 
  \sum_{k_1, k_2 \in P} \left(
    \hat{f}_{(k_1-i_1),(k_2-i_2)} \hat{g}_{(j_1-k_1),(j_2-k_2)} +
    \hat{f}_{(-k_1-i_1-1),(k_2-i_2)} \hat{g}_{(j_1+k_1+1),(j_2-k_2)} \right. \notag \\ &\quad+ \left.
    \hat{f}_{(k_1-i_1),(-k_2-i_2-1)} \hat{g}_{(j_1-k_1),(j_2+k_2+1)} +
    \hat{f}_{(-k_1-i_1-1),(-k_2-i_2-1)} \hat{g}_{(j_1+k_1+1),(j_2+k_2+1)} \right. \notag \\ &\quad+ \left.
    \hat{f}_{(k_1-i_1+n),(-k_2-i_2-1)} \hat{g}_{(j_1-k_1-n),(j_2+k_2+1)} +
    \hat{f}_{(k_1-i_1+n),(k_2-i_2)} \hat{g}_{(j_1-k_1-n),(j_2-k_2)} \right. \notag \\ &\quad+ \left.
    \hat{f}_{(k_1-i_1),(k_2-i_2+n)} \hat{g}_{(j_1-k_1),(j_2-k_2-n)} +
    \hat{f}_{(k_1-i_1+n),(k_2-i_2+n)} \hat{g}_{(j_1-k_1-n),(j_2-k_2-n)} \right. \notag \\ &\quad+ \left.
    \hat{f}_{(-k_1-i_1-1),(k_2-i_2+n)} \hat{g}_{(j_1+k_1+1),(j_2-k_2-n)}  \right)
\end{align}
where $P = \{ (k_1, k_2)\ |\ k_1, k_2 \in \Zbb, 0 \leq k_1 \leq n-1,  0 \leq k_2 \leq n-1 \}$.

\noindent
Furthermore, we can observe the following:
\begin{equation*}
  \leftmat \Dmat(f) \Dmat(g) \rightmat_{i, j} = \sum_{k = 0}^{n^2} \leftmat\Dmat(f)\rightmat_{i, k} \leftmat\Dmat(g)\rightmat_{k, j}  = \sum_{k_1, k_2 \in P} \hat{f}_{(k_1-i_1),(k_2-i_2)} \hat{g}_{(j_1-k_1),(j_2-k_2)}
\end{equation*}

\begin{flalign*}
  \leftmat \Gmat^{\alpha_1 \top}(f^*) \Gmat^{\alpha_1}(g) \rightmat_{i, j} &=  \sum_{k_1, k_2 \in P} \hat{f}^*_{(k_1+i_1+1),(i_2-k_2)} \hat{g}_{(j_1+k_1+1),(j_2-k_2)} \\
  &=  \sum_{k_1, k_2 \in P} \hat{f}_{(-k_1-i_1-1),(k_2-i_2)} \hat{g}_{(j_1+k_1+1),(j_2-k_2)}
\end{flalign*}

\begin{flalign*}
  \leftmat \Gmat^{\alpha_2 \top}(f^*) \Gmat^{\alpha_2}(g) \rightmat_{i, j} &=  \sum_{k_1, k_2 \in P} \hat{f}^*_{(i_1-k_1),(k_2+i_2+1)} \hat{g}_{(j_1-k_1),(j_2+k_2+1)} \\
  &=  \sum_{k_1, k_2 \in P} \hat{f}_{(k_1-i_1),(-k_2-i_2-1)} \hat{g}_{(j_1-k_1),(j_2+k_2+1)}
\end{flalign*}

\begin{flalign*}
  \leftmat \Gmat^{\alpha_3 \top}(f^*) \Gmat^{\alpha_3}(g) \rightmat_{i, j} &=  \sum_{k_1, k_2 \in P} \hat{f}^*_{(k_1+i_1+1),(k_2+i_2+1)} \hat{g}_{(j_1+k_1+1),(k_2+j_2+1)} \\
  &= \sum_{k_1, k_2 \in P} \hat{f}_{(-k_1-i_1-1),(-k_2-i_2-1)} \hat{g}_{(j_1+k_1+1),(k_2+j_2+1)}
\end{flalign*}

\begin{flalign*}
  \leftmat \Gmat^{\alpha_4 \top}(f^*) \Gmat^{\alpha_4}(g) \rightmat_{i, j} &= \sum_{k_1, k_2 \in P} \hat{f}^*_{(i_1-k_1-n),(k_2+i_2+1)} \hat{g}_{(j_1-k_1-n),(j_2+k_2+1)} \\
  &=  \sum_{k_1, k_2 \in P} \hat{f}_{(k_1-i_1+n),(-k_2-i_2-1)} \hat{g}_{(j_1-k_1-n),(j_2+k_2+1)}
\end{flalign*}
\noindent
Let us define the matrix $\Jmat_{n^2}$ of size $n^2 \times n^2$ as the anti-identity matrix. We have the following:

\begin{flalign*}
  \leftmat \Gmat^{\alpha_1 \top}(f) \Gmat^{\alpha_1}(g^*) \rightmat_{i, j} &= \sum_{k_1, k_2 \in P} \hat{f}_{(k_1+i_1+1),(i_2-k_2)} \hat{g}^*_{(j_1+k_1+1),(j_2-k_2)} \\
  &= \sum_{k_1, k_2 \in P} \hat{f}_{(k_1+i_1+1),(i_2-k_2)} \hat{g}_{(-j_1-k_1-1),(k_2-j_2)} \\
  \Leftrightarrow \leftmat \Jmat_{n^2} \Gmat^{\alpha_1 \top}(f) \Gmat^{\alpha_1}(g^*) \Jmat_{n^2} \rightmat_{i, j} &= \sum_{k_1, k_2 \in P} \hat{f}_{(k_1-i_1+n),(k_2-i_2)} \hat{g}_{(j_1-k_1-n),(j_2-k_2)}
\end{flalign*}

\begin{flalign*}
  \leftmat \Gmat^{\alpha_2 \top}(f) \Gmat^{\alpha_2}(g^*) \rightmat_{i, j} &=  \sum_{k_1, k_2 \in P} \hat{f}_{(i_1-k_1),(k_2+i_2+1)} \hat{g}^*_{(j_1-k_1),(j_2+k_2+1)} \\
  &=  \sum_{k_1, k_2 \in P} \hat{f}_{(i_1-k_1),(k_2+i_2+1)} \hat{g}_{(k_1-j_1),(-j_2-k_2-1)} \\
  \Leftrightarrow \leftmat \Jmat^{n^2} \Gmat^{\alpha_2 \top}(f) \Gmat^{\alpha_2}(g^*) \Jmat_{n^2} \rightmat_{i, j} &=  \sum_{k_1, k_2 \in P} \hat{f}_{(k_1-i_1),(k_2-i_2+n)} \hat{g}_{(j_1-k_1),(j_2-k_2-n)}
\end{flalign*}

\begin{flalign*}
  \leftmat \Gmat^{\alpha_3 \top}(f) \Gmat^{\alpha_3}(g^*) \rightmat_{i, j} &=  \sum_{k_1, k_2 \in P}  \hat{f}_{(k_1+i_1+1),(k_2+i_2+1)} \hat{g}^*_{(j_1+k_1+1),(k_2+j_2+1)} \\
  &=  \sum_{k_1, k_2 \in P} \hat{f}_{(k_1+i_1+1),(k_2+i_2+1)} \hat{g}_{(-j_1-k_1-1),(-k_2-j_2-1)} \\
  \Leftrightarrow \leftmat \Jmat_{n^2} \Gmat^{\alpha_3 \top}(f) \Gmat^{\alpha_3}(g^*) \Jmat_{n^2} \rightmat_{i, j} &=  \sum_{k_1, k_2 \in P} \hat{f}_{(k_1-i_1+n),(k_2-i_2+n)} \hat{g}_{(j_1-k_1-n),(-k_2+j_2-n)}
\end{flalign*}

\begin{flalign*}
  \leftmat \Gmat^{\alpha_4 \top}(f) \Gmat^{\alpha_4}(g^*) \rightmat_{i, j} &=  \sum_{k_1, k_2 \in P}  \hat{f}_{(-k_1+i_1-n),(k_2+i_2+1)} \hat{g}^*_{(j_1-k_1-n),(j_2+k_2+1)} \\
  &= \sum_{k_1, k_2 \in P} \hat{f}_{(-k_1+i_1-n),(k_2+i_2+1)} \hat{g}_{(-j_1+k_1+n),(-j_2-k_2-1)} \\
  \Leftrightarrow \leftmat \Jmat_{n^2} \Gmat^{\alpha_4 \top}(f) \Gmat^{\alpha_4}(g^*) \Jmat_{n^2} \rightmat_{i, j} &= \sum_{k_1, k_2 \in P} \hat{f}_{(-k_1-i_1-1),(k_2-i_2+n)} \hat{g}_{(j_1+k_1+1),(j_2-k_2-n)}
\end{flalign*}

\noindent
Now, we can observe from Equation~\ref{equation:split_double_sum} that:
\begin{equation}
  \Dmat(fg) = \Dmat(f)\Dmat(g) + \sum_{p=0}^3 \Gmat^{\alpha_p \top}(f^*) \Gmat^{\alpha_p}(g) + \Jmat_{n^2} \left( \sum_{p=0}^3 \Gmat^{\alpha_p \top}(f) \Gmat^{\alpha_p}(g^*) \right) \Jmat_{n^2}.
\end{equation}
which concludes the proof. 
\end{proof}

%% file: sources/appendix/ap2-training_video_classification.tex
\chapter{Diagonal Circulant Neural Networks for Video Classification}
\label{appendix:ap2-diagonal_circulant_neural_networks_for_video_classification}
\localtoc

\noindent
\emph{This Appendix reports some additional experiments on video classification with diagonal-circulant neural networks.
These experiments have been done in the context of the \yt\footnote{https://www.kaggle.com/c/youtube8m} video classification challenge.
This work was recognized as one of the 5 original approaches by the Google AI team that organized the workshop \cite{lee20182nd}.}

\vspace{\fill}

\section{Introduction}
\label{section:ap2-introduction}

Classification of unlabeled videos streams is one of the challenging tasks for machine learning algorithms.
Research in this field has been stimulated by the recent release of several large annotated video datasets such as \emph{Sports-1M}~\cite{karpathy2014large}, \emph{FCVID}~\cite{FCVID} or the \yt~\cite{abu2016youtube} dataset.

The naive approach to achieve video classification is to perform frame-by-frame image recognition, and to average the results before the classification step.
However, it has been shown by~\citet{abu2016youtube,miech2017learnable} that better results can be obtained by building features across different frames and several deep learning architectures have been designed to learn embeddings for sets of frames.
For example Deep Bag-of-Frames~(DBoF)~\cite{abu2016youtube}, NetVLAD~\cite{arandjelovic2016netvlad} or architectures based on Fisher Vectors~\cite{perronnin2007fisher}. 

The DBoF embedding layer processes videos in two steps.
First, a learned transformation projects all the frames together into a high dimensional space. 
Then, a max (or average) pooling operation aggregates all the embedded frames into a single discriminative vector representation of the video.
The NetVLAD embedding layer is built on a technique called \emph{vector of locally aggregated descriptors} (VLAD)~\cite{jegou2010aggregating}.
This technique that aggregates a large number of local frame descriptors into a compact representation using a codebook of visual words.
In NetVlad, the codebook is directly learned end-to-end during training.
Finally, NetFisherVector (NetFV) is inspired by~\citet{perronnin2007fisher} and uses first and second-order statistics as video descriptors also gathered in a codebook.
The technique can benefit from deep learning by using a deep neural network to learn the codebook \cite{miech2017learnable}.

All the architectures mentioned above can be used to build video features in the sense of features that span across several frames, but they are not designed to exploit the sequential nature of videos and capture motion.
In order to truly learn spatio-temporal features and account for motion in videos, several researchers have looked into recurrent neural networks (\eg LSTM~\cite{yue2015beyond,li2017temporal}) and 3D convolutions~\cite{karpathy2014large} (in space and time).
However, these approaches do not outperform non-sequential models, and the single best model proposed by~\citet{miech2017learnable} (winner of the first \yt competition) is based on NetVLAD. 

The \emph{2nd YouTube-8M Video Understanding Challenge} includes a constraint on the model size and many competitors have been looking into building efficient memory models with high accuracy.
There are two kinds of techniques to reduce the memory required for training and/or inference in neural networks.
The first kind aims at \emph{compressing} an existing neural network into a smaller one, (thus it only impacts the size of the model at inference time).
The second one aims at \emph{constructing models that are compact} by design.

\section{Compact Architecture using Diagonal and Circulant Matrices}
\label{section:ap2-compact_video_classification_architecture_using_diagonal_and_circulant_matrices}

Building on the matrix decomposition presented in~\Cref{chapter:ch4-diagonal_circulant_neural_network}, we introduce a compact neural network architecture for video classification where dense matrices have been replaced by products of circulant and diagonal matrices.

\subsection{Base Model}
\label{subsection:ap2-base_model}

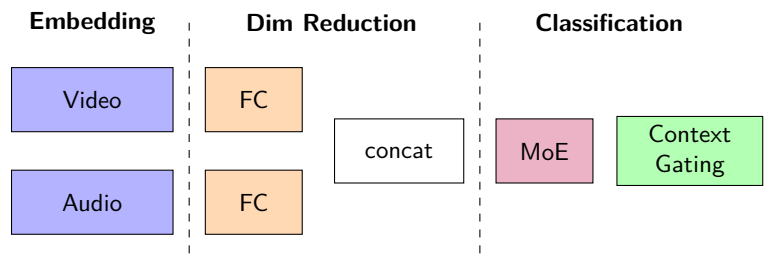
\begin{figure}[ht]
  \centering
  \input{figures/appendix/ap2-training_video_classification/fig_baseline}
  \caption{Diagram of the architecture proposed by~\citet{miech2017learnable} and used for the experiments. The sample goes through an embedding layer and is reduced with a Fully Connected layer. The results are then concatenated and classified with a Mixture-of-Experts and Context Gating layer.}
  \label{figure:ap2-model_baseline}
\end{figure}

We demonstrate the benefit of the decomposition into diagonal and circulant matrices using a base model which has been proposed by~\citet{miech2017learnable}.
This architecture can be decomposed into three blocks of layers, as illustrated in~\Cref{figure:ap2-model_baseline}. 
The first block of layers, composed of the Deep Bag-of-Frames embedding, is meant to process audio and video frames independently.
The DBoF layer computes two embeddings: one for the audio and one for the video.
In the next paragraph, we will only focus on describing the video embedding (The audio embedding is computed in a very similar way).
We represent a video $\Vmat$ as a set of $m$ frames $\{\vvec^{(1)}, \ldots, \vvec^{(m)}\}$ where each frame $\vvec^{(i)} \in \Rbb^k$ is a vector of visual features extracted from the frame image.
In the context of the \yt competition, each $\vvec_i$ is a vector of 1024 visual features extracted using the last fully connected layer of an Inception network trained on ImageNet.
The DBoF layer then embeds a video $\Vmat$ into a vector $\vvec'$ drawn from a $\pvec$ dimensional vector space as follows:
\begin{equation}
  \vvec' = \max \left\{\Wmat \vvec^{(i)} \mid \vvec^{(i)} \in \Vmat \right\}
\end{equation}
where $\Wmat$ is a matrix in $\Rbb^{p \times k}$ (learned) and max is the element-wise maximum operator.
We typically choose $p > k$, (\eg $p = 8192$).
Note that because this formulation is framed in term of sets, it can process videos of different lengths (\ie, a different value of $m$).
A second block of layers reduces the dimensionality of each embedding layer (audio and video), and merges the result into a single vector by using a simple concatenation operation.
We chose to reduce the dimensionality of each embedding layer separately \emph{before} the concatenation operation to avoid the concatenation of two high dimensional vectors.

Finally, the classification block uses a combination of Mixtures-of-Experts (MoE) and Context Gating to calculate the final probabilities.
The Mixtures-of-Experts layer introduced by~\citet{jordan1993hierarchical} and proposed for video classification by~\citet{abu2016youtube} is used to predict each label independently.
It consists of a gating and experts networks which are concurrently learned.
The gating network learns which experts to use for each label and the experts layers learn how to classify each label.
The context gating operation was introduced by~\citet{miech2017learnable} and captures dependencies among features and re-weight probabilities based on the correlation of the labels.
For example, it can capture the correlation of the labels \emph{ski} and \emph{snow} and re-adjust the probabilities accordingly.

\subsection{Robust Deep Bag-of-Frames pooling method}
\label{subsection:ap2-robust_deep_bag-of-frames_pooling_method}

We propose a technique to extract more performance from the base model with DBoF embedding.
The maximum pooling is sensitive to outliers and noise whereas the average pooling is more robust.
We propose a method which consists in taking several samples of frames, applying the upsampling followed by the maximum pooling to these samples, and then averaging over all samples.
More formally, assume a video contains $m$ frames $\vvec_1, \ldots, \vvec_m \in \Rbb^{1024}$.
We first draw $n$ random samples $\Smat^{(1)} \ldots \Smat^{(n)}$ of size $k$ from the set $\left\{ \vvec^{(1)}, \ldots, \vvec^{(m)} \right\}$.
The output of the robust-DBoF layer is:
\begin{equation}
  \frac{1}{n} \sum_{i=1}^{n} \max \left\{ \Wmat \vvec :\vvec \in S_{i} \right\} 
\end{equation}
\noindent
Depending on $n$ and $k$, this pooling method is a trade-off between the max pooling and the average pooling.
Thus, it is more robust to noise, as will be shown in the experiments section.

\subsection{Compact Representation of the Base Model}
\label{subsection:ap2-compact_representation_of_the_base_model}


In order to build compact layers, we use the diagonal-circulant matrix decomposition.
The layers are represented as follows:
\begin{equation}
  \phi(\xvec) = \rho\left(\left[\prod_{i=1}^{m} \Dmat^{(i)} \Cmat^{(i)} \right] \xvec + \bvec \right)
\end{equation}
where the parameters of each matrix $\Dmat^{(i)}$ and $\Cmat^{(i)}$ are trained using a gradient based optimization algorithm, and $m$ defines the number of factors.
Increasing the value of $m$ increases the number of trainable parameters and therefore the modeling capabilities of the layer.
In our experiments, we chose the number of factors $m$ empirically to achieve the best trade-off between model size and accuracy.

To measure the impact of the size of the model and its accuracy, we represent layers in their compact form independently.
Given that circulant and diagonal matrices are square, we use concatenation and slicing to achieve the desired dimension.
As such, with $m=1$, the weight matrix of size $1024 \times 8192$ of the video embedding is represented by a concatenation of 8 DC matrices and the weight matrix of size $8192 \times 512$ is represented by a single DC matrix with size $8192 \times 8192$ and the resulting output is sliced at the 512 dimension.
We denote layers in their classic form as \emph{``Dense''} and layers represented with circulant and diagonal factors as \emph{``Compact''}.

\subsection{Leveraging Architectural Diversity}
\label{subsection:ap2-leveraging_architectural_diversity}

In order to benefit from architectural diversity, we also devise a single model architecture that combines different types of embedding layers.
As we can see in~\Cref{figure:ap2-diverstiy_architecture}, video and audio frames are processed by several embedding layers before being reduced by a series of compact fully connected layers.
The output of the compact fully connected layers are then averaged, concatenated and fed into the final classification block.
\Cref{figure:ap2-models} shows the result of different models given the number of parameters.
The use of circulant matrices allow us to fit this model in GPU memory.
For example, the diversity model with a NetVLAD embedding (cluster size of 256) and NetFV embedding (cluster size of 128) has 160 millions parameters (600 Mo) in the compact version and 728M (2.7 Go) in the dense version. 

\begin{figure}[ht]
  \centering
  \input{figures/appendix/ap2-training_video_classification/fig_ensemble_with_fc}
  \caption{Diagram of architecture with several embeddings devised to leverage the diversity of an Ensemble in a single model.}
  \label{figure:ap2-diverstiy_architecture}
\end{figure}
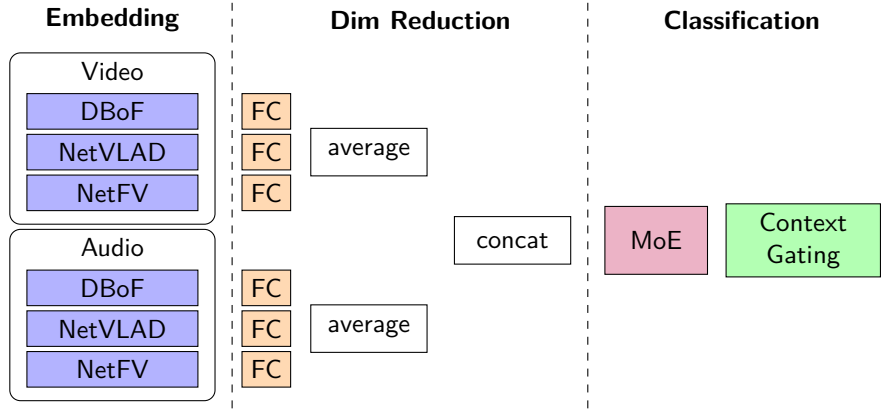

\section{Experiments}
\label{section:ap2-experiments}

In this section, we first evaluate the pooling technique proposed in~\Cref{subsection:ap2-robust_deep_bag-of-frames_pooling_method}.
Then, we conduct experiments to evaluate the accuracy of our compact models.
In particular, we investigate which layer benefits the most from a circulant representation and show that our approach where both the diagonal and the circulant is learned performs better than the approach from~\citet{cheng2015exploration} for the video classification problem.
Finally, we compare all our models on a two dimensional size \vs accuracy scale in order to evaluate the trade-off between size and accuracy of each one of our models.
All the figures of this section can be found at the end of the chapter.

\paragraph{Experimental Setup}

All the experiments of this appendix have been done in the context of the \emph{2nd YouTube-8M Video Understanding Challenge} with the \yt dataset.
We trained our models with the full training set and 70\% of the validation set which corresponds to a total of \numprint{4822555} examples.
We used the data augmentation technique proposed by~\citet{skalic2017deep} to virtually double the number of inputs. 
The method consists in splitting the videos into two equal parts.
This approach is motivated by the observation that a human could easily label the video by watching either the beginning or the ending of the video. 
All our experiments are developed with TensorFlow Framework~\cite{tensorflow2015-whitepaper}.
We trained our models with the CrossEntropy loss and used Adam optimizer with a 0.0002 learning rate and a 0.8 exponential decay every 4 million examples.
All fully connected layers are composed of 512 units.
DBoF, NetVLAD and NetFV are respectively 8192, 64 and 64 of cluster size for video frames and 4096, 32, 32 for audio frames.
We used 4 mixtures for the MoE Layer.
We used all the available 150 frames and robust max pooling introduced in~\Cref{subsection:ap2-robust_deep_bag-of-frames_pooling_method} for the DBoF embedding.
In order to stabilize and accelerate the training, we used batch normalization before each nonlinear activation and gradient clipping. 
We used the GAP (Global Average Precision), as used in the \emph{2nd YouTube-8M Video Understanding Challenge}, to compare our experiments.
The GAP metric is defined as follows:
\begin{equation}
  \text{GAP} = \sum_{i=1}^{P}p(i) \Delta r(i)
\end{equation}
where $P$ is the number of final predictions, $p(i)$ the precision, and $r(i)$ the recall.
We limit our evaluation to 20 predictions for each video. 
All experiments have been realized on a cluster of 12 nodes. Each node has 160 POWER8 processor, 128 Go of RAM and 4 Nividia Titan P100.

\paragraph{Robust Deep Bag-of-Frames pooling method}
\label{subsection:ap2-robust_deep_bag-of-frames_pooling_method_experiements}

We evaluate the performance of our Robust DBoF embedding.
In accordance with the work of~\citet{abu2016youtube}, we find that average pooling performs better than maximum pooling. 
\Cref{figure:ap2-learning_curve_bagging} shows that the proposed robust maximum pooling method outperforms both maximum and average pooling.

\paragraph{Impact of Circulant Matrices on Different Layers}

\begin{table}[ht]
  \centering
  \begin{tabular}{lccccc}
    \toprule
    \textbf{Model} & \textbf{\#Parameters} & \textbf{CR} & \textbf{GAP@20} & \textbf{Diff.} \\
    \midrule
    Dense Model  & 45M &   -- & \textbf{0.846} & -- \\
    Compact DBoF & 36M & 18.4 & 0.838 & -0.008\\
    Compact FC   & 41M &  9.2 & 0.845 & \textbf{-0.001} \\
    Compact MoE  & 12M & 72.0 & 0.805 & -0.041 \\
   \bottomrule
  \end{tabular}
  \caption{Effect of the compactness of different layers.} 
  \label{table:ap4-circulant_layer}
\end{table}

This series of experiments aims at understanding the effect of compactness over different layers.
\Cref{table:ap4-circulant_layer} shows the result in terms of the number of weights, compression ratio (CR) with respect to the dense model and GAP score.
In these experiments, for speeding-up the training phase, we did not use the audio features and exploited only the video information.
The compact fully connected layer achieves a compression rate of 9.5 while having a very similar performance, whereas the compact DBoF and MoE achieve a higher compression rate at the expense of accuracy. 
\Cref{figure:ap2-learning_curve_layers} shows that the model with a compact FC converges faster than the dense model.
The model with a compact DBoF shows a big variance over the validation GAP which can be associated with a difficulty to train.
The model with a compact MoE is more stable but at the expense of its performance.

Another series of experiments investigates the effect of adding factors of diagonal-circulant layers.
\Cref{table:ap2-factors} shows that there is no gain in accuracy even if the number of weights increases.
It also shows that adding factors has an important effect on the speed of training.
On the basis of this result, \ie given the performance and compression ratio, we can consider that representing the fully connected layer of the base model in a compact fashion can be a good trade-off.



\begin{table}[ht]
  \centering
  \begin{tabular}{ccccc}
    \toprule
    \multirow{2}{*}{\textbf{\#Factors}} & \multicolumn{2}{c}{\textbf{FC Layer}} & \multirow{2}{*}{\textbf{GAP@20}} \\
    \cmidrule{2-3}
     & \textbf{\#Parameters} & \textbf{CR} \\
    \midrule
    -- & 6.29M & -- & 0.861 \\
    1 & 12K & 99.8 & 0.861 \\
    3 & 73K & 98.8 & 0.861 \\
    6 & 147K & 97.6 & 0.859 \\
    \bottomrule
  \end{tabular}
  \caption{Evolution of the number of parameters and accuracy according to the number of factors.} 
  \label{table:ap2-factors}
\end{table}

\paragraph{Comparison with Related Works}

Circulant matrices have already been used in neural networks by~\citet{cheng2015exploration}.
They proposed to replace fully connected layers by a circulant and diagonal matrices where the circulant matrix is learned by a gradient based optimization algorithm and the diagonal matrix is random with values in \{-1, 1\}.
We compare our more general framework with their approach.
\Cref{figure:ap2-learning_dc_cd} shows the validation GAP according to the number of epochs of the base model with a compact fully connected layer implemented with both approaches.


\paragraph{Compact Baseline Model with Different Embeddings}

To compare the performance and the compression ratio we can expect, we consider different settings where the compact fully connected layer is used together with different embeddings.
Figures~\ref{figure:ap2-validation_gap_compact_dbof}, \ref{figure:ap2-validation_gap_compact_netvlad}, 
\ref{figure:ap2-validation_gap_compact_netfv} and \Cref{table:ap2-fc_circulant_with_diff_embedding} 
show the performance of the base model with DBoF, NetVLAD and NetFV embeddings with a \emph{Dense} and \emph{Compact} layer.
Notice that we can get a bigger compression rate with NetVLAD and NetFV due to the fact that the output of the embedding is in a higher dimensional space which implies a larger weight matrix for the fully connected layer.
Although the compression rate is higher, it is at the expense of accuracy.

\paragraph{Tradeoff Between Model Size and Accuracy}

To conclude our experimental evaluation, we compare all our models in terms of size and accuracy.
The results are presented in~\Cref{figure:ap2-models}. 
As we can see in this figure, the most compact models are obtained with the circulant NetVLAD and NetFV.
We can also see that the complex architectures described in~\Cref{subsection:ap2-leveraging_architectural_diversity} (DBoF + NetVLAD) achieve top performance but at the cost of a large number of weights.
Finally, the best trade-off between size and accuracy is obtained using the DBoF embedding layer and achieves a GAP of 0.861 for only 60 millions weights.

\begin{table}[t]
  \centering
  \begin{tabular}{llcccc}
    \toprule
    \textbf{Embedding} & \textbf{Method} & \textbf{\#Parameters} & \textbf{CR} & \textbf{GAP@20} \\
    \midrule
    \multirow{2}{*}{\textbf{DBoF}} & FC Dense & 65M & -- & 0.861 \\
     & FC Circulant & 59M & 9.56 & 0.861 \\
    \midrule
    \multirow{2}{*}{\textbf{NetVLAD}} & FC Dense & 86M & -- & 0.864 \\
     &FC Circulant & 50M & 41.1 & 0.851 \\
    \midrule
    \multirow{2}{*}{\textbf{NetFisher}} & FC Dense & 122M & -- & 0.860 \\
     & FC Circulant & 51M & 58.1 & 0.848 \\
    \bottomrule
  \end{tabular}
  \caption{Impact of the compression of the fully connected layer of the model architecture with Audio and Video features vector and different types of embeddings.} 
  \label{table:ap2-fc_circulant_with_diff_embedding}
\end{table}

\section{Concluding Remarks}
\label{section:ap2-conclusion}

In this appendix, we demonstrated that circulant matrices and diagonal matrices can be a great tool to design compact neural network architectures for video classification tasks.
Our experiments demonstrate that the best trade-off between size and accuracy is obtained using circulant DBoF embedding layers.
We investigated a model with multiple embeddings to leverage the performance of an Ensemble but found it ineffective.
The good performance of Ensemble models, \ie, why aggregating different distinct models performs better that incorporating all the diversity in a single architecture is still an open problem.
Our future work will be devoted to address this challenging question and to pursue our effort to devise compact models achieving the same accuracy as larger one, and to study their theoretical properties.

\begin{figure}[p!]
  \centering
  \includegraphics[width=\scalefigure\textwidth]{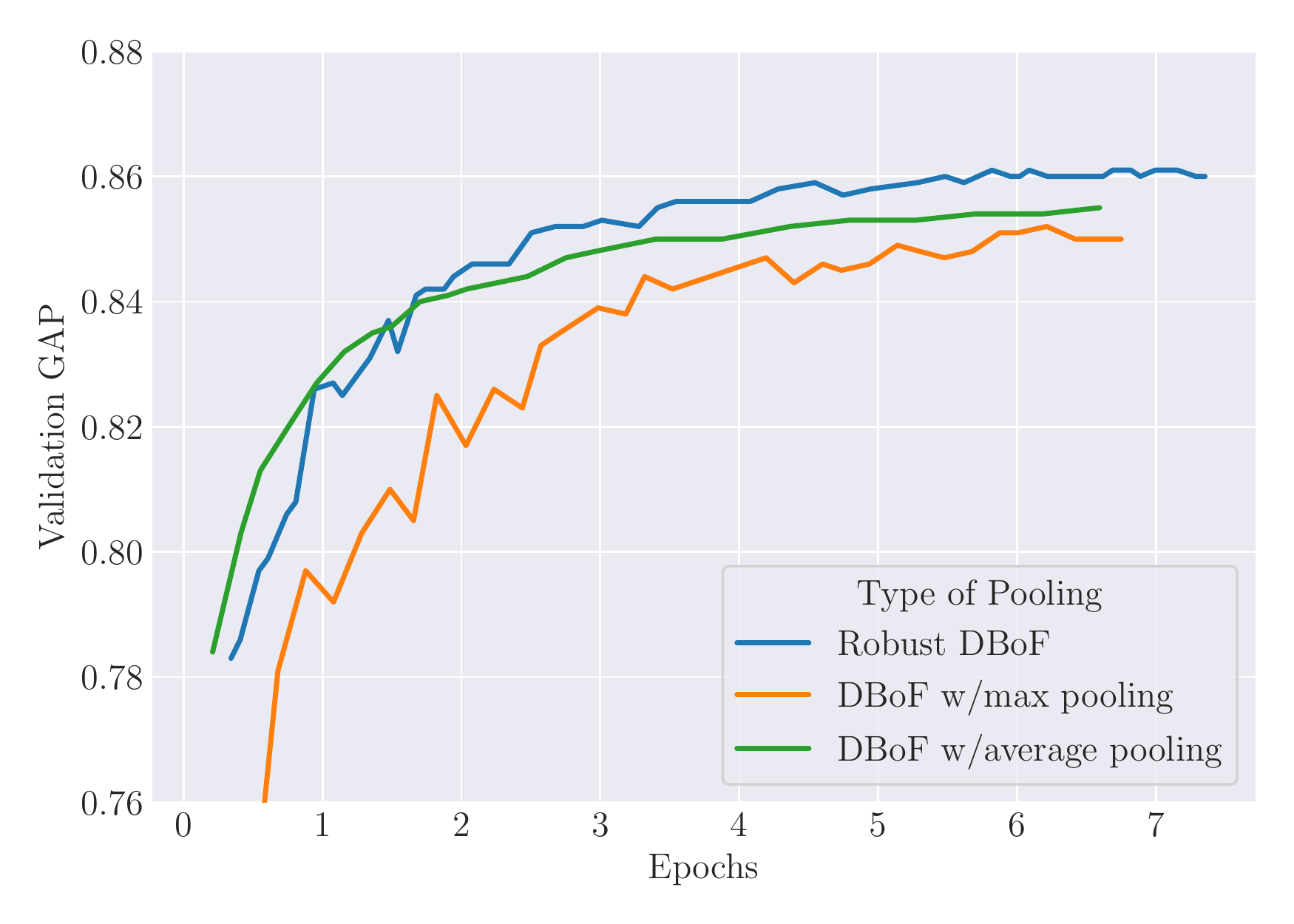}
  \caption{Impact of \emph{robust DBoF} with $n=10$ and $k=15$ on the Deep Bag-of-Frames embedding compared to max and average pooling.}
  \label{figure:ap2-learning_curve_bagging}
\end{figure}

\begin{figure}[p!]
  \centering
  \includegraphics[width=\scalefigure\textwidth]{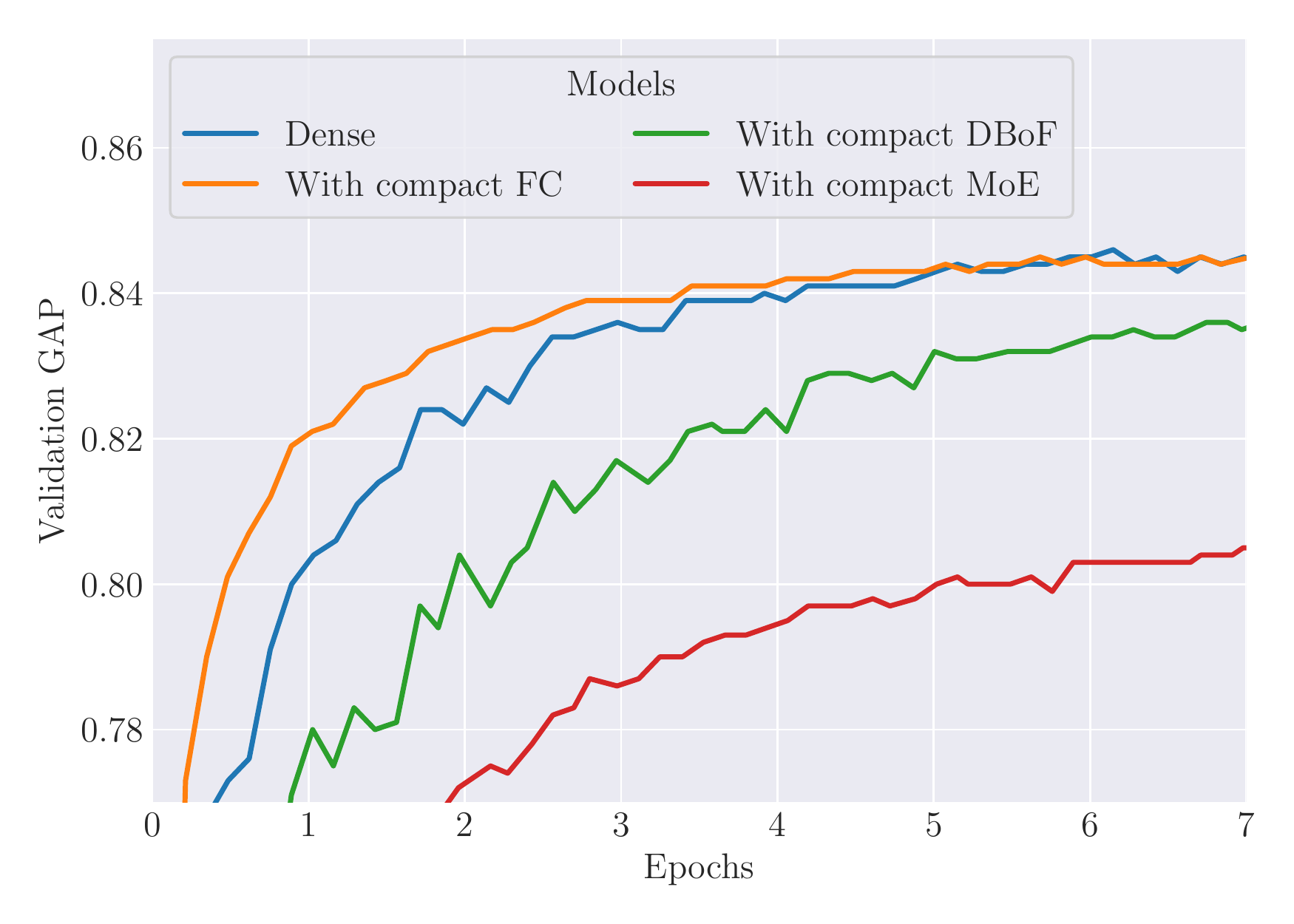}
  \caption{GAP score of models according to the number of epochs for different compact models.}
  \label{figure:ap2-learning_curve_layers}
  \vspace{2cm}
  \includegraphics[width=\scalefigure\textwidth]{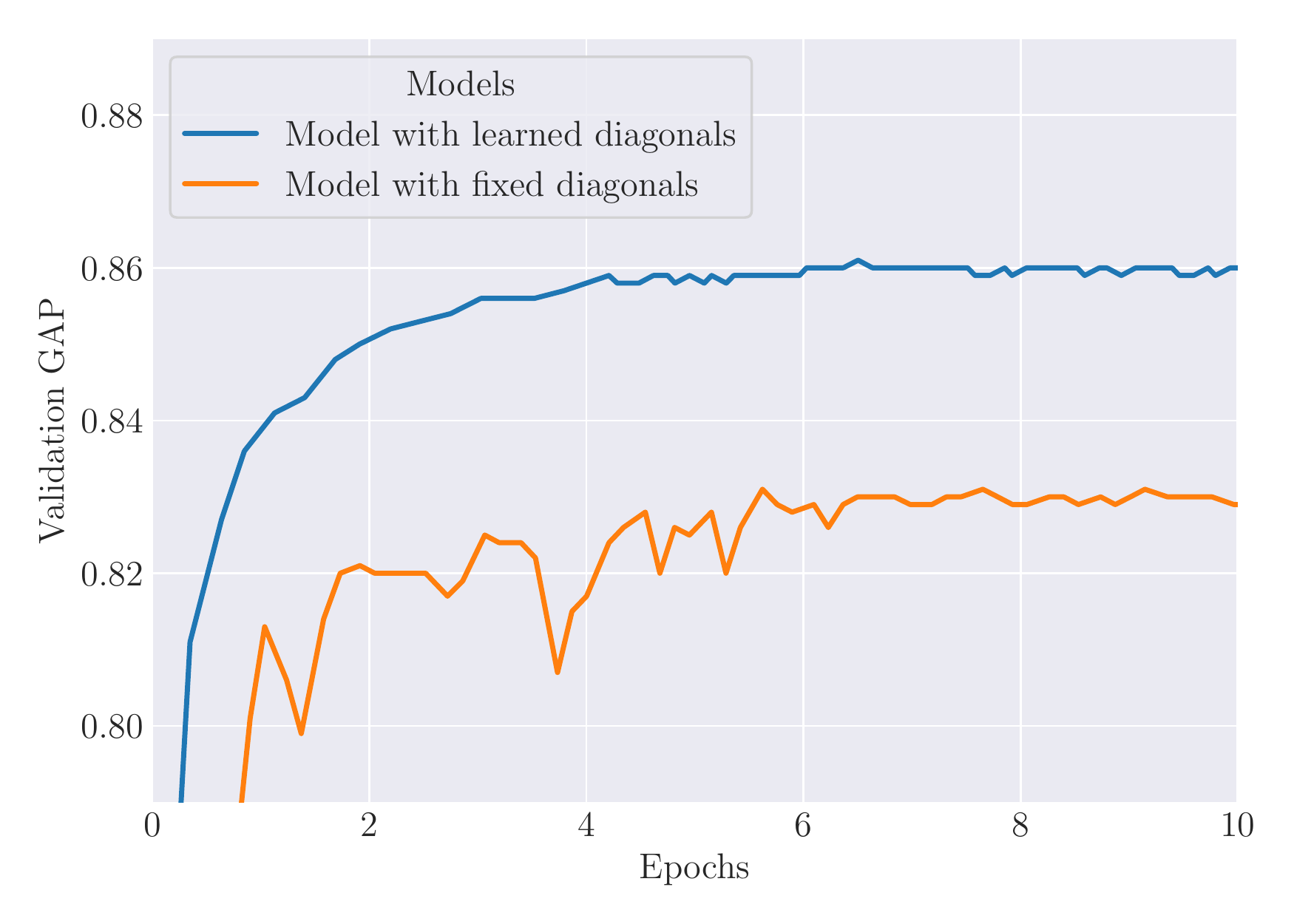}
  \caption{GAP difference between the approach proposed by~\citet{cheng2015exploration} where the diagonals from the decomposition are initialized from the set $\{-1, +1\}$ and kept fixed and our approach where the values of the diagonals are learned.} 
  \label{figure:ap2-learning_dc_cd}
\end{figure}

\begin{figure}[p!]
  \center
  \includegraphics[width=\scalefigure\textwidth]{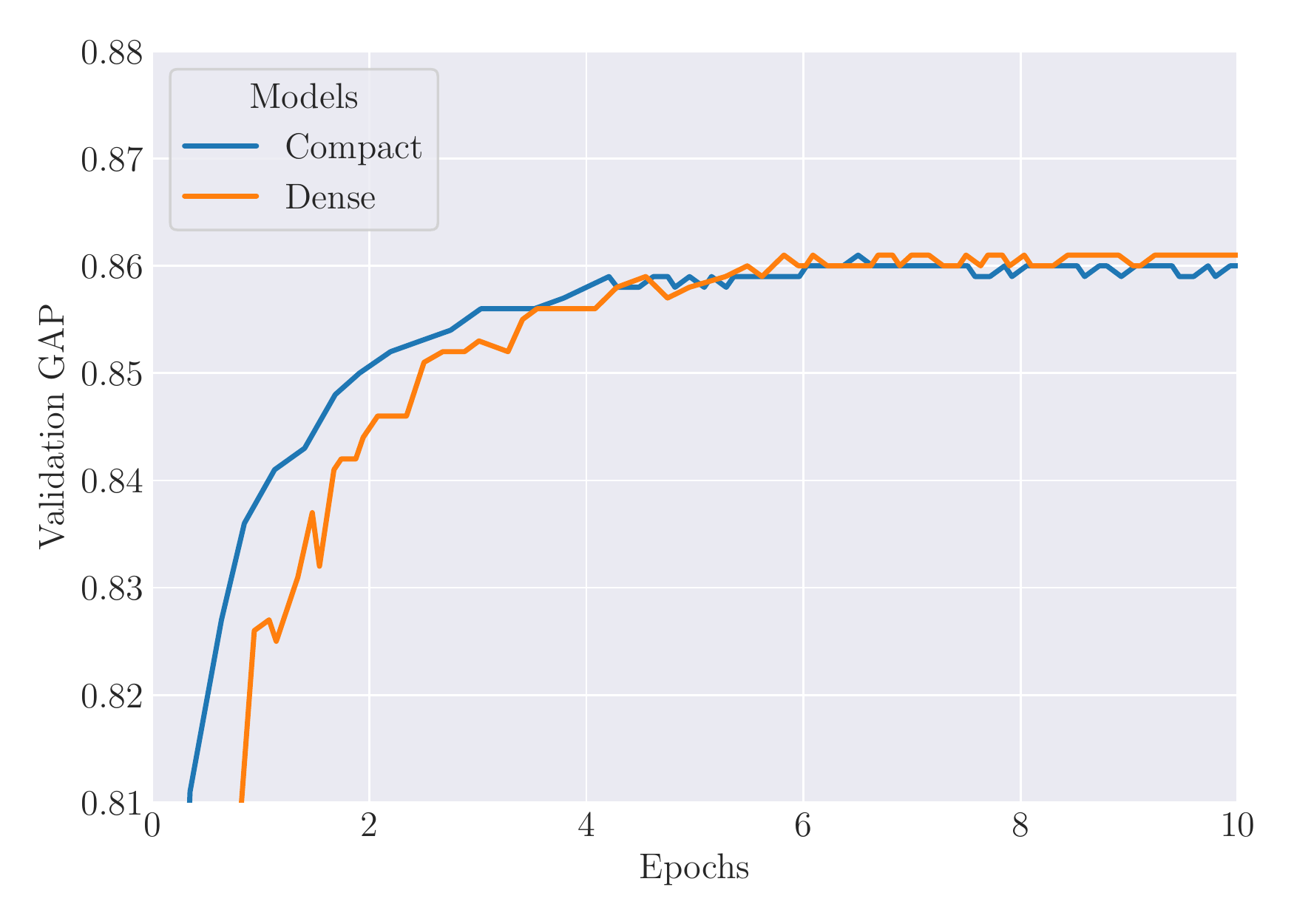}
  \caption{GAP score of models with compact DBoF embedding and dense fully connected layer.}
  \label{figure:ap2-validation_gap_compact_dbof}
  \vspace{2cm}
  \includegraphics[width=\scalefigure\textwidth]{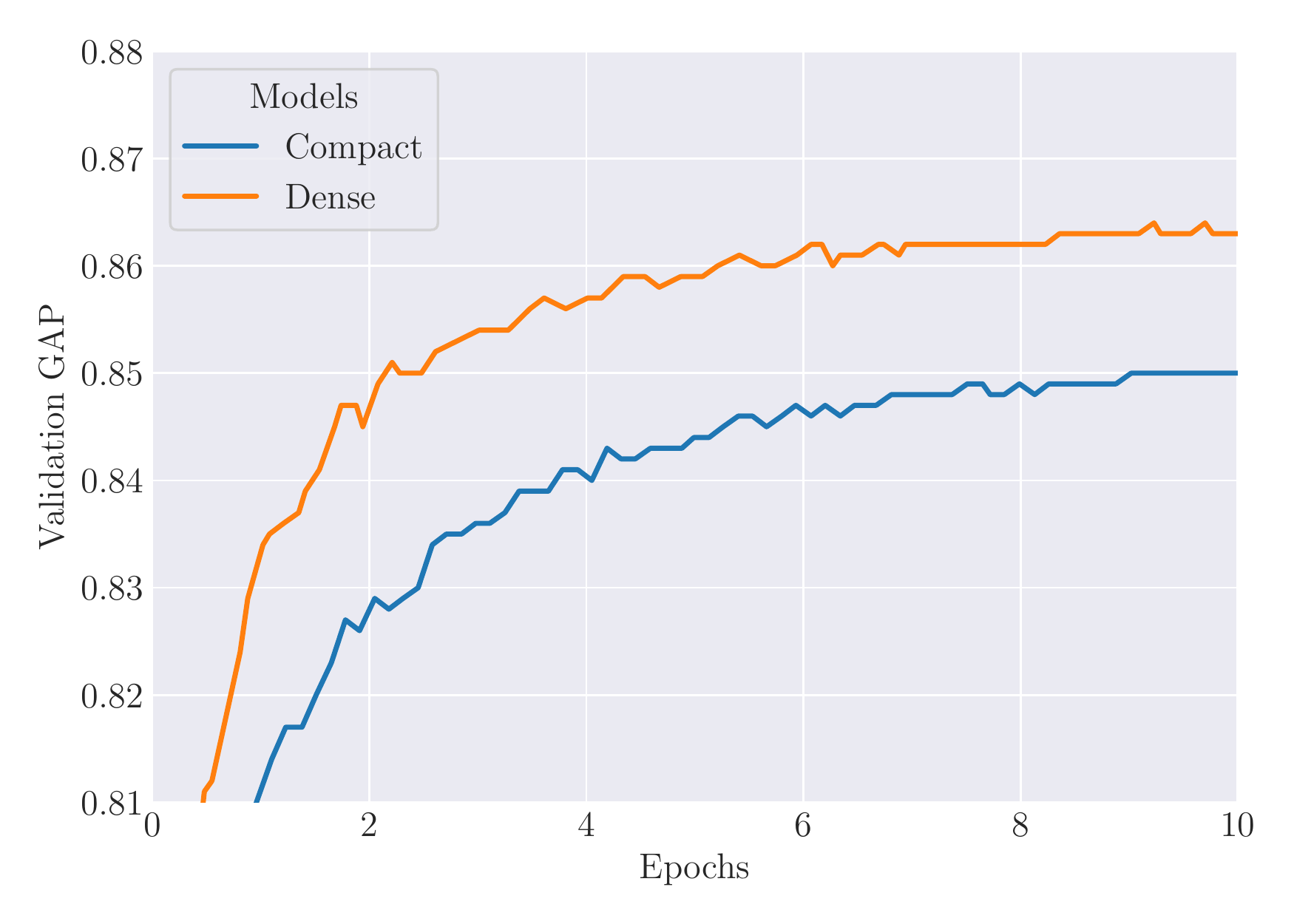}
  \caption{GAP score of models with compact NetVLAD embedding and dense fully connected layer.}
  \label{figure:ap2-validation_gap_compact_netvlad}
\end{figure}

\begin{figure}[p!]
  \center
  \includegraphics[width=\scalefigure\textwidth]{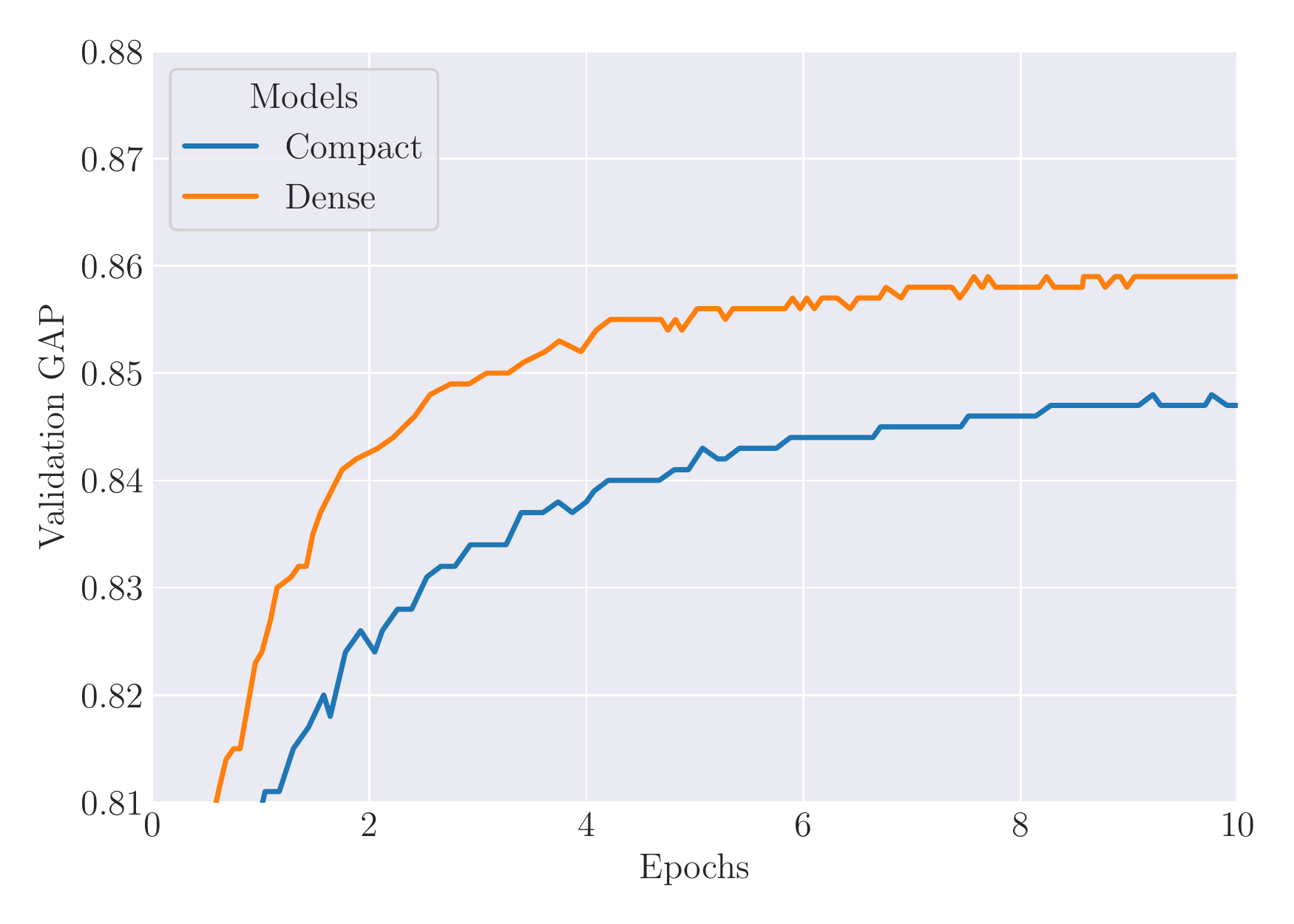}
  \caption{GAP score of models with compact NetFV embedding and dense fully connected layer.}
  \label{figure:ap2-validation_gap_compact_netfv}
  \vspace{2cm}
  \includegraphics[width=\scalefigure\textwidth]{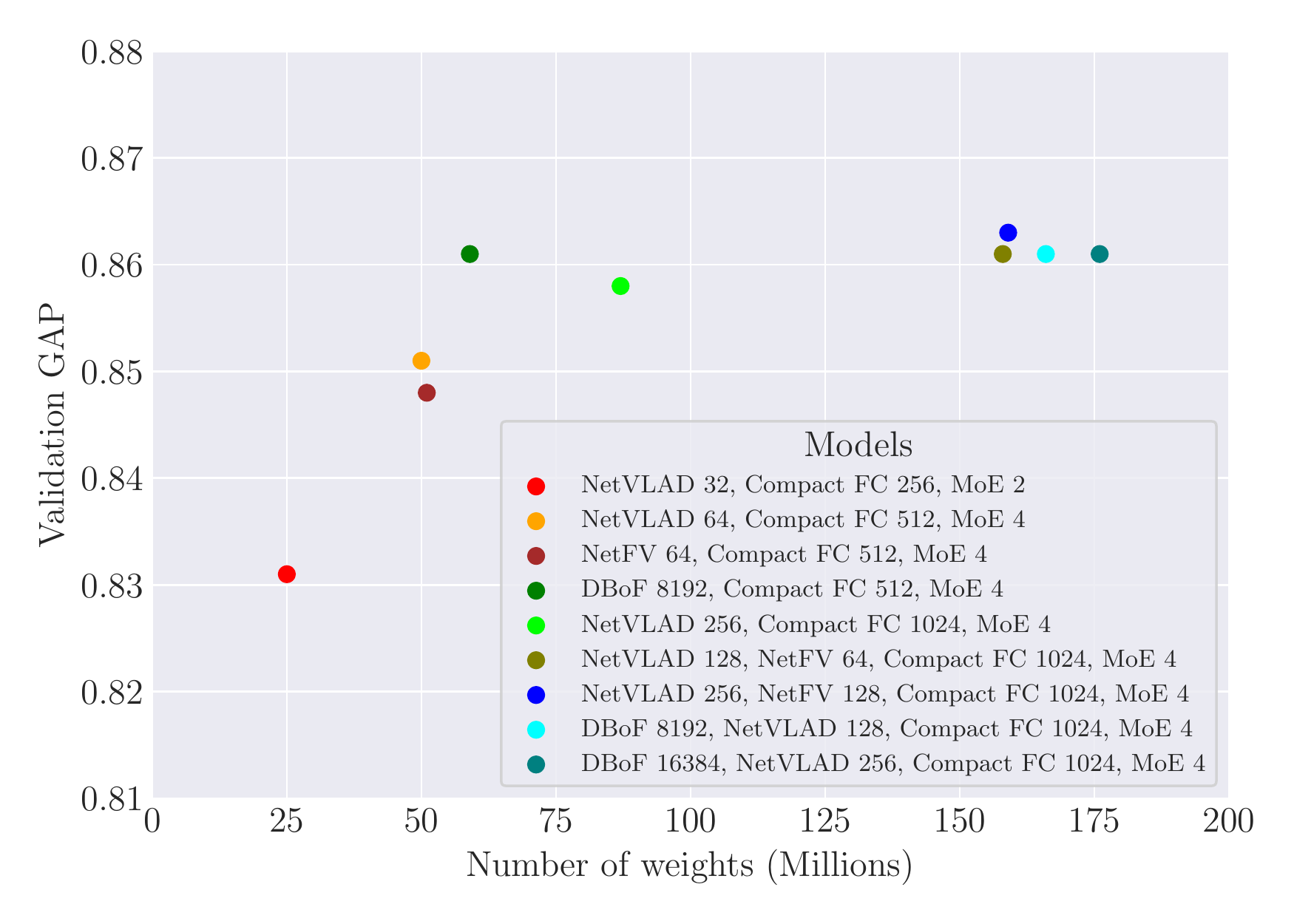}
  \caption{Comparison between different models with compact fully connected layers.}
  \label{figure:ap2-models}
\end{figure}

%% file: figures/appendix/ap2-training_video_classification/fig_baseline.tex
\tikzset{%
  >={Latex[width=2mm,length=2mm]},
            base/.style = {rectangle, draw=black, text centered, font=\sffamily},
             box/.style = {base, rounded corners, text depth=3cm, minimum height=4cm, minimum width=3cm},
     transparent/.style = {rectangle, draw=black},
       circulant/.style = {base, fill=yellow!30},
       embedding/.style = {base, fill=blue!30, minimum width=2.5cm, minimum height=1cm},
           other/.style = {base, fill=white!30,  minimum width=2cm, minimum height=1cm},
              fc/.style = {base, fill=orange!30, minimum width=1.5cm, minimum height=1cm},
          gating/.style = {base, fill=green!30, minimum width=2cm, text width=2cm, minimum height=1cm},
             moe/.style = {base, fill=purple!30, minimum width=1.5cm, minimum height=1cm},
}

\begin{tikzpicture}[every node/.style={fill=white, font=\sffamily}, align=center, scale=0.85, every node/.style={scale=0.85}]

  \draw (0.0, +2.)  node [other, draw=none] {\textbf{Embedding}};
  \draw (+3.7, +2.)  node [other, draw=none] {\textbf{Dim Reduction}};
  \draw (+8.0, +2.)  node [other, draw=none] {\textbf{Classification}};

  \draw (0, +0.8)  node [embedding] {Video};
  \draw (0, -0.8)  node [embedding] {Audio};

  \draw (+2.5, +0.8)  node (fc) [fc] {FC};
  \draw (+2.5, -0.8)  node (fc) [fc] {FC};

  \draw (+4.75, 0)  node (fc) [other] {concat};
  \draw (+7.0, 0)  node (moe) [moe] {MoE};
  \draw (+9.25, 0)  node (gating2) [gating] {Context Gating};
 
  \draw (+1.5, +2) [dashed] -- (+1.5, -1.7);
  \draw (+6, +2) [dashed] -- (+6, -1.7);
  
\end{tikzpicture}

%% file: figures/appendix/ap2-training_video_classification/fig_ensemble_with_fc.tex
\tikzset{%
  >={Latex[width=2mm,length=2mm]},
            base/.style = {rectangle, draw=black, text centered, font=\sffamily},
             box/.style = {base, rounded corners, text depth=2cm, minimum height=2cm, minimum width=3cm},
     transparent/.style = {rectangle, draw=black},
        fc_small/.style = {base, fill=orange!30, minimum width=0.5cm},
       embedding/.style = {base, fill=blue!30, minimum width=2.5cm},
          concat/.style = {base, fill=white!30, minimum height=1cm},
           other/.style = {base, fill=white!30,  minimum width=1.7cm, minimum height=0.7cm},
              fc/.style = {base, fill=orange!30, minimum width=1.5cm, minimum height=1cm},
          gating/.style = {base, fill=green!30, minimum width=2cm, text width=2cm, minimum height=1cm},
             moe/.style = {base, fill=purple!30, minimum width=1.5cm, minimum height=1cm},
}
\begin{tikzpicture}[every node/.style={fill=white, font=\sffamily}, align=center, scale=0.9, every node/.style={scale=0.9}]

  \draw (0,   0) node (box1) [box] {Video};
  \draw (0, +0.4)  node [embedding] {DBoF};
  \draw (0, -0.2)  node [embedding] {NetVLAD};
  \draw (0, -0.8)  node [embedding] {NetFV};
  \draw (+2.25, +0.4)  node (fc) [fc_small] {FC};
  \draw (+2.25, -0.2)  node (fc) [fc_small] {FC};
  \draw (+2.25, -0.8)  node (fc) [fc_small] {FC};
  
  \draw (0, -2.6) node (box1) [box] {Audio};
  \draw (0, +0.4-2.6)  node [embedding] {DBoF};
  \draw (0, -0.2-2.6)  node [embedding] {NetVLAD};
  \draw (0, -0.8-2.6)  node [embedding] {NetFV};
  \draw (+2.25, +0.4-2.6)  node (fc) [fc_small] {FC};
  \draw (+2.25, -0.2-2.6)  node (fc) [fc_small] {FC};
  \draw (+2.25, -0.8-2.6)  node (fc) [fc_small] {FC};
  
  \draw (+3.75, -0.2-2.6)  node (concat) [other] {average};
  \draw (+3.75, -0.2-0.0)  node (concat) [other] {average};
  \draw (+5.85, -0.8-0.7)  node (concat) [other] {concat};

  \draw (+7.95, -0.8-0.7)  node (moe) [moe] {MoE};
  \draw (+10.1, -0.8-0.7)  node (gating2) [gating] {Context Gating};
  
  \draw (+0, +1.75)  node [other, draw=none] {\textbf{Embedding}};
  \draw (+4.5, +1.75)  node [other, draw=none] {\textbf{Dim Reduction}};
  \draw (+9.2, +1.75)  node [other, draw=none] {\textbf{Classification}};
  
  \draw (+1.75, +2.0) [dashed] -- (+1.75, -4.0);
  \draw (+6.95, +2.0) [dashed] -- (+6.95, -4.0);
    
\end{tikzpicture}

%% file: sources/appendix/ap3-randomized_inference.tex
\chapter{Theoretical Evidence for Adversarial Robustness through Randomization}
\label{appendix:ap3-theoretical_evidence_for_adversarial_robustness_through_randomization}
\localtoc

\vspace{\fill}

\noindent
\emph{
This Appendix concerns a collaboration with Rafael Pinot, Laurent Meunier, Hisashi Kashima, Florian Yger, Cédric Gouy-Pailler and Jamal Atif.
This work has been published at the Conference on Neural Information Processing Systems (NeurIPS) 2019.
It investigates the theory of robustness against adversarial attacks. It focuses on the family of randomization techniques that consist in injecting noise in the network at inference time.
All proofs of this appendix can be found in the long version of the paper.\footnote{https://arxiv.org/abs/1902.01148}
For simplification, we left the notation as in the original paper.
}

\vspace{\fill}


\section{Introduction}
\label{section:ap3-introduction}

Adversarial attacks are some of the most puzzling and burning issues in modern machine learning.
An adversarial attack refers to a small, imperceptible change of an input maliciously designed to fool the result of a machine learning algorithm.
Since the seminal work of~\citet{szegedy2013intriguing} exhibiting this intriguing phenomenon in the context of deep learning, a wealth of results have been published on designing attacks~\cite{goodfellow2014explaining,papernot2016limitations,moosavi2016deepfool,kurakin2016adversarial,carlini2017towards,moosavi2017universal} and defenses~\cite{goodfellow2014explaining,papernot2016distillation,guo2017countering,meng2017magnet,samangouei2018defense,madry2018towards}), or on trying to understand the very nature of this phenomenon~\cite{fawzi2018empirical,simon2018adversarial,fawzi2018adversarial,moosavi2016robustness}.
Most methods remain unsuccessful to defend against powerful adversaries~\cite{carlini2017towards,madry2018towards,athalye2018obfuscated}.
Among the defense strategies, randomization has proven effective in some contexts.
It consists in injecting random noise (both during training and inference phases) inside the network architecture, \ie, at a given layer of the network.
Noise can be drawn either from Gaussian~\cite{xuanqing2018towards,lecuyer2018certified,rakin2018parametric}, Laplace~\cite{lecuyer2018certified}, Uniform~\cite{xie2017mitigating}, or Multinomial~\cite{dhillon2018stochastic} distributions.
Remarkably, most of the considered distributions belong to the Exponential family.
Albeit these significant efforts, several theoretical questions remain unanswered.
Among these, we tackle the following, for which we provide principled and theoretically-founded answers:
\begin{itemize}
    \item[\textbf{Q1:}] To what extent does a noise drawn from the Exponential family preserve robustness (in a sense to be defined) to adversarial attacks?
\end{itemize}

\paragraph{A1:}
We introduce a definition of robustness to adversarial attacks that is suitable to the randomization defense mechanism.
As this mechanism can be  described as a non-deterministic querying process, called probabilistic mapping in the sequel, we propose a formal definition of robustness relying on a metric/divergence between probability measures.
A key question arises then about the appropriate metric/divergence for our context.
This requires tools for comparing divergences \wrt the introduced robustness definition.
Renyi divergence turned out to be a measure of choice, since it satisfies most of the desired properties  (coherence, strength, and computational tractability).
Finally, thanks to the existing links between the Renyi divergence and the Exponential family, we were able to prove  that methods based on noise injection from the Exponential family  ensures robustness to adversarial examples (cf. \Cref{theorem:ap3-netrob}).
\begin{itemize}
    \item[\textbf{Q2:}] Can we guarantee a good accuracy under attack for classifiers defended with this kind of noise? 
\end{itemize}

\paragraph{A2:}
We present an upper bound on the drop of accuracy (under attack) of the methods defended with noise drawn from the Exponential family (\cf \Cref{theorem:ap3-bound}).
Then, we illustrate this result by training different randomized models with Laplace and Gaussian distributions on CIFAR-10 dataset.
These experiments highlight the trade-off between accuracy and robustness that depends on the amount of noise one injects in the network.
Our theoretical and experimental conclusion is that randomized defenses are competitive (with the current state-of-the-art~\cite{madry2018towards}) given the intensity of noise injected in the network.

\paragraph{Outline of the chapter:}
We present in \Cref{section:ap3-relatedwork} the related work on randomized defenses to adversarial examples.
\Cref{section:ap3-definition} introduces the definition of robustness relying on a metric/divergence between probability measures, and discusses the key role of the Renyi divergence.
We state in~\Cref{sec:main_result} our main results on the robustness and accuracy of Exponential family-based defenses.
\Cref{section:ap3-experiment} presents extensive experiments supporting our theoretical findings.
\Cref{section:ap3-conclusion_remarks} provides concluding remarks.

\section{Related works}
\label{section:ap3-relatedwork}

Injecting noise into algorithms to improve their robustness has been used for ages in detection and signal processing tasks~\cite{zozor1999stochastic,chapeau2004noise,mitaim1998adaptive}.
It has also been extensively studied in several machine learning and optimization fields, \eg robust optimization~\cite{ben2009robust} and data augmentation techniques~\cite{perez2017effectiveness}.
Recently, noise injection techniques have been adopted by the adversarial defense community, especially for neural networks, with very promising results.
Randomization techniques are generally oriented towards one of the following objectives: experimental robustness or provable robustness.

\paragraph{Experimental robustness:}
The first technique explicitly using randomization at inference time as a defense appeared during the 2017 NeurIPS defense challenge~\cite{xie2017mitigating}.
This method uniformly samples over geometric transformations of the image to select a substitute image to feed the network.
Then~\citet{dhillon2018stochastic} proposed to use stochastic activation pruning based on a multinomial distribution for adversarial defense.
Several works~\cite{xuanqing2018towards,rakin2018parametric} propose to inject Gaussian noise directly on the activation of selected layers both at training and inference time.
While these works hypothesize that noise injection makes the network robust to adversarial perturbations, they do not provide any formal justification on the nature of the noise they use or on the loss of accuracy/robustness of the  network.

\paragraph{Provable robustness:}
\citet{lecuyer2018certified} proposed a randomization method by exploiting the link between differential privacy~\cite{dwork2014algorithmic} and adversarial robustness.
Their framework, called ``randomized smoothing'' \footnote{Name introduced by~\citet{cohen2019certified} after the work of~\cite{lecuyer2018certified}.}, inherits some theoretical results from the differential privacy community allowing them to evaluate the level of accuracy under attack of their method.
Initial results by~\citet{lecuyer2018certified} have been refined by~\citet{li2018second}, and by~\citet{cohen2019certified}.
Our work belongs to this line of research.
However, our framework does not treat exactly the same class of defenses.
Notably, we provide theoretical arguments supporting the defense strategy based on randomization techniques relying on the exponential family, and derive a new bound on the adversarial generalization gap, which completes the results obtained so far on certified robustness.
Furthermore, our focus is on the network randomized by noise injection, ``randomized smoothing'' instead uses this network to create a \emph{new} classifier robust to attacks.

Since the initial discovery of adversarial examples, a wealth of non randomized defense approaches have also been proposed, inspired by various machine learning domains such as adversarial training~\cite{goodfellow2014explaining,madry2018towards}, image reconstruction~\cite{meng2017magnet,samangouei2018defense} or robust learning~\cite{goodfellow2014explaining,madry2018towards}.
Even if these methods have their own merits, a thorough evaluation made by~\citet{athalye2018obfuscated} shows that most defenses can be easily broken with known powerful attacks~\cite{madry2018towards,carlini2017towards,chen2018ead}.
Adversarial training, which consists in training a model directly on adversarial examples, came out as the best defense in average.
Defense based on randomization could be overcome by the Expectation Over Transformation technique proposed by~\citet{athalye2017synthesizing} which consists in taking the expectation over the network to craft the perturbation.
In this chapter, to ensure that our results are not biased by obfuscated gradients, we follow the principles provided by~\cite{athalye2018obfuscated,carlini2019evaluating} and evaluate our randomized networks with this technique.
We show that randomized defenses are still competitive given the intensity of noise injected in the network.

\section{General definitions of risk and robustness}
\label{section:ap3-definition}

\subsection{Risk, robustness and probabilistic mappings}

Let us consider two spaces $\mathcal{X}$ (with norm $\norm{\ \cdot\ }_{\mathcal{X}}$), and $\mathcal{Y}$.
We consider the classification task that seeks a hypothesis (classifier) $h: \mathcal{X} \rightarrow \mathcal{Y}$ minimizing the risk of $h$ \wrt some ground-truth distribution $\mathcal{D}$ over $\mathcal{X}\times\mathcal{Y}$.
The risk of $h$ \wrt $\mathcal{D}$ is defined as 
\begin{align*}
  \Risk(h) \triangleq \mathbb{E}_{(x,y)\sim \mathcal{D}}\left[ \mathds{1} \left( h(x) \neq y \right)\right].
\end{align*}
Given a classifier $h: \mathcal{X} \rightarrow \mathcal{Y}$, and some input $x \in \mathcal{X}$ with true label $y_{true} \in \mathcal{Y}$, to generate an adversarial example, the adversary seeks a $\tau$ such that $h(x+\tau) \neq y_{true}$, with some budget $\alpha$ over the perturbation (\ie, with $\norm{\tau}_{\mathcal{X}} \leq\alpha$).
$\alpha$ represents the maximum amount of perturbation one can add to $x$ without being spotted (the perturbation remains humanly imperceptible). 
The overall goal of the adversary is to find a perturbation crafting strategy that both maximizes the risk of $h$, and keeps the values of $\norm{\tau}_{\mathcal{X}}$ small.
To measure this risk "under attack" we define the notion of adversarial $\alpha$-radius risk of $h$ \wrt $\mathcal{D}$ as follows
\begin{align}
    \advRisk(h) \triangleq \mathbb{E}_{(x,y) \sim \mathcal{D}} \left[ \sup_{\norm{\tau}_{\mathcal{X}} \leq \alpha} \mathds{1}\left(h(x+\tau) \neq y\right) \right]\enspace.
\end{align}

In practice, the adversary does not have any access to the ground-truth distribution.
The literature proposed several surrogate versions of $\advRisk(h)$ (see~\citet{diochnos2018adversarial} for more details) to overcome this issue.
We focus our analysis on the one used by~\citet{szegedy2013intriguing,fawzi2018adversarial} denoted $\alpha$-radius prediction-change risk of $h$ \wrt $\mathcal{D}_{\mathcal{X}}$ (marginal of $\mathcal{D}$ for $\mathcal{X}$), and defined as   
\begin{align}
    \PCadvRisk(h) \triangleq \mathbb{P}_{x\sim \mathcal{D}_{\mathcal{X}}}\left[\exists \tau \in \B \text{ s.t. } h(x+\tau)\neq h(x) \right]
\end{align}
where for any $\alpha \geq 0$, \quad $\B \triangleq \{\tau \in \mathcal{X} \text{ s.t. } \norm{\tau}_{\mathcal{X}} \leq \alpha\}\enspace.$

As we will inject some noise in our classifier in order to defend against adversarial attacks, we need to introduce the notion of ``probabilistic mapping''. Let $\mathcal{Y}$ be the output space, and $\mathcal{F}_{\mathcal{Y}}$ a $\sigma$-$ algebra$ over $\mathcal{Y}$. Let us also denote $\mathcal{P}(\mathcal{Y})$ the set of probability measures over $(\mathcal{Y},\mathcal{F}_{\mathcal{Y}})$.

\begin{definition}[Probabilistic mapping] Let $\mathcal{X}$ be an arbitrary space, and $(\mathcal{Y},\mathcal{F}_{\mathcal{Y}})$ a measurable space. A \emph{probabilistic mapping} from $\mathcal{X}$ to $\mathcal{Y}$ is a mapping $\probmap: \mathcal{X} \to \mathcal{P}(\mathcal{Y})$.
To obtain a numerical output out of this \emph{probabilistic mapping}, one needs to sample $y$ according to $\probmap(x)$. 
\end{definition} 

This definition does not depend on the nature of $\mathcal{Y}$ as long as $(\mathcal{Y},\mathcal{F}_{\mathcal{Y}})$ is measurable.
In that sense, $\mathcal{Y}$ could be either the label space or any intermediate space corresponding to the output of an arbitrary hidden layer of a neural network.
Moreover, any mapping can be considered as a probabilistic mapping, whether it explicitly injects noise (see~\citet{lecuyer2018certified,rakin2018parametric,dhillon2018stochastic}) or not.
In fact, any deterministic mapping can be considered as a probabilistic mapping, since it can be characterized by a Dirac measure.
Accordingly, the definition of a probabilistic mapping is fully general and equally treats networks with or without noise injection.
There exists no definition of robustness against adversarial attacks that comply with the notion of probabilistic mappings.
We settle that by generalizing the notion of prediction-change risk initially introduced by~\citet{diochnos2018adversarial} for deterministic classifiers.
Let $\probmap$ be a probabilistic mapping from $\mathcal{X}$ to $\mathcal{Y}$, and $d_{\mathcal{P}(\mathcal{Y})}$ some metric/divergence on $\mathcal{P}(\mathcal{Y})$.
We define the $(\alpha,\epsilon)$-radius prediction-change risk of $\probmap$ \wrt $\mathcal{D}_{\mathcal{X}}$ and $d_{\mathcal{P}(\mathcal{Y})}$ as 
\begin{equation}
  \PCadvRisk(\probmap,\epsilon) \triangleq  \mathbb{P}_{x\sim \mathcal{D}_{\mathcal{X}}}\left[ \exists \tau \in B(\alpha) \text{ s.t. } d_{\mathcal{P}(\mathcal{Y})}(\probmap(x+\tau),\probmap(x)) > \epsilon \right] \enspace.
\end{equation}

\noindent
These three generalized notions allow us to analyze noise injection defense mechanisms (Theorems~\ref{theorem:ap3-netrob}, and~\ref{theorem:ap3-bound}).
We can also define adversarial robustness (and later adversarial gap) thanks to these notions. 
\begin{definition}[Adversarial robustness] \label{def::GeneralizedRobustness}
  Let $d_{\mathcal{P}(\mathcal{Y})}$ be a metric/divergence on $\mathcal{P}(\mathcal{Y})$.
  The probabilistic mapping $\probmap$ is said to be $d_{\mathcal{P}(\mathcal{Y})}$-$(\alpha, \epsilon, \gamma)$ robust if 
  \begin{equation}
    \PCadvRisk(\probmap,\epsilon) \leq \gamma \enspace.
  \end{equation}
  \removespace
\end{definition}

\noindent
It is difficult in general to show that a classifier is $d_{\mathcal{P}(\mathcal{Y})}$-$(\alpha, \epsilon, \gamma)$ robust.
However, we can  derive some bounds for particular divergences that will ensure robustness up to a certain level (Theorem~\ref{theorem:ap3-netrob}).
It is worth noting that our definition of robustness depends on the considered metric/divergence between probability measures.
Lemma~\ref{th::PropimpliesRobustness} gives some insights on the monotony of the robustness according to the parameters, and the probability metric/divergence at hand.

\begin{lemma} \label{th::PropimpliesRobustness}
  Let $\probmap$ be a probabilistic mapping, and let  $d_{1}$ and $d_{2}$ be two metrics on $\mathcal{P}(\mathcal{Y})$.
  If there exists a non decreasing function $ \phi: \mathbb{R} \to \mathbb{R}$ such that $\forall \mu_1,\mu_2 \in \mathcal{P}(\mathcal{Y})$, $d_{1}(\mu_1,\mu_2) \leq \phi(d_{2}(\mu_1,\mu_2)) $, then the following assertion holds: 
  \begin{equation}
    \probmap \text{ is } d_{2}\text{-}(\alpha, \epsilon, \gamma)\text{-robust} \implies \probmap \text{ is }d_{1}\text{-}(\alpha, \phi(\epsilon), \gamma)\text{-robust}
  \end{equation}
  \removespace
\end{lemma}
\noindent
As suggested in Definition~\ref{def::GeneralizedRobustness} and Lemma~\ref{th::PropimpliesRobustness}, any given choice of metric/divergence will instantiate a particular notion of adversarial robustness and it should be carefully selected. 

\subsection{On the choice of the metric/divergence for robustness}
\label{subsec:div}

The aforementioned formulation naturally raises the question of the choice of the metric used to defend against adversarial attacks. 
The main notions that govern the selection of an appropriate metric/divergence are  \emph{coherence}, \emph{strength}, and \emph{computational tractability}.
A metric/divergence is said to be coherent if it naturally fits the task at hand (\eg classification tasks are intrinsically linked to discrete/trivial metrics, conversely to regression tasks).
The strength of a metric/divergence refers to its ability to cover (dominate) a wide class of others in the sense of Lemma~\ref{th::PropimpliesRobustness}. 
In the following, we will focus on both the total variation metric and the Renyi divergence, that we consider as respectively the most coherent with the classification task using probabilistic mappings, and the strongest divergence.
We first discuss how total variation metric is \emph{coherent} with randomized classifiers but suffers from computational issues.
Hopefully, the Renyi divergence provides good guarantees about adversarial robustness, enjoys nice \emph{computational properties}, in particular when considering  Exponential family distributions, and is \emph{strong} enough to dominate a wide range of metrics/divergences including total variation.

Let  $\mu_1$ and $\mu_2$ be two measures in $\mathcal{P}(\mathcal{Y})$, both dominated by a third measure $\nu$.
The trivial distance $ d_{T}(\mu_1,\mu_ \triangleq \mathds{1}\left(\mu_1 \neq \mu_2\right)$ is the simplest distance one can define between $\mu_1$ and $\mu_2$.
In the deterministic case, it is straightforward to compute (since the numerical output of the algorithm characterizes its associated measure), but this is not the case in general.
In fact one might not have access to the true distribution of the mapping, but just to the numerical outputs.
Therefore, one needs to consider more sophisticated metrics/divergences, such as the total variation distance $d_{TV}(\mu_1,\mu_2) \triangleq \sup_{Y \in \mathcal{F}_{\mathcal{Y}}} |\mu_1 (Y) - \mu_2(Y)|$.
The total variation distance is one of the most broadly used probability metrics.
It admits several very simple interpretations, and is a very useful tool in many mathematical fields such as probability theory, Bayesian statistics, coupling or transportation theory.
In transportation theory, it can be rewritten as the solution of the Monge-Kantorovich problem with the cost function $c(y_1,y_2) =\mathds{1}\left(y_1 \neq y_2\right)$: $ \inf\int_{\mathcal{Y}^{2}}\mathds{1}\left(y_1 \neq y_2\right) \,\diff \pi(y_1,y_2)\, ,$ where the infimum is taken over all joint probability measures $\pi$ on $(\mathcal{Y}\times \mathcal{Y}, \mathcal{F}_{\mathcal{Y} } \otimes \mathcal{F}_{\mathcal{Y}})$ with marginals $\mu_1$ and $\mu_2$.
According to this interpretation, it seems quite natural to consider the total variation distance as a relaxation of the trivial distance on $[0,1]$ (see the book of~\citet{villani2008optimal} for details).
In the deterministic case, the total variation and the trivial distance coincides.
In general, the total variation allows a finer analysis of the probabilistic mappings than the trivial distance.
But it suffers from a high computational complexity.
In the following of the chapter we will show how to ensure robustness regarding TV distance.

Finally, denoting by $g_1$ and $g_2$ the respective probability distributions \wrt $\nu$, the Renyi divergence of order $\lambda$~\cite{renyi1961} writes as  
\begin{equation}
  d_{R,\lambda}(\mu_1,\mu_2) \triangleq \frac{1}{\lambda -1}\log \int_{\mathcal{Y}} g_2(y)  \left(\frac{g_1(y)}{g_2(y)}\right)^{\lambda} \,\diff \nu(y).
\end{equation}
The Renyi divergence is a generalized measure defined on the interval $(1,\infty)$, where it equals the Kullback-Leibler divergence when $\lambda \rightarrow 1$ (that will be denoted $d_{KL}$), and the maximum divergence when $\lambda \rightarrow \infty$.
It also has the very special property of being non decreasing \wrt $\lambda$.
This divergence is very common in machine learning, especially in its Kullback-Leibler form as it is widely used as the loss function (cross entropy) of classification algorithms.
It enjoys the desired properties  since it bounds the TV distance, and is tractable.
Furthermore, Proposition~\ref{proposition:ap3-RobustTV} proves that Renyi-robustness implies TV-robustness, making it a suitable surrogate for the trivial distance.

\begin{proposition}[Renyi-robustness implies TV-robustness] \label{proposition:ap3-RobustTV}
  Let $\probmap$ be a probabilistic mapping, then $\forall\lambda\geq1$:
  \begin{equation}
    \probmap \text{ is }  d_{R,\lambda}\text{-}(\alpha, \epsilon, \gamma)\text{-robust} \implies \probmap \text{ is } d_{TV}\text{-}(\alpha, \epsilon', \gamma)\text{-robust}
  \end{equation}
  \begin{equation}
    \textnormal{ with } \epsilon' = \min \left(\frac{3}{2}\left(\sqrt{1 + \frac{4\epsilon}{9}} - 1\right)^{1/2}, \frac{\exp(\epsilon +1) -1}{\exp(\epsilon +1) +1}\right) \enspace.
  \end{equation}
  \removespace
\end{proposition}

\noindent
A crucial property of Renyi-robustness is the \textit{Data processing inequality}.
It is a well-known inequality from information theory which states that \textit{``post-processing cannot increase information''}~\cite{cover2012elements,beaudry2011intuitive}.
In our case, if we consider a Renyi-robust probabilistic mapping, composing it with a deterministic mapping maintains Renyi-robustness with the same level.

\begin{proposition}[Data processing inequality] \label{proposition:ap3-postprocessing}
  Let us consider a probabilistic mapping $\probmap:\mathcal{X}\rightarrow\mathcal{P}(\mathcal{Y})$. Let us also denote $\rho:\mathcal{Y}\rightarrow\mathcal{Y}'$ a deterministic function.
  If $U \sim \probmap(x)$ then the probability measure $M'(x)$ s.t $\rho(U) \sim M'(x)$ defines a probabilistic mapping $M':\mathcal{X}\rightarrow\mathcal{P}(\mathcal{Y}')$.
  For any $\lambda>1$ if $\probmap$ is $d_{R,\lambda}$-$(\alpha,\epsilon,\gamma)$ robust then $M'$ is also is $d_{R,\lambda}$-$(\alpha,\epsilon,\gamma)$ robust.
\end{proposition}
\noindent
Data processing inequality will allow us later to inject some additive noise in any layer of a neural network and to ensure Renyi-robustness.

\section{Defense mechanisms based on  Exponential family noise injection}
\label{sec:main_result}

\subsection{Robustness through Exponential family noise injection}

For now, the question of which class of noise to add is treated \textit{ad hoc}.
We choose here to investigate one particular class of noise closely linked to the Renyi divergence, namely Exponential family distributions, and demonstrate their interest.
Let us first recall what the Exponential family is.

\begin{definition}[Exponential family]
  Let $\Theta$ be an open convex set of $\mathbb{R}^{n}$, and $\theta \in \Theta$.
  Let $\nu$ be a measure dominated by $\mu$ (either by the Lebesgue or counting measure), it is said to be part of the \emph{Exponential family} of parameter $\theta$ (denoted $E_{F}(\theta,t,k)$) if it has the following probability density function 
  \begin{equation}
    p_{F}(z,\theta)=\exp\left\{ \langle t(z),\theta \rangle -u(\theta) +k(z) \right\}
  \end{equation}
  where $t(z)$ is a sufficient statistic, $k$ a carrier measure (either for a Lebesgue or a counting measure) and $u(\theta) = \log \int_{z} \exp\left\{ <t(z),\theta> +k(z) \right\} \,\diff z $.
\end{definition}

\noindent
To show the robustness of randomized networks with noise injected from the Exponential family, one needs to define the notion of sensitivity for a given deterministic function:
\begin{definition}[Sensitivity of a function]
  For any $\alpha\geq0$ and for any $\norm{\ \cdot\ }_A$ and $\norm{\ \cdot\ }_B$ two norms, the $\alpha$-sensitivity of $f$ \wrt $\norm{\ \cdot\ }_A$ and $\norm{\ \cdot\ }_B$ is defined as
  \begin{equation}
    \Delta^{A,B}_\alpha(f) \triangleq \sup\limits_{ x,y \in \mathcal{X}, \norm{x-y}_{A} \leq \alpha} \norm{f(x) - f(y)}_B \enspace.
  \end{equation}
  \removespace
\end{definition}

\noindent
Let us consider an  $n$-layer feedforward neural network  $\mathcal{N}(\ \cdot\ ) = \phi^n \circ \cdots \circ \phi^1(\ \cdot\ )$.
For any $i\in\left[n\right]$, we define $\mathcal{N}_{|i}(\ \cdot\ ) = \phi^i\circ \cdots \circ \phi^1(\ \cdot\ )$ the neural network truncated at layer $i$.
Theorem~\ref{theorem:ap3-netrob} shows that, injecting noise drawn from an Exponential family distribution ensures robustness to adversarial example attacks in the sense of Definition~\ref{def::GeneralizedRobustness}.

\begin{theorem}[Exponential family ensures robustness] \label{theorem:ap3-netrob}
  Let us denote $\mathcal{N}_{X}^i(\ \cdot\ ) = \phi^n\circ \cdots \circ\phi^{i+1}(\mathcal{N}_{|i}(\ \cdot\ )+X)$ with $X$ a random variable.
  Let us also consider two arbitrary norms $\norm{\ \cdot\ }_{A}$ and $\norm{\ \cdot\ }_{B}$  respectively on $\mathcal{X}$ and on the output space of $\mathcal{N}_{X}^i$.
  \begin{itemize}
    \item If $X\sim E_{F}(\theta,t,k)$ where $t$ and $k$ have non-decreasing modulus of continuity $\omega_t$ and $\omega_k$.
    Then for any $\alpha \geq 0$, $\mathcal{N}_{X}^i(\ \cdot\ )$ defines a probabilistic mapping that is $d_{R,\lambda}$-$(\alpha,\epsilon)$ robust with $\epsilon = \norm{\theta}_2 \omega^{B,2}_t(\Delta^{A,B}_{\alpha}(\mathcal{N}_{|i})) +\omega_k^{B,1}(\Delta^{A,B}_{\alpha}(\mathcal{N}_{|i})) $ where $\norm{\ \cdot\ }_2$ is the norm corresponding to the scalar product in the definition of the exponential family density function and $\norm{\ \cdot\ }_1$ is the absolute value on $\mathbb{R}$.
    \footnote{The notion of continuity modulus is defined in the arxiv version of this chapter: https://arxiv.org/abs/1902.01148.}
    \item If $X$ is a centered Gaussian random variable with a non degenerated matrix parameter $\Sigma$.
      Then for any $\alpha \geq 0$, $\mathcal{N}_{X}^i(\ \cdot\ )$ defines a probabilistic mapping that is $d_{R,\lambda}$-$(\alpha,\epsilon)$ robust with $ \epsilon = \frac{\lambda \Delta^{A,2}_{\alpha}(\phi)^2 }{2 \sigma_{min}(\Sigma) } $ where $\norm{\ \cdot\ }_2$ is the canonical Euclidean norm on $\mathbb{R}^n$.
  \end{itemize}
  \removespace
\end{theorem}

In simpler words, the previous theorem ensures stability in the neural network when injecting noise \wrt the distribution of the output.
Intuitively, if two inputs are close \wrt $\norm{\ \cdot\ }_{A}$, the output distributions of the network will be close in the sense of Renyi divergence.
It is well known that in the case of deterministic neural networks, the Lipschitz constant becomes bigger as the number of layers increases~\cite{gouk2018regularisation}.
By injecting noise at layer $i$, the notion of robustness only depends on the sensitivity of the first $i$ layers of the network and not the following ones.
In that sense, randomization provides a more precise control on the ``continuity'' of the neural network.
In the next section, we show that thanks to the notion of robustness \wrt probabilistic mappings, one can bound the loss of accuracy of a randomized neural network when it is attacked. 

\subsection{Bound on the generalization gap under attack}

The notions of risk and adversarial risk can easily be generalized to encompass probabilistic mappings.
\begin{definition}[Risks for probabilistic mappings]
  Let $\probmap$ be a probabilistic mapping from $\mathcal{X}$ to $\mathcal{Y}$, the risk and the $\alpha$-radius adversarial risk of $\probmap$ \wrt $\mathcal{D}$ are defined as 
  \begin{align}
    \Risk(\probmap) &\triangleq \mathbb{E}_{(x,y) \sim \mathcal{D}} \left[ \mathbb{E}_{y'\sim \probmap(x)} \left[ \mathds{1} \left( y' \neq y \right)\right]\right] \\
    \advRisk(\probmap) &\triangleq \mathbb{E}_{(x,y)\sim \mathcal{D}}\left[ \sup_{\norm{\tau}_{\mathcal{X}} \leq \alpha}\mathbb{E}_{y'\sim \probmap(x+\tau)} \left[ \mathds{1} \left( y' \neq y \right)\right]\right]\enspace.
  \end{align}
  \removespace
\end{definition}

\noindent
The definition of adversarial risk for a probabilistic mapping can be matched with the concept of Expectation over Transformation (EoT) attacks~\cite{athalye2018obfuscated}.
Indeed, EoT attacks aim at computing the best opponent in expectation for a given random transformation.
In the adversarial risk definition, the adversary chooses the perturbation which has the greatest probability to fool the model, which is a stronger objective than the EoT objective.
Theorem~\ref{theorem:ap3-bound} provides a bound on the gap between the adversarial risk and the regular risk:
\begin{theorem}[Adversarial generalization gap bound in the randomized setting]
  Let $\probmap$ be the probabilistic mapping at hand.
  Let us suppose that  $\probmap$ is $d_{R,\lambda}$-$(\alpha,\epsilon)$ robust for some $\lambda\geq1$ then:
  \begin{equation}
    |\advRisk(\probmap)-\Risk(\probmap)|\leq 1-e^{-\epsilon}\mathbb{E}_x\left[e^{-H(\probmap(x))}\right]
  \end{equation}
  where $H$ is the Shannon entropy $H(p)=-\sum_i p_i \log(p_i)\enspace.$
\label{theorem:ap3-bound}
\end{theorem}

\noindent
This theorem gives a control on the loss of accuracy under attack \wrt the robustness parameter $\epsilon$ and the entropy of the predictor.
It provides a trade-off between the quantity of noise added in the network and the accuracy under attack.
Intuitively, when the noise increases, for any input, the output distribution tends towards the uniform distribution, then, $\epsilon\rightarrow0$ and $H(\probmap(x))\rightarrow \log(K)$, and the risk and the adversarial risk both tends to $\frac{1}{K}$ where $K$ is the number of classes in the classification problem.
On the opposite, if no noise is injected, for any input, the output distribution is a  Dirac distribution, then, if the prediction for the adversarial example is not the same as for the regular one, $\epsilon\rightarrow\infty$ and $H(\probmap(x))\rightarrow 0$.
Hence, the noise needs to be designed both to preserve accuracy and robustness to adversarial attacks.
In the Section~\ref{section:ap3-experiment}, we give an illustration of this bound when $\probmap$ is a neural network with noise injection at input level as presented in Theorem~\ref{theorem:ap3-netrob}.

\section{Experiments}
\label{section:ap3-experiment}

To illustrate our theoretical findings, we train randomized neural networks with a simple method which consists in injecting a noise drawn from an Exponential family distribution in the image during training and inference.
This section aims to answer \textbf{Q2} stated in the introduction, by tackling the following sub-questions:
\begin{itemize}[leftmargin=12mm]
  \item[\textbf{Q2.1:}] How does the randomization impact the accuracy of the network? And, how does the theoretical trade-off between accuracy and robustness apply in practice? 
  \item[\textbf{Q2.2:}] What is the accuracy under attack of randomized neural networks against powerful iterative attacks? And how does randomized neural networks compare to state-of-the-art defenses given the intensity of the injected noise? 
\end{itemize}

\paragraph{Experimental setup}

We present our results and analysis on CIFAR-10 \cite{krizhevsky2009learning}.
We used a Wide ResNet architecture \cite{zagoruyko2016wide} which is a variant of the ResNet model proposed by~\citet{he2016deep}.
We use 28 layers with a widen factor of 10.
We train all networks for 200 epochs, a batch size of 400, dropout 0.3 and Leaky Relu activation with a slope on $\mathbb{R}^-$ of 0.1.
We minimize the Cross Entropy Loss with Momentum 0.9 and use a piecewise constant learning rate of 0.1, 0.02, 0.004 and 0.00008 after respectively 7500, 15000 and 20000 steps.
The networks achieve for CIFAR10 and 100 a TOP-1 accuracy of 95.8\% and 79.1\% respectively on test images.

To transform these classical networks to probabilistic mappings, we inject noise drawn from Laplace and Gaussian distributions, each with various standard deviations.
While the noise could theoretically be injected anywhere in the network, we inject the noise on the image for simplicity.
More experiments with noise injected in the first layer of the network are presented in the supplementary material.
To evaluate our models under attack, we use three powerful iterative attacks with different norms: \emph{ElasticNet} attack (EAD)~\cite{chen2018ead} with $\ell_1$ distortion, \emph{Carlini\&Wagner} attack (C\&W)~\cite{carlini2017towards} with $\ell_2$ distortion and \emph{Projected Gradient Descent} attack (PGD)~\cite{madry2018towards} with $\ell_\infty$ distortion.
All standard deviations and attack intensities are in between $-1$ and $1$.
Precise descriptions of our numerical experiments and of the attacks used for evaluation are deferred to the supplementary material.

\paragraph{Attacks against randomized defenses:}
It has been pointed out by~\citet{athalye2017synthesizing,carlini2019evaluating} that in a white box setting, an attacker with a complete knowledge of the system will know the distribution of the noise injected in the network.
As such, to create a stronger adversarial example, the attacker can take the expectation of the loss or the logits of the randomized network during the computation of the attack.
This technique is called Expectation Over Transformation ($\EoT$) and we use a Monte Carlo method with $80$ simulations to approximate the best perturbation for a randomized network. 

\afterpage{
\begin{figure}[H]
  \centering
  \includegraphics[scale=0.55]{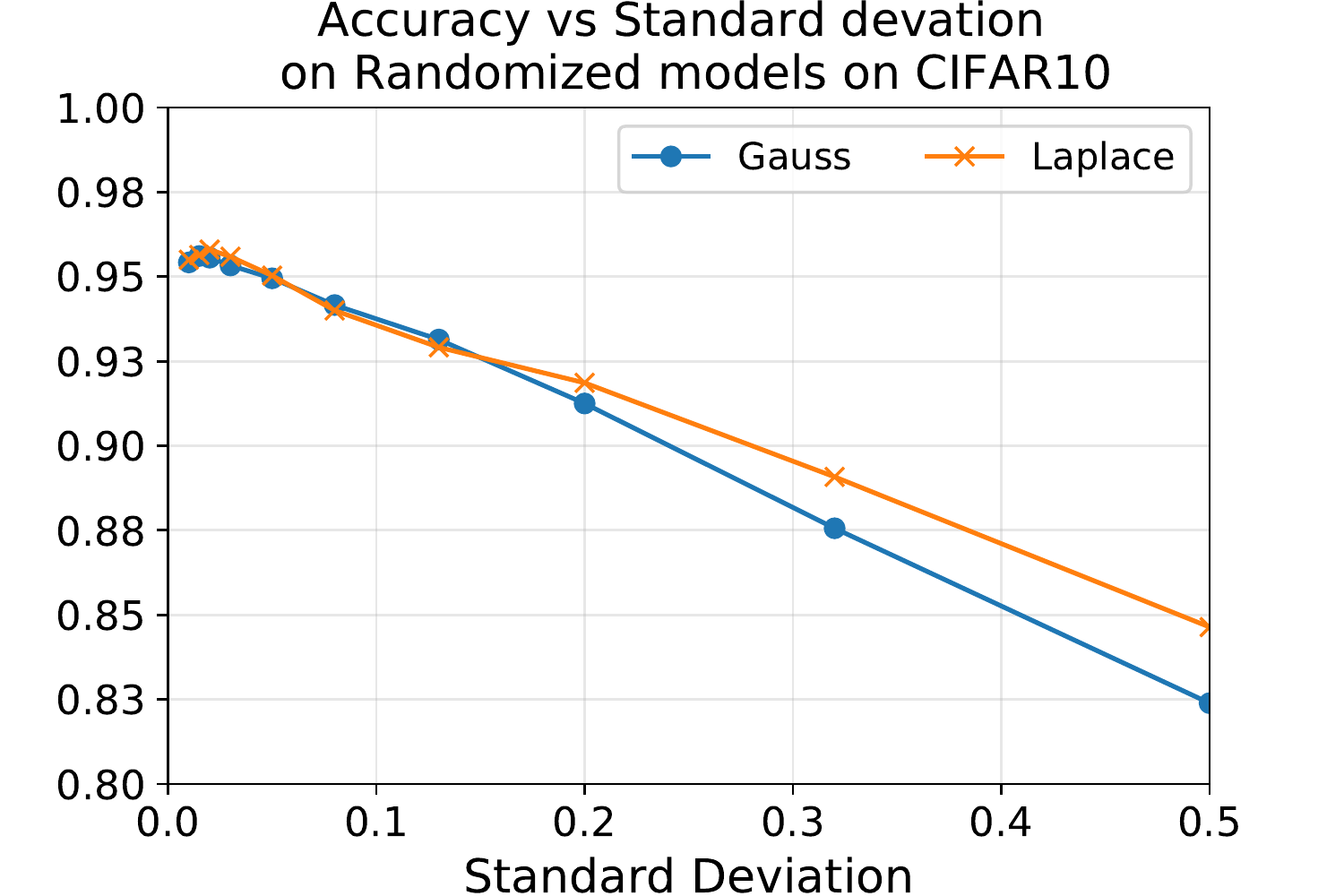}
  \caption{Impact of the standard deviation of the injected noise on accuracy in a randomized model on CIFAR-10 dataset with a Wide ResNet architecture.} 
  \label{figure:ap3-acc_sd_CIFAR10}
  \vspace{0.8cm}
  \includegraphics[scale=0.55]{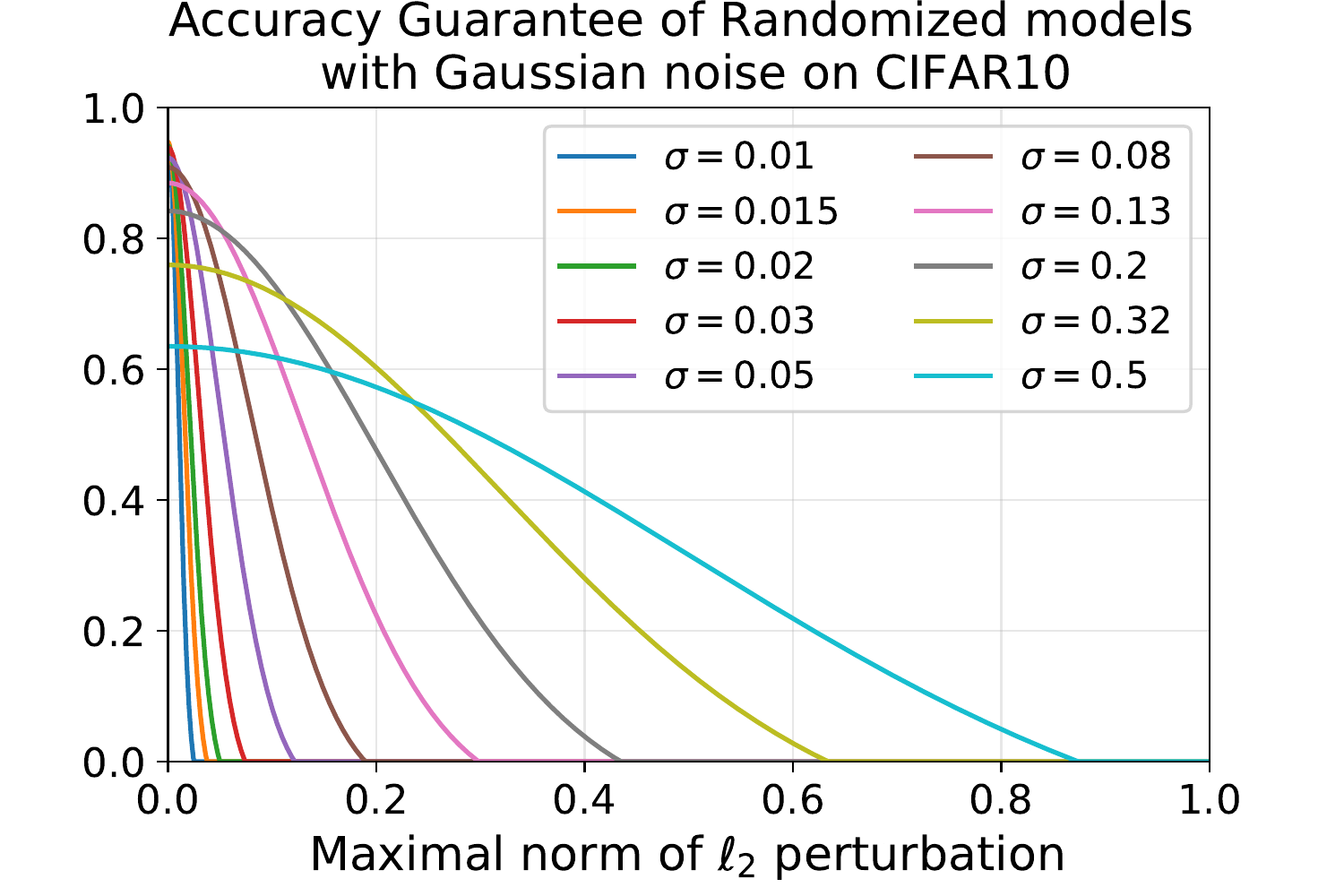}
\caption{Illustration of the guaranteed accuracy of different randomized models with Gaussian noises given the norm of the adversarial perturbation.}
  \label{figure:ap3-gauss_certif_CIFAR10}
  \vspace{0.8cm}
  \includegraphics[scale=0.55]{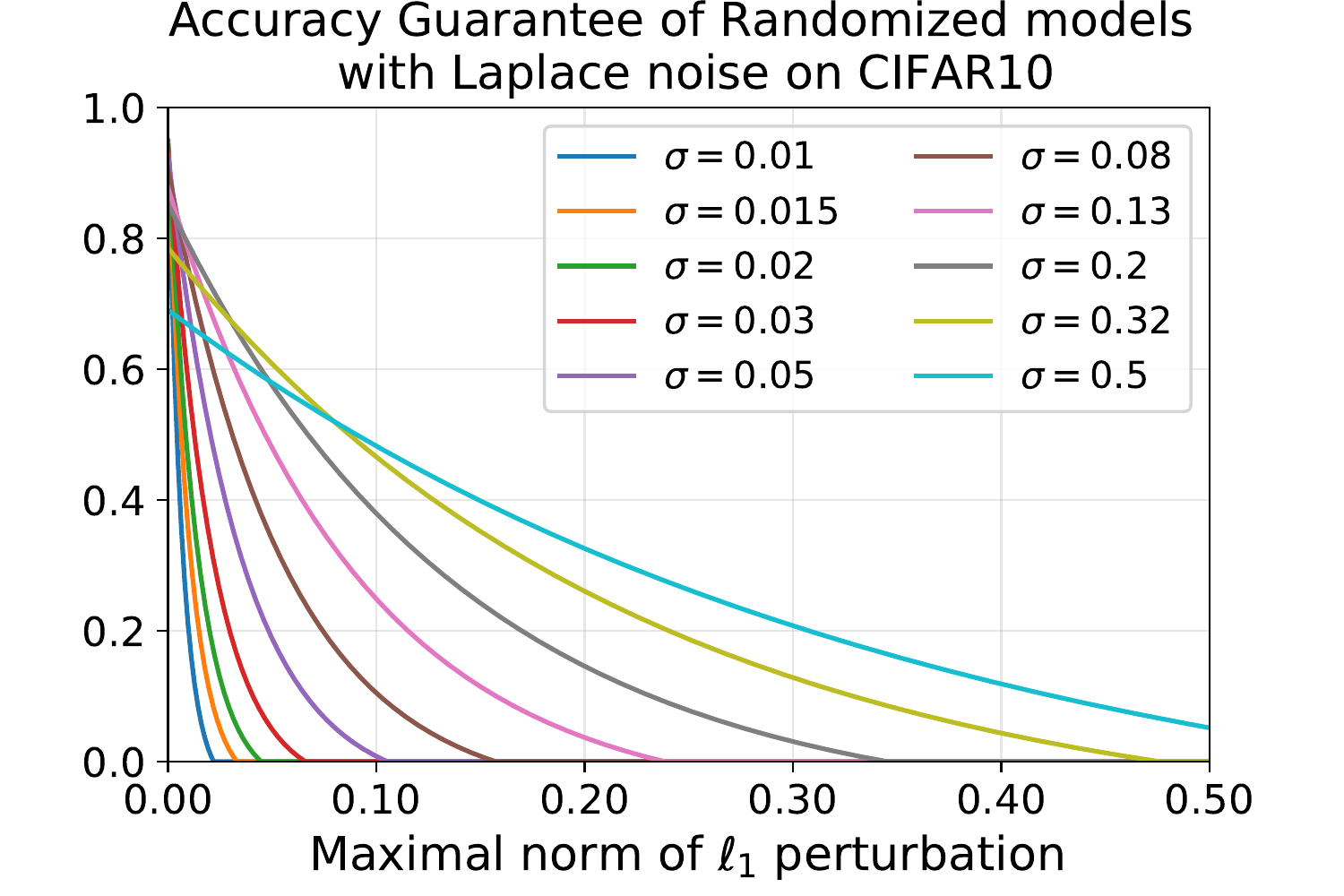}
  \caption{Illustration of the guaranteed accuracy of different randomized models with Laplace noises given the norm of the adversarial perturbation.}
  \label{figure:ap3-laplace_certif_CIFAR10}
  \label{figure:ap3-cifar10_results}
\end{figure}
}

\paragraph{Trade-off between accuracy and intensity of noise (Q2.1):}

When injecting noise as a defense mechanism, regardless of the distribution it is drawn from, we observe (as in Figure~\ref{figure:ap3-acc_sd_CIFAR10}) that the accuracy decreases when the noise intensity grows.
In that sense, noise needs to be calibrated to preserve both accuracy and robustness against adversarial attacks, \ie, it needs to be large enough to preserve robustness and small enough to preserve accuracy.
Figure~\ref{figure:ap3-acc_sd_CIFAR10} shows the loss of accuracy on CIFAR10 from $0.95$ to $0.82$ (respectively $0.95$ to $0.84$) with noise drawn from a Gaussian distribution (respectively Laplace) with a standard deviation from $0.01$ to $0.5$.
Figure~\ref{figure:ap3-gauss_certif_CIFAR10} and \ref{figure:ap3-laplace_certif_CIFAR10} illustrate the theoretical lower bound on accuracy under attack of Theorem~\ref{theorem:ap3-bound} for different distributions and standard deviations.
The term in entropy of Theorem~\ref{theorem:ap3-bound} has been estimated using a Monte Carlo method with $10^4$ simulations.
The trade-off between accuracy and robustness from Theorem~\ref{theorem:ap3-bound} thus appears \wrt the noise intensity.
With small noises, the accuracy is high, but the guaranteed accuracy drops fast \wrt the magnitude of the adversarial perturbation.
Conversely, with bigger noises, the accuracy is lower but decreases slowly \wrt the magnitude of the adversarial perturbation.
These Figures also show that Theorem~\ref{theorem:ap3-bound} gives strong accuracy guarantees against small adversarial perturbations.
Next paragraph shows that in practice, randomized networks achieve much higher accuracy under attack than the theoretical bound, and keep this accuracy against much larger perturbations.

\paragraph{Performance of randomized networks under attacks and comparison to state of the art (Q2.2):}

While Figure~\ref{figure:ap3-gauss_certif_CIFAR10} and \ref{figure:ap3-laplace_certif_CIFAR10} illustrated a theoretical robustness against growing adversarial perturbations, Table~\ref{table:ap3-accuracy_under_attack} illustrates this trade-off experimentally.
It compares the accuracy obtained under attack by a deterministic network with the one obtained by randomized networks with Gaussian and Laplace noises both with low ($0.01$) and high ($0.5$) standard deviations.
Randomized networks with a small noise lead to no loss in accuracy with a small robustness while high noise leads to a higher robustness at the expense of loss of accuracy ($\sim11$ points). 

\begin{table}[t]
  \centering
  \begin{tabular}{lccccc}
    \toprule
    \textbf{Distribution} & \textbf{Sd} & \textbf{Natural} & \textbf{$\ell_1$ EAD 60} & \textbf{$\ell_2$ C\&W 60} & \textbf{$\ell_\infty$ PGD 20} \\
    \midrule
    \multicolumn{1}{c}{--} & -- & 0.958 & 0.035 & 0.034 & 0.384 \\
    \midrule
    \multirow{2}[0]{*}{Normal} & 0.01 & 0.954 & 0.193 & 0.294 & 0.408 \\
	  & 0.50 & 0.824 & 0.448 & 0.523 & 0.587 \\
    \midrule
    \multirow{2}[0]{*}{Laplace} & 0.01 & 0.955 & 0.208 & 0.313 & 0.389 \\
	  & 0.50 & 0.846 & 0.464 & 0.494 & 0.589 \\
    \bottomrule
  \end{tabular}%
  \caption{Accuracy under attack on the CIFAR-10 dataset with a randomized Wide ResNet architecture. We compare the accuracy on natural images and under attack with different noise over 3 iterative attacks (the number of steps is next to the name) made with 80 Monte Carlo simulations to compute EoT attacks. The first line is the baseline, no noise has been injected.}
  \label{table:ap3-accuracy_under_attack}%
\end{table}%


\begin{table}[ht]
  \centering
  \begin{tabular}{cccccccccc}
    \toprule
    \multirow{2}{*}{\textbf{Attack}} & \multirow{2}{*}{\textbf{Steps}} & & \multirow{2}[0]{*}{\textbf{Madry et al.}} & &    \multicolumn{2}{c}{\textbf{Normal}} &  & \multicolumn{2}{c}{\textbf{Laplace}} \\
    \cmidrule{6-7} \cmidrule{9-10}
    & & & & & \textbf{0.32} & \textbf{0.5} &  & \textbf{0.32} & \textbf{0.5} \\
    \midrule
    --                    & -- & & 0.873 & & 0.876 & 0.824 & & 0.891 & 0.846 \\ 
    $\ell_\infty$ -- PGD  & 20 & & 0.456 & & 0.566 & 0.587 & & 0.576 & 0.589 \\
    $\ell_2$ -- C\&W      & 30 & & 0.468 & & 0.512 & 0.489 & & 0.502 & 0.479 \\
    \bottomrule
  \end{tabular}
  \caption{Accuracy under attack of randomized neural network with different distributions and standard deviations versus adversarial training by~\citet{madry2018towards}. The PGD attack has been made with 20 step, an epsilon of 0.06 and a step size of 0.006 (input space between $-1$ and $+1$). The Carlini\&Wagner attack uses 30 steps, 9 binary search steps and a 0.01 learning rate. The first line refers to the baseline without attack.}
  \label{table:madry_vs_random}
\end{table}

Finally, Table~\ref{table:madry_vs_random} compares the accuracy and the accuracy under attack of randomized networks with Gaussian and Laplace distributions for different standard deviations against adversarial training from~\citet{madry2018towards}.
We observe that the accuracy on natural images of both noise injection methods are similar to the one from~\citet{madry2018towards}.
Moreover, both methods are more robust than adversarial training to PGD and C\&W attacks.
As with all the experiments, to construct an EoT attack,  we use 80 Monte Carlo simulations at every step of PGD and C\&W attacks.
These experiments show that randomized defenses can be competitive given the intensity of noise injected in the network.
Note that these experiments have been led with $\EoT$ of size 80.
For much bigger sizes of $\EoT$ these results would be mitigated.
Nevertheless, the accuracy would never drop under the bounds illustrated in the Figures~\ref{figure:ap3-gauss_certif_CIFAR10} and \ref{figure:ap3-laplace_certif_CIFAR10}, since Theorem~\ref{theorem:ap3-bound} gives a bound that on the worst case attack strategy (including $\EoT$).

\section{Concluding Remarks}
\label{section:ap3-conclusion_remarks}

This chapter brings new contributions to the field of provable defenses to adversarial attacks.
Principled answers have been provided to key questions on the interest of randomization techniques, and on their loss of accuracy under attack.
The obtained bounds have been illustrated in practice by conducting thorough experiments on baseline datasets such as CIFAR.
We show in particular that a simple method based on injecting noise drawn from the Exponential family is competitive compared to baseline approaches while leading to provable guarantees.
Future work will focus on investigating other noise distributions belonging or not to the Exponential family, combining randomization with more sophisticated defenses and on devising new tight bounds on the adversarial generalization gap.

%% file: sources/appendix/ap4-advocating_for_multiple_defense_strategies.tex
\chapter{Advocating for Multiple Defense Strategies against Adversarial Examples}
\label{appendix:ap4-advocating_multiple_defense_strategies_against_adversarial_examples}
\localtoc

\vspace{\fill}

\emph{
This Appendix concerns a collaboration with Rafael Pinot, Laurent Meunier and Benjamin Negrevergne.
This work has been published in the European Conference on Machine Learning Workshop for CyberSecurity.
It conducts a geometrical analysis of defense mechanisms designed to protect neural networks against.
This work shows that neural networks designed to be robust against one type of adversarial example offers poor robustness against other types of attacks.
}

\vspace{\fill}

\section{Introduction}
\label{section:ap4-introduction}


We have seen that deep neural networks are vulnerable to \emph{adversarial examples}.
Because it is difficult to characterize the space of visually imperceptible variations of a natural image, existing adversarial attacks use surrogates that can differ from one attack to another.
For example, \citet{goodfellow2014explaining} use the $\linf$ norm to measure the distance between the original image and the adversarial image whereas \citet{carlini2017towards} use the $\ltwo$ norm.
When the input dimension is low, the choice of the norm is of little importance because the $\linf$ and $\ltwo$ balls overlap by a large margin, and the adversarial examples lie in the same space.
An important insight in this chapter is to observe that the overlap between the two balls diminishes exponentially quickly as the dimensionality of the input space increases.
For typical image datasets with large dimensionality, the two balls are mostly disjoint.
As a consequence, the $\linf$ and the $\ltwo$ adversarial examples lie in different areas of the space, and it explains why $\linf$ defense mechanisms perform poorly against $\ltwo$ attacks and vice versa. 

Building on this insight, we advocate for designing models that incorporate defense mechanisms against both $\linf$ and $\ltwo$ attacks and review several ways of mixing existing defense mechanisms.
In particular, we evaluate the performance of \emph{Mixed Adversarial Training} (MAT)~\cite{goodfellow2014explaining} which consists of  augmenting training batches using \emph{both} $\linf$ and $\ltwo$ adversarial examples, and {\em Randomized Adversarial Training} (RAT)~\cite{salman2019provably}, a solution to benefit from the advantages of both $\linf$ adversarial training, and $\ltwo$ randomized defense.

\section{No Free Lunch for Adversarial Defenses}
\label{section:ap4-no_free_lunch}

In this Section, we show both theoretically and empirically that defenses mechanisms intending to defend against $\linf$ attacks cannot provide suitable defense against $\ltwo$ attacks.
Our reasoning is perfectly general; hence we can similarly demonstrate the reciprocal statement, but we focus on this side for simplicity.

\begin{figure}
   \centering
   \begin{subfigure}[t]{0.32\textwidth}
       \centering
       \includegraphics[scale=0.22]{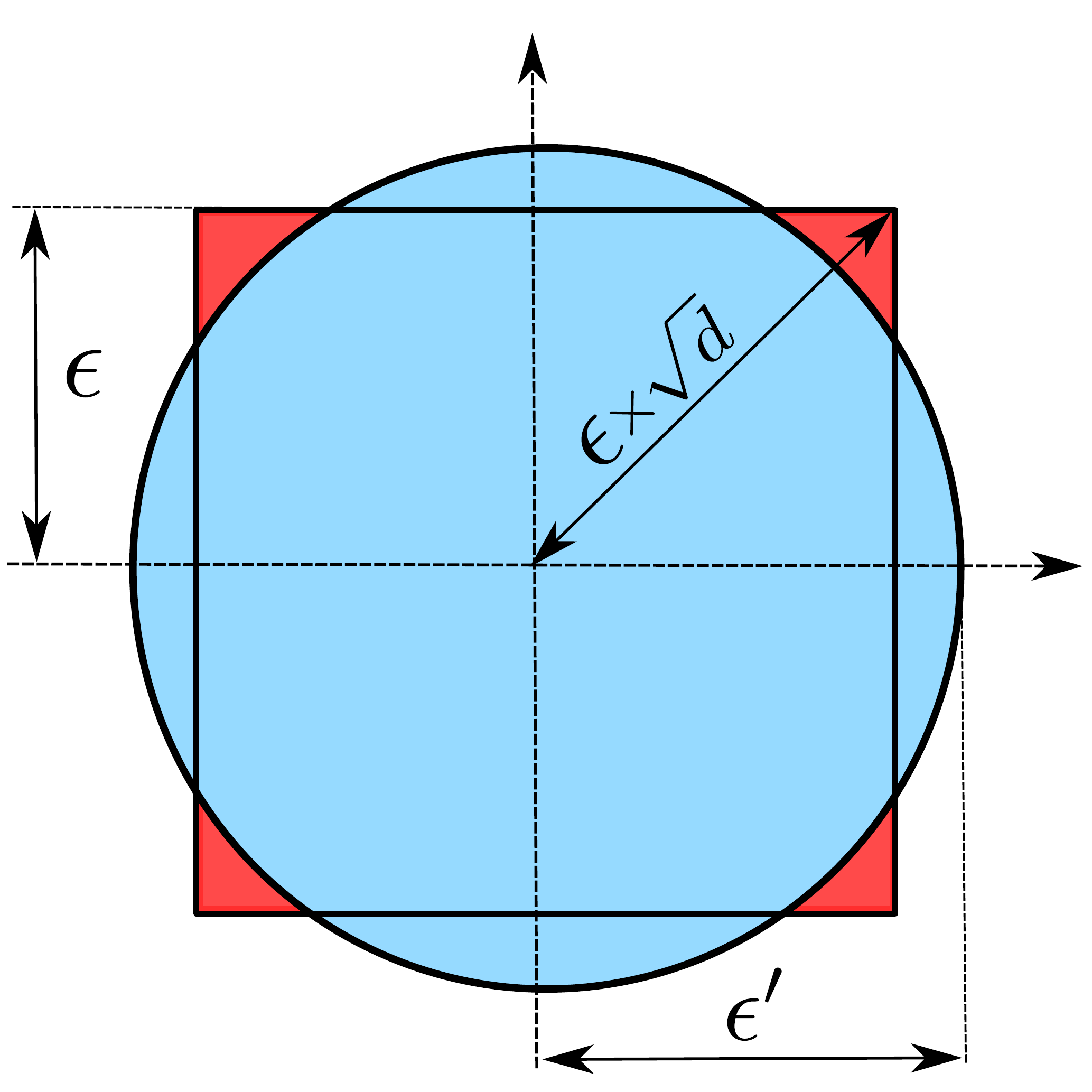}
       \caption{2D representation of the $\linf$ and $\ltwo$ balls of respective radius $\epsilon$ and $\epsilon'$}
       \label{figure:ap4-ball_inclusion_adversarial_training}
   \end{subfigure}
   \hfill
   \begin{subfigure}[t]{0.32\textwidth}
       \centering
       \includegraphics[scale=0.22]{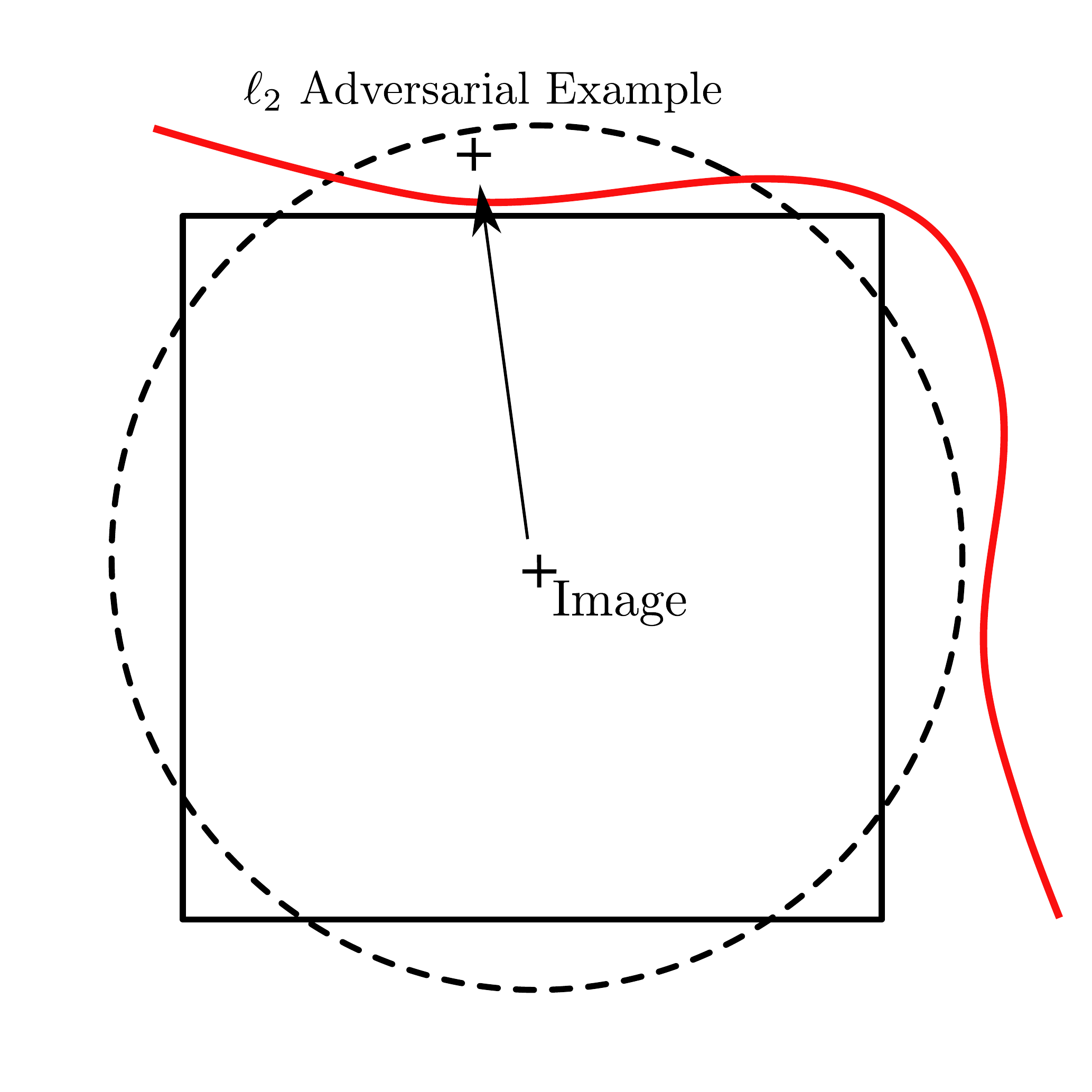}
       \caption{A classifier trained with $\linf$ adversarial perturbations  (materialized by the red line) remains vulnerable to $\ltwo$ attacks.}
       \label{figure:ap4-ball_adversarial_l2}
   \end{subfigure}
   \hfill
   \begin{subfigure}[t]{0.32\textwidth}
       \centering
       \includegraphics[scale=0.22]{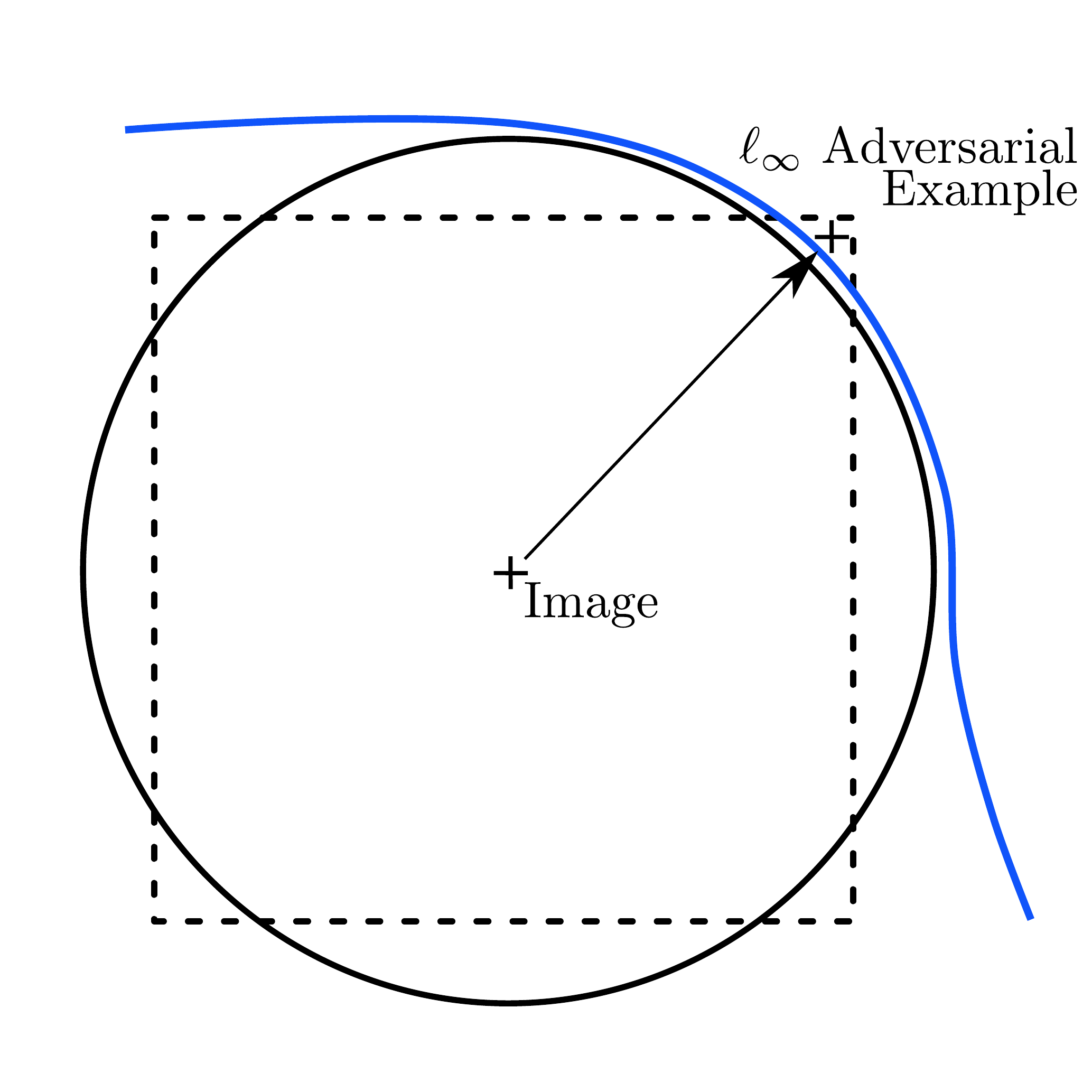}
       \caption{A classifier trained with $\ltwo$ adversarial perturbations (materialized by the blue line) remains vulnerable to $\linf$ attacks.}
       \label{figure:ap4-ball_adversarial_linf}
   \end{subfigure}
   \caption{2-dimensional representation of $\linf$ and $\ltwo$ balls}
\end{figure}

\subsection{Theoretical analysis}
\label{subsection:ap4-theoretical_analysis}

Let us consider a classifier $f_{\infty}$ that is provably robust against adversarial examples with maximum $\linf$ norm of value $\epsilon_\infty$.
It guarantees that for any input-output pair $(x,y) \sim \Dset$ and for any perturbation $\tau$ such that $\norm{\tau}_\infty \leq \epsilon_\infty$, $f_{\infty}$ is not misled by the perturbation, \ie, $f_{\infty}(x + \tau) = f_{\infty}(x)$.
We now focus our study on the performance of this classifier against adversarial examples bounded with a $\ltwo$ norm of value $\epsilon_2$.
Using Figure~\ref{figure:ap4-ball_inclusion_adversarial_training}, we observe that any $\ltwo$ adversarial example that is also in the $\linf$ ball, will not fool $f_{\infty}$.
Conversely, if it is outside the ball, we have no guarantee.

To characterize the probability that such an  $\ltwo$ perturbation fools an $\linf$ defense mechanism in the general case (\emph{i.e.}, any dimension $d$), we measure the ratio between the volume of the intersection of the $\linf$ ball of radius $\epsilon_\infty$ and the $\ltwo$ ball of radius $\epsilon_2$. As Theorem~\ref{theorem:ap4-nullvolume} shows, this ratio depends on the dimensionality $d$ of the input vector $x$, and  rapidly converges to zero when $d$ increases. 
Therefore a defense mechanism that protects against all $\linf$ bounded adversarial examples is unlikely to be efficient against $\ltwo$ attacks.

\begin{theorem}[Probability of the intersection goes to $0$] ~\\
Let
\begin{equation}
  B_{2,d}(\epsilon) \triangleq \left \{\tau \in \Rbb^d \ |\  \norm{\tau}_2 \leq \epsilon \right \}
\end{equation}
and
\begin{equation}
  B_{\infty,d}(\epsilon') \triangleq \left \{\tau \in \Rbb^d \ |\  \norm{\tau}_\infty \leq \epsilon' \right\}.
\end{equation}
If for all $d$, we select $\epsilon$ and $\epsilon$' such that $\Vol\left(B_{2,d}(\epsilon)\right) = \Vol\left(B_{\infty,d}(\epsilon')\right)$, then
\begin{equation}
  \frac{\Vol\left(B_{2,d}(\epsilon)\bigcap B_{\infty,d}(\epsilon')\right)}{\Vol\left(B_{\infty,d}(\epsilon')\right)} \rightarrow 0 \text{ when } d\rightarrow \infty.
\end{equation}
\label{theorem:ap4-nullvolume}
\end{theorem} 

\begin{proof}[\Cref{theorem:ap4-nullvolume}] 
Without loss of generality, let us fix $\epsilon = 1$. One can show that for all $d$, 
\begin{equation}
    \Vol\left( B_{2,d}\left(\frac{2}{\sqrt{\pi}}\Gamma\left(\frac{d}{2}+1\right)^{1/d}\right)\right) = \Vol\left(B_{\infty,d}\left(1\right)\right)
\end{equation}
where $\Gamma$ is the gamma function. Let us denote 
\begin{equation}
    r_2(d)=\frac{2}{\sqrt{\pi}}\Gamma\left(\frac{d}{2}+1\right)^{1/d}.
\end{equation}

\noindent
If we denote $\Uset$, the uniform distribution on $B_{\infty,d}(1)$, we get:
\begin{align}
  & \frac{\Vol(B_{2,d}(r_2(d)) \bigcap B_{\infty,d}(1))}{\Vol(B_{\infty,d}(1))} \\ 
  &= \Pbb_{\xvec \sim \Uset} \left[ x \in B_{2,d}(r_2(d)) \right] = \Pbb_{\xvec \sim \Uset} \left[ \sum_{i=1}^d |x_i|^2 \leq r_2(d)^2 \right] \\
  &= \Pbb_{\xvec \sim \Uset} \left[ \sum_{i=1}^d |\xvec_i|^2  - \Ebb_{\xvec \sim \Uset}\left[\sum_{i=1}^d |\xvec_i|^2 \right] \leq r_2(d)^2 - \Ebb_{\xvec \sim \Uset}\left[\sum_{i=1}^d |\xvec_i|^2 \right] \right] 
\end{align}
Note that when $d$ is sufficiently large we get
\begin{equation}
  r_2(d)^2 - \Ebb_{\xvec \sim \Uset}\left[\sum_{i=1}^d |\xvec_i|^2 \right] = r_2(d)^2 - \frac{d}{3} < 0
\end{equation}
Then, with Hoeffding inequality, we finally obtain:
\begin{equation}
  \frac{\Vol(B_{2,d}(r_2(d)) \bigcap B_{\infty,d}(1))}{\Vol(B_{\infty,d}(1))} \leq \exp\left(- \frac{\left(r_2(d)^2 - \frac{d}{3}\right)^2}{d} \right)
\end{equation}
Then, thanks to Stirling's formula
\begin{equation}
    r_2(d) \underset{d \rightarrow \infty}{\sim} \sqrt{\frac{2}{\pi e}} d^{1/2}.
\end{equation}
Then the ratio between the volume of the intersection of the ball and the volume of the ball converges towards $0$ when $d$ goes to $\infty$.
\end{proof}

Theorem~\ref{theorem:ap4-nullvolume} states that, when $d$ is large enough, $\ltwo$ bounded perturbations have a null probability of being also in the $\linf$ ball of the same volume.
As a consequence, for any value of $d$ that is large enough, a defense mechanism that offers full protection against $\linf$ adversarial examples is not guaranteed to offer any protection against $\ltwo$ attacks \footnote{Theorem~\ref{theorem:ap4-nullvolume} can easily be extended to any two balls with different norms. For clarity, we restrict to the case of $\linf$ and $\ltwo$ norms.}.

\begin{table}[ht]
  \centering
  \begin{tabular}{c r r r l}
    \toprule
    \textbf{Dataset\ } & \phantom{....} & \textbf{Dim.} $\mathbf{(d)}$ & \phantom{....} & \textbf{Vol. of the intersection }\\
    \midrule
    -- & & 2\ \ & & \ \ $e^{-0.183}$ \quad ($\approx$ 0.83) \\
    MNIST & & 784\ \  & & \ \ $e^{-7.344}$\\
    CIFAR & & 3072\ \ & &  \ \ $e^{-29.76}$\\
    ImageNet & & 150528\ \ & & \ \ $e^{-1478.71}$\\
    \bottomrule
  \end{tabular}
  \caption{ Bounds of Theorem~\ref{theorem:ap4-nullvolume} on the volume of the intersection of  $\ltwo$ and $\linf$ balls at equal volume for typical image classification datasets. When $d=2$, the bound is $10^{-0.183}\approx 0.83$.}
  \label{table:ap4-datadim}
\end{table}

Note that this result defeats the 2-dimensional intuition: if we consider a 2 dimensional problem setting, the $\linf$ and the $\ltwo$ balls have an important overlap (as illustrated in Figure~\ref{figure:ap4-ball_inclusion_adversarial_training}) and the probability of sampling at the intersection of the two balls is bounded by approximately 83\%.
However, as we increase the dimensionality $d$, this probability quickly becomes negligible, even for very simple image datasets such as MNIST.
An instantiation of the bound for classical image datasets is presented in Table~\ref{table:ap4-datadim}.
The probability of sampling at the intersection of the $\linf$ and $\ltwo$ balls is close to zero for any realistic image setting.
In large dimensions, the volume of the corner of the $\linf$ ball is much bigger than it appears in Figure~\ref{figure:ap4-ball_inclusion_adversarial_training}.

\subsection{No Free Lunch in Practice}
\label{subsection:ap4-no_free_lunch_in_practice}

Our theoretical analysis shows that if adversarial examples were uniformly distributed in a high-dimensional space, then any mechanism that perfectly defends against $\linf$ adversarial examples has a null probability of protecting against $\ltwo$-bounded adversarial attacks.
Although existing defense mechanisms do not necessarily assume such a distribution of adversarial examples, we demonstrate that whatever distribution they use, it offers no favorable bias with respect to the result of Theorem~\ref{theorem:ap4-nullvolume}.
As we discussed in Chapter~\ref{chapter:ch5-lipschitz_bound}, there are two distinct attack settings: loss maximization (PGD) and perturbation minimization (C\&W).
Our analysis is mainly focusing on loss maximization attacks.
However, these attacks have a very strict geometry\footnote{Due to the projection operator, all PGD attacks saturate the constraint, which makes them all lies in a very small part of the ball.}.
This is why, to present a deeper analysis of the behavior of adversarial attacks and defenses, we also present a set of experiments that use perturbation minimization attacks.

\begin{table}[htbp]
  \centering 
  \begin{tabular}{lccccccc}
    \toprule
    &  & \multicolumn{2}{c}{\textbf{Attack PGD-}$\ltwo$} & & \multicolumn{2}{c}{\textbf{Attack PGD-}$\linf$} \\
  \cmidrule{3-4} \cmidrule{6-7}
  &  & \textbf{Unprotected} & \textbf{AT-}$\linf$ & & \textbf{Unprotected} & \textbf{AT-}$\ltwo$ \\
    \midrule
    \textbf{Average $\ltwo$ norm} &   & 0.830 & 0.830 &   & 1.400 & 1.640 \\
    \textbf{Average $\linf$ norm} &   & 0.075 & 0.200 &   & 0.031 & 0.031 \\
    \bottomrule
  \end{tabular}%
  \caption{Average norms of PGD-$\ltwo$ and PGD-$\linf$ adversarial examples with and without $\linf$ adversarial training on CIFAR-10 ($d=3072$).}
  \label{table:ap4-mean_norm_pgd_attack}
\end{table}%

\paragraph{Adversarial training vs. loss maximization attacks}

To demonstrate that $\linf$ adversarial training is not robust against PGD-$\ltwo$ attacks we measure the evolution of $\ltwo$ norm of adversarial examples generated with PGD-$\linf$ between an unprotected model and a model trained with AT-$\linf$, \ie, AT where adversarial examples are generated with PGD-$\linf$ \footnote{To do so, we use the same experimental setting as in Section~\ref{section:ap4-reviewing_defenses_against_multiple_attacks} with $\epsilon_\infty$ and $\epsilon_2$ such that the volumes of the two balls are equal.}. 
Results are presented in  Table~\ref{table:ap4-mean_norm_pgd_attack}.

The analysis is unambiguous: the average $\linf$ norm of a bounded $\ltwo$ perturbation more than double between an unprotected model and a model trained with AT PGD-$\linf$.
This phenomenon perfectly reflects the illustration of Figure~\ref{figure:ap4-ball_adversarial_linf}.
The attack will generate an adversarial example on the corner of the $\linf$ ball thus increasing the $\linf$ norm while maintaining the same $\ltwo$ norm. 
We can observe the same phenomenon with AT-$\ltwo$ against PGD-$\linf$ attack (see Figure~\ref{figure:ap4-ball_adversarial_l2} and Table \ref{table:ap4-mean_norm_pgd_attack}).
PGD-$\linf$ attack increases the $\ltwo$ norm while maintaining the same $\linf$ perturbation thus generating the perturbation in the upper area. 

As a consequence, we cannot expect adversarial training $\linf$ to offer any guaranteed protection against $\ltwo$ adversarial examples .

\paragraph{Adversarial training vs. perturbation minimization attacks.}
To better capture the behavior of $\ltwo$ adversarial examples, we now study the performances of an $\ltwo$ perturbation minimization attack (C\&W) with and without AT-$\linf$.
It allows us to understand in which area C\&W discovers adversarial examples and the impact of AT-$\linf$.
In high dimensions, the red corners (see Figure~\ref{figure:ap4-ball_inclusion_adversarial_training}) are very far away from the $\ltwo$ ball.
Therefore, we hypothesize that a large proportion of the $\ltwo$ adversarial examples will remain unprotected.
To validate this assumption, we measure the proportion of adversarial examples inside of the $\ltwo$ ball before and after $\linf$ adversarial training.
The results are presented in Figure~\ref{fig:calotte} (left: without adversarial training, right: with adversarial training). 

\begin{figure}[htb]
    \centering
    \input{figures/appendix/ap4-advocating_for_multiple_defense_strategies/graph.tex}
    \caption{Comparison of the number of adversarial examples found by C\&W, inside the $\linf$ ball (lower, blue area), outside the $\linf$ ball but inside the $\ltwo$ ball (middle, red area) and outside the $\ltwo$ ball (upper gray area). $\epsilon$ is set to $0.3$ and $\epsilon'$ varies along the x-axis. Left: without adversarial training, right: with adversarial training. Most adversarial examples have shifted from the $\linf$ ball to the cap of the $\ltwo$ ball, but remain at the same $\ltwo$ distance from the original example.}
    \label{fig:calotte}
\end{figure}
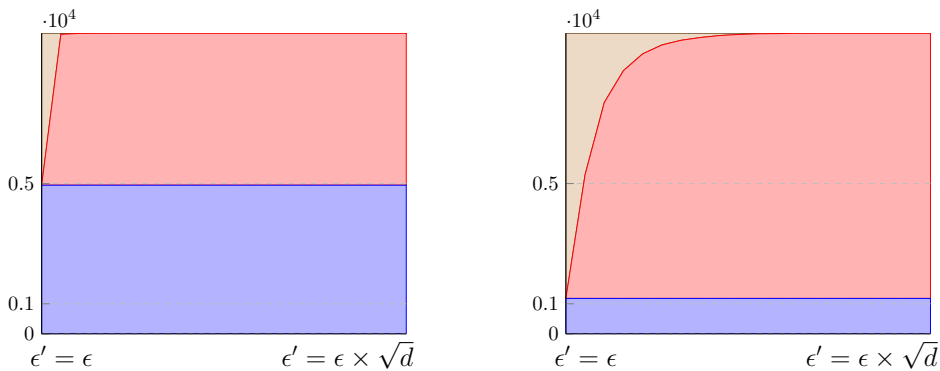

On both charts, the blue area represents the proportion of adversarial examples that are inside the $\linf$ ball.
The red area represents the adversarial examples that are outside the $\linf$ ball but still inside the $\ltwo$ ball (valid $\ltwo$ adversarial examples).
Finally, the brown-beige area represents the adversarial examples that are beyond the $\ltwo$ bound.
The radius $\epsilon'$ of the $\ltwo$ ball varies along the x-axis from $\epsilon'$ to $\epsilon' \sqrt{d}$.
On the left chart (without adversarial training) most $\ltwo$ adversarial examples generated by C\&W are inside both balls.
On the right chart most of the adversarial examples have been shifted out the $\linf$ ball.
This is the expected consequence of $\linf$ adversarial training.
However, these adversarial examples remain in the $\ltwo$ ball, \ie, they are in the cap of the $\ltwo$ ball.
These examples are equally good from the $\ltwo$ perspective.
This means that even after adversarial training, it is still easy to find good $\ltwo$ adversarial examples, making the $\ltwo$ robustness of AT-$\linf$ almost null. 

\section{Reviewing Defenses Against Multiple Attacks}
\label{section:ap4-reviewing_defenses_against_multiple_attacks}

\begin{table}[ht]
  \centering
  \tabcolsep=0.13cm
  {\small
  \begin{tabular}{lccccccccccccccccc}
    \toprule
     & & \multirow{2}{*}{\textbf{Baseline}} & & \multicolumn{2}{c}{\textbf{AT}} & & \multicolumn{2}{c}{\textbf{MAT}} & & \multicolumn{2}{c}{\textbf{NI}} & & \multicolumn{2}{c}{\textbf{RAT}-$\linf$} & & \multicolumn{2}{c}{\textbf{RAT}-$\ltwo$} \\
    \cmidrule{5-6} \cmidrule{8-9} \cmidrule{11-12} \cmidrule{14-15} \cmidrule{17-18}
     & & & & $\linf$ & $\ltwo$ &   & Max & Rand &   & $\mathcal{N}$ & $\mathcal{U}$ &   & $\mathcal{N}$ & $\mathcal{U}$ &   & $\mathcal{N}$ & $\mathcal{U}$ \\
    \midrule
    \textbf{Natural}     &   & 0.94 &   & 0.85 & 0.85 &   & 0.80 & 0.80 &   & 0.79 & 0.87 &   & 0.74 & 0.80 &   & 0.79 & 0.87 \\
    \textbf{PGD-$\linf$} &   & 0.00 &   & 0.43 & 0.37 &   & 0.37 & 0.40 &   & 0.23 & 0.22 &   & 0.35 & 0.40 &   & 0.23 & 0.22 \\
    \textbf{PGD-$\ltwo$} &   & 0.00 &   & 0.37 & 0.52 &   & 0.50 & 0.55 &   & 0.34 & 0.36 &   & 0.43 & 0.39 &   & 0.34 & 0.37 \\
    \bottomrule
  \end{tabular}%
  }
  \caption{Comprehensive list of results consisting of the accuracy of several defense mechanisms against $\ltwo$ and $\linf$ attacks.}
  \label{table:ap4-results}
\end{table}%

Adversarial attacks have been an active topic in the machine learning community since their discovery~\cite{globerson2006nightmare, biggio2013evasion,szegedy2013intriguing}.
Many attacks have been developed.
Most of them solve a loss maximization problem with either $\linf$~\cite{goodfellow2014explaining,kurakin2016adversarial,madry2018towards}, $\ltwo$~\cite{carlini2017towards,kurakin2016adversarial,madry2018towards}, $\lone$~\cite{tramer2019adversarial} or $\lzero$~\cite{papernot2016limitations} surrogate norms.
As we showed, these norms are really different in high dimension.
Hence, defending against one norm-based attack is not sufficient to protect against another one. 
In order to solve this problem, we review several strategies to build defenses against multiple adversarial attacks.
These strategies are based on the idea that both types of defense must be used simultaneously in order for the classifier to be protected against multiple attacks.

\subsection{Experimental Setting}
\label{section:ap4-experimental_settings}

To compare the robustness provided by the different defense mechanisms, we use strong adversarial attacks and a conservative setting: the attacker has a total knowledge of the parameters of the model (white-box setting) and we only consider untargeted attacks  (a misclassification from one target to any other will be considered as adversarial).
To evaluate defenses based on Noise Injection, we use \emph{Expectation Over Transformation} (EOT), the rigorous experimental protocol  proposed by~\citet{athalye2017synthesizing} and later used by~\citet{athalye2018obfuscated,carlini2019evaluating} to identify flawed defense mechanisms. 

To attack the models, we use state-of-the-art algorithms PGD.
We run PGD with 20 iterations to generate adversarial examples and with 10 iterations when it is used for adversarial training.
The maximum $\linf$ bound is fixed to $0.031$ and the maximum $\ltwo$ bound is fixed to $0.83$.
We chose these values so that the $\linf$ and the $\ltwo$ balls have similar volumes.
Note that $0.83$ is slightly above the values typically used in previous publications in the area, meaning the attacks are stronger, and thus  more difficult to defend against.

All experiments are conducted on CIFAR-10 with the Wide-Resnet 28-10 architecture.
We use the training procedure and the hyper-parameters described in the original paper by~\citet{zagoruyko2016wide}.
Training time varies from 1 day (AT) to 2 days (MAT) on 4 GPUs-V100 servers.

\subsection{MAT -- Mixed Adversarial Training}
\label{subsection:ap4-mixed_adversarial_training}

Earlier results have shown that AT-$\lp$ improves the robustness against corresponding $\lp$-bounded adversarial examples, and the experiments we present in this section corroborate this observation (See Table~\ref{table:ap4-results}, column: AT).
Building on this, it is natural to examine the efficiency of \emph{Mixed Adversarial Training} (MAT) against mixed $\linf$ and $\ltwo$ attacks.
MAT is a variation of AT that uses both $\linf$-bounded adversarial examples and $\ltwo$-bounded adversarial examples as training examples.
As discussed by~\citet{tramer2019adversarial}, there are several possible strategies to mix the adversarial training examples.
The first strategy (MAT-Rand) consists in randomly selecting one adversarial example among the two most damaging $\linf$ and $\ltwo$, and to use it as a training example:

\paragraph{MAT-Rand}:
\begin{equation} \label{equation:ap4-mat_rand}
  \min_{\Omega} \Ebb_{(\xvec, y) \sim \Dset} \left[\Ebb_{p \sim \Uset({\{2, \infty\})}} \max_{\norm{\adv}_p \leq \epsilon} L \left( N_{\Omega}(\xvec+\adv), y \right) \right].
\end{equation}

An alternative strategy is to systematically train the model with the most damaging adversarial example ($\linf$ or $\ltwo$):
\paragraph{MAT-Max}:
\begin{equation} \label{equation:ap4-mat_max}
  \min_{\Omega} \Ebb_{(\xvec, y) \sim \Dset} \left[ \max_{p \in \{2, \infty\}} \max_{\norm{\adv}_p \leq \epsilon} L \left( N_{\Omega}(\xvec+\adv), y \right) \right].
\end{equation}

The accuracy of MAT-Rand and MAT-Max are reported in~\Cref{table:ap4-results} (Column: MAT).
As expected, we observe that MAT-Rand and MAT-Max offer better robustness both against PGD-$\ltwo$ and PGD-$\linf$ adversarial examples than the original AT does.
More  generally, we can see that AT is a good strategy against loss maximization attacks, and thus it is not surprising that MAT is a good strategy against mixed loss maximization attacks.
However efficient in practice, MAT (for the same reasons as AT) lacks theoretical arguments.
In order to get the best of both worlds, \citet{salman2019provably} proposed to mix adversarial training with randomization.

\subsection{RAT -- Randomized Adversarial Training}
\label{subsection:ap4-randomized_adversarial_training}

We now examine the performance of Randomized Adversarial Training (RAT) first introduced by~\citet{salman2019provably}.
This technique mixes Adversarial Training with Noise Injection.
The corresponding loss function is defined as follows:
\begin{equation}
  \min_{\Omega} \Ebb_{(\xvec, y) \sim \Dset} \left[ \max_{\norm{\tau}_p \leq \epsilon} L\left( \tilde{N}_{\Omega}(\xvec+\tau), y)  \right) \right].
\end{equation}
where $\tilde{N}_\Omega$ is a randomized neural network with noise injection as described in Appendix~\ref{appendix:ap3-theoretical_evidence_for_adversarial_robustness_through_randomization}, and $\norm{\ \cdot\ }_p$ define which kind of AT is used.
For each setting, we consider two noise distributions, Gaussian and Uniform as we did with NI.
We also consider two different Adversarial training AT-$\linf$ as well as AT-$\ltwo$. 

The results of RAT are reported in Table~\ref{table:ap4-results}~(Columns: RAT-$\linf$ and RAT-$\ltwo$).
We can observe that RAT-$\linf$ offers the best extra robustness with both noises, which is consistent with previous experiments, since AT is generally more effective against $\linf$ attacks whereas NI is more effective against $\ltwo$-attacks.
Overall, RAT-$\linf$ and a noise from uniform distribution offers the best performances but is still weaker than MAT-Rand.
These results are also consistent with the literature, since adversarial training (and its variants) is the best defense against adversarial examples so far.

\section{Concluding Remarks}
\label{section:ap4-conclusion}

In this chapter, we tackled the problem of protecting neural networks against multiple attacks crafted from different norms.
We demonstrated and gave a geometrical interpretation to explain why most defense mechanisms can only protect against one type of attack.
Then we reviewed existing strategies that mix defense mechanisms in order to build models that are robust against multiple adversarial attacks.
We conduct a rigorous and full comparison of \emph{Randomized Adversarial Training} and \emph{Mixed Adversarial Training} as defenses against multiple attacks.

We could argue that both techniques offer benefits and limitations.
We have observed that MAT offers the best empirical robustness against multiples adversarial attacks but this technique is computationally expensive which hinders its use in large-scale applications.
Randomized techniques have the important advantage of providing theoretical guarantees of robustness and being computationally cheaper.
However, the certificate provided by such defenses is still too small for strong attacks.
Furthermore, certain Randomized defenses also suffer from the curse of dimensionality as recently shown by~\citet{kumar2020curse}. 

Although, randomized defenses based on noise injection seem limited in terms of accuracy under attack and scalability, they could be improved either by Learning the best distribution to use or by leveraging different types of randomization such as discrete randomization first proposed by~\citet{pinot2020randomization}.
We believe that these certified defenses are the best solution to ensure the robustness of classifiers deployed into real-world applications.

%% file: figures/appendix/ap4-advocating_for_multiple_defense_strategies/graph.tex
\begin{tikzpicture}[scale=0.7]
    \begin{groupplot}[group style={
                        group name=myplot,
			group size= 2 by 1,
		        horizontal sep=3cm},
                      grid style=dashed,
		      ymajorgrids=true]
       
    \nextgroupplot[
       stack plots=y,
       area style,
       ytick={0,5000,1000},
       ymin=0,
       ymax=10000,
       xmin=0.3,
       xmax=16.63,
       axis x line*=bottom,
       axis y line*=left,
       xtick={2,14},
       xticklabels={\Large $\epsilon'=\epsilon\phantom{\sqrt{d}}$, \Large $\epsilon'=\epsilon\times\sqrt{d}$},
       xtick style={draw=none}]
        \addplot table [x=eps,y=linf_ball] {figures/appendix/ap4-advocating_for_multiple_defense_strategies/data/ball_l2_base.dat}\closedcycle;
        \addplot table [x=eps,y=callote] {figures/appendix/ap4-advocating_for_multiple_defense_strategies/data/ball_l2_base.dat}\closedcycle;
        \addplot table [x=eps,y=outside] {figures/appendix/ap4-advocating_for_multiple_defense_strategies/data/ball_l2_base.dat}\closedcycle;

    \nextgroupplot[
       stack plots=y,
       area style,
       ytick={0,5000,1000},
       ymin=0,
       ymax=10000,
       xmin=0.3, 
       xmax=16.63,
       axis x line*=bottom,
       axis y line*=left,
       xtick={2,14},
       xticklabels={\Large $\epsilon'=\epsilon\phantom{\sqrt{d}}$, \Large $\epsilon'=\epsilon\times\sqrt{d}$},
       xtick style={draw=none}]
        \addplot table [x=eps,y=linf_ball] {figures/appendix/ap4-advocating_for_multiple_defense_strategies/data/ball_l2_at.dat}\closedcycle;
        \addplot table [x=eps,y=callote] {figures/appendix/ap4-advocating_for_multiple_defense_strategies/data/ball_l2_at.dat}\closedcycle;
        \addplot table [x=eps,y=outside] {figures/appendix/ap4-advocating_for_multiple_defense_strategies/data/ball_l2_at.dat}\closedcycle;

    \end{groupplot}
\end{tikzpicture}

%% file: sources/appendix/ap5-resume_these_fr.tex
\chapter{Résumé de la thèse en Français}
\label{appendix:ap7-resume_de_la_thèse_en_français}

\begingroup
\etocsettocstyle{
  \addsec*{Contenus \\ \vspace{-0.5cm}
    \rule{\textwidth}{\tocrulewidth}
    \vspace{-1cm plus0mm minus0mm}
  }
}{
  \noindent\rule{\linewidth}{\tocrulewidth}
}
\localtoc
\endgroup

\section{Introduction}
\label{section:ap7-introduction}

\subsection{Contexte et Motivation}
\label{subsection:ap7-context_and_motivation}

\begin{figure}[t]
  \centering
  \includegraphics[scale=0.2]{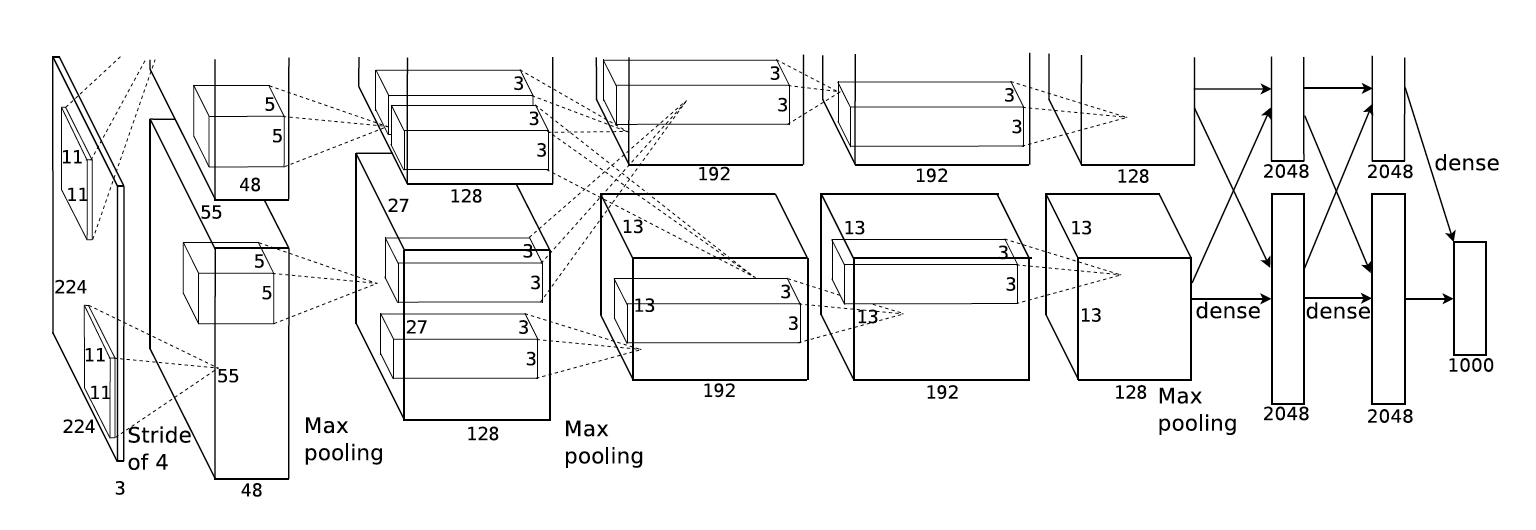}
  \caption{L'architecture de réseau de neurones convolutifs (AlexNet), proposée par ~\citet{krizhevsky2012imagenet}, qui a remporté la compétition de reconnaissance d'image ImageNet en 2012.}
  \label{figure:ap7-alexnet_network}
\end{figure}


L'une des percées les plus remarquables de l'apprentissage profond s'est produite en 2012 lors de la compétition de reconnaissance d'image ImageNet~\cite{russakovsky2015imagenet}.
Cette compétition vise à évaluer différents algorithmes pour la détection d'objets et la classification d'images.
En 2012, \citeauthor{krizhevsky2012imagenet} ont obtenu la première place et ont battu tous les autres participants avec une marge de plus de 10,8\% grâce à un réseau de neurones appelé AlexNet.
Les raisons principales de ce succès sont doubles.
Premièrement, ils ont utilisé un réseau de neurones convolutif (CNN) avec plus de 60 millions de paramètres, qui était l'un des plus grands modèles de l'époque.
Deuxièmement, ils ont conçu une architecture spécifique pour exploiter deux cartes graphiques en parallèle (GPU) afin d'accélérer les opérations arithmétiques, ce qui leur a permis de réduire considérablement le temps d'apprentissage du réseau.
La figure~\ref{figure:ap7-alexnet_network} montre un schéma de l'architecture d'AlexNet qui se compose de cinq couches convolutives avec deux des couches entièrement connectées à la fin.
La figure montre également la répartition de la charge de travail entre les deux GPU.

\begin{table}[t]
  \selectlanguage{french}
  \centering
  \sisetup{
    table-number-alignment = center,
    table-space-text-pre = \ \ \ ,
  }
  \begin{subfigure}[b]{\textwidth}
    \centering
    \begin{tabular}{
      L{5cm}
      L{3.5cm}
      S[table-format=3.0, table-text-alignment=left]@{\,}
      s[table-unit-alignment=left]
      c
    }
      \toprule
      \textbf{Auteurs} & \textbf{Modèles} & \multicolumn{2}{c}{\textbf{\#Params}} & \textbf{TOP-5 Préc.} \\
      \midrule
      \citet{krizhevsky2012imagenet} & AlexNet             &  61 & \si{M} & 84.7\% \\
      \citet{simonyan2014very}       & VGG                 & 144 & \si{M} & 92.0\% \\
      \citet{he2016deep}             & ResNet-152          &  60 & \si{M} & 93.8\% \\
      \citet{szegedy2017inception}   & Inception-ResNet-v2 &  56 & \si{M} & 95.1\% \\
      \citet{xie2017aggregated}      & ResNeXt-101         &  84 & \si{M} & 95.6\% \\
      \citet{hu2018squeeze}          & SENet               & 146 & \si{M} & 96.2\% \\
      \citet{real2019regularized}    & AmoebaNet-A         & 469 & \si{M} & 96.7\% \\
      \citet{huang2019gpipe}         & AmoebaNet-B         & 556 & \si{M} & 97.0\% \\
      \bottomrule
    \end{tabular}
    \caption{Modèles de reconnaissance d'images}
    \label{table:ap7-networks_parameters_cv}
  \end{subfigure}
  \par\bigskip
  \begin{subfigure}[b]{\textwidth}
    \centering
    \begin{tabular}{
      L{4.5cm}
      L{4.5cm}
      S[table-format=3.0, table-text-alignment=left]@{\,}
      s[table-unit-alignment=left]
    }
      \toprule
      \textbf{Auteurs} & \textbf{Modèles} & \multicolumn{2}{c}{\textbf{\#Params}} \\
      \midrule
      \citet{peters2018deep}         & ELMo            &  94 & \si{M} \\
      \citet{radford2018improving}   & GPT             & 110 & \si{M} \\
      \citet{devlin2019bert}         & BERT            & 340 & \si{M} \\
      \citet{yang2019xlnet}          & XLNet (Large)   & 340 & \si{M} \\
      \citet{liu2019roberta}         & RoBERTa (Large) & 355 & \si{M} \\
      \citet{radford2019language}    & GPT-2           &   1 & \si{B} \\
      \citet{shoeybi2019megatron}    & MegatronLM      &   8 & \si{B} \\
      \citet{raffel2020exploring}    & T5-11B          &  11 & \si{B} \\
      \citet{rosset2020turingnlg}    & T-NLG           &  17 & \si{B} \\
      \citet{brown2020language}      & GPT-3           & 175 & \si{B} \\
      \citet{fedus2021switch}        & Switch Transformers & 1 & \si{T} \\
      \bottomrule
    \end{tabular}
    \caption{Modèles de traitement automatique des langues}
    \label{table:ap7-networks_parameters_nlp}
  \end{subfigure}
  \par\bigskip
  \caption{Évolution du nombre de paramètres des modèles de reconnaissance d'image et de traitement du langage naturel développés dans les années qui ont suivi l'architecture AlexNet.}
  \label{table:ap7-networks_parameters}
\end{table}


Après l'introduction d'AlexNet, de nombreuses architectures avec un nombre croissant de paramètres ont été développées.
Cette augmentation du nombre de paramètres a conduit à une augmentation de la précision des modèles, dépassant même les performances humaines, sur l'ensemble de données d'ImageNet~\cite{he2015delving}.
Le Tableau~\ref{table:ap7-networks_parameters} montre une liste des différentes architectures de pointe avec leur taille et leur précision.
Comme on peut le voir, la précision des modèles s'améliore généralement au prix de la taille du modèle.
Pour les modèles de vision par ordinateur, \citet{tan2019efficientnet} ont montré que la relation entre la taille du modèle et la précision semble obéir à une loi de puissance.
Cette relation a également été observée pour les réseaux neuronaux de traitement du langage naturel (NLP) \cite{rosenfeld2020a,kaplan2020scaling} aidés par la disponibilité de larges ensembles de données tels que le Common Crawl~\cite{raffel2020exploring} qui constitue près d'un trillion de mots.


Grâce à leur taille et à leur précision accrue, les réseaux de neurones profonds atteignent désormais des performances de pointe dans divers domaines tels que la reconnaissance d'images~\cite{lecun1998gradient,krizhevsky2012imagenet,he2016deep,tan2019efficientnet}, la détection d'objets~\cite{redmon2016you, liu2016ssd,redmon2017yolo9000}, le traitement du langage naturel~\cite{merity2016pointer,vaswani2017attention,radford2019language,brown2020language}, speech recognition~\cite{hinton2012deep,abdel2014convolutional,yu2016automatic}, le domaine de la santé \cite{faust2018deep} etc.
Les modèles de vision par ordinateur et de traitement du langage naturel ont atteint des performances suffisantes pour être utilisés dans des applications du monde réel telles que les véhicules autonomes~\cite{fagnant2015preparing,sharma2021automating}, la traduction~\cite{wu2016google}, les assistants vocaux~\cite{li2017acoustic}, etc.

Cependant, la précision des modèles ne devrait pas être la seule préoccupation, lorsqu'ils sont mis en œuvre dans un processus de décision critique, les réseaux de neurones doivent être compacts, efficaces et sécurisés.
Bien que précis, les grands réseaux de neurones n'ont souvent pas ces propriétés.
En effet, l'entraînement de modèles de pointe sur des tâches de reconnaissance d'image ou de traitement du langage naturel nécessite des gigaoctets de mémoire et peut prendre plusieurs mois sur un seul GPU~\cite{krizhevsky2012imagenet,brown2020language}.
Par exemple, le modèle GPT-3 proposé par~\citet{brown2020language}, culmine à 175 milliards de paramètres et l'entraînement durerait 355 ans sur un seul GPU et coûterait \$\numprint{4600000} sur une plateforme de cloud computing \cite{li2020overview}.
Il a également été estimé par \citet{strubell2019energy} que la formation et le développement du modèle Transformer proposé par~\citet{vaswani2017attention} avec l'optimisation des hyperparamètres émettraient environ \numprint{284019} kg de $\mathrm{CO}_2$ alors qu'une vie humaine consomme en moyenne seulement \numprint{5000} kg de $\mathrm{CO}_2$ pendant un an. 
En outre, avec l'essor des smartphones et des objets connectés aux ressources de calcul et de mémoire limitées, les réseaux de neurones doivent également être efficaces pendant la phase d'inférence, c'est-à-dire, la phase d'exécution du modèle.
De plus, avec la préoccupation croissante concernant la confidentialité des données, des méthodes telles que l'``apprentissage collaboratif'' gagnent du terrain.
L'apprentissage collaboratif consiste à entraîner un modèle sur plusieurs appareils décentralisés (par exemple les smartphones) avec des échantillons de données locales.
Cela permet d'éviter l'étape de centralisation de toutes les données des utilisateurs sur un seul serveur, ce qui permet de résoudre le problème de la confidentialité des données.
Ainsi, la construction de réseaux de neurones compacts et efficaces reste un objectif important afin de réduire le temps d'entraînement, de diminuer les coûts et de permettre une R\&D plus rapide.


En plus d'être compacts et efficients, les réseaux de neurones doivent également être sécurisés.
En raison de leur grande complexité et expressivité, les larges réseaux de neurones sont instables aux petites perturbations.
Ainsi, cette instabilité mène à des vulnérabilités face aux \emph{exemples antagonistes}, c'est-à-dire aux variations imperceptibles des exemples naturels, conçus pour tromper délibérément les modèles~\cite{globerson2006nightmare,biggio2013evasion,szegedy2013intriguing}.
La Figure~\ref{figure:ap7-adversarial_image_example} présente un exemple antagoniste sur une image.
La petite perturbation (au centre) est ajoutée à l'image originale (à gauche), ce qui donne une image contradictoire (à droite).
Ce comportement peut causer de graves problèmes de sécurité lorsque des réseaux neuronaux sont utilisés pour des prises de décisions critiques (par exemple, les décisions judiciaires, les voitures autonomes, etc.).


Cette thèse se concentre sur l'entraînement de réseaux de neurones qui sont non seulement précis, mais aussi compacts, efficients, faciles à entraîner, fiables et robustes aux exemples antagonistes.

\subsection{Problématiques et Contributions}
\label{subsection:ap7-problem_statement_and_contributions}

\input{figures/main/ch1-introduction/example_structure_matrices_fr}


Les réseaux de neurones, qui trouvent leurs racines dans les travaux de \citet{mcculloch1943logical,rosenblatt1958perceptron}, peuvent être décrits analytiquement comme une composition de fonctions linéaires entrelacées avec des fonctions non linéaires (également appelées fonctions d'activation).
Plus formellement, un réseau de neurones est une fonction $N_{\Omega} : \Rbb^n \rightarrow \Rbb^m$ paramétrée par un ensemble de poids $\Omega$ de la forme:
\begin{equation} \label{equation:ap7-neural_network}
  N_{\Omega}(\xvec) = \psi^{(\depth)} \circ \rho \circ \psi^{(\depth-1)} \cdots \circ \psi^{(2)} \circ \rho \circ \psi^{(1)} (\xvec) \enspace.
\end{equation}
Ici, $\depth$ correspond à la \emph{profondeur} du réseau (c'est-à-dire, le nombre de couches) et $\rho$ est une fonction non linéaire.
Enfin, chaque $\psi^{(i)}$ est une fonction linéaire multidimensionnelle $\psi^{(i)} : \xvec \mapsto \Wmat^{(i)} \xvec + \bvec^{(i)}$ paramétrée par une matrice de poids $\Wmat^{(i)}$ et un biais $\bvec^{(i)}$ et $\Omega$ est l'union des paramètres de chaque couche.


Si les réseaux de neurones n'ont pas de restriction sur les matrices de poids $\Wmat^{(i)}$, on dit que les couches sont \emph{entièrement connectées}.
En règle générale, les réseaux de neurones entièrement connectés ont un grand nombre de paramètres.
Par exemple, un réseau de neurones entièrement connecté avec $\depth$ couches et des $n$ neurones sur chaque couche ($\Wmat^{(i)} \in \Rbb^{n \times n}$) aura $\bigO\left(pn (n + 1)\right)$ paramètres.
Comme les dimensions d'entrée et de sortie sont généralement importantes (par exemple, le jeu de données ImageNet a une dimension d'entrée de $224^2 \times 3$ et une sortie de 1000), les réseaux de neurones entièrement connectés avec peu de couches peuvent facilement accumuler des centaines de millions de paramètres.
Il a été montré que ce type de réseau est peu performant, car l'entraînement n'optimise pas suffisamment bien les paramètres en raison d'un grand espace de recherche.
En outre, l'entraînement est long et complexe, ce qui les rend peu pratiques pour un certain nombre de cas d'usage (smartphones, objets connectés, etc.).
Pour réduire le nombre de paramètres sur chaque couche, de nombreux chercheurs ont mis au point des opérations linéaires spécifiques qui réduisent le nombre de paramètres et ont de nombreuses propriétés intéressantes.

\begin{figure}[t]
  \centering
  \includegraphics[width=\textwidth]{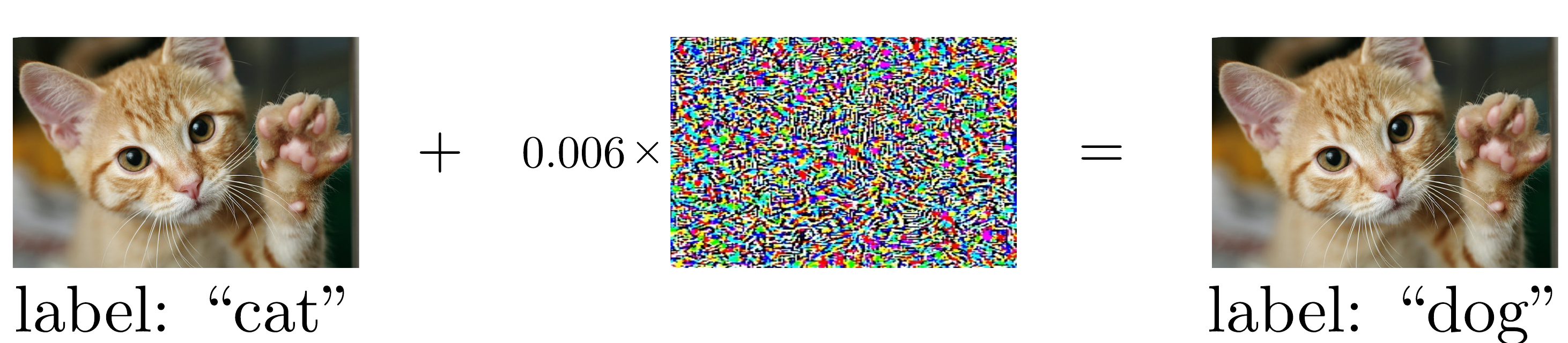}
  \caption{Exemple d'exemple antagoniste avec une image.}
  \label{figure:ap7-adversarial_image_example}
\end{figure}


Les réseaux de neurones convolutifs (CNN), qui utilisent des opérations linéaires spécialisées et plus compactes, sont considérés comme l'état de l'art concernant les tâches de vision par ordinateur~\cite{lecun1998gradient,krizhevsky2012imagenet,he2016deep,tan2019efficientnet}. 
Ces réseaux neuronaux convolutifs utilisent des matrices de poids spécifiques qui permettent une invariance du traitement par translation, ce qui est souhaitable pour traiter des images.
Alors qu'une couche linéaire classique avec une matrice dense a $n \times n$ paramètres, une couche convolutionnelle n'a que $k \times k$ paramètres avec $k$ qui correspond à la taille du noyau et qui est généralement petite (3 ou 5 pour les couches convolutionnelles classiques).
Un réseau neuronal convolutif est le type le plus courant de réseau neuronal \emph{structuré}.
En effet, l'opération de convolution peut être représentée par une matrice structurée, c'est-à-dire, une matrice qui peut être représentée avec moins de $n^2$ de paramètres.


En plus d'offrir une représentation plus compacte, la structure de certaines matrices peut être exploitée afin d'obtenir de meilleurs algorithmes pour multiplier la matrice avec un vecteur, cela permet d'optimiser la mémoire et de réduire le nombre d'opérations réalisées.
En se basant sur le succès des réseaux de neurones convolutifs, les chercheurs ont étudié et proposé d'autres types de réseaux basés sur des matrices de poids avec différentes structures~\cite{moczulski2016acdc,sindhwani2015structured}.
La Figure~\ref{figure:ap7-example_structure_matrices} montre différents types de matrices structurées qui ont été utilisées pour l'apprentissage profond.
Bien que les réseaux de neurones convolutifs soient état de l'art pour les tâches de vision par ordinateur, il reste à savoir si d'autres types de réseaux structurés pourraient être utiles à d'autres types d'applications et quel type de structure pourraient fournir à la fois précision et efficacité de calcul.


Les contributions de cette thèse se situent à l'intersection de l'algèbre linéaire, l'analyse de Fourier et de l'apprentissage profond.
En conséquence, nous construisons des réseaux de neurones compacts et sécurisés en exploitant les propriétés des matrices structurées issues de la famille de Toeplitz.
Ci-après, nous détaillons nos contributions.

\subsubsection{Entraînement de Réseaux de Neurones Compacts}
\label{subsubsection:ap7-training_compact_neural_networks}

Comme première contribution, nous étudions les réseaux de neurones dans lesquels les matrices de poids sont le produit des matrices diagonales et circulantes.
Les matrices circulantes sont un type particulier de matrice de Toeplitz.
Cette nouvelle architecture compacte permet de remplacer les réseaux de neurones entièrement connectés tout en maintenant la performance.
Outre une analyse théorique de leur expressivité, nous introduisons de nouvelles techniques pour l'entraînement de ces modèles : nous concevons une procédure d'initialisation et proposons une utilisation intelligente des fonctions de non-linéarité afin de faciliter leur entraînement.
Nous montrons que ces modèles sont plus précis que les autres approches structurées tout en nécessitant deux fois moins de poids que les meilleures approches.
Enfin, nous entraînons des réseaux de neurones profonds basés sur des matrices diagonales et circulantes sur un ensemble de données de classification vidéo qui contient plus de 3.8 millions d'exemples.

L'analyse expérimentale des réseaux de neurones profonds basée sur des matrices diagonales et circulantes sur l'ensemble de données de classification vidéo a été publiée dans le cadre de l'\textbf{\color{mydarkblue} Atelier sur la reconnaissance de vidéo de la Conférence Européenne de Vision par Ordinateur}.
Ce travail a été réalisé dans le cadre de la compétition \yt organisée par Google.
Ensuite, l'analyse théorique de l'expressivité de ces réseaux a été publiée dans un deuxième article dans le cadre de la \textbf{\color{mydarkblue} 24e Conférence Européenne sur l'Intelligence Artificielle.}

\subsubsection{Entraînement de Réseaux de Neurones Robustes}
\label{subsubsection:ap7-training_robust_neural_networks}

Comme deuxième contribution, nous proposons une procédure pour entraîner des réseaux de neurones robustes en étudiant les propriétés de la structure des convolutions.
Nous concevons une nouvelle borne supérieure des valeurs singulières des couches de convolution, qui est à la fois précise et facile à calculer.
Notre travail est basé sur le résultat de~\citet{gray2006toeplitz} qui indique qu'une borne supérieure des valeurs singulières des matrices de Toeplitz peut être calculée à partir de la transformée de Fourier inverse de la séquence caractéristique de ces matrices.
De notre analyse découle immédiatement un algorithme de régularisation de la constante de Lipschitz d'une couche convolutive, et par extension de la constante de Lipschitz de l'ensemble du réseau.
Enfin, nous utilisons notre approche pour améliorer la robustesse des réseaux de neurones convolutifs.
Des travaux récents ont montré que les méthodes empiriques telles que l'entraînement contradictoire offrent une faible généralisation~\cite{schmidt2018adversarially} et peuvent être améliorées en appliquant une régularisation Lipschitz~\cite{farnia2018generalizable}.
Pour illustrer l'avantage de notre nouvelle méthode, nous entraînons des réseaux de neurones avec la régularisation Lipschitz et montrons qu'elle offre une amélioration significative de robustesse.

Le principal résultat des travaux décrits dans le Chapitre~\ref{chapter:ch5-lipschitz_bound} a été publié dans le cadre de la  \textbf{\color{mydarkblue}35e Conférence AAAI sur l'Intelligence Artificielle}.
D'autres contributions conjointes ont également été publiées sur le thème des réseaux neuronaux robustes.
La première, publiée dans le cadre de la \textbf{\color{mydarkblue} Conférence en Intelligence Artificielle et Neurosciences Computationnelles}, étudie l'efficacité de l'injection de bruit à l'entraînement et à l'inférence dans le réseau pour protéger contre les attaques adverses.
Dans ce travail, nous montrons que le bruit tiré de la famille exponentielle offre une protection garantie contre les attaques adverses. 
La deuxième contribution conjointe, publiée dans le cadre de l'\textbf{\color{mydarkblue} Atelier de Cybersécurité de la Conférence Européenne de l'Apprentissage Automatique}, effectue une analyse géométrique des mécanismes de défense destinés à protéger les réseaux neuronaux contre différents types d'attaques.
Ce travail montre que les réseaux neuronaux conçus pour être robustes contre un type d'attaque adverse offrent peu de protection contre d'autres types d'attaques.

\section{Réseaux de Neurones Compacts basés sur les matrices Diagonales et Circulantes}
\label{section:ap7-diagonal_circulant_neural_network}

Ces dernières années, la conception de réseaux neuronaux compacts et performants a été un sujet de recherche actif.
Ce domaine est motivé par des applications pratiques dans les systèmes embarqués (pour réduire l'empreinte mémoire \cite{sainath2015convolutional}), l'apprentissage fédéré et distribué (pour réduire la communication \cite{konecny2016federated}), etc.
Outre un certain nombre d'applications pratiques, la question de savoir si les modèles doivent réellement être aussi larges ou si des réseaux plus petits peuvent atteindre une précision similaire est également une question de recherche importante.

Les matrices structurées sont au cœur même de la plupart des travaux sur les réseaux compacts.
Dans ces modèles, les matrices de poids dense sont remplacées par des matrices ayant une structure précise (par exemple, les matrices de rang faible, les matrices de Toeplitz, les matrices circulantes, LDR, etc.)
Malgré des efforts importants (\citet{cheng2015exploration,moczulski2016acdc}), les performances des modèles compacts sont encore loin d'atteindre une précision acceptable motivant leur utilisation dans des scénarios du monde réel.
Cela soulève plusieurs questions sur l'efficacité de ces modèles et sur notre capacité à les entraîner.
En particulier, deux questions principales appellent à investigation :
\begin{enumerate}[leftmargin=0.5cm]
  \item Quelle est l'expressivité des couches structurées par rapport aux couches denses ?
  \item Comment entraîner efficacement des réseaux neuronaux profonds avec un grand nombre de couches structurées ?
\end{enumerate}
Dans cette thèse, nous nous efforçons de répondre à ces questions en étudiant les réseaux neuronaux basés sur les matrices diagonales et circulantes (\aka DCNN), qui sont des réseaux neuronaux profonds dans lesquels les matrices de poids sont le produit des matrices diagonales et circulantes.
L'idée d'utiliser ensemble des matrices diagonales et circulantes vient d'une série de résultats en algèbre linéaire par~\citet{muller1998algorithmic} et~\citet{huhtanen2015factoring}.

Pour répondre à la première question, nous proposons une analyse de l'expressivité des DCNN en étendant les résultats obtenus par~\citet{huhtanen2015factoring} qui indique que toute matrice peut être décomposée en un produit de $2n-1$ matrices diagonales et circulantes.
Nous introduisons une nouvelle borne sur le nombre de produits requis pour approcher une matrice qui dépend de son rang.
Sur la base de ce résultat, nous démontrons qu'un DCNN avec une largeur limitée et une faible profondeur peut être autant expressif que n'importe quel réseau de neurones dense avec des activations ReLU. 

Pour répondre à la deuxième question, nous décrivons d'abord une procédure d'initialisation pour les DCNN qui permet au signal de se propager à travers le réseau sans disparaître ou exploser.
En outre, nous fournissons un certain nombre d'expériences pour expliquer le comportement des DCNN et montrer l'impact du nombre de non-linéarités dans le réseau sur le taux de convergence et la précision. 
En combinant toutes ces connaissances, nous sommes en mesure de former des DCNN de grande taille et de grande profondeur.
Pour finir, nous démontrons les bonnes performances de ces réseaux dans le contexte de la reconnaissance de vidéo.

\section{Constante de Lipschitz des Couches Convolutionnelles}
\label{section:ap7-lipschitz_bound}

Ces dernières années ont vu un intérêt croissant pour la régularisation Lipschitz des réseaux de neurones, dans le but d'améliorer leur généralisation~\cite{bartlett2017spectrally}, leur robustesse aux attaques adverses~\cite{tsuzuku2018lipschitz, farnia2018generalizable}, ou leurs capacités de génération (par exemple pour les GANs : \citet{miyato2018spectral,arjovsky2017wasserstein}).
Malheureusement, le calcul exact de la constante de Lipschitz d'un réseau de neurones est un problème NP-complet~\cite{scaman2018lipschitz} et en pratique, les techniques existantes telles que celles proposées par~\citet{scaman2018lipschitz}, \citet{fazlyab2019efficient} ou~\citet{latorre2020lipschitz} sont difficiles à mettre en œuvre pour les réseaux neuronaux à plus d'une ou deux couches, ce qui entrave leur utilisation dans les applications d'apprentissage profond.

Pour surmonter cette difficulté au lieu de calculer la constante globale, la plupart des travaux se sont concentrés sur le calcul de la constante de Lipschitz des \emph{couches} du réseau.
Le produit des constantes de Lipschitz de chaque couche est une borne supérieure de la constante de Lipschitz de l'ensemble du réseau, et elle peut être utilisée comme substitut pour effectuer une régularisation Lipschitz.
Comme la plupart des fonctions d'activation courantes (telles que la ReLU) ont une constante de Lipschitz égale à un, la principale difficulté consiste à calculer la constante de Lipschitz de l'application linéaire sous-jacente qui est égale à sa plus grande valeur singulière.
Les travaux dans ce domaine de recherche s'appuient principalement sur un célèbre algorithme itératif appelé \emph{méthode de la puissance itérée} \cite{golub2000eigenvalue} utilisée pour approximer la valeur singulière maximale d'une fonction linéaire.
Bien que générique et précise, cette technique est également coûteuse en termes de calcul, ce qui en empêche son utilisation pour l'entraînement de larges réseaux de neurones. 

Dans cette thèse, nous introduisons une nouvelle borne supérieure des valeurs singulières des couches de convolution, qui est à la fois précise et facile à calculer.
Au lieu d'utiliser la méthode de la puissance itérée pour approximer cette valeur, nous nous appuyons sur la théorie des matrices de Toeplitz et ses liens avec l'analyse de Fourier.
Notre travail est basé sur le résultat de~\citet{gray2006toeplitz} qui indique qu'une borne supérieure des valeurs singulières des matrices de Toeplitz peut être calculée à partir de la transformée de Fourier inverse de la séquence caractéristique de ces matrices.
Nous étendons d'abord ce résultat aux matrices de Toeplitz par blocs de Toeplitz (c'est-à-dire une matrice de Toeplitz par blocs où chaque bloc est également Toeplitz) et ensuite aux opérateurs convolutionnels.
De notre analyse découle immédiatement un algorithme de régularisation de la constante de Lipschitz d'une couche convolutive, et par extension de la constante de Lipschitz de l'ensemble du réseau.
Nous étudions théoriquement l'approximation de cet algorithme et montrons expérimentalement qu'il est plus efficace et plus précis que les approches concurrentes.

Enfin, nous illustrons notre approche sur la robustesse aux exemples antagonistes.
Des travaux récents ont montré que les méthodes empiriques, telles que la formation contradictoire (\emph{Adversarial Training} ou AT), offrent une faible généralisation~\cite{schmidt2018adversarially} et peuvent être améliorées en appliquant une régularisation Lipschitz~\cite{farnia2018generalizable}.
Pour illustrer les avantages de notre nouvelle méthode, nous entraînons un large réseau de neurones avec AT et la régularisation Lipschitz et montrons qu'elle offre une amélioration significative par rapport à un entraînement contradictoire seul et par rapport aux autres méthodes de régularisation Lipschitz.
En résumé, nous apportons les trois contributions suivantes :
\begin{enumerate}[leftmargin=0.8cm]
  \item Nous proposons une nouvelle borne supérieure des valeurs singulières des couches convolutionnelles en nous appuyant sur la théorie des matrices de Toeplitz et ses liens avec l'analyse de Fourier.
  \item Nous proposons un algorithme efficace pour calculer cette borne qui permet son utilisation dans le contexte des réseaux neuronaux convolutifs.
  \item Nous utilisons notre méthode pour régulariser la constante de Lipschitz des réseaux de neurones et montrons qu'elle permet un gain significatif de robustesse face aux attaques adverses.
\end{enumerate}

%% file: figures/main/ch1-introduction/example_structure_matrices_fr.tex
\begin{figure}[t]
   \centering
   \begin{subfigure}[t]{0.24\textwidth}
       \centering
       \begin{equation*}
	  \leftmatrix
	    a &   &   &   \\
	      & b &   &   \\
	      &   & c &   \\
	      &   &   & d
	  \rightmatrix
       \end{equation*}
       \caption*{diagonal}
   \end{subfigure}
   \hfill
   \begin{subfigure}[t]{0.24\textwidth}
       \centering
       \begin{equation*}
	  \leftmatrix
	    a & b & c & d \\
	    e & a & b & c \\
	    f & e & a & b \\
	    g & f & e & a
	  \rightmatrix
       \end{equation*}
       \caption*{Toeplitz}
   \end{subfigure}
   \hfill
   \begin{subfigure}[t]{0.24\textwidth}
       \centering
       \begin{equation*}
	  \leftmatrix
	    ae & af & ag & ah \\
	    be & bf & bg & bh \\
	    ce & cf & cg & ch \\
	    de & df & dg & dh
	  \rightmatrix
       \end{equation*}
       \caption*{Low Rank}
   \end{subfigure}
   \hfill
   \begin{subfigure}[t]{0.24\textwidth}
       \centering
       \begin{equation*}
	  \leftmatrix
	    a & a^2 & a^3 & a^4 \\
	    b & b^2 & b^3 & b^4 \\
	    c & c^2 & c^3 & c^4 \\
	    d & d^2 & d^3 & d^4
	  \rightmatrix
       \end{equation*}
       \caption*{Vandermonde}
   \end{subfigure}
  \caption{Exemples de matrices structurées.}
  \label{figure:ap7-example_structure_matrices}
\end{figure}